\documentclass[10pt,twocolumn,letterpaper]{article}

\usepackage[pagenumbers]{iccv} 

\usepackage{url}            
\usepackage{booktabs}       
\usepackage{amsfonts}       
\usepackage{nicefrac}       
\usepackage{microtype}      
\usepackage{xcolor}         
\usepackage{import}
\usepackage{amsmath}
\usepackage{amsthm}
\usepackage{graphicx}
\usepackage{cancel}
\usepackage{tikz}
\usetikzlibrary{spy}
\usepackage{subcaption}
\usepackage{float}
\usepackage{wrapfig}
\usepackage{verbatim}
\usepackage{rotating}
\usepackage{mwe}
\usepackage{wrapfig}
\usepackage{algorithm}
\usepackage{float}
\usepackage{algpseudocode}
\usepackage{afterpage}
\usepackage{multirow}
\usepackage{relsize} 

\usepackage{amsmath,amsfonts,bm}

\def\eqref#1{equation~\ref{#1}}

\def\1{\bm{1}}

\def\vg{{\bm{g}}}

\def\vm{{\bm{m}}}
\def\vn{{\bm{n}}}

\def\vu{{\bm{u}}}
\def\vv{{\bm{v}}}
\def\vw{{\bm{w}}}
\def\vx{{\bm{x}}}
\def\vy{{\bm{y}}}
\def\vz{{\bm{z}}}

\DeclareMathAlphabet{\mathsfit}{\encodingdefault}{\sfdefault}{m}{sl}
\SetMathAlphabet{\mathsfit}{bold}{\encodingdefault}{\sfdefault}{bx}{n}

\newtheorem{proposition}{Proposition}
\newtheorem{definition}{Definition}

\definecolor{alessiogreen}{RGB}{0, 192, 0}
\definecolor{addred}{RGB}{255, 0, 0}

\newcommand\restr[2]{{
  \left.\kern-\nulldelimiterspace 
  #1 
  \littletaller 
  \right|_{#2} 
  }}

\newcommand{\zoomedImage}[3]{%
    \begin{tikzpicture}[spy using outlines={circle,red,magnification=2,size=1.0cm, connect spies}]
        \node {\includegraphics[width=\textwidth]{#1}};
        \spy on (#2) in node [left] at (#3);
    \end{tikzpicture}
}

\newcommand{\scaleY}{0.40\textwidth}
\newcommand{\sidecap}[1]{{\begin{sideways}\parbox{\scaleY}{\centering #1}\end{sideways}}}

\usepackage{amsmath,amsfonts,bm}

\def\vn{{\bm{n}}} 
\def\vx{{\bm{x}}} 
\def\vy{{\bm{y}}} 
\def\vz{{\bm{z}}} 

\newcommand{\encoder}{{\mathcal{E}}}
\newcommand{\decoder}{{\mathcal{D}}}

\newcommand{\Id}{{\mathrm{Id}}}

\renewcommand{\eqref}[1]{\textup{(\ref{#1})}}

\definecolor{iccvblue}{rgb}{0.21,0.49,0.74}
\usepackage[pagebackref,breaklinks,colorlinks,allcolors=iccvblue]{hyperref}

\def\paperID{*****} 
\def\confName{ICCV}
\def\confYear{2025}

\title{LATINO-PRO: LAtent consisTency INverse sOlver with PRompt Optimization\vspace{-0.2cm}}

\author{
Alessio Spagnoletti\textsuperscript{1}\thanks{Equal contribution.} \quad
Jean Prost\textsuperscript{2}\footnotemark[1] \quad
Andrés Almansa\textsuperscript{1} \quad
Nicolas Papadakis\textsuperscript{3} \quad
Marcelo Pereyra\textsuperscript{4} \\[2mm]
\footnotesize
\textsuperscript{1}Université Paris Cité, CNRS, MAP5 UMR 8145, F-75006 Paris, France\  
\footnotesize
\textsuperscript{2}Université de Lille, CRIStAL UMR 9189, F-59655 Villeneuve d'Ascq, France\\
\footnotesize
\textsuperscript{3}Univ. Bordeaux, CNRS, INRIA, Bordeaux INP, IMB, UMR 5251, F-33400 Talence, France\\
\footnotesize
\textsuperscript{4}\footnotesize Heriot-Watt University, MACS \& Maxwell Institute for Mathematical Sciences, EH14 4AS, Edinburgh, United Kingdom\\[2mm]
}

\begin{document}
\maketitle
\begin{abstract}
Text-to-image latent diffusion models (LDMs) have recently emerged as powerful generative models with great potential for solving inverse problems in imaging. However, leveraging such models in a Plug \& Play (PnP), zero-shot manner remains challenging because it requires identifying a suitable text prompt for the unknown image of interest. Also, existing text-to-image PnP approaches are highly computationally expensive. We herein address these challenges by proposing a novel PnP inference paradigm specifically designed for embedding generative models within stochastic inverse solvers, with special attention to Latent Consistency Models (LCMs), which distill LDMs into fast generators. We leverage our framework to propose LAtent consisTency INverse sOlver (LATINO), the first zero-shot PnP framework to solve inverse problems with priors encoded by LCMs. Our conditioning mechanism avoids automatic differentiation and reaches SOTA quality in as little as 8 neural function evaluations. As a result, LATINO delivers remarkably accurate solutions and is significantly more memory and computationally efficient than previous approaches. We then embed LATINO within an empirical Bayesian framework that automatically calibrates the text prompt from the observed measurements by marginal maximum likelihood estimation. Extensive experiments show that prompt self-calibration greatly improves estimation, allowing LATINO with PRompt Optimization to define new SOTAs in image reconstruction quality and computational efficiency. The code is available at \href{https://latino-pro.github.io/}{latino-pro.github.io}.\vspace{-0.4cm}
\end{abstract}

\section{Introduction}
\label{sec:intro}
We seek to recover on an unknown image of interest $\vx$, taking values in $\mathbb{R}^n$, from a measurement
\begin{equation*}
    \vy = \mathcal{A}\vx + \vn\, ,
\end{equation*}
where $\mathcal{A}$ is a linear measurement operator, and $\vn$ is additive Gaussian noise with covariance $\sigma_n^2\Id$. We consider situations in which the recovery of $\vx$ from $\vy$ is ill-conditioned or ill-posed, leading to significant uncertainty about the solution. Adopting a Bayesian approach, we leverage prior knowledge about $\vx$ in order to regularize the problem and deliver meaningful inferences that are well-posed. This is achieved by specifying the marginal distribution $p(\vx)$, so-called prior distribution, together with Bayes' theorem to obtain the posterior distribution $
    p(\vx|\vy) = {p(\vy|\vx)p(\vx)}/{p(\vy)}$
which models our knowledge about $\vx$ after observing $\vy$.

\begin{figure}[t]
\centering
\begin{minipage}{0.16\textwidth}
    \centering \textbf{Measurement} \\ 
    \includegraphics[width=\textwidth]{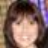} \\
    \includegraphics[width=\textwidth]{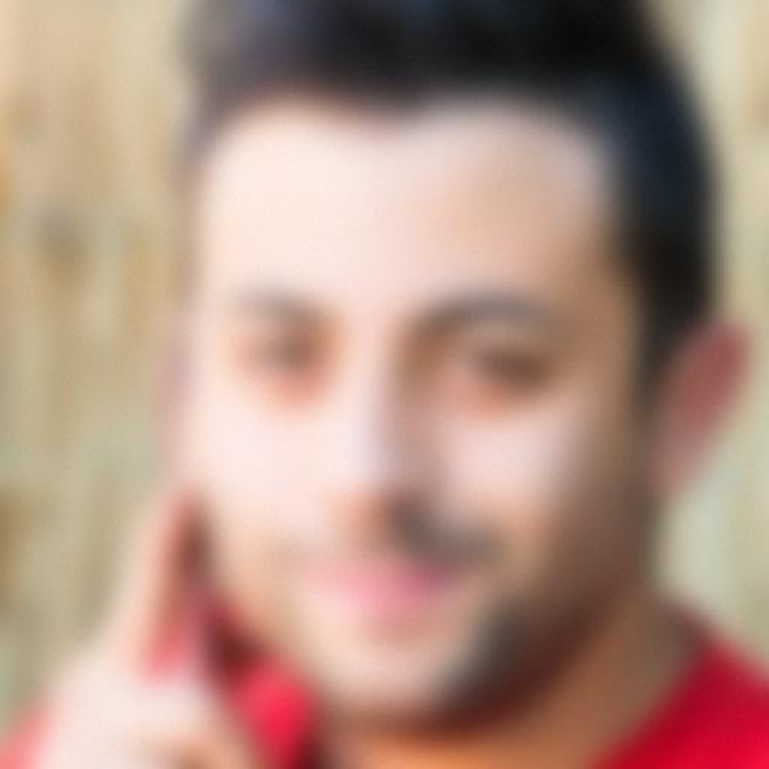}\\
    \includegraphics[width=\textwidth]{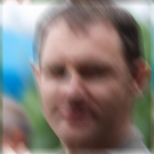}
\end{minipage}%
\begin{minipage}{0.16\textwidth}
    \centering \textbf{GT} \\ 
    \includegraphics[width=\textwidth]{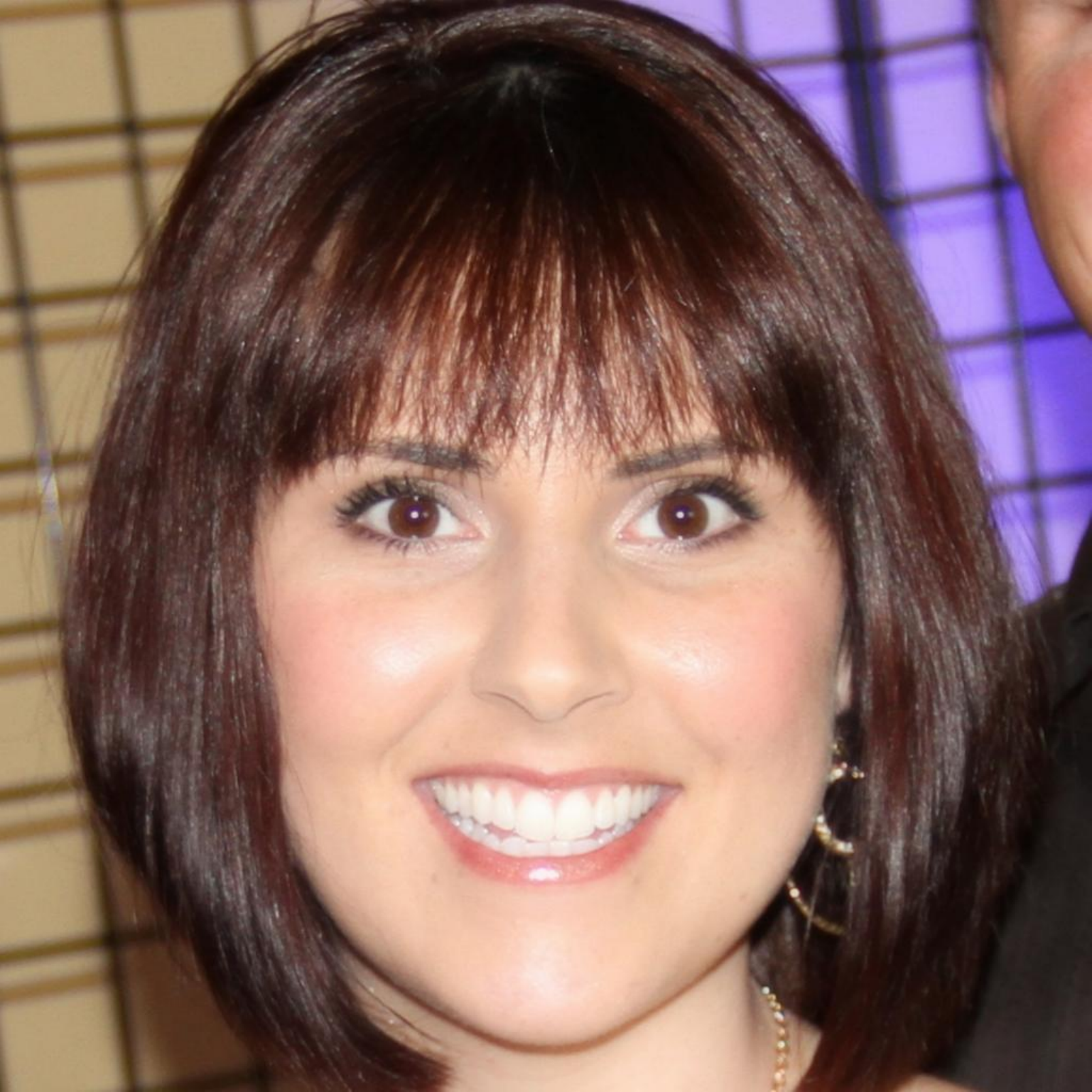} \\
    \includegraphics[width=\textwidth]{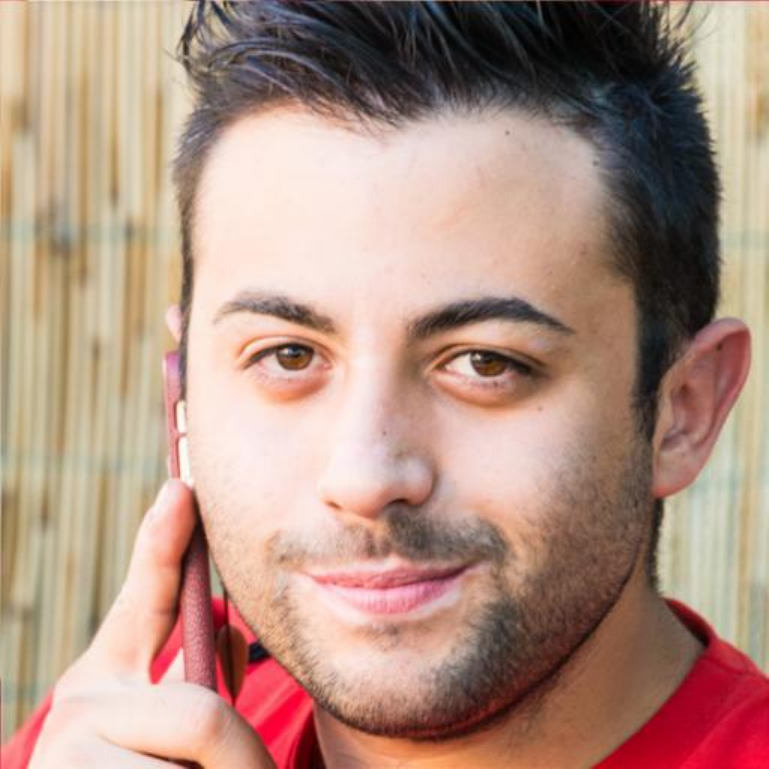}\\
    \includegraphics[width=\textwidth]{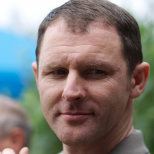}
\end{minipage}%
\begin{minipage}{0.16\textwidth}
    \centering \textbf{Ours} \\ 
    \includegraphics[width=\textwidth]{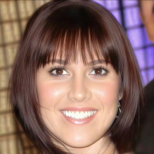} \\
    \includegraphics[width=\textwidth]{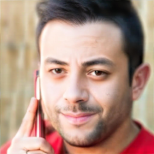}\\
    \includegraphics[width=\textwidth]{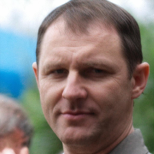}
\end{minipage}\vspace*{-0.2cm}
\caption{Qualitative comparison of LATINO-PRO on the FFHQ-1024 val dataset. Tasks: $\times 32$ super-resolution, Gaussian deblur $\sigma=20.0$ pixels, Motion deblur.\vspace{-0.2cm}}
\label{fig:preview}
\end{figure}

Modern imaging techniques rely increasingly on pre-trained foundational generative models as image priors. Developing methodology to leverage such models in a so-called Plug \& Play (PnP) manner, in combination with a likelihood function $p(\vy|\vx)$ specified during test time, is a highly active research area. In particular, strategies relying on diffusion models (DM), pioneered in \cite{Song2019, Song2020improved, Song2020SDE}, superseded GANs \cite{Dhariwal2021} and established themselves as the default approach; with notable examples including, e.g., DPS \cite{Chung2022}, DDRM \cite{kawar2022denoising}, DiffPIR \cite{Zhu2023DenoisingDM}, $\Pi$GDM \cite{Song2023PseudoinverseGuidedDM}, and mid-point~\cite{moufad2025variational}. In order to reduce computational efficiency, modern DMs operate in the latent space  of a variational autoencoder (VAE) \cite{Kingma2013AutoEncodingVB, Lopez2020AUTOENCODINGVB}; leading to Latent Diffusion Models (LDMs) \cite{Rombach2021}. LDMs simply project the data into a latent space using an encoder, $\mathcal{E}$, perform the diffusion in the latent space, and use a decoder, $\mathcal{D}$, to move back to pixel space. However, since the likelihood is specified in pixel space, naive adaptations of DM-based imaging algorithms fail when applied directly to LDMs. Instead, the specific properties of the latent space need to be considered, as proposed in PSLD \cite{Rout2023}.

Another fundamental advantage of using LDMs like the celebrated Stable Diffusion models \cite{Rombach2021, Podell2023SDXLIL} as image priors is the possibility to condition on a textual prompt input. Prompts can be highly informative, so it is natural to seek to exploit them for imaging. Specifying a suitable prompt a-priori is difficult, so P2L \cite{Chung2023PrompttuningLD} includes a strategy to calibrate the prompt automatically, while TReg \cite{Kim2023RegularizationBT} exploits classifier-free guidance to tune the null prompt.

A main limitation of all existing (L)DM-based imaging methods is that they rely on iterative sampling strategies that require hundreds or thousands of neural function evaluations (NFEs) to produce a single sample from $p(\vx|\vy)$. As a result, they are very computationally expensive. Also, they cannot support Bayesian inferences that require drawing a large number of samples from $p(\vx|\vy)$, e.g., to compute statistical estimators, quantify uncertainty, etc. This is in sharp contrast with the latest generative models, such as Consistency Models (CM) \cite{Song2023ConsistencyM, Luo2023LatentCM} and Distilled Diffusion Models (DDM) \cite{Luhman2021KnowledgeDI, salimans2022progressivedistillationfastsampling, Meng2022OnDO}, which can produce remarkable samples from the prior $p(\vx)$ in a few NFEs.   
Some concurrent works explore CMs within a PnP sampling framework~\cite{Garber_2025_CVPR, xu2024consistencymodeleffectiveposterior, li2025decoupleddataconsistencydiffusion}, and one could also consider fine-tuning the models to learn to sample from $p(\vx | \vy)$ instead of $p(\vx)$, e.g., see CoSIGN \cite{Zhao2024CoSIGNFG}. Unfortunately, again we are not aware of any existing imaging methods capable of achieving this for high-resolution latent CMs, exploiting the prompt conditioning.

We aim at developing a method that is simultaneously zero-shot, fast (low number of NFEs), light in GPU memory usage (no auto-grad required), that can sample high-resolution images, and which can be conditioned by a text prompt. In particular, the last two conditions are currently only satisfied by  LDM-based methods TReg and P2L requiring $\ge 200$ NFE per sample as well as the use of additional correction terms. Our two main contributions are:
\begin{enumerate}
    \item We propose LATINO, a new PnP approach that leverages pre-trained text-to-image LCMs to sample from the posterior $p(\vx \mid \vy,c)$ in a gradient-free way, and with few NFEs per sample.
    This allows scaling to large images ($\ge 1024^2$) 
    with low inference time and low GPU memory footprint.
    LATINO is by nature prompt-conditioned, allowing users to input semantic information. LATINO is derived from a careful discretization of a Langevin diffusion and a novel PnP approach tailored for LCMs, which involves stochastic auto-encoders instead of denoisers.

    \item We equip LATINO with automatic prompt optimization (LATINO-PRO), via a stochastic approximation proximal gradient scheme to simultaneously find $\hat{c}(\vy) = \arg\max_{c\in \mathbb{R}^k} \, p(\vy \mid c)$ and sample from $p(\vx \mid \vy,\hat{c}(\vy))$. This method can correct incomplete or misleading prompts while still requiring only 68 NFEs.
\end{enumerate}

\section{Related works}

\paragraph{Data-driven regularization for inverse problem.}
Deep learning regularization methods demonstrated a significant performance improvement over classical inverse problem solvers on a large number of applications.
Denoising-based regularization such as RED and Plug-and-play implicitly defines the prior model via a denoising neural network~\cite{romano2017little, meinhardt2017learning, Zhang2020PlugandPlayIR}. 
Deep generative model can also be used to define the prior model on the solution, 
 including GANs~\cite{ menon2020pulse, pan2021exploiting, daras2021intermediate}, variational autoencoders~\cite{bora2017compressed, gonzalez2022solving, prost2023inverse}, normalizing flows~\cite{cai2023nf}. Denoising diffusion models can also be used repurposed to solve inverse problem by defining an additional guidance term to enforce consistency with measurements~\cite{Chung2022, kawar2022denoising, Song2023PseudoinverseGuidedDM}. In this work, we instead use a latent consistency model.\vspace{-0.1cm}

\paragraph{Latent diffusion inverse problem solvers (LDIS).}
Latent diffusion models are significantly faster than classical diffusion models which operate in the pixel space. 
Various strategies were proposed to condition the latent diffusion process to a measurement, including optimizing the latent iterates to force consistency with the measurements
\cite{Rout2023, he2024faststablediffusioninverse, songsolving}, or using a second-order sampler~\cite{rout2023secondorder}. Unlike the proposed LATINO solver, those methods necessitate to backpropagate through the decoder and the latent score network, and are thus less efficient in terms of memory usage.\vspace{-0.1cm}

\paragraph{Inverse problem solvers with text-to-image generative models.}
Latent diffusion inverse solvers (LDIS) originally relied on null text embeddings $c_{\varnothing}$, while P2L \cite{Chung2023PrompttuningLD} has been the first to try to optimize it to improve the reconstruction. TReg \cite{Kim2023RegularizationBT} instead adopts the CFG \cite{Ho2022ClassifierFreeDG} framework to optimize the null text embedding while still enabling the user to input a custom conditioning prompt.
Both TReg and P2L use the proximal operator of the data fidelity, similar to our LATINO solver, although in those works the proximal operator is motivated by connection with ADMM \cite{Boyd2011DistributedOA,Parikh2013ProximalA}.
 
In Appendix~\ref{sec:related_work} we detail   (text-based) LDIS algorithms.\vspace{-0.1cm}

\paragraph{Langevin PnP inverse problem solvers.} DMs can be used as PnP priors within Langevin Markov chain Monte Carlo samplers. Notably, ~\cite{mbakam2024} embeds a DM within a PnP unadjusted Langevin algorithm (ULA)~\cite{laumont2022bayesian} that computes the posterior mean $\textrm{E}(\vx|\vy)$, whereas \cite{coeurdoux2024plug} embeds a DM within a split-Gibbs sampler~\cite{Vono2019} that is equivariant to a noisy ULA~\cite{Vargas-Mieles2022}. Moreover, ~\cite{Coeurdoux2024} and ~\cite{Melidonis2024} consider PnP priors encoded by a normalizing flow, whereas \cite{Holden2022} provides a general theoretical framework for using VAE and GAN priors. 
\section{Background}
\label{sec:background}

In this section, we will first recall the main properties of DMs and LDMs. Then, we will introduce the Consistency Model formulation that is used by methods like Distribution Matching Distillation (DMD) to sample from the prior distribution in as few as 4 steps. This last ingredient will be crucial for our implementation, allowing for fast IS.

\paragraph{Diffusion Models}
 are generative models designed to learn the reverse of a forward noising process \cite{SohlDickstein2015DeepUL, Ho2020Denoising, Song2020SDE, Song2020improved}. Starting from an initial distribution $p_0(\vx), \vx \in \mathbb{R}^n$, the forward process gradually transforms it into a standard Gaussian distribution $p_T(\vx) = \mathcal{N}(0, \Id)$ as $T \to \infty$. In the variance-preserving (VP) framework proposed by Ho et al. \cite{Ho2020Denoising}, both forward and reverse processes are described using Ito stochastic differential equations (SDEs)
\cite{Song2020SDE}:
\begin{align}
    d \vx_t & = -\frac{\beta_t}{2} \vx_t dt + \sqrt{\beta_t} d\vw, \nonumber\\
    d \vx_t & = \left[ -\frac{\beta_t}{2} \vx_t - \beta_t \nabla_{\vx_t} \log p_t(\vx_t) \right] dt + \sqrt{\beta_t} d\overline{\vw}, \label{eq:VP-SDE}
\end{align}
where $\beta_t$ represents the noise schedule, and $\vw, \overline{\vw}$ are Brownian motions for forward and reverse directions, respectively. The term $\nabla_{\vx_t} \log p_t(\vx_t)$ is typically approximated using a score network $s_{\theta}(\cdot)$, trained with denoising score matching (DSM) \cite{Vincent2011}. The main property of \eqref{eq:VP-SDE} is to possess also an equivalent Probability Flow ODE formulation \cite{Song2020SDE}, in which all the stochasticity is concentrated in the initial step $p_T(\vx)$:

\begin{equation}\label{eq:PF-ODE}
    d \vx_t = \left[ -\frac{\beta_t}{2} \vx_t - \frac{\beta_t}{2} \nabla_{\vx_t} \log p_t(\vx_t) \right] dt.
\end{equation}

While image diffusion models in pixel space $\vx$ are computationally intensive, a more efficient alternative is to operate in a lower-dimensional latent space using an autoencoder \cite{Rombach2021, Kingma2013AutoEncodingVB}. This approach is represented as follows:
\begin{equation*}
    \mathcal{E}: \mathbb{R}^n \mapsto \mathbb{R}^d, \quad \mathcal{D}: \mathbb{R}^d \mapsto \mathbb{R}^n, \quad \vx \approx \mathcal{D}(\mathcal{E}(\vx)),
\end{equation*}
where $\mathcal{E}$ and $\mathcal{D}$ denote the encoder and decoder, respectively, with $d \ll n$. By encoding images into a latent space $\vz = \mathcal{E}(\vx)$ \cite{Rombach2021}, one can train a diffusion model on the reduced representation, significantly decreasing computational costs and making it feasible to model high-resolution images (e.g., $\geq 512^2$ pixels). Latent diffusion models (LDMs) have thus become the standard for generative image models, notably under the name Stable Diffusion (SD).

A key distinction between SD and traditional image diffusion models \cite{Dhariwal2021} lies in the integration of text conditioning $s_{\theta}(\cdot, c)$, where $c$ is the prompt conditioning, which in practice is continuously embedded as a vector by the CLIP text encoder \cite{Radford2021LearningTV}. Since SD is trained on the large-scale LAION-5B dataset \cite{Schuhmann2022LAION5BAO}, which consists of image-text pairs, it can be conditioned at inference time to generate text-aligned images through $s_{\theta}(\cdot, c)$ or classifier-free guidance (CFG) \cite{Ho2022ClassifierFreeDG}.

\paragraph{Consistency Models}

are obtained by distillation of a pre-trained DM or trained from scratch, allow few or even single-step generation~\cite{Song2023ConsistencyM}. CMs are defined by using a so-called \textit{consistency function}, a pushforward map defined as follows:

\begin{definition}[Consistency function]
    Given a small $\eta>0$ and a trajectory $\{ \vx_t\}_{t\in[\eta, T]}$ of the PF-ODE \eqref{eq:PF-ODE}, we define the \textit{consistency function} as $f:(\vx_t,t)\rightarrow\vx_\eta$.   
\end{definition}
Given such a function, it is straightforward to sample $\vx_\eta = f_\theta(\vx_T,T)$ with $\vx_T \sim \mathcal{N}(0,\Id)$ in a single step.
This function is indeed self-consistent in the sense that $f(\vx_t, t) = f(\vx_{t'}, t') \ \ \forall \ t,t' \in [\eta,T]$. 
Given timesteps $t_1 > t_2 > \cdots > t_{N-1} > \eta$, the multistep consistency sampling process is
\begin{align*}
    &\hat{\vx}_T \sim \mathcal{N}(0,\Id), \quad \vx = f_\theta(\hat{\vx}_T, T) \\
    &\text{For } n = 1 \text{ to } N-1: \\
    &\quad \hat{\vx}_{t_n} = \vx + \sqrt{(1 - \alpha_{t_n}) - (1 - \alpha_\eta)}\epsilon \ \text{ with } \  \epsilon \sim \mathcal{N}(0, \Id) \\
    &\quad \vx = f_\theta(\hat{\vx}_{t_n}, t_n),
\end{align*}
where $\alpha_t := \prod_{s=1}^t (1 -\beta_s)$. As done with DMs, it is possible to define a CM in the latent space \cite{Luo2023LatentCM} and train it by distilling a LDM \cite{Luo2023LCMLoRAAU} to obtain a LCM, e.g.  by leveraging a LoRA \cite{Hu2021LoRALA} fine-tuning as done in LCM-LoRA \cite{Luo2023LCMLoRAAU}.


\tikzset{every picture/.style={line width=0.75pt}}         
\begin{figure*}
\begin{tikzpicture}[x=0.75pt,y=0.75pt,yscale=-1,xscale=1]

\draw (57.8,51.5) node  {\includegraphics[width=52.5pt,height=52.5pt]{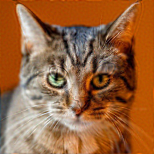}};
\draw (524.3,51.5) node  {\includegraphics[width=52.5pt,height=52.5pt]{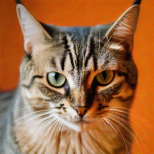}};
\draw  [color={rgb, 255:red, 245; green, 166; blue, 35 }  ,draw opacity=0.33 ][fill={rgb, 255:red, 245; green, 166; blue, 35 }  ,fill opacity=0.33 ] (97,4.5) -- (271,33.14) -- (271,71.32) -- (97,99.96) -- cycle ;
\draw  [color={rgb, 255:red, 80; green, 227; blue, 194 }  ,draw opacity=0.33 ][fill={rgb, 255:red, 80; green, 227; blue, 194 }  ,fill opacity=0.33 ] (481.96,100.42) -- (307.98,71.54) -- (308,33.2) -- (482.02,4.56) -- cycle ;
\draw    (561.3,52.06) -- (602.3,52.54) ;
\draw [shift={(604.3,52.56)}, rotate = 180.67] [color={rgb, 255:red, 0; green, 0; blue, 0 }  ][line width=0.75]    (10.93,-3.29) .. controls (6.95,-1.4) and (3.31,-0.3) .. (0,0) .. controls (3.31,0.3) and (6.95,1.4) .. (10.93,3.29)   ;
\draw (641.8,51.5) node  {\includegraphics[width=52.5pt,height=52.5pt]{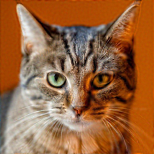}};
\draw  [color={rgb, 255:red, 0; green, 0; blue, 0 }  ,draw opacity=1 ] (102.6,44.98) .. controls (102.6,40.02) and (113.46,36) .. (126.85,36) .. controls (140.24,36) and (151.1,40.02) .. (151.1,44.98) .. controls (151.1,49.94) and (140.24,53.96) .. (126.85,53.96) .. controls (113.46,53.96) and (102.6,49.94) .. (102.6,44.98) -- cycle ;
\draw  [color={rgb, 255:red, 0; green, 0; blue, 0 }  ,draw opacity=1 ] (159.1,44.98) .. controls (159.1,40.02) and (169.96,36) .. (183.35,36) .. controls (196.74,36) and (207.6,40.02) .. (207.6,44.98) .. controls (207.6,49.94) and (196.74,53.96) .. (183.35,53.96) .. controls (169.96,53.96) and (159.1,49.94) .. (159.1,44.98) -- cycle ;
\draw  [color={rgb, 255:red, 0; green, 0; blue, 0 }  ,draw opacity=1 ] (216.6,44.98) .. controls (216.6,40.02) and (227.46,36) .. (240.85,36) .. controls (254.24,36) and (265.1,40.02) .. (265.1,44.98) .. controls (265.1,49.94) and (254.24,53.96) .. (240.85,53.96) .. controls (227.46,53.96) and (216.6,49.94) .. (216.6,44.98) -- cycle ;
\draw  [color={rgb, 255:red, 0; green, 0; blue, 0 }  ,draw opacity=1 ] (430.2,44.68) .. controls (430.2,39.72) and (441.06,35.7) .. (454.45,35.7) .. controls (467.84,35.7) and (478.7,39.72) .. (478.7,44.68) .. controls (478.7,49.64) and (467.84,53.66) .. (454.45,53.66) .. controls (441.06,53.66) and (430.2,49.64) .. (430.2,44.68) -- cycle ;
\draw  [color={rgb, 255:red, 0; green, 0; blue, 0 }  ,draw opacity=1 ] (373.7,44.68) .. controls (373.7,39.72) and (384.56,35.7) .. (397.95,35.7) .. controls (411.34,35.7) and (422.2,39.72) .. (422.2,44.68) .. controls (422.2,49.64) and (411.34,53.66) .. (397.95,53.66) .. controls (384.56,53.66) and (373.7,49.64) .. (373.7,44.68) -- cycle ;
\draw  [color={rgb, 255:red, 0; green, 0; blue, 0 }  ,draw opacity=1 ] (317.2,44.68) .. controls (317.2,39.72) and (328.06,35.7) .. (341.45,35.7) .. controls (354.84,35.7) and (365.7,39.72) .. (365.7,44.68) .. controls (365.7,49.64) and (354.84,53.66) .. (341.45,53.66) .. controls (328.06,53.66) and (317.2,49.64) .. (317.2,44.68) -- cycle ;
\draw    (126.85,53.96) .. controls (143.18,70.54) and (164.5,70.95) .. (182.01,55.2) ;
\draw [shift={(183.35,53.96)}, rotate = 136.24] [color={rgb, 255:red, 0; green, 0; blue, 0 }  ][line width=0.75]    (10.93,-3.29) .. controls (6.95,-1.4) and (3.31,-0.3) .. (0,0) .. controls (3.31,0.3) and (6.95,1.4) .. (10.93,3.29)   ;
\draw    (184.35,53.96) .. controls (200.68,70.54) and (222,70.95) .. (239.51,55.2) ;
\draw [shift={(240.85,53.96)}, rotate = 136.24] [color={rgb, 255:red, 0; green, 0; blue, 0 }  ][line width=0.75]    (10.93,-3.29) .. controls (6.95,-1.4) and (3.31,-0.3) .. (0,0) .. controls (3.31,0.3) and (6.95,1.4) .. (10.93,3.29)   ;
\draw    (341.45,53.66) .. controls (357.78,70.24) and (379.1,70.65) .. (396.61,54.9) ;
\draw [shift={(397.95,53.66)}, rotate = 136.24] [color={rgb, 255:red, 0; green, 0; blue, 0 }  ][line width=0.75]    (10.93,-3.29) .. controls (6.95,-1.4) and (3.31,-0.3) .. (0,0) .. controls (3.31,0.3) and (6.95,1.4) .. (10.93,3.29)   ;
\draw    (397.95,53.66) .. controls (414.28,70.24) and (435.6,70.65) .. (453.11,54.9) ;
\draw [shift={(454.45,53.66)}, rotate = 136.24] [color={rgb, 255:red, 0; green, 0; blue, 0 }  ][line width=0.75]    (10.93,-3.29) .. controls (6.95,-1.4) and (3.31,-0.3) .. (0,0) .. controls (3.31,0.3) and (6.95,1.4) .. (10.93,3.29)   ;

\draw (50,0) node [anchor=north west][inner sep=0.75pt]  [font=\small]  {$x^{(k-1)}$};
\draw (515,0) node [anchor=north west][inner sep=0.75pt]  [font=\small]  {$\Tilde{x}^{(k)}$};
\draw (630,0) node [anchor=north west][inner sep=0.75pt]  [font=\small]  {$x^{(k)}$};
\draw (561,59) node [anchor=north west][inner sep=0.75pt]  [font=\scriptsize]  {$prox_{\delta_{k} g_y}$};
\draw (273,46.7) node [anchor=north west][inner sep=0.75pt]  [font=\tiny]  {$z_{t} =z_{t}'$};
\draw (110.6,41.4) node [anchor=north west][inner sep=0.75pt]  [font=\small]  {\scalebox{.9}{$\quad x$}};
\draw (166.1,41.4) node [anchor=north west][inner sep=0.75pt]  [font=\small]  {\scalebox{.9}{$\quad z_{0}$}};
\draw (224.6,41.4) node [anchor=north west][inner sep=0.75pt]  [font=\small]  {\scalebox{.9}{$\quad z_{t}$}};
\draw (438.7,37.4) node [anchor=north west][inner sep=0.75pt]  [font=\small]  {\scalebox{.9}{$\quad x'$}};
\draw (380.7,37.4) node [anchor=north west][inner sep=0.75pt]  [font=\small]  {\scalebox{.9}{$\quad z_{0}'$}};
\draw (324.2,37.4) node [anchor=north west][inner sep=0.75pt]  [font=\small]  {\scalebox{.9}{$\quad z_{t}'$}};
\draw (145,68) node [anchor=north west][inner sep=0.75pt]  [font=\tiny]  {$\mathcal{E}( \cdot )$};
\draw (418.2,68) node [anchor=north west][inner sep=0.75pt]  [font=\tiny]  {$\mathcal{D}( \cdot )$};
\draw (189,68) node [anchor=north west][inner sep=0.75pt]  [font=\tiny]  {$p( z_{t} \ \mid z_{0})$};
\draw (348,68) node [anchor=north west][inner sep=0.75pt]  [font=\tiny]  {$G_{\theta }( \cdot ,t,c)$};

\end{tikzpicture}
\caption{
One step of the LATINO solver, a discretization of the Langevin SDE~\eqref{eq:Langevin} which targets the posterior $p(\vx|\vy, c)$. The current iterate $\vx_k$ is encoded by the VAE encoder and propagated forward via a noising diffusion kernel $p(\vz_t|\vz_0)$. This process is then reversed via the latent consistency model and the VAE decoder, followed by the proximal operator to involve the likelihood $p(\vy|\vx)$.
}
\label{fig:LATINO}
\end{figure*}
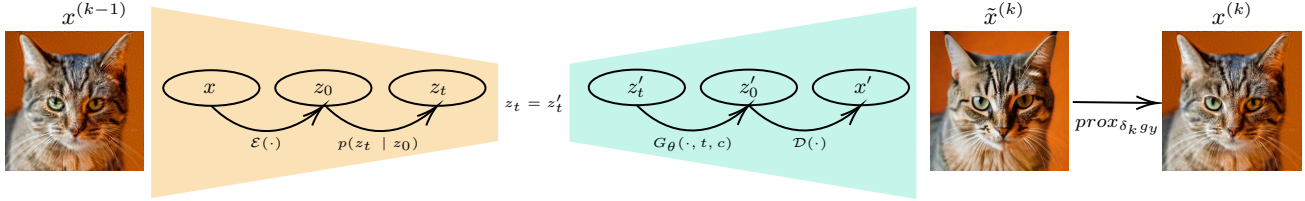


A recent paradigm to distill a LDM into a one-step generator is the Distribution Matching Distillation (DMD) \cite{Yin2023OneStepDW} in which a generator $G_\theta$ with the same architecture of the DM denoiser is trained. DMD distills many-step diffusion models into a one-step generator $G$ by minimizing the expectation over $t$ of approximate Kullback-Liebler (KL) divergences between the diffused latent target distribution $p_{\text {real }, t}$ and the diffused latent generator output distribution $p_{\text {fake }, t}$. Since DMD trains $G$ by gradient descent, it only requires the gradient of this loss, which can be computed as the difference of two score functions. 
DMD uses a frozen pre-trained DM $\mu_{\text {real }}$ (the teacher), and dynamically updates a fake DM $\mu_{\text {fake }}$ while training $G_\theta$, using a denoising score-matching loss on samples from the one-step generator, i.e., fake data.

DMD's improved version,  DMD2 \cite{Yin2024ImprovedDM}, adds during training a GAN-based loss term, which increases the stability and quality, allowing the obtained distilled model to beat his teacher one. Furthermore, the generator $G_\theta$ can be conditioned on $t$ leading to  $G_\theta(\vz_{t_i},t_i)$ that predicts the final $\vz_0$ given a noisy input $\vz_{t_i}$. By deploying the generator following the CM framework $\hat{\vz}_{0\mid t_i} = G_\theta(\vz_{t_i},t_i)$ and $\vz_{t_{i+1}} = \sqrt{\alpha_{t_{i+1}}}\hat{\vz}_{0\mid t_i} + \sqrt{1 - \alpha_{t_{i+1}}}\epsilon$, one obtains a 4-step sampler. In practice, what is done is to consider $G_\theta$ as the pre-trained SDXL U-Net \cite{Podell2023SDXLIL}, that is fine-tuned for the times $t_i \in \{ 999, 749, 499, 249 \}$, the 4 used by the CM. The final $\hat{\vz}$ is then passed through the decoder to get $\hat{\vx} = \mathcal{D}(\hat{\vz})$.


\section{LATINO}\label{sec:LATINO}
\subsection{A new Plug \& Play Langevin sampling method}

We propose a LAtent consisTency INverse sOlver (LATINO) to approximately draw samples from the posterior distribution $p(\vx|\vy,c)$, parametrized by the degraded observation $\vy$ and a text prompt $c$. Following a PnP philosophy, LATINO is constructed by combining an analytical likelihood function $p(\vy|\vx)$ with a prior distribution $p(\vx|c)$ that is implicitly encoded by an text-to-image CM, obtained by distilling a foundational SD model. However, unlike the conventional PnP approach that exploits a denoising operator in lieu of a gradient or proximal mapping of a clean image prior $p(\vx)$, LATINO pioneers a new PnP approach suitable for leveraging generative models such as CMs. In this  paradigm, the CM is first auto-encoded to define a Markov kernel that admits $p(\vx|c)$ as invariant distribution. The kernel is subsequently used in place of an SDE targeting $p(\vx|c)$ within a stochastic sampler for $p(\vx|\vy,c)$.

More precisely, consider an overdamped Langevin diffusion process to sample from $p(\vx|\vy,c)$, given by the SDE
\begin{equation}
    \textrm{d}\vx_s = \nabla\log p(\vy|\vx_s)\textrm{d}s + \nabla\log p(\vx_s|c)\textrm{d}s + \sqrt{2}\textrm{d}\bm{w}_s\,, \label{eq:Langevin}
\end{equation}
where $\bm{w}_s$ denotes a $n$-dimensional Brownian motion. Under mild regularity assumptions on $p(\vx|\vy,c)$~\cite{Durmus2017}, starting from an initial condition $\vx_0$, the process $\vx_s$ converges to $p(\vx|\vy,c)$ exponentially fast as $s$ increases ~\cite{Durmus2017}. While solving \eqref{eq:Langevin} exactly is not generally possible, it provides a powerful computational engine to derive samplers that approximately sample $p(\vx|\vy,c)$ by mimicking $\vx_s$. For example, an Euler-Maruyama approximation of \eqref{eq:Langevin} leads to the widely used unadjusted Langevin algorithm (ULA)~\cite{Durmus2017}. 

LATINO stems from the following approximation of \eqref{eq:Langevin}
\begin{align}
    {\vu} = \vx_k + \int^{\delta}_{0} &\nabla\log p(\tilde{\vx}_s|c)\textrm{d}s + \sqrt{2}\textrm{d}\bm{w}_s\,,\, \tilde{\vx}_{0} = \vx_k\,,\nonumber\\
    \vx_{k+1} & = {\vu} + \delta \nabla\log p(\vy|\vx_{k+1})\,, \label{eq:split-Langevin}
\end{align}
where we note the first half of the splitting, encoded in ${\vu}$, solves \eqref{eq:Langevin} exactly with likelihood term removed, and the second step involves the likelihood via an implicit Euler step. Note that \eqref{eq:split-Langevin} has two key advantages over ULA: it is potentially very accurate, as the first step introduces no discretisation bias; and it is numerically stable for all $\delta >0$, so it can be made to converge arbitrarily quickly by increasing $\delta$, at the expense of some bias. Conversely, ULA only integrates the Brownian term $\bm{w}_s$ exactly, it involves the drift via an explicit Euler step, and is explosive unless $\delta$ is sufficiently small. Also note that the implicit Euler step in \eqref{eq:split-Langevin} can be reformulated as an explicit proximal point step, i.e., $\tilde{\vx}_{k+1} = \operatorname{prox}_{\delta g_y}(\tilde{\vx}_{k+1})$, where $g_y : \vx \mapsto -\log p(\vy|\vx)$\cite{pereyra2016proximal}, which is computationally tractable for common inverse problems such as deblurring or super-resolution\footnote{Note that  {\scriptsize
$\operatorname{prox}_{\delta g_y}(\vx) = \left(\delta\mathcal{A}^t\mathcal{A} + \sigma_n^{2} \Id\right)^{-1}\left(  \delta^{1}A^t \vy +\sigma_n^{2}\vx\right)$}, which can be computed cheaply when we have access to the singular value decomposition of $\mathcal{A}$ (see, e.g.,~\cite{Zhang2020PlugandPlayIR}), or by using a specialised subiteration.}.

The computation of the intermediate step ${\vu}$ from $\vx_k$ in \eqref{eq:split-Langevin} is usually intractable. LATINO circumvents this difficulty by noting that this step defines a contractive Markov kernel with $p(\vx|c)$ as unique invariant distribution. Our proposed PnP approach replaces this kernel in \eqref{eq:split-Langevin} with a Stochastic Auto-Encoder (SAE) 
derived from a CM, which we design such that it also contracts random variables towards $p(\vx|c)$.

\subsection{Auto-encoding stable diffusion}

Consider a distilled SD model which samples from $p(\vx|c)$ by using a CM, $G_\theta$, operating on the latent space of a deterministic auto-encoder $(\mathcal{E},\mathcal{D})$. Assume $G_\theta$ solves the flow
\begin{equation}\label{eq:PF-ODE-Z}
    d \vz_t = \left[ -\frac{\beta_t}{2} \vz_t - \frac{\beta_t}{2} \nabla_{\vz_t} \log p_t(\vz_t) \right] dt\,,
\end{equation}
which generates a set of distributions $p_0(\vz_0|c)$ by pushing forward a latent normal random variable through $G_\theta(\cdot,T,c)$ for a sufficiently large $T$, which in turn produces the target distribution $p(\vx|c)$ through the action of the decoder $\mathcal{D}$.

Our aim is to use $G_\theta$ and 
$(\mathcal{E},\mathcal{D})$
to construct a SAE $(\mathfrak{E}_t, \mathfrak{D}_{t,c})$ such that it contracts random variables on $\mathbb{R}^n$ towards $p(\vx|c)$, its fixed point. In order to achieve this, we define the stochastic encoder parametrized by $t \in (0,T]$
$$
\mathfrak{E}_t:\quad \vz_t|\vx \sim \mathcal{N}(\sqrt{\alpha_t}\mathcal{E}(\vx),(1-\alpha_t)\Id_d)\,.
$$
$\mathfrak{E}$ maps $\vx$ to the output $\vz_t$ by first applying the deterministic encoder to compute $\vz_0 = \mathcal{E}(\vx)$ and subsequently drawing $\vz_t$ conditionally to $\vz_0$, as given by the SDE  $d\vz_t = -\frac{\beta_t}{2} \vz_t dt + \sqrt{\beta_t} d\vw^\prime$, which inverts the probability flow \eqref{eq:PF-ODE-Z}~\cite{Song2023ConsistencyM}. 

This stochastic encoder $\mathfrak{E}_t$ is paired with the decoder $\mathfrak{D}_{t,c}$, which takes as input $\vz^\prime_t = \vz_t$ and pushes forward to the ambient space through the deterministic mapping
$$
\mathfrak{D}_{t,c}:\quad \vx^\prime = \mathcal{D}(G_\theta(\vz^\prime_t,t,c))\,.
$$
The proposed AE architecture is summarized in Figure \ref{fig:LATINO}.

Assume that $\mathcal{E}$ inverts $\mathcal{D}$ exactly. If $\vx$ is distributed according to $p(\vx|c)$, then, for any $t >0$, $\vz_t = \mathfrak{E}_t(\vx)$ has distribution $p_t$ as described by the flow \eqref{eq:PF-ODE-Z}. Equally, if $\vz^\prime_t$ follows $p_t$, then $\vz^\prime_0 = G_\theta(\vz^\prime_t,t,c)$ follows $p_0$ and hence $\vx^\prime = \mathfrak{D}(\vz^\prime_t,t)$ follows $p(\vx|c)$, implying that $p(\vx|c)$ is a fixed point as required (in praxis, for numerical stability, $G_\theta$ targets $p_\tau(\vz_\tau|c) \approx p_0(\vz_0|c)$ for $\tau > 0$ small and $\mathcal{E}$ does not perfectly invert $\mathcal{D}$, leading to some small bias).

Alternatively, if $\vx$ is not distributed according to $p(\vx|c)$, $(\mathfrak{E}_{t},\mathfrak{D}_{t,c})$ will transport it towards $p(\vx|c)$. In this situation, the parameter $t$ plays a key role, analogous to a regularization parameter. Suppose that $t$ is very large, such that the encoded random variable $\vz_t \approx \mathcal{N}(0,\Id)$ regardless of the distribution of $\vx$. Then $\mathfrak{D}$  behaves as a standard SD generative model and output $\vx^\prime$ distributed according to $p(\vx|c)$. Conversely, now suppose that $t$ is very small. In that case, 
$(\mathfrak{E}_{t},\mathfrak{D}_{t,c})$ reduce to $(\mathcal{E},\mathcal{D})$. Because of the consistency property of $G_\theta$, we have $\vx^\prime\approx \mathcal{D}[\mathcal{E}(\vx)]$. This implies that $\vx^\prime \approx \vx$ if the distribution of $\vx$ is concentrated on the range of $\mathcal{D}$. For intermediate values of $t$, not infinitely large or infinitesimally small, the stochastic autoencoder $\vx$ is transported some extent towards $p(\vx|c)$. We conclude that $(\mathfrak{E},\mathfrak{D})$ has the desired fixed point and contraction properties, with $t$ controlling the contraction strength analogously to  the step-size or integration period $\delta$ in \eqref{eq:split-Langevin}. This is summarized in Figure \ref{fig:LATINO-tests} and by the recursion
\begin{align}
    \tilde{\vx}_{k+1} &= \mathfrak{D}_{t,c}\circ\mathfrak{E}_t(\vx_{k})\,,\nonumber\\
    {\vx}_{k+1} & = \operatorname{prox}_{\delta g_y}(\tilde{\vx}_{k+1})\,. \label{eq:pnp-split-Langevin}
\end{align}
Note that, because \eqref{eq:pnp-split-Langevin} approximates \eqref{eq:split-Langevin}, $t$ and $\delta$ must be in agreement, similarly to conventional PnP methods that need to balance the step-size and the denoiser's parameters~\cite{laumont2022bayesian}.

\begin{figure}
    \centering
    \includegraphics[width=\linewidth]{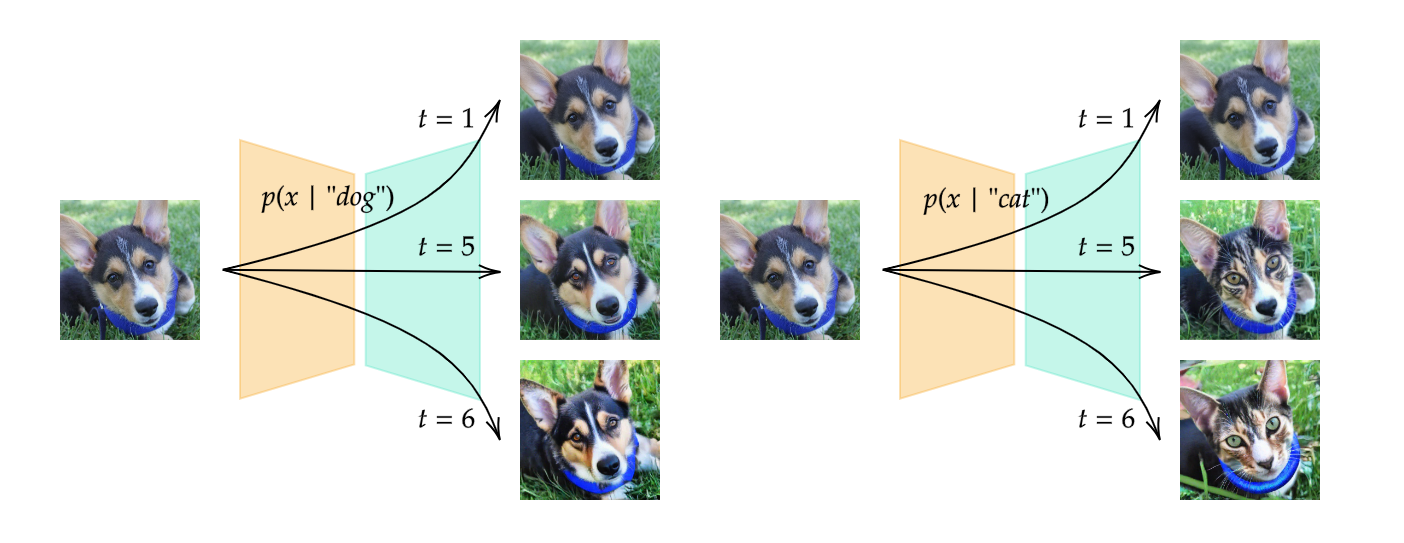}\vspace{-0.2cm}
    \caption{SAE applied to images in and out of distribution for different values of $t$, illustrating contraction towards $p(\vx|c)$.}
    \label{fig:LATINO-tests}
\end{figure}

\subsection{The LATINO algorithm} Algorithm~\ref{alg:lcm-pir} below summarizes the proposed LATINO sampling scheme, which is based on recursion \eqref{eq:pnp-split-Langevin} (see also Figure \ref{fig:LATINO}). We warm-start the algorithm by using the observation-informed initialization $\vx^{(0)} = \mathcal{A}^\dagger\vy$. For computational efficiency, we use an annealing strategy that involves iteration-dependent parameters $t_k,\delta_k$. We now introduce the choice of $t_k$ and leave to Appendix \ref{sec:hyperparam} the details on $\delta_k$. In our experiments, we find that $N=8$ suffice to deliver remarkably accurate samples when $G_\theta$ is given by the DMD2 \cite{Yin2024ImprovedDM}.

\begin{algorithm}
\caption{LATINO}
\label{alg:lcm-pir}
\begin{algorithmic}[1]
    \State \textbf{given} $\vx^{(0)} = \mathcal{A}^\dagger\vy$, text prompt $c$, number of step $N=4$ or $8$, latent consistency model $G_\theta$, latent space decoder $\mathcal{D}$, latent space encoder $\mathcal{E}$, sequences $\{t_k,\delta_k\}_{k=1}^N$.
    \For{$k = 1,\ldots,N$}
        \State $\boldsymbol{\epsilon} \sim \mathcal{N}(0, \text{Id})$
        \State $\vz_{t_k}^{(k)} \gets \sqrt{\alpha_{t_{k}}}\encoder(\vx^{(k-1)})+ \sqrt{1 - \alpha_{t_{k}}} \boldsymbol{\epsilon}$ \Comment{Encode}
        \State $\vu^{(k)} \gets \decoder(G_\theta(\vz_{t_k}^{(k)},t_k,c))$ \Comment{Decode}
        \State $\vx^{(k)} \gets \operatorname{prox}_{\delta_k g_y} (\vu^{(k)})$ \Comment{$g_{\vy}: \vx \mapsto -\log p(\vy|\vx)$}
    \EndFor
    
    \State \Return $\vx^{(N)}$
\end{algorithmic}
\end{algorithm}
\vspace{-0.2cm}
\paragraph{The LCM prior choice.}
We use DMD2 \cite{Yin2024ImprovedDM}, a pretrained model based on the SDXL architecture \cite{Podell2023SDXLIL}. That said, our method is agnostic to the choice of LCM; e.g., see appendix~\ref{sec:priors} for a comparison of our method using DMD2 and SD1.5-LoRA. We can either consider the original DMD2 4-step prior or an 8-step variant of DMD2, where the timesteps are $t_i \in \{ 999, 874, 749, 624, 499, 374, 249, 124\}$. We mix the 4 distilled steps $t_i \in \{ 999, 749, 499, 249\}$ with the 4 steps from the original SDXL model. This choice has been adopted to increase the quality of reconstructions by doubling the number of steps. Still, 4 distilled time steps gives excellent results, and we make use of this version in our LATINO-PRO scheme that self-calibrates the prompt $c$.

Interestingly, if instead of an LCM we construct $(\mathfrak{E}_t,\mathfrak{D}_{t,k})$ with a DM we recover the DM-based PnP-ULA scheme of~\cite{mbakam2024} as a special case, whereas if we use the stochastic denoiser ~\cite{Renaud2024} we recover a PnP-ULA~\cite{laumont2022bayesian}.

\section{LATINO-PRO: Prompt optimization via maximum marginal likelihood estimation}\label{sec:SOUL}

State-of-the-art generative models such as DMD2 provide a highly informative prior $p(\vx|c)$ that concentrates its probability mass in very narrow regions of the solution space. Conversely, in inverse problems that are ill-conditioned or ill-posed, the likelihood function $p(\vy|\vx)$ has identifiability issues and is only weakly informative as a result. Hence, the prior $p(\vx|c)$ dominates key aspects of the posterior $p(\vx|\vy,c)$. Consequently, it is essential to calibrate $p(\vx|c)$ by identifying a suitable text prompt $c$, as otherwise $p(\vx|\vy,c)$ exhibits strong bias. Unfortunately, setting $c$ a priori is difficult.

Adopting an empirical Bayesian approach, we propose to set $c$ by maximum marginal likelihood estimation (MMLE),
\begin{equation}\label{eq:MMLE}
    \hat{c}(\vy) = \arg\max_{c\in \mathbb{R}^k} \, p(\vy \mid c)\,,
\end{equation}
followed by inference with the empirical Bayesian posterior distribution $p(\vx\mid\vy,\hat{c}(\vy))$. The marginal likelihood
\begin{equation*}
p(\vy \mid c) = \textrm{E}_{\vx|c} [p(\vy|\vx)]\,
\end{equation*}
measures the model's goodness-of-fit to the data $\vy$.
A main challenge in solving \eqref{eq:MMLE} is that the likelihood $p(\vy \mid c)$ is computationally intractable, as it requires integration over the solution space. We address this difficulty by using LATINO to implement a stochastic approximation proximal gradient (SAPG) scheme which seeks to simultaneously solve \eqref{eq:MMLE} and draw samples from $p(\vx\mid\vy,\hat{c}(\vy))$.

More precisely, in a manner akin to \cite{Bortoli2019EfficientSO}, we consider a stochastic optimization scheme that mimics the following projected gradient algorithm to solve \eqref{eq:MMLE}
\begin{equation}\label{eq:PG}
c_{m+1} = \Pi_C [c_{m} + \gamma_m \nabla_c \log p(\vy \mid c_m)]
\end{equation}
where $\gamma_m$ is a sequence of decreasing positive step-sizes and $C \subset \mathbb{R}^k$ is a convex set of admissible values for $c$. While \eqref{eq:PG} is also intractable, from Fisher's identity, the gradient 
\begin{align*}
\nabla_c \log p(\vy \mid c) &= \textrm{E}_{\vx|\vy,c}[\nabla_c \log p(\vy,\vx \mid c)]\,,
\\ &= \textrm{E}_{\vx|\vy,c}[\nabla_c \log p(\vx \mid c)]\,,
\end{align*}
which suggests implementing a stochastic variant of \eqref{eq:PG} with 
\begin{equation*}
\nabla_c \log p(\vy \mid c_m) \approx \nabla_c \log p(\vx^{(1)},\ldots,\vx^{(N)} \mid c_m),
\end{equation*}
where $\{\vx^{(k)}\}_{k=1}^N$ is a Markov chain targeting $p(\vx|\vy,c_m)$, as generated by LATINO (Alg. \ref{alg:lcm-pir}). The resulting scheme alternates between using LATINO to generate a batch $\{\vx^{(k)}\}_{k=1}^N$ of samples, and then updating $c_m$ through the recursion
\begin{equation}\label{eq:SAPG}
c_{m+1} = \Pi_C \left[c_{m} \hspace{-1pt}+\hspace{-1pt} \gamma_m \nabla_c \log p(\vx^{(1)}\hspace{-1pt},\ldots,\vx^{(N)}|c_m)\right].
\end{equation}
In praxis, the computation of \eqref{eq:SAPG} is tractable by automatic differentiation on the latent space (see Appendix \ref{sec:gradient_LATINOPRO} for details). LATINO with prompt optimization (LATINO-PRO), is summarized in Algorithm \ref{alg:latino-pro} below. Under some technical assumptions, \cite[Theorem 5]{Bortoli2019EfficientSO} guarantees that the iterates $c_m$ generated by SAPG schemes driven by Langevin samplers converge to $\hat{c}(\vy)$ as $m$ increases. LATINO-PRO automatically verifies the conditions of \cite[Theorem 5]{Bortoli2019EfficientSO} under simplifying assumptions on target $p(\vx|\vy,c)$ and $(\mathfrak{E}_t,\mathfrak{D}_{t,c})$ (e.g., smoothness and log-concavity of $p(\vx|\vy,c)$).  Extending \cite[Theorem 5]{Bortoli2019EfficientSO} to derive convergence guarantees for LATINO-PRO under realistic assumptions for CM-based priors is a main perspective for future work.

We recommend warm-starting LATINO-PRO by specifying an initial value for $c_0$ that describes the expected solution (e.g., ``a sharp photo of a dog''). We find that the quality of the samples generated by LATINO-PRO increases rapidly as we start to update $c_m$, and that stopping iterations early introduces some regularization, which is beneficial. In our experiments, we perform $M=15$ iterations of \eqref{eq:SAPG}, with $N=4$ sub-iterations of Algorithm \ref{alg:lcm-pir}, except in the final stage when we use $N=8$ iterations of Algorithm \ref{alg:lcm-pir} to obtain a higher quality sample from $p(\vx|\vy,\hat{c}(\vy))$.

\begin{algorithm}[htbp]
\caption{LATINO-PRO}
\label{alg:latino-pro}
\begin{algorithmic}[1]
    \State \textbf{given} $\vx^{(0)} = \mathcal{A}^\dagger\vy$, text prompt $c_0$ and admissible set $C$, number of step SAPG steps $M$, sub-iteration parameters $\{N_m,\gamma_m\}_{m=1}^M, \{t_k,\delta_k\}_{k=1}^{N_m}$, latent consistency model $G_\theta$, latent space decoder $\mathcal{D}$ and encoder $\mathcal{E}$.
    \For{$m = 1,\ldots,M$}
    \For{$k = 1,\ldots,N_m$} \Comment{LATINO}
        \State $\boldsymbol{\epsilon} \sim \mathcal{N}(0, \text{Id})$
        \State $\vz_{t_k}^{(k)} \gets \sqrt{\alpha_{t_{k}}}\encoder(\vx^{(k-1)})+ \sqrt{1 - \alpha_{t_{k}}} \boldsymbol{\epsilon}$ 
        \State $\vu^{(k)} \gets \decoder(G_\theta(\vz_{t_k}^{(k)},t_k,c_m))$ 
        \State $\vx^{(k)} \gets \operatorname{prox}_{\delta_k g_y} (\vu^{(k)})$
    \EndFor
    \State $h(c_m) \gets \nabla_c \log p(\vz_{t_1}^{(1)},\ldots,\vz_{t_{N_m}}^{(N_m)}|c_m)$
    \State $c_{m+1} = \Pi_C \left[c_{m} + \gamma_m h(c_m)\right]$ \Comment{SAPG}
    \State $\vx^{(0)} \gets \vx^{(N_{m})}$ \Comment{Carry state forward}
    \EndFor
    \State \Return $\vx^{(N_M)}$
\end{algorithmic}
\end{algorithm}
\section{Experiments}
\label{sec:experiments}

\begin{figure*}[!h]
\centering
\hspace*{-0.4cm}\begin{minipage}{0.03\textwidth}
    \centering
    \begin{tabular}{c}
        \rotatebox{90}{\textbf{\smaller SR $\times 16 \quad \ \ \ $}} \\[10.5mm]
        \rotatebox{90}{\textbf{\smaller Gaussian deblur}}
    \end{tabular}
\end{minipage}%
\begin{minipage}{0.135\textwidth}
    \centering \textbf{Measurement} \\
    \zoomedImage{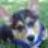}{0.15,0.1}{1.2,0.7} \\
    \zoomedImage{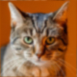}{-0.15,0.0}{1.2,0.7}
\end{minipage}%
\begin{minipage}{0.135\textwidth}
    \centering \textbf{GT} \\ 
    \zoomedImage{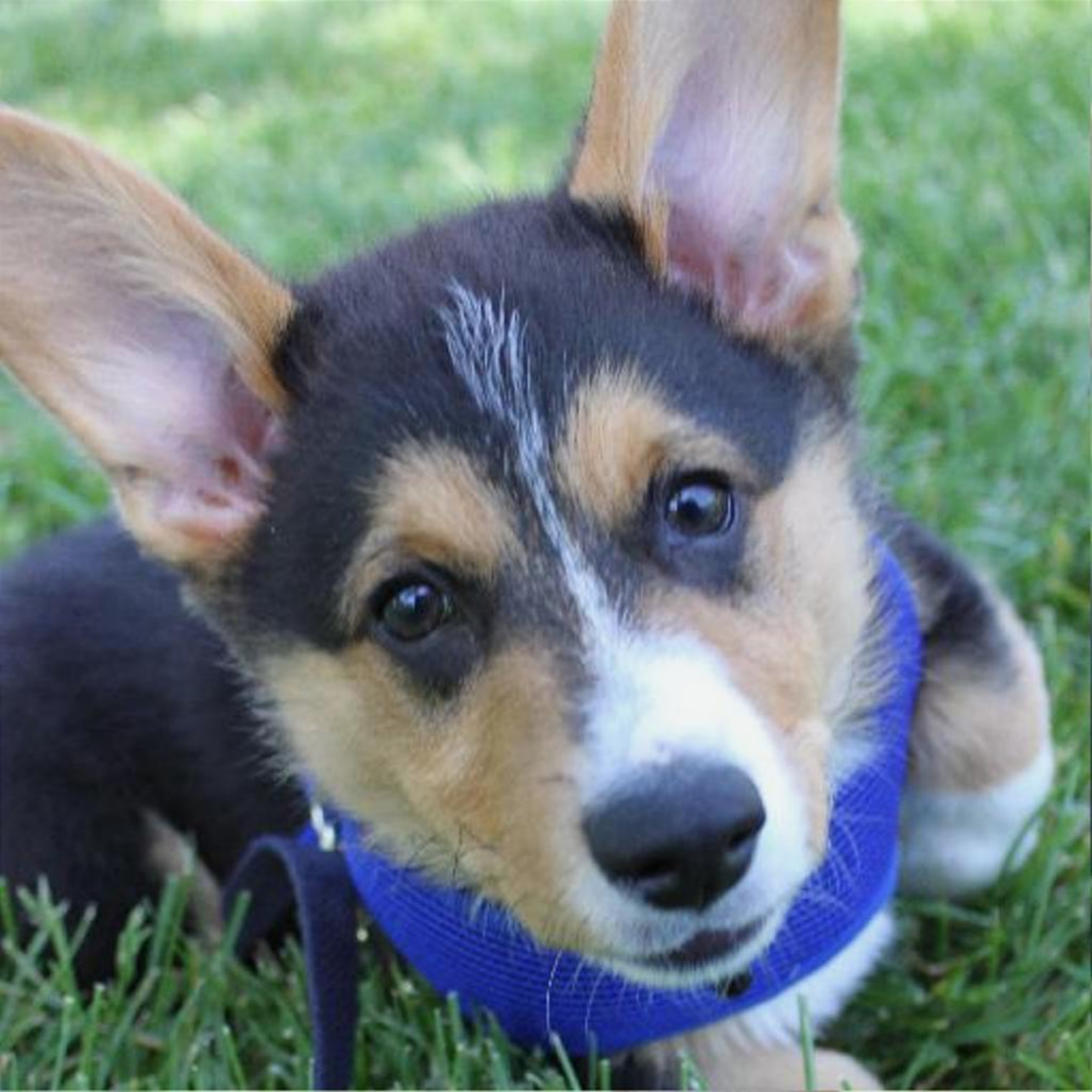}{0.15,0.1}{1.2,0.7} \\
    \zoomedImage{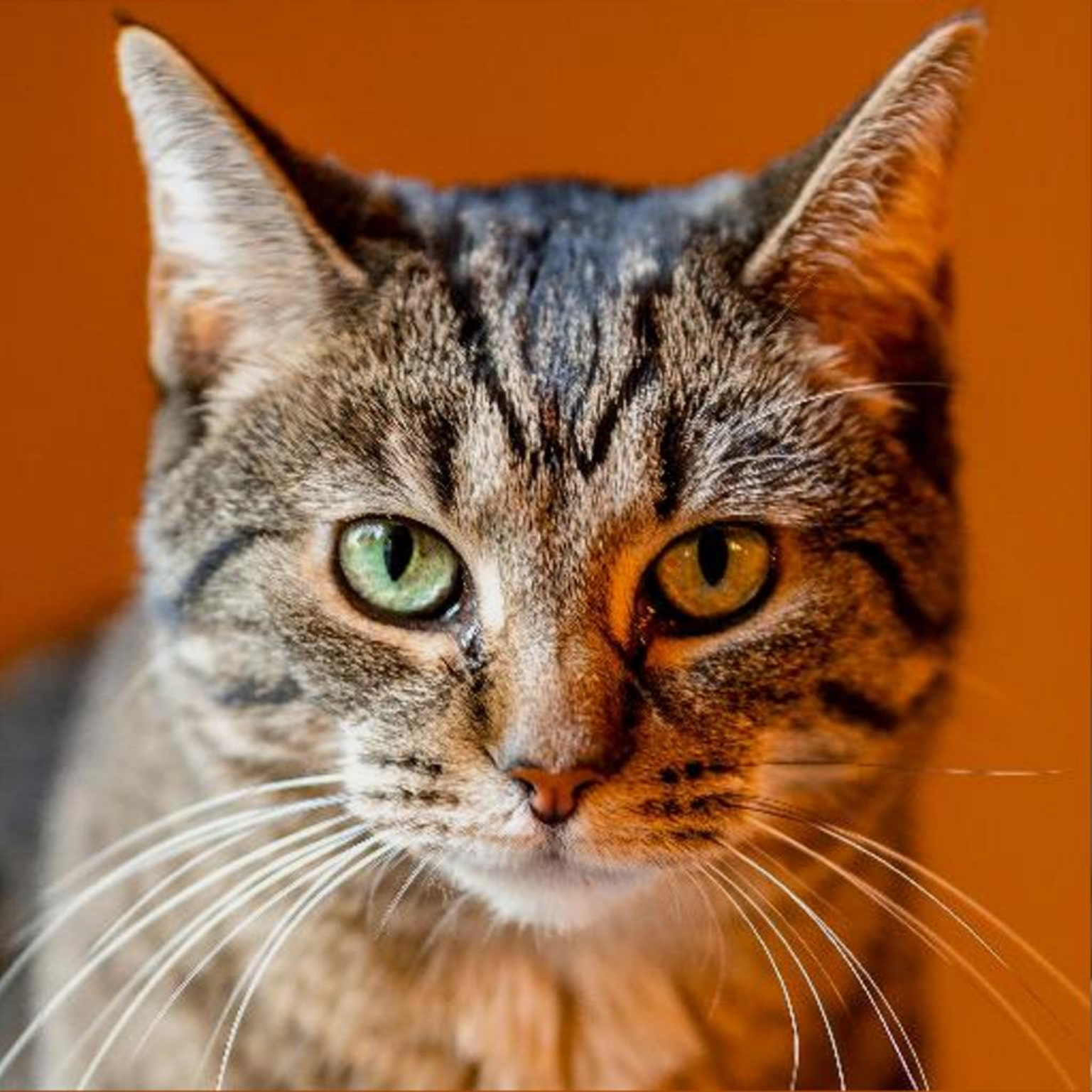}{-0.15,0.0}{1.2,0.7}
\end{minipage}%
\begin{minipage}{0.135\textwidth}
    \centering \textbf{LATINO} \\ 
    \zoomedImage{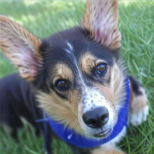}{0.15,0.1}{1.2,0.7} \\
    \zoomedImage{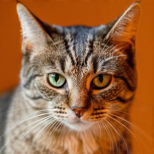}{-0.15,0.0}{1.2,0.7}
\end{minipage}%
\begin{minipage}{0.135\textwidth}
    \centering \textbf{LATINO-PRO} \\ 
    \zoomedImage{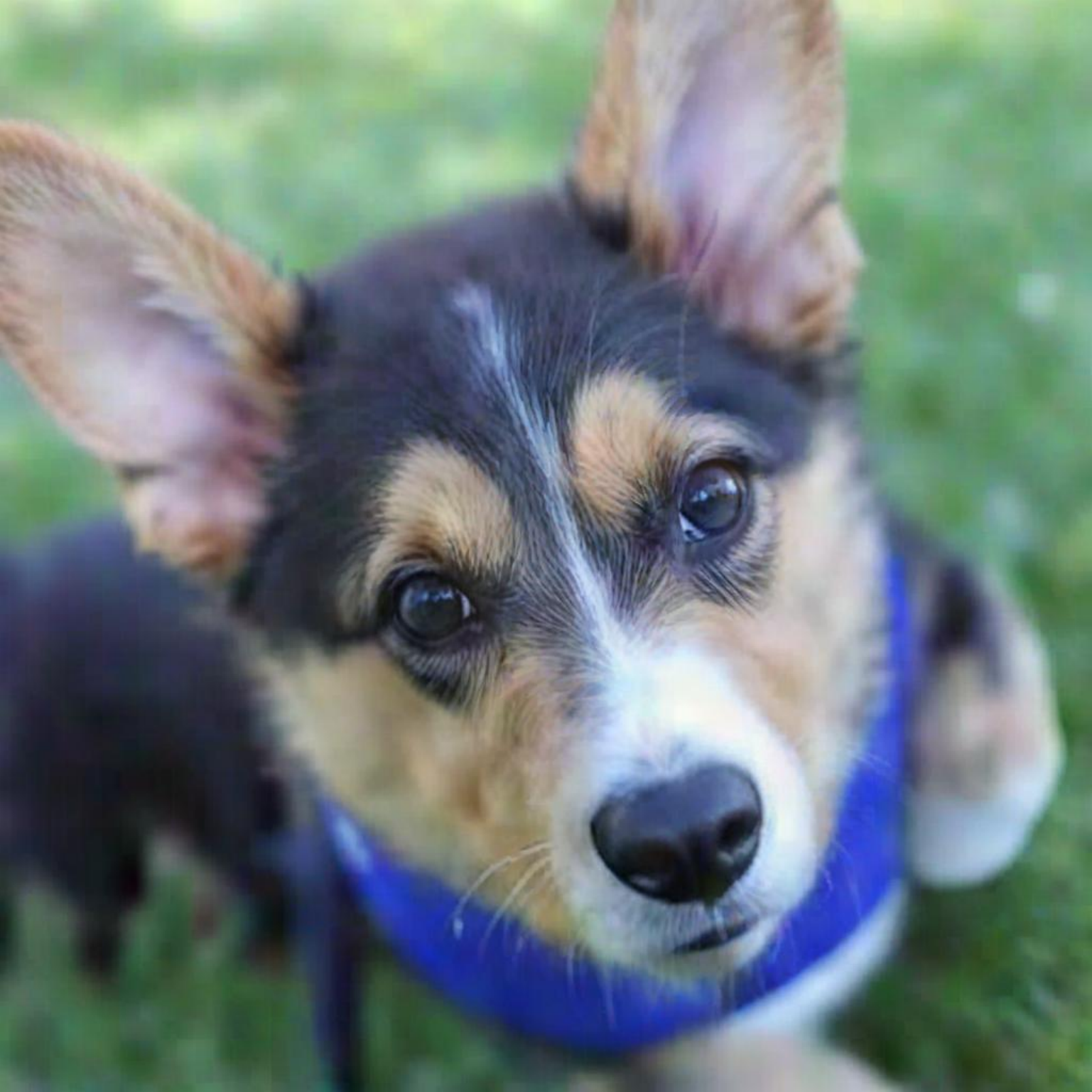}{0.15,0.1}{1.2,0.7} \\
    \zoomedImage{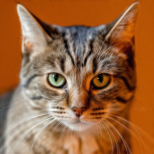}{-0.15,0.0}{1.2,0.7}
\end{minipage}%
\begin{minipage}{0.135\textwidth}
    \centering \textbf{TREG} \\ 
    \zoomedImage{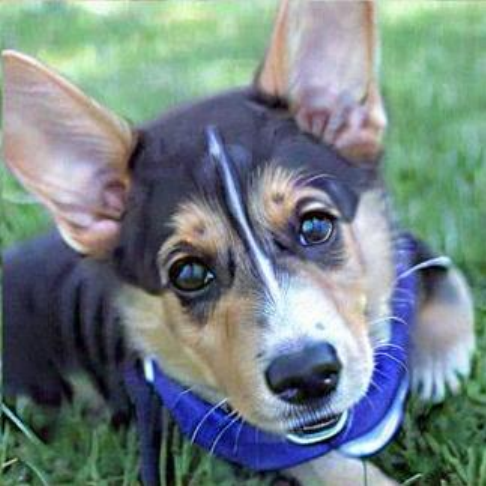}{0.15,0.1}{1.2,0.7} \\
    \zoomedImage{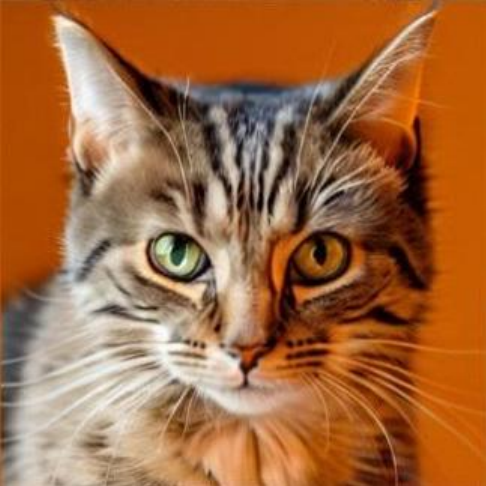}{-0.15,0.0}{1.2,0.7}
\end{minipage}%
\begin{minipage}{0.135\textwidth}
    \centering \textbf{P2L} \\ 
    \zoomedImage{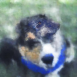}{0.15,0.1}{1.2,0.7} \\
    \zoomedImage{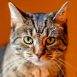}{-0.15,0.0}{1.2,0.7}
\end{minipage}%
\begin{minipage}{0.135\textwidth}
    \centering \textbf{PSLD} \\ 
    \zoomedImage{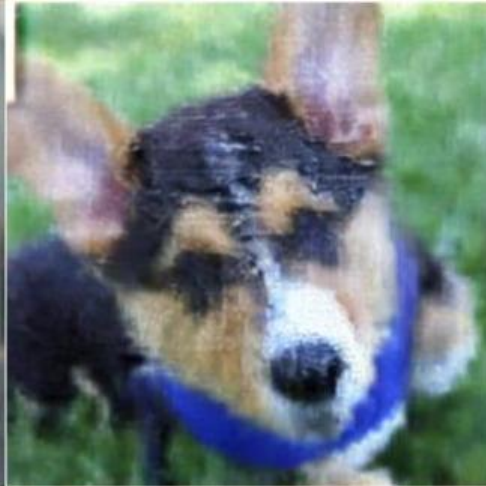}{0.15,0.1}{1.2,0.7} \\
    \zoomedImage{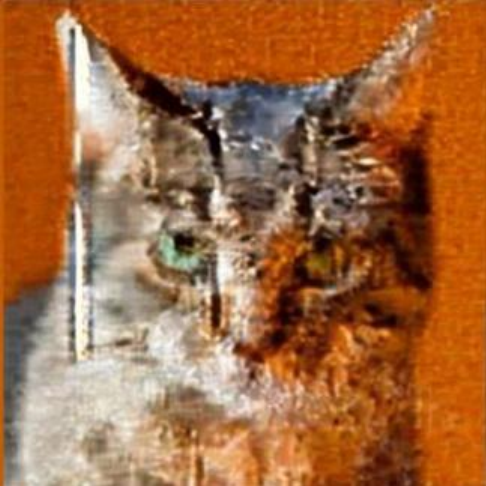}{-0.15,0.0}{1.2,0.7}
\end{minipage}\vspace*{-0.6cm}
\caption{Comparison of image restorations. Samples taken from AFHQ-512. Prompts: \texttt{a sharp photo of a dog} (resp.  \texttt{a cat}).}
\label{fig:qualitative_comparison_AFHQ}
\end{figure*}

\paragraph{Datasets and Models.}

We consider two high-quality (HQ) datasets to test our models: FFHQ \cite{Karras2018ASG} and AFHQ. For FFHQ, we consider both the $1024\times1024$ and $512\times512$ versions, depending on the objective of the tests conducted, and we take the first 1k test images, as done by \cite{Chung2023PrompttuningLD}. For AFHQ, we consider the $512\times512$ version, and we take the validation sets of \textit{dogs} and \textit{cats}, as done in \cite{Kim2023RegularizationBT}. The model used for our LATINO algorithm is the DMD2 \cite{Yin2024ImprovedDM} based on SDXL \cite{Podell2023SDXLIL}. SD1.5 \cite{Rombach2021} is instead the prior for all the other LDIS, like P2L, PSLD, LDIR, and LDPS. In Appendix \ref{sec:priors} we discuss why it is fair to compare SD1.5 with DMD2 in terms of prior quality, and we also explore an SD1.5-based CM, to show the universal applicability of our method.

\paragraph{Problems considered.}

The degradations considered are the same of \cite{Chung2023PrompttuningLD}: Gaussian deblurring with kernel of size $61\times61$ with $\sigma = 3.0$ for the tests on FFHQ and $\sigma = 5.0$ for those on AFHQ. Motion deblurring with kernel of size $61\times61$ randomly sampled with intensity $0.5$\footnote{\hyperlink{https://github.com/LeviBorodenko/motionblur}{https://github.com/LeviBorodenko/motionblur}}. We also consider various types of Super-Resolution tests, ranging from $\times 8$ upscaling with the average pooling kernel to $\times 16$ upscaling with the bicubic interpolation kernel. We add a white noise of intensity $\sigma_n = 0.01$ to all our tests. We also add in appendix \ref{sec:extreme} harder problem settings, providing visual results to show the effectiveness of our algorithms. Additional experiments are reported in Appendix~\ref{sec:Inpainting} and~\ref{sec:non-linear}.

\begin{table}[t]
\centering \footnotesize\setlength{\tabcolsep}{5pt}
\begin{tabular}{lccccc}
 \toprule
 & & \multicolumn{2}{c}{\textbf{Deblur (Gaussian)}} & \multicolumn{2}{c}{\textbf{SR$\times16$}} \\
 \cmidrule(lr){3-4} \cmidrule(lr){5-6}
  \textbf{Method} &  \textbf{NFE↓} & \textbf{FID↓} & \textbf{PSNR↑} & \textbf{FID↓} & \textbf{PSNR↑} \\    
  \hline
\textbf{LATINO-PRO } & \underline{68} & \textbf{18.37} & \textbf{26.82}  & \textbf{30.40} &  \textbf{21.52}     \\ 
\textbf{LATINO}  & \textbf{8} & \underline{20.03} &     \underline{26.25}   & 42.14 & \underline{20.05}   \\ \hline
P2L \cite{Chung2023PrompttuningLD}   &     2000     & 85.80 & 20.96   & 121.7 & 19.99      \\
TReg \cite{Kim2023RegularizationBT}   &     200     & 35.47 & 21.13   &   \underline{37.13}  &     19.60   \\
LDPS   &     1000     & 64.88 & 22.60  & 101.13 & 17.34        \\
PSLD \cite{Rout2023}    &  1000  &    125.5     & 20.52   &    113.4   &  16.48 \\
\bottomrule
\end{tabular}
\caption{Results for Gaussian Deblurring with $\sigma = 5.0$, and $\times 16$ super-resolution, both with noise $\sigma_y = 0.01$ on the AFHQ-512 val dataset. Our LATINO and LATINO-PRO models are compared to recent state-of-the-art methods. Prompts: \texttt{a sharp photo of a dog} (resp.  \texttt{a cat}) 
\textbf{Bold}:  best, \underline{underline}: second best.
}
\label{tab:512_comparison_AFHQ}
\end{table}

\begin{table*}[h]
\centering \setlength{\tabcolsep}{4pt}

\begin{tabular}{lcccccccccc}
 \toprule
 &  & \multicolumn{3}{c}{\textbf{Deblur (Gaussian)}} & \multicolumn{3}{c}{\textbf{Deblur (Motion)}} & \multicolumn{3}{c}{\textbf{SR$\times8$}} \\ 
  \cmidrule(lr){3-5} \cmidrule(lr){6-8} \cmidrule(lr){9-11}
\textbf{Method} & \textbf{NFE↓} & \textbf{FID↓} & \textbf{PSNR↑} & \textbf{LPIPS↓} & \textbf{FID↓} & \textbf{PSNR↑} & \textbf{LPIPS↓} & \textbf{FID↓} & \textbf{PSNR↑} & \textbf{LPIPS↓}\\ \hline
\textbf{LATINO-PRO} &  \underline{68}  &    \underline{31.98}    &    \textbf{29.11}    & \textbf{0.292}  & \underline{27.80}     &   \underline{27.14}       & \textbf{0.301}        & 40.95     &   26.58        & 0.355           \\ 
\textbf{LATINO}       &  \textbf{8}   & 33.94 &    \underline{28.95}                  & \underline{0.296} &  29.17 &   26.88  & 0.318   & 37.13     &   26.22      &    0.356     \\ \hline
P2L \cite{Chung2023PrompttuningLD} &  2000    &  \textbf{30.62}  &   26.97              & 0.299  &  28.34  & \textbf{27.23}            & \underline{0.302}   & \textbf{31.23}     &   \underline{28.55}       &    \textbf{0.290}    \\ 
LDPS & 1000 &  45.89  &  27.82   & 0.334  &   58.66   &   26.19      & 0.382   & 36.81     &   \textbf{28.78}       &   \underline{0.292}     \\ 
PSLD \cite{Rout2023}       &  1000   &   41.04  & 28.47   & 0.320    &   47.71   & 27.05   & 0.348   & 36.93   & 26.62   &   0.335     \\ 
LDIR \cite{he2024faststablediffusioninverse}   &  1000   &   35.61   &      25.75                  & 0.341    &    \textbf{24.40}     & 24.40        & 0.376     & \underline{36.04}   & 25.79      & 0.345     \\ 
\bottomrule
\end{tabular}
\caption{Results for Gaussian deblurring with $\sigma = 3.0$, motion deblurring, and $\times 8$ super-resolution, all with noise $\sigma_y = 0.01$ on the FFHQ-512 val dataset. Our LATINO and LATINO-PRO models are compared to recent state-of-the-art methods. Prompt: \texttt{a sharp photo of a face}. \textbf{Bold}:  best, \underline{underline}: second best.}
\label{tab:512_comparison_FFHQ}
\end{table*}

\begin{figure*}[h]
\centering
\hspace*{-0.4cm}\begin{minipage}{0.03\textwidth}
    \centering
    \begin{tabular}{c}
        \rotatebox{90}{\textbf{SR $\times 8 \quad \ $}} \\[10mm]
        \rotatebox{90}{\textbf{Gaussian deblur}} \\[3mm]
        \rotatebox{90}{\textbf{Motion deblur}}
    \end{tabular}
\end{minipage}%
\begin{minipage}{0.135\textwidth}
    \centering \textbf{Measurement} \\
    \zoomedImage{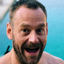}{-0.15,-0.60}{1.2,0.7} \\    \zoomedImage{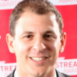}{0.15,-0.65}{1.2,0.7} \\
    \zoomedImage{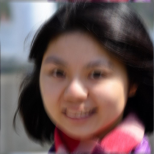}{0.15,0.10}{1.2,0.7}
\end{minipage}%
\begin{minipage}{0.135\textwidth}
    \centering \textbf{GT} \\
    \zoomedImage{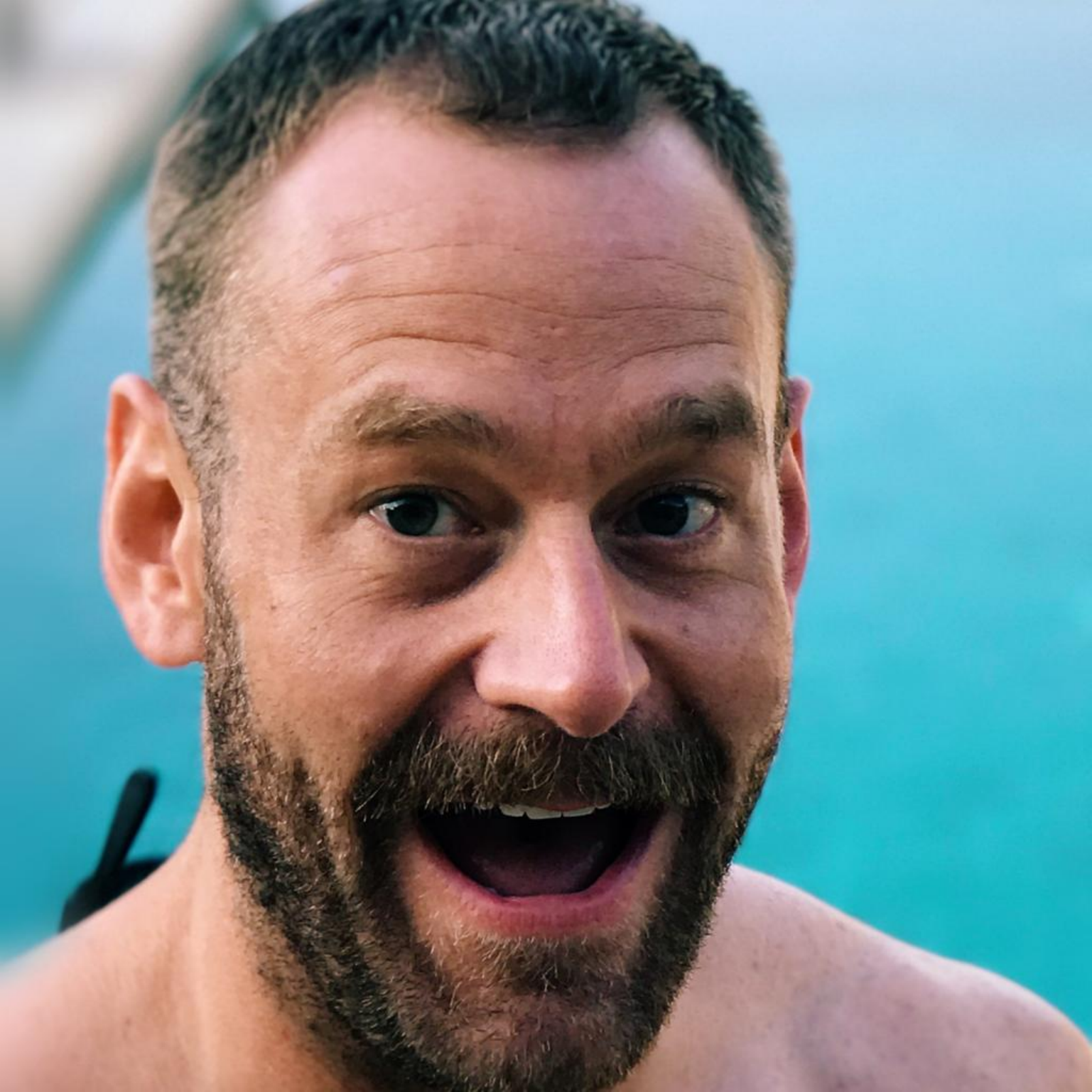}{-0.15,-0.60}{1.2,0.7} \\
    \zoomedImage{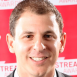}{0.15,-0.65}{1.2,0.7} \\
    \zoomedImage{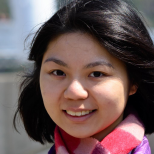}{0.15,0.10}{1.2,0.7}
\end{minipage}%
\begin{minipage}{0.135\textwidth}
    \centering \textbf{LATINO} \\
    \zoomedImage{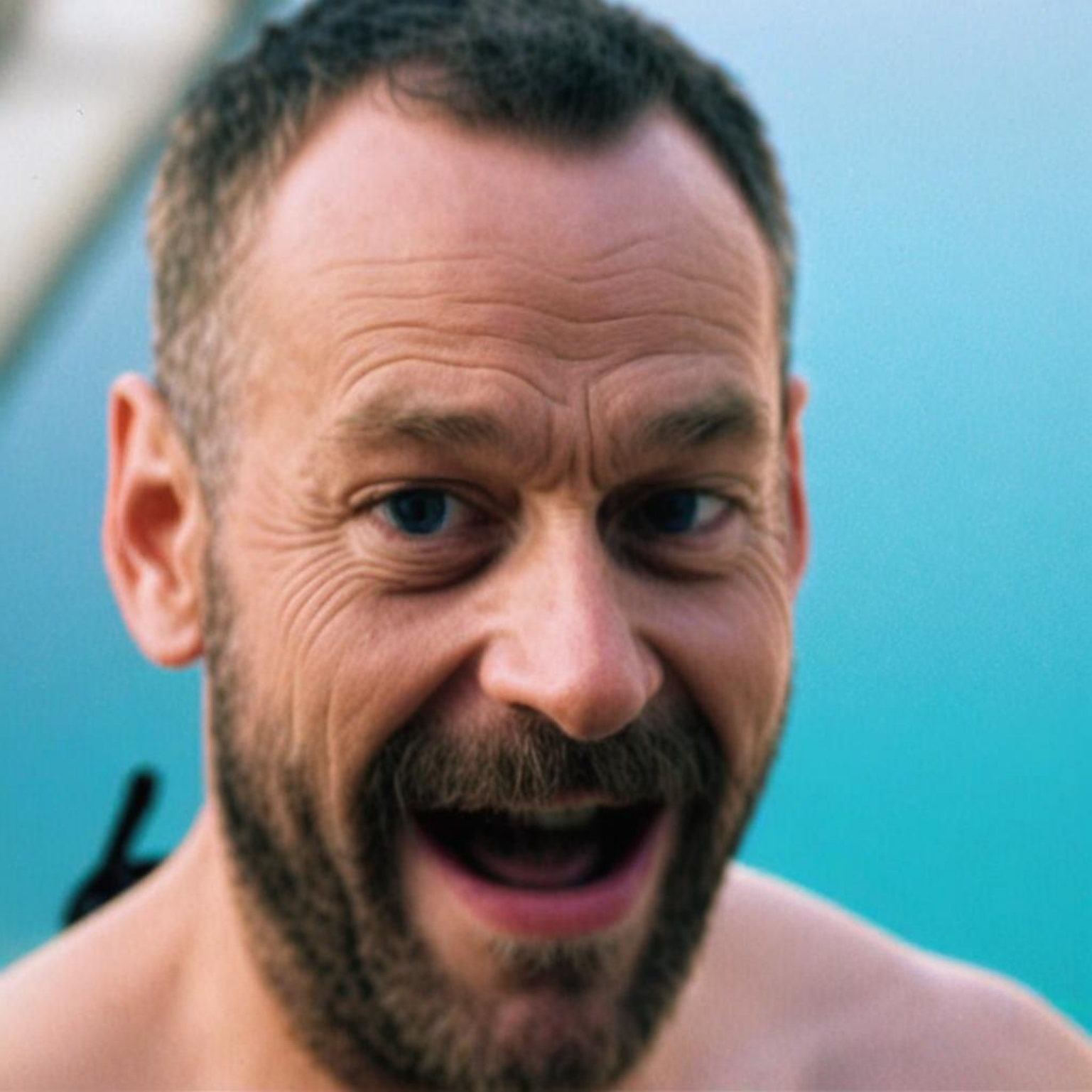}{-0.15,-0.60}{1.2,0.7} \\    \zoomedImage{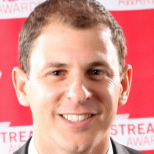}{0.15,-0.65}{1.2,0.7} \\
    \zoomedImage{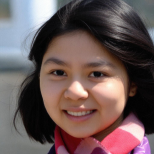}{0.15,0.10}{1.2,0.7}
\end{minipage}%
\begin{minipage}{0.135\textwidth}
    \centering \textbf{LATINO-PRO} \\ 
    \zoomedImage{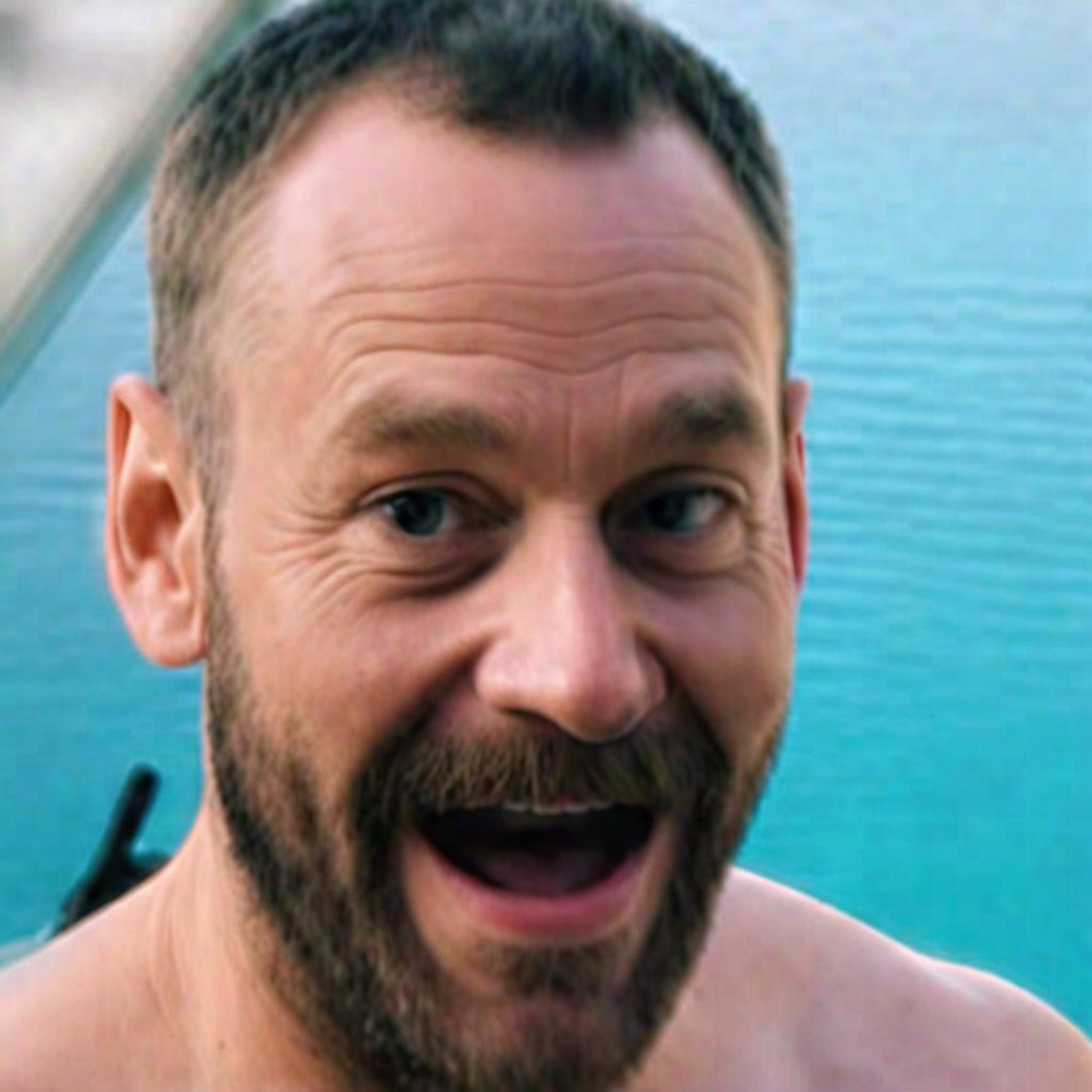}{-0.15,-0.60}{1.2,0.7} \\    \zoomedImage{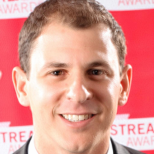}{0.15,-0.65}{1.2,0.7} \\
    \zoomedImage{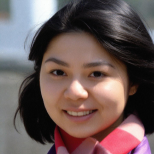}{0.15,0.10}{1.2,0.7}
\end{minipage}%
\begin{minipage}{0.135\textwidth}
    \centering \textbf{P2L} \\ 
    \zoomedImage{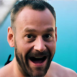}{-0.15,-0.60}{1.2,0.7} \\    \zoomedImage{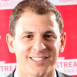}{0.15,-0.65}{1.2,0.7} \\
    \zoomedImage{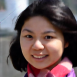}{0.15,0.10}{1.2,0.7}
\end{minipage}%
\begin{minipage}{0.135\textwidth}
    \centering \textbf{LDPS} \\ 
    \zoomedImage{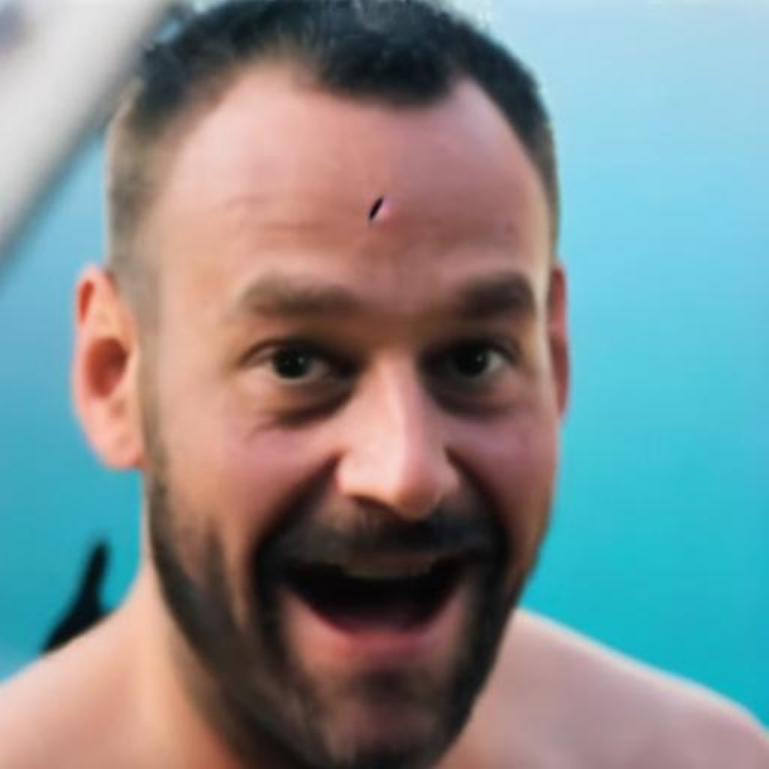}{-0.15,-0.60}{1.2,0.7} \\    \zoomedImage{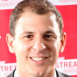}{0.15,-0.65}{1.2,0.7} \\
    \zoomedImage{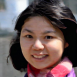}{0.15,0.10}{1.2,0.7}
\end{minipage}%
\begin{minipage}{0.135\textwidth}
    \centering \textbf{PSLD} \\ 
    \zoomedImage{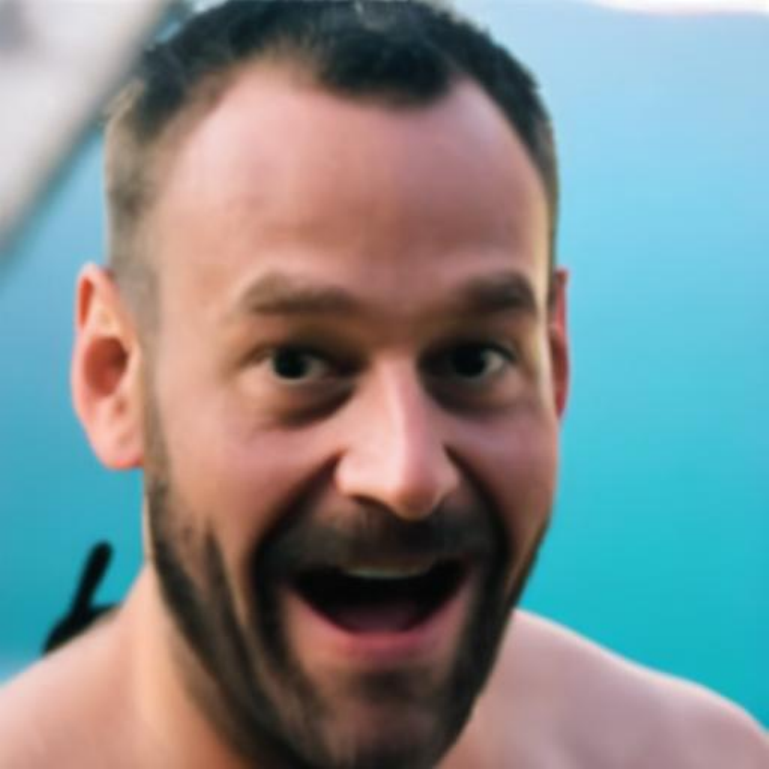}{-0.15,-0.60}{1.2,0.7} \\    \zoomedImage{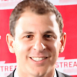}{0.15,-0.65}{1.2,0.7} \\
    \zoomedImage{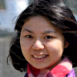}{0.15,0.10}{1.2,0.7}
\end{minipage}\vspace*{-0.6cm}
\caption{Qualitative comparison of image restoration results. Samples taken from FFHQ-512. Prompt: \texttt{a sharp photo of a face}.}
\label{fig:qualitative_comparison_FFHQ}
\end{figure*}

\paragraph{Evaluation on inverse problems tasks.}
Since the current SOTA methods that employ LDMs to solve inverse problems work only with $512\times512$ resolution, whereas the pretrained model we use works at $1024\times 1024$ resolution, we must adapt our model for fair comparison purposes. Indeed, when comparing with P2L, PSLD, LDIR and LDPS we adapt the inverse problems as follows:
\begin{enumerate}
    \item Super-resolution $\times8$ becomes $\times16$ ($\times16$ becomes $\times32$), so the image fed to the algorithm always has the same size, i.e. $64\times 64$. The output of our model is then downsampled to $512\times512$. \hbox{More details in Appendix \ref{sec:super-res}.}
    \item The Deblurring task is converted into a simultaneous super-resolution and Deblurring problem since we first have to downsample the clean image $\vx$ to $512\times512$ and then apply the actual blur operator $\mathcal{A}$. The formal details of this specific case are illustrated in Appendix \ref{sec:deblur}.\\
\end{enumerate}
In Appendix \ref{sec:additional} we also provide visual results and FID, PSNR, and LPIPS metrics for the FFHQ1024 case, allowing future work that will handle this resolution to be compared with our results. 
We can see from Table \ref{tab:512_comparison_AFHQ} and Table \ref{tab:512_comparison_FFHQ} how both our LATINO and LATINO-PRO models have similar performances compared to current SOTA, if not even better in terms of PSNR and LPIPS. Especially on the AFHQ dataset, we can beat all the SOTA methods in all the metrics considered.
On the FFHQ dataset we can beat SOTA in most cases. Importantly, we can see a huge gain in terms of NFEs. Furthermore, LATINO requires around only $ 13$ Gb of GPU memory to run, thanks to the absence of any gradient computation in the steps (more details on memory and time consumption in Appendix \ref{sec:Memory}). Considering that the DMD2 prior takes around $10.7$Gb, the overhead is minimum. Figures \ref{fig:qualitative_comparison_AFHQ} and \ref{fig:qualitative_comparison_FFHQ} show a qualitative comparison of the methods proposed against common latent DIS like LDPS, PSLD and P2L.
We refer to the TReg original work \cite{Kim2023RegularizationBT} to compare on the AFHQ dataset with both algorithms, see Appendix \ref{sec:TReg}.

\paragraph{Prompt Optimization.}

As discussed in Section \ref{sec:SOUL}, we can use the LATINO-PRO SAPG scheme to optimize the text prompt $c$. This greatly improves accuracy, both when the prompt is already partially aligned with the image and when it is misleading, as shown in Table \ref{tab:512_comparison_AFHQ_prompt} in the Appendix \ref{sec:prompt-exp}. We initialize the prompt $c$ by concatenating two strings, a fixed string \texttt{a sharp photo of} whose tokens will remain frozen, plus a descriptive string, e.g. \texttt{a dog}, which is updated via SAPG updates. The hyperparameters used are $M=15$, $C = B(c_0,15)$, and $\gamma_m = 0.1\cdot0.9^{\max(0,m - 10)}$. 

Tables \ref{tab:512_comparison_AFHQ} and \ref{tab:512_comparison_FFHQ} show that LATINO-PRO can significantly outperform LATINO, with a computational cost that remains $15\times$ lower than competing SOTA methods. Although LATINO-PRO is not gradient-free due to the terms $\nabla_c$, these gradients are computed in the latent space and do not require backpropagation through $\decoder$ or $\encoder$. We provide in Table \ref{tab:gpu_time_comparison} in Appendix \ref{sec:Memory} an extensive comparison of the memory consumption and running times of different IR algorithms compared to ours. In Appendix~\ref{sec:Inpainting}, we provide inpainting experiments highlighting the capacity of LATINO-PRO to produce diverse image restorations with different prompts.
\section{Conclusion}
\label{sec:conclusion}
We introduced LATINO, a fast, light, zero-shot, prompt-guided IS that can leverage SOTA latent consistency models such as DMD2 \cite{Yin2024ImprovedDM}. Its effectiveness has been compared to current SOTA methods, which achieve comparable quality but require two to three orders of magnitude more NFEs. We also presented LATINO-PRO, which self-calibrates the text prompt by maximum likelihood estimation. LATINO-PRO outperforms SOTA on AFHQ512 and FFHQ512 datasets on multiple inverse problems. 

Future work will seek to analyze the theoretical properties of LATINO and LATINO-PRO, with special attention to non-asymptotic convergence results. The development of strategies to automatically adjust the parameters of LATINO and LATINO-PRO is another main perspective for future work. Future work will also explore strategies for decoding the prompt embedding to reveal the optimized text prompt.

\section*{Acknowledgments and Disclosure of Funding}

MP acknowledges support 
by UKRI Engineering and Physical Sciences Research Council (EPSRC) (EP/V006134/1, EP/Z534481/1). AS, AA, and NP acknowledge support from the France 2030 research program on artificial intelligence via the PEPR PDE-AI grant (ANR-23-PEIA-0004).
JP acknowledges support by the ``Chaire IA Sherlock’' held by P. Chainais, which includes the ANR project  ANR-20-CHIA-0031-01, the national support within the {\em programme d'investissements d'avenir} ANR-16-IDEX-0004 ULNE and Région HdF.
HPC resources provided by GENCI-IDRIS Jean-Zay (Grant 2024-AD011014557).
{
    \small
    \bibliographystyle{ieeenat_fullname}
    \bibliography{main}

\begin{thebibliography}{67}
\providecommand{\natexlab}[1]{#1}
\providecommand{\url}[1]{\texttt{#1}}
\expandafter\ifx\csname urlstyle\endcsname\relax
  \providecommand{\doi}[1]{doi: #1}\else
  \providecommand{\doi}{doi: \begingroup \urlstyle{rm}\Url}\fi

\bibitem[Blanes et~al.(2024)Blanes, Casas, and Murua]{Blanes2024}
Sergio Blanes, Fernando Casas, and Ander Murua.
\newblock Splitting methods for differential equations.
\newblock \emph{Acta Numerica}, 33:\penalty0 1–161, 2024.

\bibitem[Bora et~al.(2017)Bora, Jalal, Price, and Dimakis]{bora2017compressed}
Ashish Bora, Ajil Jalal, Eric Price, and Alexandros~G Dimakis.
\newblock Compressed sensing using generative models.
\newblock In \emph{International conference on machine learning}, pages 537--546. PMLR, 2017.

\bibitem[Bortoli et~al.(2019)Bortoli, Durmus, Pereyra, and Vidal]{Bortoli2019EfficientSO}
Valentin~De Bortoli, Alain Durmus, Marcelo Pereyra, and Ana~Fernandez Vidal.
\newblock Efficient stochastic optimisation by unadjusted langevin monte carlo.
\newblock \emph{Statistics and Computing}, 31, 2019.

\bibitem[Bossard et~al.(2014)Bossard, Guillaumin, and Gool]{Bossard2014Food101M}
Lukas Bossard, Matthieu Guillaumin, and Luc~Van Gool.
\newblock Food-101 - mining discriminative components with random forests.
\newblock In \emph{European Conference on Computer Vision}, 2014.

\bibitem[Boyd et~al.(2011)Boyd, Parikh, Chu, Peleato, and Eckstein]{Boyd2011DistributedOA}
Stephen~P. Boyd, Neal Parikh, Eric Chu, Borja Peleato, and Jonathan Eckstein.
\newblock Distributed optimization and statistical learning via the alternating direction method of multipliers.
\newblock \emph{Found. Trends Mach. Learn.}, 3:\penalty0 1--122, 2011.

\bibitem[Cai et~al.(2023)Cai, Tang, Mukherjee, Li, Sch{\"o}nlieb, and Zhang]{cai2023nf}
Ziruo Cai, Junqi Tang, Subhadip Mukherjee, Jinglai Li, Carola~Bibiane Sch{\"o}nlieb, and Xiaoqun Zhang.
\newblock Nf-ula: Langevin monte carlo with normalizing flow prior for imaging inverse problems.
\newblock \emph{preprint arXiv:2304.08342}, 2023.

\bibitem[Chung et~al.(2022)Chung, Kim, Mccann, Klasky, and Ye]{Chung2022}
Hyungjin Chung, Jeongsol Kim, Michael~Thompson Mccann, Marc~Louis Klasky, and Jong~Chul Ye.
\newblock Diffusion posterior sampling for general noisy inverse problems.
\newblock In \emph{The Eleventh International Conference on Learning Representations}, 2022.

\bibitem[Chung et~al.(2024)Chung, Ye, Milanfar, and Delbracio]{Chung2023PrompttuningLD}
Hyungjin Chung, Jong~Chul Ye, Peyman Milanfar, and Mauricio Delbracio.
\newblock Prompt-tuning latent diffusion models for inverse problems.
\newblock In \emph{Forty-first International Conference on Machine Learning}, 2024.

\bibitem[Coeurdoux et~al.(2024{\natexlab{a}})Coeurdoux, Dobigeon, and Chainais]{Coeurdoux2024}
Florentin Coeurdoux, Nicolas Dobigeon, and Pierre Chainais.
\newblock Normalizing flow sampling with langevin dynamics in the latent space.
\newblock \emph{Machine Learning}, 113\penalty0 (11):\penalty0 8301--8326, 2024{\natexlab{a}}.

\bibitem[Coeurdoux et~al.(2024{\natexlab{b}})Coeurdoux, Dobigeon, and Chainais]{coeurdoux2024plug}
Florentin Coeurdoux, Nicolas Dobigeon, and Pierre Chainais.
\newblock Plug-and-play split gibbs sampler: embedding deep generative priors in bayesian inference.
\newblock \emph{IEEE Transactions on Image Processing}, 2024{\natexlab{b}}.

\bibitem[Daras et~al.(2021)Daras, Dean, Jalal, and Dimakis]{daras2021intermediate}
Giannis Daras, Joseph Dean, Ajil Jalal, and Alex Dimakis.
\newblock Intermediate layer optimization for inverse problems using deep generative models.
\newblock In \emph{International Conference on Machine Learning}, pages 2421--2432. PMLR, 2021.

\bibitem[Dhariwal and Nichol(2021)]{Dhariwal2021}
Prafulla Dhariwal and Alexander Nichol.
\newblock Diffusion models beat gans on image synthesis.
\newblock In \emph{(NeurIPS) Advances in Neural Information Processing Systems}, pages 8780--8794. Curran Associates, Inc., 2021.

\bibitem[Durmus and Moulines(2017)]{Durmus2017}
Alain Durmus and {\'E}ric Moulines.
\newblock {Nonasymptotic convergence analysis for the unadjusted Langevin algorithm}.
\newblock \emph{The Annals of Applied Probability}, 27\penalty0 (3):\penalty0 1551 -- 1587, 2017.

\bibitem[Garber and Tirer(2025)]{Garber_2025_CVPR}
Tomer Garber and Tom Tirer.
\newblock Zero-shot image restoration using few-step guidance of consistency models (and beyond).
\newblock In \emph{Proceedings of the Computer Vision and Pattern Recognition Conference (CVPR)}, pages 2398--2407, 2025.

\bibitem[Gonz{\'a}lez et~al.(2022)Gonz{\'a}lez, Almansa, and Tan]{gonzalez2022solving}
Mario Gonz{\'a}lez, Andr{\'e}s Almansa, and Pauline Tan.
\newblock Solving inverse problems by joint posterior maximization with autoencoding prior.
\newblock \emph{SIAM Journal on Imaging Sciences}, 15\penalty0 (2):\penalty0 822--859, 2022.

\bibitem[He et~al.(2024)He, Yan, Luo, Wu, Luo, Wang, Du, Chen, Yang, Zhang, and Lv]{he2024faststablediffusioninverse}
Linchao He, Hongyu Yan, Mengting Luo, Hongjie Wu, Kunming Luo, Wang Wang, Wenchao Du, Hu Chen, Hongyu Yang, Yi Zhang, and Jiancheng Lv.
\newblock Fast and stable diffusion inverse solver with history gradient update, 2024.

\bibitem[Ho and Salimans()]{Ho2022ClassifierFreeDG}
Jonathan Ho and Tim Salimans.
\newblock Classifier-free diffusion guidance.
\newblock In \emph{NeurIPS 2021 Workshop on Deep Generative Models and Downstream Applications}.

\bibitem[Ho et~al.(2020)Ho, Jain, and Abbeel]{Ho2020Denoising}
Jonathan Ho, Ajay Jain, and Pieter Abbeel.
\newblock Denoising diffusion probabilistic models.
\newblock \emph{Advances in neural information processing systems}, 33:\penalty0 6840--6851, 2020.

\bibitem[Holden et~al.(2022)Holden, Pereyra, and Zygalakis]{Holden2022}
Matthew Holden, Marcelo Pereyra, and Konstantinos~C. Zygalakis.
\newblock Bayesian imaging with data-driven priors encoded by neural networks.
\newblock \emph{SIAM Journal on Imaging Sciences}, 15\penalty0 (2):\penalty0 892--924, 2022.

\bibitem[Hu et~al.(2022)Hu, yelong shen, Wallis, Allen-Zhu, Li, Wang, Wang, and Chen]{Hu2021LoRALA}
Edward~J Hu, yelong shen, Phillip Wallis, Zeyuan Allen-Zhu, Yuanzhi Li, Shean Wang, Lu Wang, and Weizhu Chen.
\newblock Lo{RA}: Low-rank adaptation of large language models.
\newblock In \emph{International Conference on Learning Representations}, 2022.

\bibitem[Karras et~al.(2018)Karras, Laine, and Aila]{Karras2018ASG}
Tero Karras, Samuli Laine, and Timo Aila.
\newblock A style-based generator architecture for generative adversarial networks.
\newblock \emph{2019 IEEE/CVF Conference on Computer Vision and Pattern Recognition (CVPR)}, pages 4396--4405, 2018.

\bibitem[Kawar et~al.(2022)Kawar, Elad, Ermon, and Song]{kawar2022denoising}
Bahjat Kawar, Michael Elad, Stefano Ermon, and Jiaming Song.
\newblock Denoising diffusion restoration models.
\newblock \emph{Advances in Neural Information Processing Systems}, 35:\penalty0 23593--23606, 2022.

\bibitem[Kim et~al.(2025)Kim, Park, Chung, and Ye]{Kim2023RegularizationBT}
Jeongsol Kim, Geon~Yeong Park, Hyungjin Chung, and Jong~Chul Ye.
\newblock Regularization by texts for latent diffusion inverse solvers.
\newblock In \emph{The Thirteenth International Conference on Learning Representations}, 2025.

\bibitem[Kingma and Ba(2014)]{Kingma2014AdamAM}
Diederik~P. Kingma and Jimmy Ba.
\newblock Adam: A method for stochastic optimization.
\newblock \emph{preprint arXiv:1412.6980}, 2014.

\bibitem[Kingma and Welling(2013)]{Kingma2013AutoEncodingVB}
Diederik~P. Kingma and Max Welling.
\newblock Auto-encoding variational bayes.
\newblock \emph{CoRR}, abs/1312.6114, 2013.

\bibitem[Laumont et~al.(2022)Laumont, Bortoli, Almansa, Delon, Durmus, and Pereyra]{laumont2022bayesian}
R{\'e}mi Laumont, Valentin~De Bortoli, Andr{\'e}s Almansa, Julie Delon, Alain Durmus, and Marcelo Pereyra.
\newblock Bayesian imaging using plug \& play priors: when langevin meets tweedie.
\newblock \emph{SIAM Journal on Imaging Sciences}, 15\penalty0 (2):\penalty0 701--737, 2022.

\bibitem[Li et~al.(2025)Li, Kwon, Liang, Alkhouri, Ravishankar, and Qu]{li2025decoupleddataconsistencydiffusion}
Xiang Li, Soo~Min Kwon, Shijun Liang, Ismail~R. Alkhouri, Saiprasad Ravishankar, and Qing Qu.
\newblock Decoupled data consistency with diffusion purification for image restoration, 2025.

\bibitem[Lopez et~al.(2020)Lopez, Boyeau, Yosef, Jordan, and Regier]{Lopez2020AUTOENCODINGVB}
Romain Lopez, Pierre Boyeau, Nir Yosef, Michael~I. Jordan, and Jeffrey Regier.
\newblock Auto-encoding variational bayes.
\newblock In \emph{Proceedings of the Neural Information Processing Systems Conference}, 2020.

\bibitem[Luhman and Luhman(2021)]{Luhman2021KnowledgeDI}
Eric Luhman and Troy Luhman.
\newblock Knowledge distillation in iterative generative models for improved sampling speed.
\newblock \emph{ArXiv}, abs/2101.02388, 2021.

\bibitem[Luo et~al.(2023{\natexlab{a}})Luo, Tan, Huang, Li, and Zhao]{Luo2023LatentCM}
Simian Luo, Yiqin Tan, Longbo Huang, Jian Li, and Hang Zhao.
\newblock Latent consistency models: Synthesizing high-resolution images with few-step inference.
\newblock \emph{ArXiv}, abs/2310.04378, 2023{\natexlab{a}}.

\bibitem[Luo et~al.(2023{\natexlab{b}})Luo, Tan, Patil, Gu, von Platen, Passos, Huang, Li, and Zhao]{Luo2023LCMLoRAAU}
Simian Luo, Yiqin Tan, Suraj Patil, Daniel Gu, Patrick von Platen, Apolinário Passos, Longbo Huang, Jian Li, and Hang Zhao.
\newblock Lcm-lora: A universal stable-diffusion acceleration module, 2023{\natexlab{b}}.

\bibitem[Mbakam et~al.(2024)Mbakam, Giovannelli, and Pereyra]{mbakam2024}
Charlesquin~Kemajou Mbakam, Jean-Francois Giovannelli, and Marcelo Pereyra.
\newblock Empirical bayesian image restoration by langevin sampling with a denoising diffusion implicit prior, 2024.

\bibitem[Meinhardt et~al.(2017)Meinhardt, Moller, Hazirbas, and Cremers]{meinhardt2017learning}
Tim Meinhardt, Michael Moller, Caner Hazirbas, and Daniel Cremers.
\newblock Learning proximal operators: Using denoising networks for regularizing inverse imaging problems.
\newblock In \emph{Proceedings of the IEEE International Conference on Computer Vision}, pages 1781--1790, 2017.

\bibitem[Melidonis et~al.(2024)Melidonis, Holden, Altmann, Pereyra, and Zygalakis]{Melidonis2024}
S Melidonis, M Holden, Y Altmann, M Pereyra, and K~C Zygalakis.
\newblock Empirical bayesian imaging with large-scale push-forward generative priors.
\newblock \emph{IEEE Signal Process. Lett.}, 31:\penalty0 631--635, 2024.

\bibitem[Meng et~al.(2022)Meng, Gao, Kingma, Ermon, Ho, and Salimans]{Meng2022OnDO}
Chenlin Meng, Ruiqi Gao, Diederik~P. Kingma, Stefano Ermon, Jonathan Ho, and Tim Salimans.
\newblock On distillation of guided diffusion models.
\newblock \emph{2023 IEEE/CVF Conference on Computer Vision and Pattern Recognition (CVPR)}, pages 14297--14306, 2022.

\bibitem[Menon et~al.(2020)Menon, Damian, Hu, Ravi, and Rudin]{menon2020pulse}
Sachit Menon, Alexandru Damian, Shijia Hu, Nikhil Ravi, and Cynthia Rudin.
\newblock Pulse: Self-supervised photo upsampling via latent space exploration of generative models.
\newblock In \emph{Proceedings of the ieee/cvf conference on computer vision and pattern recognition}, pages 2437--2445, 2020.

\bibitem[MOUFAD et~al.(2025)MOUFAD, Janati, Bedin, Durmus, randal douc, Moulines, and Olsson]{moufad2025variational}
Badr MOUFAD, Yazid Janati, Lisa Bedin, Alain~Oliviero Durmus, randal douc, Eric Moulines, and Jimmy Olsson.
\newblock Variational diffusion posterior sampling with midpoint guidance.
\newblock In \emph{The Thirteenth International Conference on Learning Representations}, 2025.

\bibitem[Pan et~al.(2021)Pan, Zhan, Dai, Lin, Loy, and Luo]{pan2021exploiting}
Xingang Pan, Xiaohang Zhan, Bo Dai, Dahua Lin, Chen~Change Loy, and Ping Luo.
\newblock Exploiting deep generative prior for versatile image restoration and manipulation.
\newblock \emph{IEEE Transactions on Pattern Analysis and Machine Intelligence}, 44\penalty0 (11):\penalty0 7474--7489, 2021.

\bibitem[Parikh and Boyd(2013)]{Parikh2013ProximalA}
Neal Parikh and Stephen~P. Boyd.
\newblock Proximal algorithms.
\newblock \emph{Found. Trends Optim.}, 1:\penalty0 127--239, 2013.

\bibitem[Pereyra(2016)]{pereyra2016proximal}
Marcelo Pereyra.
\newblock Proximal markov chain monte carlo algorithms.
\newblock \emph{Statistics and Computing}, 26:\penalty0 745--760, 2016.

\bibitem[Pereyra et~al.(2023)Pereyra, Vargas-Mieles, and Zygalakis]{Vargas-Mieles2022}
Marcelo Pereyra, Luis~A. Vargas-Mieles, and Konstantinos~C. Zygalakis.
\newblock The split gibbs sampler revisited: Improvements to its algorithmic structure and augmented target distribution.
\newblock \emph{SIAM Journal on Imaging Sciences}, 16\penalty0 (4):\penalty0 2040--2071, 2023.

\bibitem[Podell et~al.()Podell, English, Lacey, Blattmann, Dockhorn, M{\"u}ller, Penna, and Rombach]{Podell2023SDXLIL}
Dustin Podell, Zion English, Kyle Lacey, Andreas Blattmann, Tim Dockhorn, Jonas M{\"u}ller, Joe Penna, and Robin Rombach.
\newblock Sdxl: Improving latent diffusion models for high-resolution image synthesis.
\newblock In \emph{The Twelfth International Conference on Learning Representations}.

\bibitem[Prost et~al.(2023)Prost, Houdard, Almansa, and Papadakis]{prost2023inverse}
Jean Prost, Antoine Houdard, Andr{\'e}s Almansa, and Nicolas Papadakis.
\newblock Inverse problem regularization with hierarchical variational autoencoders.
\newblock In \emph{Proceedings of the IEEE/CVF International Conference on Computer Vision}, pages 22894--22905, 2023.

\bibitem[Radford et~al.(2021)Radford, Kim, Hallacy, Ramesh, Goh, Agarwal, Sastry, Askell, Mishkin, Clark, Krueger, and Sutskever]{Radford2021LearningTV}
Alec Radford, Jong~Wook Kim, Chris Hallacy, Aditya Ramesh, Gabriel Goh, Sandhini Agarwal, Girish Sastry, Amanda Askell, Pamela Mishkin, Jack Clark, Gretchen Krueger, and Ilya Sutskever.
\newblock Learning transferable visual models from natural language supervision.
\newblock In \emph{International Conference on Machine Learning}, 2021.

\bibitem[Renaud et~al.(2024)Renaud, Prost, Leclaire, and Papadakis]{Renaud2024}
Marien Renaud, Jean Prost, Arthur Leclaire, and Nicolas Papadakis.
\newblock Plug-and-play image restoration with stochastic denoising regularization.
\newblock In \emph{Proceedings of the 41st International Conference on Machine Learning}. JMLR.org, 2024.

\bibitem[Romano et~al.(2017)Romano, Elad, and Milanfar]{romano2017little}
Yaniv Romano, Michael Elad, and Peyman Milanfar.
\newblock The little engine that could: Regularization by denoising (red).
\newblock \emph{SIAM Journal on Imaging Sciences}, 10\penalty0 (4):\penalty0 1804--1844, 2017.

\bibitem[Rombach et~al.(2021)Rombach, Blattmann, Lorenz, Esser, and Ommer]{Rombach2021}
Robin Rombach, A. Blattmann, Dominik Lorenz, Patrick Esser, and Bj{\"o}rn Ommer.
\newblock High-resolution image synthesis with latent diffusion models.
\newblock \emph{2022 IEEE/CVF Conference on Computer Vision and Pattern Recognition (CVPR)}, pages 10674--10685, 2021.

\bibitem[Rout et~al.(2023)Rout, Raoof, Daras, Caramanis, Dimakis, and Shakkottai]{Rout2023}
Litu Rout, Negin Raoof, Giannis Daras, Constantine Caramanis, Alex Dimakis, and Sanjay Shakkottai.
\newblock Solving linear inverse problems provably via posterior sampling with latent diffusion models.
\newblock \emph{Advances in Neural Information Processing Systems}, 36:\penalty0 49960--49990, 2023.

\bibitem[Rout et~al.(2024)Rout, Chen, Kumar, Caramanis, Shakkottai, and Chu]{rout2023secondorder}
L Rout, Y Chen, A Kumar, C Caramanis, S Shakkottai, and W Chu.
\newblock Beyond first-order tweedie: Solving inverse problems using latent diffusion, 2024.

\bibitem[Salimans and Ho(2022)]{salimans2022progressivedistillationfastsampling}
Tim Salimans and Jonathan Ho.
\newblock Progressive distillation for fast sampling of diffusion models, 2022.

\bibitem[Schuhmann et~al.(2022)Schuhmann, Beaumont, Vencu, Gordon, Wightman, Cherti, Coombes, Katta, Mullis, Wortsman, et~al.]{Schuhmann2022LAION5BAO}
Christoph Schuhmann, Romain Beaumont, Richard Vencu, Cade Gordon, Ross Wightman, Mehdi Cherti, Theo Coombes, Aarush Katta, Clayton Mullis, Mitchell Wortsman, et~al.
\newblock Laion-5b: An open large-scale dataset for training next generation image-text models.
\newblock \emph{Advances in neural information processing systems}, 35:\penalty0 25278--25294, 2022.

\bibitem[Sohl-Dickstein et~al.(2015)Sohl-Dickstein, Weiss, Maheswaranathan, and Ganguli]{SohlDickstein2015DeepUL}
Jascha~Narain Sohl-Dickstein, Eric~A. Weiss, Niru Maheswaranathan, and Surya Ganguli.
\newblock Deep unsupervised learning using nonequilibrium thermodynamics.
\newblock \emph{ArXiv}, abs/1503.03585, 2015.

\bibitem[Song et~al.({\natexlab{a}})Song, Kwon, Zhang, Hu, Qu, and Shen]{songsolving}
Bowen Song, Soo~Min Kwon, Zecheng Zhang, Xinyu Hu, Qing Qu, and Liyue Shen.
\newblock Solving inverse problems with latent diffusion models via hard data consistency.
\newblock In \emph{The Twelfth International Conference on Learning Representations}, {\natexlab{a}}.

\bibitem[Song et~al.(2020)Song, Meng, and Ermon]{Song2020DenoisingDI}
Jiaming Song, Chenlin Meng, and Stefano Ermon.
\newblock Denoising diffusion implicit models.
\newblock \emph{ArXiv}, abs/2010.02502, 2020.

\bibitem[Song et~al.(2023{\natexlab{a}})Song, Vahdat, Mardani, and Kautz]{Song2023PseudoinverseGuidedDM}
Jiaming Song, Arash Vahdat, Morteza Mardani, and Jan Kautz.
\newblock Pseudoinverse-guided diffusion models for inverse problems.
\newblock In \emph{International Conference on Learning Representations}, 2023{\natexlab{a}}.

\bibitem[Song and Ermon(2019)]{Song2019}
Yang Song and Stefano Ermon.
\newblock Generative modeling by estimating gradients of the data distribution.
\newblock In \emph{Neural Information Processing Systems}, 2019.

\bibitem[Song and Ermon(2020)]{Song2020improved}
Yang Song and Stefano Ermon.
\newblock Improved techniques for training score-based generative models.
\newblock \emph{Advances in neural information processing systems}, 33:\penalty0 12438--12448, 2020.

\bibitem[Song et~al.({\natexlab{b}})Song, Sohl-Dickstein, Kingma, Kumar, Ermon, and Poole]{Song2020SDE}
Yang Song, Jascha Sohl-Dickstein, Diederik~P Kingma, Abhishek Kumar, Stefano Ermon, and Ben Poole.
\newblock Score-based generative modeling through stochastic differential equations.
\newblock In \emph{International Conference on Learning Representations}, {\natexlab{b}}.

\bibitem[Song et~al.(2023{\natexlab{b}})Song, Dhariwal, Chen, and Sutskever]{Song2023ConsistencyM}
Yang Song, Prafulla Dhariwal, Mark Chen, and Ilya Sutskever.
\newblock Consistency models.
\newblock In \emph{International Conference on Machine Learning}, 2023{\natexlab{b}}.

\bibitem[Vincent(2011)]{Vincent2011}
Pascal Vincent.
\newblock A connection between score matching and denoising autoencoders.
\newblock \emph{Neural Computation}, 23:\penalty0 1661--1674, 2011.

\bibitem[Vono et~al.(2019)Vono, Dobigeon, and Chainais]{Vono2019}
Maxime Vono, Nicolas Dobigeon, and Pierre Chainais.
\newblock Split-and-augmented {G}ibbs sampler - {A}pplication to large-scale inference problems.
\newblock \emph{IEEE Transactions on Signal Processing}, 67\penalty0 (6):\penalty0 1648--1661, 2019.

\bibitem[Xu et~al.(2024)Xu, Zhu, Li, He, Wang, Sun, Li, Qin, Wang, Liu, and Zhang]{xu2024consistencymodeleffectiveposterior}
Tongda Xu, Ziran Zhu, Jian Li, Dailan He, Yuanyuan Wang, Ming Sun, Ling Li, Hongwei Qin, Yan Wang, Jingjing Liu, and Ya-Qin Zhang.
\newblock Consistency model is an effective posterior sample approximation for diffusion inverse solvers, 2024.

\bibitem[Yin et~al.(2023)Yin, Gharbi, Zhang, Shechtman, Durand, Freeman, and Park]{Yin2023OneStepDW}
Tianwei Yin, Michael Gharbi, Richard Zhang, Eli Shechtman, Fr{\'e}do Durand, William~T. Freeman, and Taesung Park.
\newblock One-step diffusion with distribution matching distillation.
\newblock \emph{2024 IEEE/CVF Conference on Computer Vision and Pattern Recognition (CVPR)}, pages 6613--6623, 2023.

\bibitem[Yin et~al.(2024)Yin, Gharbi, Park, Zhang, Shechtman, Durand, and Freeman]{Yin2024ImprovedDM}
Tianwei Yin, Micha{\"e}l Gharbi, Taesung Park, Richard Zhang, Eli Shechtman, Fredo Durand, and Bill Freeman.
\newblock Improved distribution matching distillation for fast image synthesis.
\newblock \emph{Advances in Neural Information Processing Systems}, 37:\penalty0 47455--47487, 2024.

\bibitem[Zhang et~al.(2020)Zhang, Li, Zuo, Zhang, Gool, and Timofte]{Zhang2020PlugandPlayIR}
K. Zhang, Yawei Li, Wangmeng Zuo, Lei Zhang, Luc~Van Gool, and Radu Timofte.
\newblock Plug-and-play image restoration with deep denoiser prior.
\newblock \emph{IEEE Transactions on Pattern Analysis and Machine Intelligence}, 44:\penalty0 6360--6376, 2020.

\bibitem[Zhao et~al.(2024)Zhao, Song, and Shen]{Zhao2024CoSIGNFG}
Jiankun Zhao, Bowen Song, and Liyue Shen.
\newblock Cosign: Few-step guidance of consistency model to solve general inverse problems.
\newblock In \emph{European Conference on Computer Vision}, pages 108--126. Springer, 2024.

\bibitem[Zhu et~al.(2023)Zhu, Zhang, Liang, Cao, Wen, Timofte, and Gool]{Zhu2023DenoisingDM}
Yuanzhi Zhu, K. Zhang, Jingyun Liang, Jiezhang Cao, Bihan Wen, Radu Timofte, and Luc~Van Gool.
\newblock Denoising diffusion models for plug-and-play image restoration.
\newblock \emph{2023 IEEE/CVF Conference on Computer Vision and Pattern Recognition Workshops (CVPRW)}, pages 1219--1229, 2023.

\end{thebibliography}
}

\clearpage
\setcounter{page}{1}
\maketitlesupplementary
\appendix

\section{Technical details about diffusion ISs}
\label{sec:related_work}

\paragraph{DPS \cite{Chung2022}:} Diffusion Posterior Sampling (DPS) follows this update rule:
\begin{equation*} \vx_{t-1} = \operatorname{DDIM}\left(\vx_{t}\right) - \eta \nabla_{\vx_{t}}\|\vy - \mathcal{A} \hat{\vx}_0\|_2^2, \end{equation*}
where $\operatorname{DDIM}(\cdot)$ represents a single update step of the DDIM sampling \cite{Song2020DenoisingDI}, defined as: 
\begin{equation*} \vx_{t-1} = \sqrt{\alpha_{t-1}} \hat{\vx}_0 - (1 - \alpha_{t-1}) s_\theta(\vx_t,t),
\end{equation*}
where $\hat{\vx}_0$ is estimated from $\vx_t$, and $s_\theta$ is the predicted score at time $t$. The optimal step size $\eta$ is dynamically set as $\eta = \frac{1}{\|\vy - \mathcal{A} \hat{\vx}_0\|_2^2}$, ensuring adaptive scaling of the likelihood gradient.

\paragraph{LDPS:} Latent Diffusion Posterior Sampling (LDPS) can be seen as a direct extension of the image-domain DPS approach proposed by Chung et al. \cite{Chung2022}. The update rule for LDPS is given by:
\begin{equation*}
\vz_{t-1}=\operatorname{DDIM}\left(\vz_{t}\right)-\rho \nabla_{\vz_{t}}\left\|\vy-\mathcal{A} \mathcal{D}\left(\hat{\vz}_{0}\right)\right\|_{2}
\end{equation*}
where $\rho$ denotes the step size, and $\operatorname{DDIM}(\cdot)$ represents a single step of DDIM sampling. A static step size of $\rho=1$ is employed, as is commonly adopted in the literature.

\paragraph{LDIR \cite{he2024faststablediffusioninverse}} modifies LDPS by introducing a momentum-based gradient update mechanism inspired by Adam. A single iteration of the algorithm follows:
\begin{align*}
\vg_{t} & =\nabla_{\vz_{t}}\left\|\vy-\mathcal{A} \mathcal{D}\left(\hat{\vz}_{0}\right)\right\|\\
\hat{\vm}_{t} & =\left(\beta_{1} \vm_{t-1}+\left(1-\beta_{1}\right) \vg_{t}\right) /\left(1-\beta_{1}\right)\\
\hat{\vv}_{t} & =\left(\beta_{2} \vv_{t-1}+\left(1-\beta_{2}\right)\left(\vg_{t} \circ \vg_{t}\right)\right) /\left(1-\beta_{2}\right)\\
\vz_{t-1} & =\operatorname{DDIM}\left(\vz_{t}\right)-\rho \frac{\hat{\vm}_{t}}{\sqrt{\hat{\vv}_{t}}+\varepsilon}
\end{align*}
where $\circ$ denotes element-wise multiplication, and $\beta_{1}, \beta_{2}, \varepsilon$ are hyperparameters of the method. The momentum-based approach in LDIR leads to smoother gradient updates. The parameters are set as $\beta_{1}=0.9, \beta_{2}=0.999, \varepsilon=1e-8$. The step size $\rho$ is set to be $0.05$.

\paragraph{GML-DPS, PSLD \cite{Rout2023}:} GML-DPS introduces a constraint to ensure that the estimated clean latent $\hat{\vz}_{0}$ remains stable after encoding and decoding. The update rule is:
\begin{align*}
\vz_{t-1}&=\operatorname{DDIM}\left(\vz_{t}\right)\\
&-\rho \nabla_{\vz_{t}}\left(\left\|\vy-\mathcal{A} \mathcal{D}\left(\hat{\vz}_{0}\right)\right\|_{2}+\gamma\left\|\hat{\vz}_{0}-\mathcal{E}\left(\mathcal{D}\left(\hat{\vz}_{0}\right)\right)\right\|_{2}\right)
\end{align*}
PSLD  refines this approach by incorporating an orthogonal projection step onto the subspace defined by $\mathcal{A}$ between the decoding and encoding stages to enforce fidelity:
\begin{align*}
\vz_{t-1}&=\operatorname{DDIM}\left(\vz_{t}\right)-\rho \nabla_{\vz_{t}}\left\|\vy-\mathcal{A D}\left(\hat{\vz}_{0}\right)\right\|_{2} \\
&-\gamma\nabla_{\vz_{t}}\left\|\hat{\vz}_{0}-\mathcal{E}\left(\mathcal{A}^{\top} \vy+\left(\Id-\mathcal{A}^{\top} \mathcal{A}\right) \mathcal{D}\left(\hat{\vz}_{0}\right)\right)\right\|_{2}.
\end{align*}
A static step size of $\rho=1$ is applied,and we set $\gamma=0.1$. These methods aim at guiding latents toward the natural manifold, enforcing their stability after autoencoding.

\paragraph{P2L \cite{Chung2023PrompttuningLD}:}

The P2L algorithm alternates between two main update steps: optimizing the text embedding $c$ and refining the latent variable $\vz_t$. 

The first step focuses on updating the text embedding $c$ to align it with the measurement $\vy$ and the current diffusion estimate $\vz_t$. This is done by maximizing the posterior $p(c \mid \vz_t, \vy)$, leading to the gradient update:
\begin{align*}
    \nabla_c \log p(c \mid \vz_t, \vy) \approx & \nabla_c \| \mathcal{A} \mathcal{D} (\mathbb{E}[\vz_0 \mid \vz_t, c]) - \vy \|^2_2.
\end{align*}
This optimization uses stochastic optimizers like Adam \cite{Kingma2014AdamAM}. 

In the second step, the latent variable $\vz_t$ is refined using the optimized text embedding $c^*_t$ obtained from the first step. This update aims at maximizing $p(\vz_t \mid \vy, c^*_t)$, resulting in the following gradient expression:
\begin{align*}
    \nabla_{\vz_t} \log p(\vz_t \mid \vy, c^*_t) \approx &s^*_\theta(\vz_t, c^*_t)\\
    + &\rho_t \nabla_{\vz_t} \| \mathcal{A} \mathcal{D} (\mathbb{E}[\vz_0 \mid \vz_t, c^*_t]) - \vy \|^2_2,
\end{align*}
where $s^*_\theta(\vz_t, c^*_t)$ is the score function from the diffusion model and $\rho_t$ is a step size that balances the influence of the likelihood term.

\paragraph{TReg \cite{Kim2023RegularizationBT}:}
The TReg algorithm  solves the following proximal optimization problem in an ADMM \cite{Boyd2011DistributedOA} style:
\begin{align*}
    \min_{\vx,\vz} l_{\text{MAP}}(\vz) + \gamma l_{\text{TReg}}(\vz) &= l_{\text{MAP}}(\vz) + \| \vz - \hat{\vz}_{0\mid t}\|_2^2\\
    \text{s.t.} \quad \vx &= \mathcal{D}(\vz)
\end{align*}
where the objective of the maximum a posteriori (MAP) problem is defined as
\begin{align*}
    \ell_{\text{MAP}}(\vz) &= -\log p(\vz \mid \mathcal{D}(\vz), \vy) - \log p(\vy \mid \mathcal{D}(\vz))\\
    &=
    \frac{\|\vz - \mathcal{E}(\mathcal{D}(\vz)) \|_2^2}{2\sigma_{\mathcal{E}}^2} 
    + \frac{\| \vy - \mathcal{A}(\mathcal{D}(\vz)) \|_2^2}{2\sigma^2},
\end{align*}
where $ \sigma_{\mathcal{E}}$ is the encoder variance. 
First is solved
\begin{equation*}
    \hat{\vx}_{0}(\vy) = \min_{\vx}\frac{\|\vy - \mathcal{A}(\vx)\|_2^2}{2\sigma^2} + \lambda \|\vx - \mathcal{D}(\hat{\vz}_{0\mid t})\|_2^2,
\end{equation*}
where $\vx = \mathcal{D}(\vz)$, and then:
\begin{align}
    \hat{\vz}_{0}^{ema} &= \operatorname{argmin}_\vz \zeta \| \vz - \mathcal{E}(\hat{\vx}_{0}(\vy))\|_2^2 + \gamma \|\vz - \hat{\vz}_{0\mid t} \|_2^2 \nonumber\\
    &= \alpha_{t-1}\mathcal{E}(\hat{\vx}_{0}(\vy)) + (1 - \alpha_{t-1})\hat{\vz}_{0\mid t} \label{eq:TReg}
\end{align}
where $\zeta, \gamma$ are empirically chosen to satisfy $\alpha_{t-1} = \zeta / (\zeta + \gamma)$ in order to give the second equality in \eqref{eq:TReg}.

After these two steps, a DDIM step is run and eventually the null prompt is optimized through what is called "Adaptive Negation", i.e.:
\begin{equation*}
    c_{\emptyset} \gets c_{\emptyset} - \eta \nabla_{\emptyset} (\mathcal{T}_{\text{img}}(\hat{\vx}_{0}(\vy)),c_{\emptyset})
\end{equation*}
where $\eta$ is a fixed learning rate and $\mathcal{T}_{\text{img}}$ denotes the CLIP image encoder.

\section{Hyperparameters tuning}
\label{sec:hyperparam}

As deeply studied in the theoretical derivation of our method in Section \ref{sec:LATINO}, we introduce the hyperparameter $\delta_k$ as it represents the implicit Euler step size. We will now show the values of $\delta_k$ used for each task.
\begin{enumerate}
    \item \textbf{Gaussian Deblurring:}\vspace{.1cm}
    \begin{itemize}
        \item For $k \geq 5$: 
        $\delta_k = 2\cdot10^{-5} (1 - \alpha_{t_k}){\|\mathcal{A}\vu^{(k)} - \vy \|}{}/{\sigma_n}$\vspace{.1cm}
        \item Otherwise:
       $ \delta_k = 4\cdot10^{-5} (1 - \alpha_{t_k}){\|\mathcal{A}\vu^{(k)} - \vy \|}{}/{\sigma_n}$\\
    \end{itemize}

    \item \textbf{Motion Deblurring:}\vspace{.1cm}
    \begin{itemize}
        \item For $k \geq 5$: 
       $ \delta_k = 4\cdot10^{-6} (1 - \alpha_{t_k}){\|\mathcal{A}\vu^{(k)} - \vy \|}{}/{\sigma_n}$\vspace{.1cm}
        \item Otherwise:
       $ \delta_k = 2\cdot10^{-6} (1 - \alpha_{t_k}){\|\mathcal{A}\vu^{(k)} - \vy \|}{}/{\sigma_n}$\\
    \end{itemize}

    \item \textbf{Super Resolution $\times 8$:}\vspace{.1cm}
    \begin{itemize}
        \item For $k \geq 6$:
        $\delta_k = 6\cdot10^{-3} (1 - \alpha_{t_k}){\|\mathcal{A}\vu^{(k)} - \vy \|}{}/{\sigma_n}$\vspace{.1cm}
        \item Otherwise:
        $\delta_k = 3\cdot10^{-3} (1 - \alpha_{t_k}){\|\mathcal{A}\vu^{(k)} - \vy \|}{}/{\sigma_n}$\\
    \end{itemize}

    \item \textbf{Super Resolution $\times 16$:}\vspace{.1cm}
    \begin{itemize}
        \item For $k \geq 6$: 
        $\delta_k = 2\cdot10^{-2} (1 - \alpha_{t_k}){\|\mathcal{A}\vu^{(k)} - \vy \|}{}/{\sigma_n}$\vspace{.1cm}
        \item Otherwise:
        $\delta_k = 9\cdot10^{-3} (1 - \alpha_{t_k}){\|\mathcal{A}\vu^{(k)} - \vy \|}{}/{\sigma_n}$\\
    \end{itemize}

    \item \textbf{Box Inpainting:}\vspace{.1cm}
    \begin{itemize}
        \item For $k \geq 5$: $\delta_k = (1 - \alpha_{t_k})$\vspace{.1cm}
        \item Otherwise: $\delta_k = 0.5 (1 - \alpha_{t_k})$\\
    \end{itemize}
\end{enumerate}

These choices can be motivated in the following way: the normalized $L^2$ norm acts as a regularizer that strengthens the data-fidelity term when the reconstruction is poor (i.e. big $L^2$ norm) and gives more freedom to the prior otherwise. In particular, we expect the norm to be big in the first steps, when we need to prevent the prior from deviating from the observation, and small in the final steps when it is more important to be able to generate detailed high-frequency features that cannot be recovered from the noisy observation. A similar reasoning leads to the addition of the $1-\alpha_{t_k}$ term.\\

The PSNR/LPIPS performance of our method is quite robust to the choice of $\delta_k \in (0.1\delta_k^*,10\delta_k^*)$
as shown below for gaussian deblurring:
\textbf{LATINO} uses the optimal $\delta^*_k$ as defined above;
\textbf{LATINO-s} uses $\delta_k=0.1\delta_k^*$;
\textbf{LATINO-b} uses $\delta_k=10\delta_k^*$.

\newlength{\lblw}
\setlength{\lblw}{0.03\columnwidth}
\newlength{\colimgwidth}
\setlength{\colimgwidth}{0.19\columnwidth}

\begin{figure}[htbp]
  \centering
  \noindent
  \begin{minipage}[b]{\lblw}\centering
    \rotatebox{90}{\scriptsize\bfseries Gaussian deblur}
  \end{minipage}%
  \begin{minipage}[b]{\colimgwidth}\centering
    \footnotesize\textbf{Meas.}\\[0.5ex]
    \includegraphics[width=\linewidth]{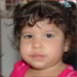}
  \end{minipage}%
  \begin{minipage}[b]{\colimgwidth}\centering
    \footnotesize\textbf{GT}\\[0.5ex]
    \includegraphics[width=\linewidth]{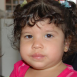}
  \end{minipage}%
  \begin{minipage}[b]{\colimgwidth}\centering
    \footnotesize\textbf{LATINO}\\[0.5ex]
    \includegraphics[width=\linewidth]{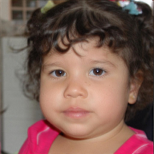}
  \end{minipage}%
  \begin{minipage}[b]{\colimgwidth}\centering
    \footnotesize\textbf{LATINO‑b}\\[0.5ex]
    \includegraphics[width=\linewidth]{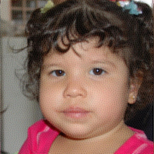}
  \end{minipage}%
  \begin{minipage}[b]{\colimgwidth}\centering
    \footnotesize\textbf{LATINO‑s}\\[0.5ex]
    \includegraphics[width=\linewidth]{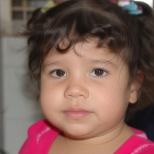}
  \end{minipage}

  \caption{Qualitative comparison on Gaussian deblurring with different $\delta_k$ schedules. LATINO uses $\delta^*_k$, LATINO‑b uses $10\delta^*_k$, and LATINO‑s uses $0.1\delta^*_k$.}
  \label{fig:qualitative_gaussian_deblur}
\end{figure}

\begin{figure}[htbp]
  \centering
  \begin{minipage}[b]{0.48\columnwidth}
    \centering
    \includegraphics[width=\linewidth]{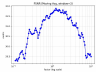}\\[\smallskipamount]
    \textbf{PSNR vs.\ $\delta_k$}
  \end{minipage}\hfill
  \begin{minipage}[b]{0.48\columnwidth}
    \centering
    \includegraphics[width=\linewidth]{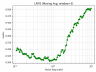}\\[\smallskipamount]
    \textbf{LPIPS vs.\ $\delta_k$}
  \end{minipage}

  \caption{Robustness of LATINO’s choice of $\delta_k$: performance curves on FFHQ‑1024 Gaussian deblurring.}
  \label{fig:performance_gaussian_deblur}
\end{figure}

\section{LATINO as a split-step Langevin sampler}
As previously noted, computing exact solutions to the Langevin diffusion process \eqref{eq:Langevin} is generally not possible. Therefore, solutions are usually obtained by using a discrete-time numerical integrator whose accuracy and cost are controlled by the size of integration time step. LATINO employs a split-step discretization of the Langevin diffusion process \eqref{eq:Langevin} in which the Brownian motion and the drift term associated with the prior density are approximately integrated via the stochastic auto-encoding step (\ref{eq:split-Langevin}.a). The likelihood term is handled via an implicit (backwards Euler) or proximal step (\ref{eq:split-Langevin}.b), hence the iterate $x_{k+1}$ appears on both sides of the second row, resulting in improved stability properties that permit larger step sizes \cite{Blanes2024}. The Langevin SDE is a time-homogeneous process, hence $g_y: x\mapsto -\log p(y|x)$ is the exact likelihood or data fidelity term, a key advantage w.r.t. DPS, $\Pi$GDM, DiffPIR, etc., which require approximations. Indeed, the use of the Langevin SDE allows employing the likelihood of $y$ w.r.t. the (noise-less) image $x$, which is usually tractable. In contrast, strategies such as DPS or $\Pi$GDM seek to embed the likelihood of $y$ w.r.t. a noisy version of $x$ within a time-inhomogeneous reverse diffusion process; such likelihoods are often intractable and require approximations. Also note that the iteration index $k$ is related to the time of the Langevin diffusion \eqref{eq:Langevin}-\eqref{eq:split-Langevin}, which goes forward as the algorithm iterations progress. It is \underline{not} the time of the diffusion SDE \eqref{eq:VP-SDE} which is encapsulated into (\eqref{eq:split-Langevin}, top row). 

With regards to convergence properties of LATINO, known theoretical convergence results for PnP Langevin sampling suggest that when $t$ is mall, LATINO should converge under a wide class of probability metrics towards a biased approximation of the posterior distribution of interest~\cite{laumont2022bayesian}. Empirically, we observe that LATINO converges very quickly, especially when $t$ is large, allowing to generate samples in very few steps. A theoretical analysis of the convergence of LATINO for large $t$ is a main perspective for future work.

\section{LATINO-PRO: gradient computation}
\label{sec:gradient_LATINOPRO}
As discussed in Section \ref{sec:SOUL}, the key step of our LATINO-PRO Algorithm \ref{alg:latino-pro} is the computation of the following quantity
\begin{equation}\label{eq:SAPG-2}
c_{m+1} = \Pi_C \left[c_{m} + \gamma_m \nabla_c \log p(\vx^{(1)},\ldots,\vx^{(N)}|c_m)\right]\,,
\end{equation}
where $\{\vx^{(k)}\}_{k=1}^N$ is a Markov chain targeting $p(\vx|\vy,c_m)$. This requires running a full iteration of our LATINO algorithm \ref{alg:lcm-pir}, and in particular, we are interested in storing the latent realizations $\{\vz_{t_k}^{(k)}\}_{k=1}^{N}$, as this leads to tractable computations by automatic differentiation (to simplify notation, we henceforth use use $\vz_{t_k} \equiv\vz_{t_k}^{(k)}$ and $c \equiv c_m$). During the optimization steps in Algorithm \ref{alg:latino-pro}, we consider $N=4$, so the computations become:
\begin{align}\label{eq:prompt_opt}
&\log p(\vz_{t_1},\vz_{t_2},\vz_{t_3},\vz_{t_4} \mid c) = \log p(\vz_{t_4} \mid \vz_{t_3}, c) +\nonumber \\ 
& + \log p(\vz_{t_3} \mid \vz_{t_2}, c) + \log p(\vz_{t_2} \mid \vz_{t_1}, c) + \log p(\vz_{t_1} \mid c).
\end{align}
All the terms can be computed through the definition of the latent part of our stochastic auto-encoder, i.e.,
\small
\begin{align*}
\vz_{t_{i+1}} = \sqrt{\alpha_{t_{i+1}}}G_\theta(\vz_{t_{i}},t_{i}, c) + \sqrt{1 - \alpha_{t_{i+1}}} \boldsymbol{\epsilon}, \quad \boldsymbol{\epsilon}\sim\mathcal{N}(0,\Id)\,,
\end{align*}
\normalsize
and hence
\begin{align*}
\vz_{t_{i+1}} \mid \vz_{t_i}, c \sim \mathcal{N}(G_\theta(\vz_{t_{i}},t_{i}, c), (1-\alpha_{t_{i+1}})\Id)\,,
\end{align*}
so
\small
\begin{equation*}
- \nabla_c \log p(\vz_{t_{i+1}}|\vz_{t_{i}},c) = \frac{\nabla_c \| \vz_{t_{i+1}} - \sqrt{\alpha_{t_{i+1}}} G_\theta(\vz_{t_{i}},t_{i}, c)\|^2}{2(1-\alpha_{t_{i+1}})}.
\end{equation*}
\normalsize
This holds for all terms in (\ref{eq:prompt_opt}), including $\log p(\vz_{t_1} \mid c)$, for which we simply have a dependence on the starting $\vz_0 \sim \mathcal{N}(\mathcal{A}^\dagger\vy, (1 - \alpha_{t_0})\Id)$ in the equation. Also, we do not include $p(\vz_{t_4} \mid \vz_{t_3}, c)$ as $\vz_{t_4}$ is deterministically determine from $\vz_{t_3}$. Instead, $\vz_{t_4}$ initializes the next iteration of LATINO within the SAPG scheme. In conclusion, \eqref{eq:SAPG-2} becomes
\small
\begin{align*}\label{eq:SAPG-extended}
&c_{m+1} =\\
&\Pi_C \left[c_{m} + \gamma_m \nabla_c \sum_{i=0}^2 \frac{\nabla_c \| \vz_{t_{i+1}} - \sqrt{\alpha_{t_{i+1}}} G_\theta(\vz_{t_{i}},t_{i}, c)\|^2}{2(1-\alpha_{t_{i+1}})}\right]\,,
\end{align*}
\normalsize

\begin{table*}[!h]
\centering \footnotesize
\begin{tabular}{lccccc}
 \toprule
 & & \multicolumn{2}{c}{\textbf{Deblur (Gaussian)}} & \multicolumn{2}{c}{\textbf{SR$\times16$}} \\ 
  \cmidrule(lr){3-4} \cmidrule(lr){5-6}
  \textbf{Method} &  \textbf{NFE↓} & \textbf{FID↓} & \textbf{PSNR↑} & \textbf{FID↓} & \textbf{PSNR↑} \\ \hline
\textbf{LATINO-PRO} & \underline{68} & \textbf{18.37} & \textbf{26.82}  & \textbf{30.40} &  \textbf{21.52}     \\ 
\textbf{LATINO}  & \textbf{8} & \underline{20.03} &     \underline{26.25}   & \underline{42.14} & \underline{20.05}   \\ 
\textbf{LATINO-LoRA} & \textbf{8} & 57.96 &   23.02   & 76.53 & 17.82    \\ \bottomrule
\end{tabular}
\caption{Results for Gaussian Deblurring with $\sigma = 5.0$, and $\times 16$ super-resolution, both with noise $\sigma_y = 0.01$ on the AFHQ-512 val dataset. Our LATINO, LATINO-PRO, and LATINO-LoRA models are compared. Prompt: \texttt{a sharp photo of a dog} or \texttt{a sharp photo of a cat}. \textbf{Bold}:  best, \underline{underline}: second best.}
\label{tab:512_comparison_AFHQ_LoRA}
\end{table*}

\begin{table*}[h]
\centering
\footnotesize
\begin{tabular}{lcccccccccc}
\toprule
 &  & \multicolumn{3}{c}{\textbf{Deblur (Gaussian)}} & \multicolumn{3}{c}{\textbf{Deblur (Motion)}} & \multicolumn{3}{c}{\textbf{SR$\times8$}} \\ 
  \cmidrule(lr){3-5} \cmidrule(lr){6-8}\cmidrule(lr){9-11}
\textbf{Method} & \textbf{NFE↓} & \textbf{FID↓} & \textbf{PSNR↑} & \textbf{LPIPS↓} & \textbf{FID↓} & \textbf{PSNR↑} & \textbf{LPIPS↓} & \textbf{FID↓} & \textbf{PSNR↑} & \textbf{LPIPS↓}\\ \hline
\textbf{LATINO-PRO } &  \underline{68}  &    \textbf{31.98}    &    \textbf{29.11}    & \textbf{0.292}  & \textbf{27.80}     &   \textbf{27.14}       & \textbf{0.301}        & \underline{40.95}     &   \textbf{26.58}        & \textbf{0.355}           \\ 
\textbf{LATINO}       &  \textbf{8}   & 33.94 &    \underline{28.95}                  & \underline{0.296} &  \underline{29.17} &   \underline{26.88}  & \underline{0.318}   & \textbf{37.13}     &   \underline{26.22}      &    \underline{0.356}     \\ 
\textbf{LATINO-LoRA} &  \textbf{8} &  \underline{33.70}  &   28.20           & 0.340   & 40.66         & 24.83            & 0.407            & 50.89   &   25.80            & 0.428           \\ \bottomrule
\end{tabular}
\caption{Results for Gaussian deblurring with $\sigma = 3.0$, motion deblurring, and $\times 8$ super-resolution, all with noise $\sigma_y = 0.01$ on the FFHQ-512 val dataset. Our LATINO, LATINO-PRO, and LATINO-LoRA models are compared. Prompt: \texttt{a sharp photo of a face}. \textbf{Bold}:  best, \underline{underline}: second best.\vspace{.0cm}}
\label{tab:512_comparison_FFHQ_LoRA}
\end{table*}
\begin{figure*}[!h]
\centering
\begin{minipage}{0.03\textwidth}
    \centering
    \begin{tabular}{c}
        \rotatebox{90}{\textbf{SR $\times 16 \quad \ \ $}} \\[11mm]
        \rotatebox{90}{\textbf{Gaussian deblur}}
    \end{tabular}
\end{minipage}%
\begin{minipage}{0.16\textwidth}
    \centering \textbf{Measurement} \\
    \zoomedImage{images/Tests/TReg_comparisons/degraded5.pdf}{0.2,0.12}{1.4,0.9} \\
    \zoomedImage{images/Tests/TReg_comparisons/degraded9.pdf}{-0.2,0.0}{1.4,0.9}
\end{minipage}%
\begin{minipage}{0.16\textwidth}
    \centering \textbf{GT} \\ 
    \zoomedImage{images/Tests/TReg_comparisons/clean5.pdf}{0.20,0.12}{1.4,0.9} \\
    \zoomedImage{images/Tests/TReg_comparisons/clean9.pdf}{-0.20,0.0}{1.4,0.9}
\end{minipage}%
\begin{minipage}{0.16\textwidth}
    \centering \textbf{LATINO} \\ 
    \zoomedImage{images/Tests/TReg_comparisons/restored5_LATINO.pdf}{0.20,0.12}{1.4,0.9} \\
    \zoomedImage{images/Tests/TReg_comparisons/restored9_LATINO.pdf}{-0.20,0.0}{1.4,0.9}
\end{minipage}%
\begin{minipage}{0.16\textwidth}
    \centering \textbf{LATINO-PRO} \\ 
    \zoomedImage{images/Tests/TReg_comparisons/restored5.pdf}{0.20,0.12}{1.4,0.9} \\
    \zoomedImage{images/Tests/TReg_comparisons/restored9.pdf}{-0.20,0.0}{1.4,0.9}
\end{minipage}%
\begin{minipage}{0.16\textwidth}
    \centering \textbf{LATINO-LoRA} \\ 
    \zoomedImage{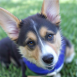}{0.20,0.12}{1.4,0.9} \\
    \zoomedImage{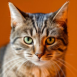}{-0.20,0.0}{1.4,0.9}
\end{minipage}%
\vspace*{-0.4cm}
\caption{Qualitative comparison of image restoration results. Samples taken from AFHQ-512. Prompt: \texttt{a sharp photo of a dog} or \texttt{a sharp photo of a cat}.}
\label{fig:qualitative_comparison_AFHQ_LoRA}
\end{figure*}

\section{Adaptation to non-linear operators}\label{sec:non-linear}

When the pseudoinverse is not accessible, the proximal operator can be computed with Conjugate Gradient in the linear case. For nonlinear operators, a direct least squares method can be adopted, as already done in \cite{Kim2023RegularizationBT}, using the Adam optimizer with learning rate $1\mathrm{e}{-3}$ and $\beta_1 = 0.9, \, \beta_2 = 0.999$ for $300$ iterations to obtain the solution of
$$
\min_{\vx} \ \frac{\| \vy - \mathcal{A}(\vx) \|_2^2}{2\sigma^2} + \lambda \| \vx - \hat{\vx}_0 \|_2^2.
$$
In Fig.~\ref{fig:phase_retrieval}, we tackle a non-linear phase retrieval task 
$$
  \vy = \bigl|\operatorname{DFT}(\vx)\bigr| + \vn
$$
on FFHQ-512, and compare LATINO-PRO with TReg and LDPS (P2L and PSLD only work for linear problems). Our method is around \(\times 4\) faster than TReg and can handle tougher cases like images with complex backgrounds, which cause failures in TReg (bottom row). We stress the fact that a key strength of our method is the possibility to use various discretization schemes in place of the implicit proximal in \ref{eq:split-Langevin} (implicit–explicit, Runge–Kutta) and even off‑the‑shelf NN to approximate \(\operatorname{prox}_{\delta g_y}\) when a closed form is not available.

\newlength{\imgw}
\setlength{\imgw}{0.2\columnwidth}

\begin{figure}[H]
  \centering

  \noindent
  \begin{minipage}[b]{\imgw}\centering \textbf{Meas.} \end{minipage}%
  \begin{minipage}[b]{\imgw}\centering \textbf{GT}    \end{minipage}%
  \begin{minipage}[b]{\imgw}\centering \textbf{LDPS}  \end{minipage}%
  \begin{minipage}[b]{\imgw}\centering \textbf{LATINO}\end{minipage}%
  \begin{minipage}[b]{\imgw}\centering \textbf{TReg}  \end{minipage}

  \vspace{0.5ex}

  \noindent
  \begin{minipage}[b]{\imgw}\centering
    \includegraphics[width=\linewidth]{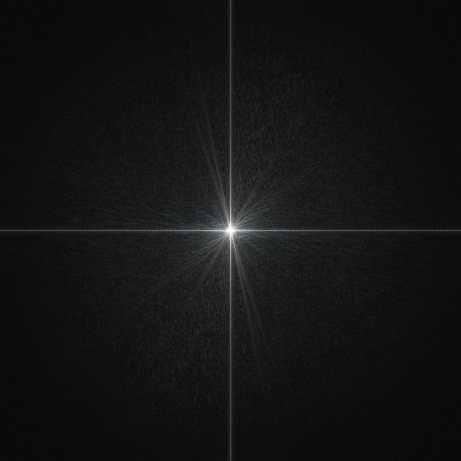}%
  \end{minipage}%
  \begin{minipage}[b]{\imgw}\centering
    \includegraphics[width=\linewidth]{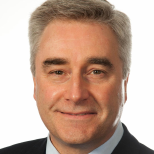}%
  \end{minipage}%
  \begin{minipage}[b]{\imgw}\centering
    \includegraphics[width=\linewidth]{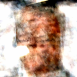}%
  \end{minipage}%
  \begin{minipage}[b]{\imgw}\centering
    \includegraphics[width=\linewidth]{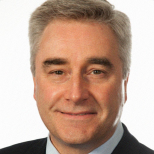}%
  \end{minipage}%
  \begin{minipage}[b]{\imgw}\centering
    \includegraphics[width=\linewidth]{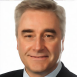}%
  \end{minipage}

  \vspace{0.5ex}

  \noindent
  \begin{minipage}[b]{\imgw}\centering
    \includegraphics[width=\linewidth]{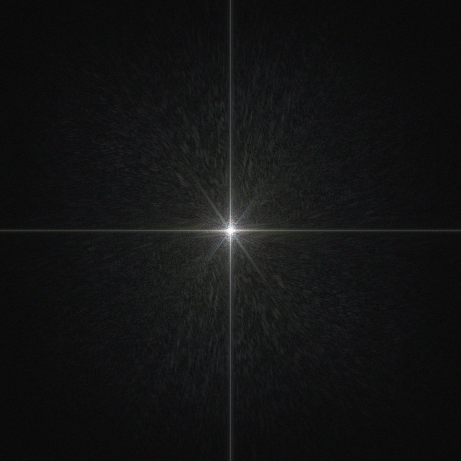}%
  \end{minipage}%
  \begin{minipage}[b]{\imgw}\centering
    \includegraphics[width=\linewidth]{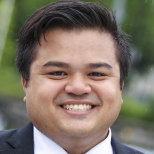}%
  \end{minipage}%
  \begin{minipage}[b]{\imgw}\centering
    \includegraphics[width=\linewidth]{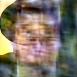}%
  \end{minipage}%
  \begin{minipage}[b]{\imgw}\centering
    \includegraphics[width=\linewidth]{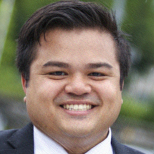}%
  \end{minipage}%
  \begin{minipage}[b]{\imgw}\centering
    \includegraphics[width=\linewidth]{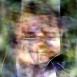}%
  \end{minipage}

  \caption{Nonlinear phase retrieval. Top row: Example 1; bottom row: Example 2.}
  \label{fig:phase_retrieval}
\end{figure}

\section{Ablation study: prompt choice}

Table~\ref{tab:latino_vs_pro} highlights the robustness of the reconstruction quality to slight semantic variations of the initial prompt. In particular, we observe that less informative prompts often yield better metrics than those that include information about the degradation operator. LATINO-PRO is more robust to variations in the prompt initialization, as it seems that the optimization scheme converges towards an optimal prompt in all three cases.

\begin{table}[H]
\centering
\resizebox{\columnwidth}{!}{%
\begin{tabular}{lcccccc}

\multirow{2}{*}{\textbf{Prompt}} & \multicolumn{3}{c}{\textbf{LATINO}} & \multicolumn{3}{c}{\textbf{LATINO-PRO}} \\
\cline{2-7}
& \textbf{LPIPS} ↓ & \textbf{PSNR} ↑ & \textbf{FID} ↓ & \textbf{LPIPS} ↓ & \textbf{PSNR} ↑ & \textbf{FID} ↓ \\
\hline
\texttt{"a photo of"} & \textbf{0.312} & \textbf{26.93} & 29.24 & \textbf{0.299} & \textbf{27.25} & 28.39 \\
\texttt{"a high resolution photo of"} & 0.319 & \underline{26.85} & \underline{29.22} & \underline{0.301} & 27.19 & \textbf{27.59} \\
\texttt{"a sharp photo of"} & \underline{0.318} & 26.88 & \textbf{29.17} & \underline{0.301} & \underline{27.14} & \underline{27.80} \vspace*{-.5cm}
\end{tabular}%
}
\caption{Performance of LATINO and LATINO-PRO on FFHQ-1024 1k test dataset, motion blur task, under different prompts.}
\label{tab:latino_vs_pro}
\end{table}

\section{Ablation study: prior choice}
\label{sec:priors}

Alongside DMD2 \cite{Yin2024ImprovedDM}, other distilled models can be found in the literature, some of which are based on the SD1.5 \cite{Rombach2021} backbone. \cite{Luo2023LCMLoRAAU} introduces a LoRA fine-tuning of SD1.5 that allows a few-step sampling following the CM scheme. We decided to also try this model to show the universal adaptability of our model.\\

We consider the 8-step version of SD1.5-LoRA, and we tried to solve the same inverse problems as done in Section \ref{sec:experiments}: Gaussian deblurring with $\sigma=5.0$ and SR$\times16$ on the AFHQ-512 1k val dataset, Gaussian deblurring with $\sigma=3.0$ and SR$\times8$ on the FFHQ-512 1k val dataset; in all cases $\sigma_n = 0.01$. Table \ref{tab:512_comparison_AFHQ_LoRA} and Table \ref{tab:512_comparison_FFHQ_LoRA} sum up these results, while Figure \ref{fig:qualitative_comparison_AFHQ_LoRA} and Figure \ref{fig:qualitative_comparison_FFHQ_LoRA} show an extended visual comparison. We call this version of our algorithm LATINO-LoRA.\\

To better understand the difference in performances between the SD1.5-LoRA \cite{Luo2023LCMLoRAAU} and the DMD2 \cite{Yin2024ImprovedDM}, we provide in Figure \ref{fig:priors} some prior-generated images from the same prompts used during the reconstructions, focusing on the faces case. It is evident how SD1.5-LoRA tends to generate cartoonish features that are good-looking but unrealistic and that this increases the perceptual distance during the reconstruction process. At the same time, we show how the original SD1.5 can generate quite realistic faces as well, comparable to the DMD2 ones.\\

To further show that the performances are not related to the improved capacities of the LCM prior, we compare LATINO with LDPS, PSLD, and P2L using SDXL, the LDM from which DMD2 \cite{Yin2024ImprovedDM} was distilled, as the prior. Since SDXL uses $50$ time steps, while the other methods use $1000$, we report in Figure~\ref{fig:qualitative_comparison_FFHQ_SDXL} results on a Gaussian deblurring example on FFHQ-1024 for both settings.

\begin{figure}[!h]
    \centering
    \makebox[0.32\linewidth]{\centering \textbf{DMD2}}
    \hfill
    \makebox[0.32\linewidth]{\centering \textbf{SD1.5-LoRA}}
    \hfill
    \makebox[0.32\linewidth]{\centering \textbf{SD1.5}}

    \vspace{2mm} 

    \begin{minipage}[H]{0.32\linewidth}
        \centering
        \includegraphics[width=\linewidth]{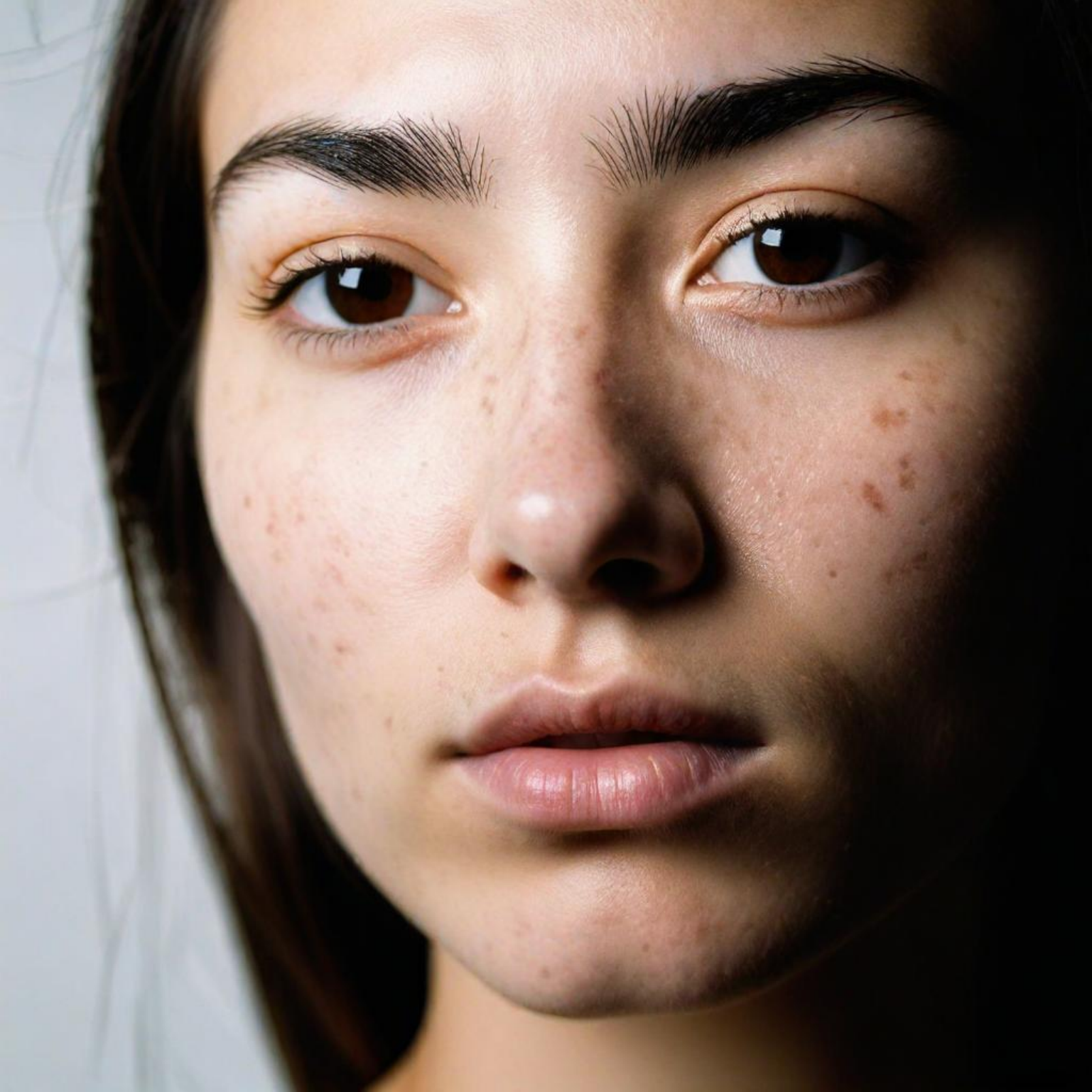}
    \end{minipage}
    \hfill
    \begin{minipage}[H]{0.32\linewidth}
        \centering
        \includegraphics[width=\linewidth]{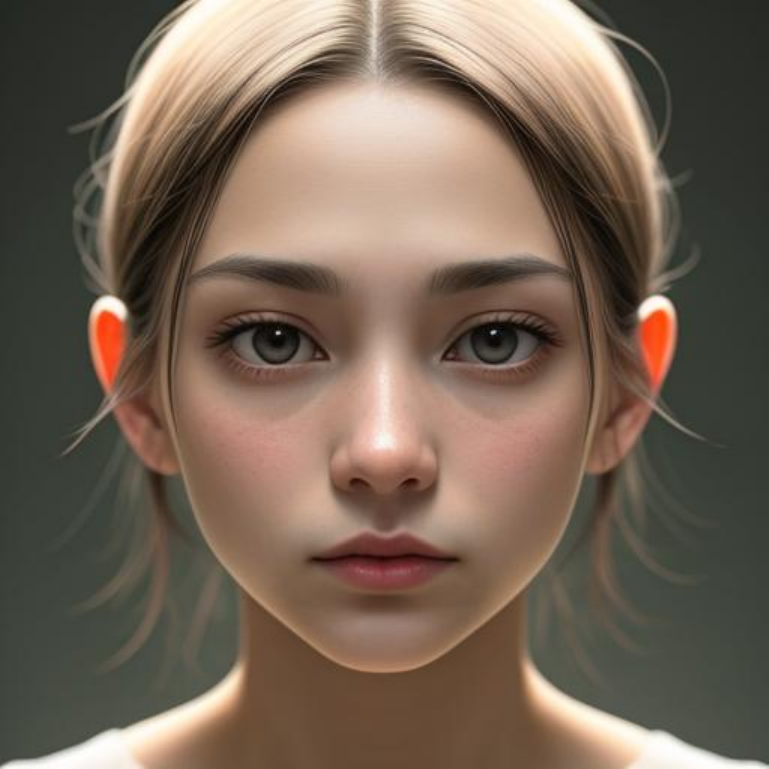}
    \end{minipage}
    \hfill
    \begin{minipage}[H]{0.32\linewidth}
        \centering
        \includegraphics[width=\linewidth]{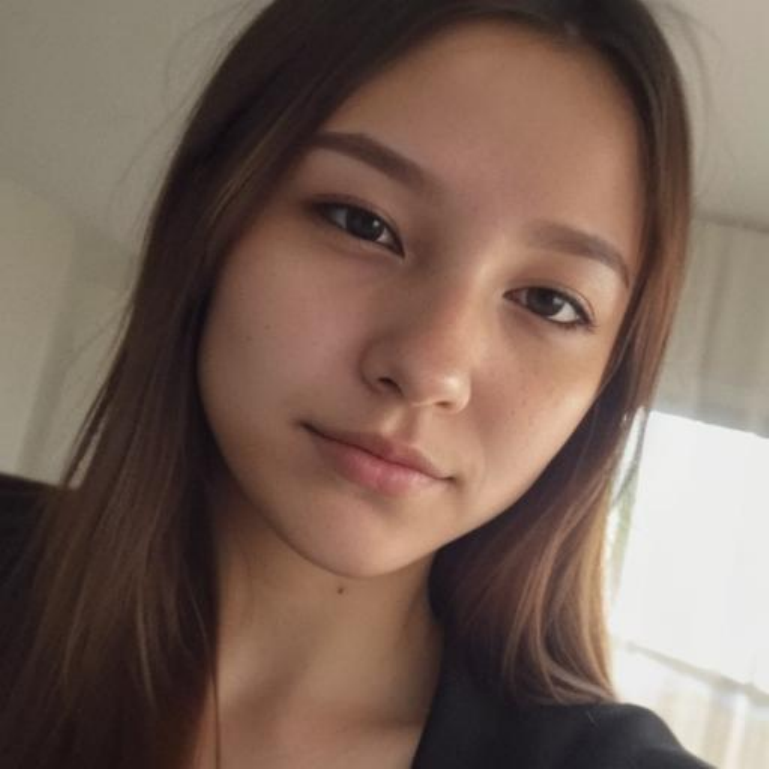}
    \end{minipage}

    \caption{Prior comparisons. Prompt: \texttt{a photo of a face}.}
    \label{fig:priors}
\end{figure}

\begin{figure*}[!h]
\centering
\resizebox{\textwidth}{!}{%
\begin{minipage}{0.03\textwidth}
    \centering
    \begin{tabular}{c}
        \rotatebox{90}{\textbf{Gaussian deblur}}
    \end{tabular}
\end{minipage}%
\begin{minipage}{0.14\textwidth}
    \centering \textbf{Measurement} \\
    \vspace{0.1cm}
    \includegraphics[width=\linewidth]{images/supp/fair/degraded.pdf}
\end{minipage}%
\begin{minipage}{0.14\textwidth}
    \centering \textbf{GT} \\ 
    \vspace{0.1cm}
    \includegraphics[width=\linewidth]{images/supp/fair/clean.pdf}
\end{minipage}%
\begin{minipage}{0.14\textwidth}
    \centering \textbf{LATINO} \\ 
    \vspace{0.1cm}
    \includegraphics[width=\linewidth]{images/supp/parameter/restored.pdf}
\end{minipage}%
\begin{minipage}{0.14\textwidth}
    \centering \textbf{LDPS-50} \\ 
    \vspace{0.1cm}
    \includegraphics[width=\linewidth]{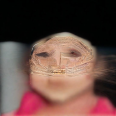}
\end{minipage}%
\begin{minipage}{0.14\textwidth}
    \centering \textbf{PSLD-50} \\ 
    \vspace{0.1cm}
    \includegraphics[width=\linewidth]{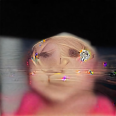}
\end{minipage}%
\begin{minipage}{0.14\textwidth}
    \centering \textbf{P2L-50} \\
    \vspace{0.1cm}
    \includegraphics[width=\linewidth]{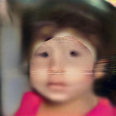}
\end{minipage}%
\begin{minipage}{0.14\textwidth}
    \centering \textbf{LDPS-1000} \\ 
    \vspace{0.1cm}
    \includegraphics[width=\linewidth]{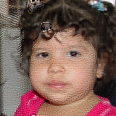}
\end{minipage}%
\begin{minipage}{0.14\textwidth}
    \centering \textbf{PSLD-1000} \\ 
    \vspace{0.1cm}
    \includegraphics[width=\linewidth]{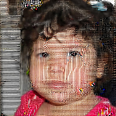}
\end{minipage}%
\begin{minipage}{0.14\textwidth}
    \centering \textbf{P2L-1000} \\
    \vspace{0.1cm}
    \includegraphics[width=\linewidth]{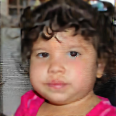}
\end{minipage}
} 
\caption{Comparison on a sample from FFHQ-1024 using SDXL as an LDM prior. Prompt: \texttt{a sharp photo of a face}.}
\label{fig:qualitative_comparison_FFHQ_SDXL}
\end{figure*}

\newcommand{\scaleYY}{0.16\textwidth}  
\newcommand{\sidecapp}[1]{
    \begin{sideways}
        \parbox{\scaleYY}{\centering #1}
    \end{sideways}
}

\begin{figure*}[!h]
\centering
\begin{minipage}{0.03\textwidth}
    \centering
    \begin{tabular}{c}
        \rotatebox{90}{\textbf{SR $\times 8 \quad \quad $}}\\[10mm]
        \rotatebox{90}{\textbf{Gaussian deblur}} \\[6mm]
        \rotatebox{90}{\textbf{Motion deblur}}
    \end{tabular}
\end{minipage}%
\begin{minipage}{0.16\textwidth}
    \centering \textbf{Measurement} \\
    \zoomedImage{images/Tests/SR_x16_bicubic/degraded2.png}{-0.15,-0.7}{1.4,0.9} \\    
    \zoomedImage{images/Tests/Deblurring_3.0_0.01_512/degraded_2_orig.pdf}{0.15,-0.75}{1.4,0.9} \\
    \zoomedImage{images/Tests/Deblurring_motion/degraded.pdf}{0.3,0.1}{1.4,0.9}
\end{minipage}%
\begin{minipage}{0.16\textwidth}
    \centering \textbf{GT} \\
    \zoomedImage{images/Tests/SR_x16_bicubic/clean2.pdf}{-0.15,-0.7}{1.4,0.9} \\
    \zoomedImage{images/Tests/Deblurring_3.0_0.01_512/clean_2_orig.pdf}{0.15,-0.75}{1.4,0.9} \\
    \zoomedImage{images/Tests/Deblurring_motion/clean.pdf}{0.3,0.1}{1.4,0.9}
\end{minipage}%
\begin{minipage}{0.16\textwidth}
    \centering \textbf{LATINO} \\
    \zoomedImage{images/Tests/SR_x16_bicubic/restored2.pdf}{-0.15,-0.7}{1.4,0.9} \\    
    \zoomedImage{images/Tests/Deblurring_3.0_0.01_512/restored_2_orig.pdf}{0.15,-0.75}{1.4,0.9} \\
    \zoomedImage{images/Tests/Deblurring_motion/restored.pdf}{0.3,0.1}{1.4,0.9}
\end{minipage}%
\begin{minipage}{0.16\textwidth}
    \centering \textbf{LATINO-PRO} \\ 
    \zoomedImage{images/Tests/SR_x16_bicubic/restored2_SOUL.pdf}{-0.15,-0.7}{1.4,0.9} \\    
    \zoomedImage{images/Tests/Deblurring_3.0_0.01_512/restored_2_SOUL_orig.pdf}{0.15,-0.75}{1.4,0.9} \\
    \zoomedImage{images/Tests/Deblurring_motion/restored_SOUL.pdf}{0.3,0.1}{1.4,0.9}
\end{minipage}%
\begin{minipage}{0.16\textwidth}
    \centering \textbf{LATINO-LoRA} \\ 
    \zoomedImage{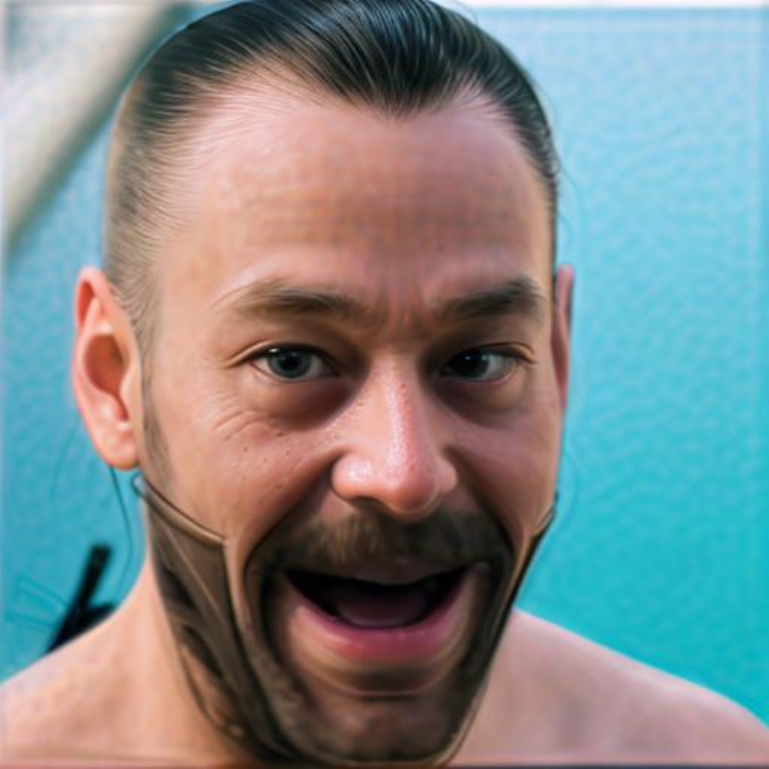}{-0.15,-0.7}{1.4,0.9} \\    
    \zoomedImage{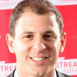}{0.15,-0.75}{1.4,0.9} \\
    \zoomedImage{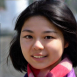}{0.3,0.1}{1.4,0.9}
\end{minipage}%
\vspace{-0.5cm}
\caption{Qualitative comparison of image restoration results. Samples taken from FFHQ-512. Prompt: \texttt{a sharp photo of a face}.\vspace{.2cm}}
\label{fig:qualitative_comparison_FFHQ_LoRA}
\end{figure*}

\section{Equivalent HR problems}\label{sec:equivalent-HR}

\subsection{The deblurring case}
\label{sec:deblur}
In order to precisely compare our deblurring results with the ones of methods that work with lower resolutions we need to transform our low-resolution problem into an equivalent high-resolution one, and go back to lower resolution.

\paragraph{LR problem:} Our original problem is the following deblurring one: Find $x$ from measurement $y$ obtained via~(1)
\begin{equation*}
(1) \quad y = x * h + n
\end{equation*}
 where $*$ denotes convolution, and $h$ is a low-resolution blur kernel.
\paragraph{HR problem:} Our proxy high-resolution problem is a joint deblurring and super-resolution problem: Find $X$ from measurements $y$ obtained via (2)
\begin{equation*}
(2) \quad y = S_s ( X * H ) + n
\end{equation*}
where $H$ is a high-resolution blur kernel. The output of our algorithm is $x=S_s(X)$, where $S_s$ is a downsampling operator to be defined below. In our case $s=2$ since our method works at resolution $1024\times 1024$ and the other algorithms work at resolution $512 \times 512$.\\

\begin{table*}[h]
\centering
\begin{tabular}{lcccccccccc}
\toprule
 &  & \multicolumn{3}{c}{\textbf{Deblur (Gaussian)}} & \multicolumn{3}{c}{\textbf{Deblur (Motion)}} & \multicolumn{3}{c}{\textbf{SR$\times16$}} \\ 
   \cmidrule(lr){3-5} \cmidrule(lr){6-8}\cmidrule(lr){9-11}
\textbf{Method} & \textbf{NFE↓} & \textbf{FID↓} & \textbf{PSNR↑} & \textbf{LPIPS↓} & \textbf{FID↓} & \textbf{PSNR↑} & \textbf{LPIPS↓} & \textbf{FID↓} & \textbf{PSNR↑} & \textbf{LPIPS↓}\\ \hline
\textbf{LATINO-PRO} &  68  &   31.98    &    28.76    & 0.372  & 27.80     &   26.89       &  0.423        & 40.95     &   25.81        & 0.445           \\ 
\textbf{LATINO}       &  8   & 33.94 &    28.41                 & 0.382 &  29.17 &   26.58  & 0.445   & 37.13     &   25.71      &    0.450    \\ \bottomrule
\end{tabular}
\caption{Results for Gaussian deblurring with $\sigma = 6.0$, motion deblurring, and $\times 16$ super-resolution, all with noise $\sigma_y = 0.01$ on the FFHQ-1024 val dataset. Our LATINO and LATINO-PRO are compared. Prompt: \texttt{a sharp photo of a face}.\vspace{1cm}}
\label{tab:1024_comparison_FFHQ}
\end{table*}

If we want problems (1) and (2) to be equivalent, we need to find a high-resolution kernel $H$ such that:
\begin{equation}\label{eq:HR-LR-equiv}
S_s(X) * h = S_s ( X * H ).    
\end{equation}    

The following definition establishes a family of subsampling operators $S_s$ for which condition \eqref{eq:HR-LR-equiv} is satisfied, as long as the HR kernel $H$ is chosen as shown in proposition~\ref{prop:alias-free-subsampling}.

\begin{definition} An alias-free subsampling operator $S_s$ is defined as:
\begin{equation*}
S_s(X) = \downarrow_s( h_s * X )
\end{equation*}
where $\downarrow_s$ is the decimation operator (which takes one sample every $s$ pixels without any filtering) and $h_s$ is a convolution kernel with spectral support in $[-\pi/s, \pi/s]$.
\end{definition}

The following choices provide alias-free subsampling operators (or approximations thereof)
\begin{enumerate}
    \item $h_s = \text{sinc}( \cdot / s ) = \text{sinc}_s$, i.e. standard Shannon subsampling.
    \item $h_s = \mathcal{F}^{-1} ( \varphi )$ where $\varphi$ is a smooth function with support in $[-\pi/s\, \pi/s]$. Avoids Gibbs artifacts that are commonly associated with Shannon subsampling.
    \item $h_s$ = kernel implicit in spline downsampling
    of order $k$. This is not exactly alias-free, but a good approximation of Shannon subsampling for sufficiently large $k$. In practice we use bicubic downsampling for $k=3$, which is a sufficiently good approximation.
\end{enumerate}

\begin{proposition}\label{prop:alias-free-subsampling}    
    If $S_s$ is an alias-free subsampling operator, then any $H$ satisfying
    \begin{equation}\label{eq:equivalent-kernel}
    h = \downarrow_s ( \text{sinc}_s * H )
    \end{equation}
    satisfies the equivalence condition~\eqref{eq:HR-LR-equiv}.
\end{proposition}

\begin{proof}
    Let $N^2$ be the size of the image $X$. Let $P_s$ denote the periodization operator with period $N/s$. Then, using the fact that convolution (respectively s-sampling) becomes a product (respectively $N/s$-periodization), we have that
    \begin{align*}
    &\mathcal{F}(S_s(X * H)) \\
    =& \mathcal{F}(\downarrow_s(h_s * X * H)) \\
    =& P_s \left[ \mathcal{F}(h_s) \mathcal{F}(X) \mathcal{F}(H) \mathcal{F}(\text{sinc}_s) \right] \\
    =& P_s \left[ \mathcal{F}(h_s) \mathcal{F}(X) \right] P_s \left[ \mathcal{F}(H) \mathcal{F}(\text{sinc}_s) \right] \\
    =& \mathcal{F}( \downarrow_s(h_s * X) ) \mathcal{F}( \downarrow_s( H * \text{sinc}_s ) ) \\
    =& \mathcal{F}( S_s(X) ) \mathcal{F}( \downarrow_s( H * \text{sinc}_s ) ).
    \end{align*}
The third line is true because both $\mathcal{F}(h_s) \mathcal{F}(X)$ and $\mathcal{F}(H) \mathcal{F}(\text{sinc}_s)$ are supported in $[-\pi/s,\pi/s]$. Taking inverse fourier transform we have
$$ S_s(X*H) = S_s(X) * \downarrow_s (H*\text{sinc}_s) = S_s(X) * h $$
and the equivalence is established.
\end{proof}

\paragraph{Practical Considerations.}
When $S_s$ is Shannon subsampling, any $H$ such that $\hat{H}|_{[-\pi/s,\pi/s]^2} = \hat{h}$ satisfies condition \eqref{eq:equivalent-kernel}. In particular we can choose Shannon (zero-padding) upsampling, or a kernel $H$ that satisfies \eqref{eq:equivalent-kernel} and minimizes the total variation (to minimize Gibbs artifacts).\\

Similarly when $S_s$ is bicubic downsampling, we choose $H$ as bicubic upsampling of $h$. In this case we have an approximation of conditions \eqref{eq:HR-LR-equiv}~and~\eqref{eq:equivalent-kernel} that is good enough for our purposes.
\newpage
\subsection{The super-resolution case}
\label{sec:super-res}
To precisely compare our super-resolution results with the ones of methods that work with lower resolutions, we need to transform our low-resolution problem into an equivalent high-resolution one, and go back to lower resolution.
\paragraph{LR problem:} Our original problem is the following super-resolution problem: Find $x$ from the measurements $y$ obtained via (1)
\begin{equation*}
(1) \quad y = S_a(x) + n,
\end{equation*}
where $S_s$ is a downsampling operator of factor $s>0$.
\paragraph{HR problem:} Our proxy high-resolution problem is a super-resolution problem: Find $X$ from measurements $y$ obtained via (2):
\begin{equation*}
(2) \quad y = S_{ab} ( X ) + n.
\end{equation*}
The output of our algorithm is $x=S_b(X)$. In our examples $a=8$ and $b=2$ so that $ab=16$, or $a=16$ and $b=2$ so that $ab=32$.

If we want problems (1) and (2) to be equivalent, we need to find a subsampling operator such that:
\begin{equation}\label{eq:HR-LR-equiv-SR}
S_{ab}(X) = S_a ( S_b(X) ).    
\end{equation}

Any subsampling operator derived from a wavelet transform (including average pooling) satisfies condition~\eqref{eq:HR-LR-equiv-SR}. So does Shannon sub-sampling. Bicubic downsampling approximately satisfies condition~\eqref{eq:HR-LR-equiv-SR} since it is a good approximation of Shannon subsampling.

\begin{figure*}[!h]
    \centering
    \begin{center} 
        \begin{tabular}{c c c} 
            \centering
             \sidecap{Measurement}&\includegraphics[width=0.40\textwidth]{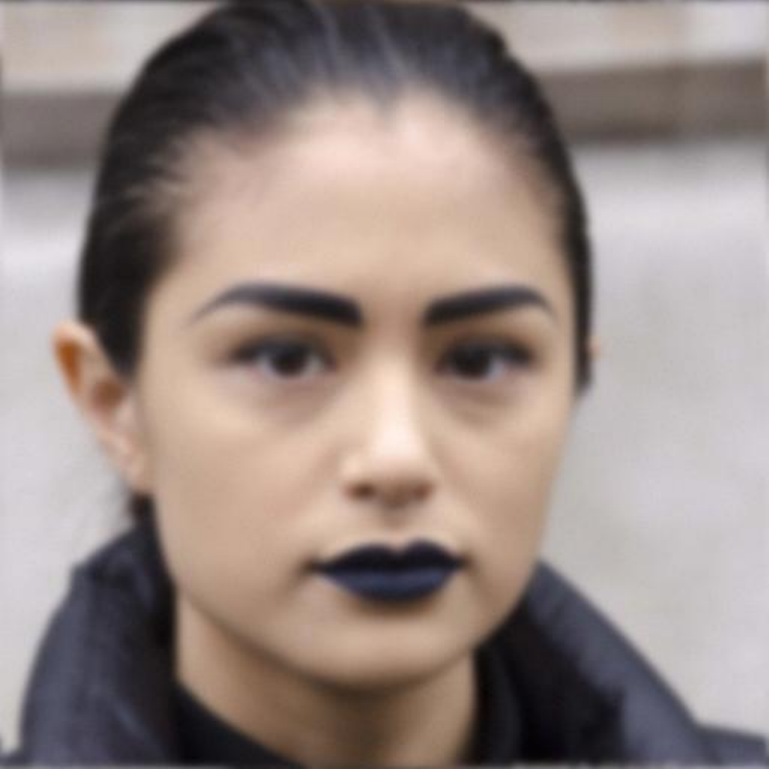} & 
            \includegraphics[width=0.40\textwidth]{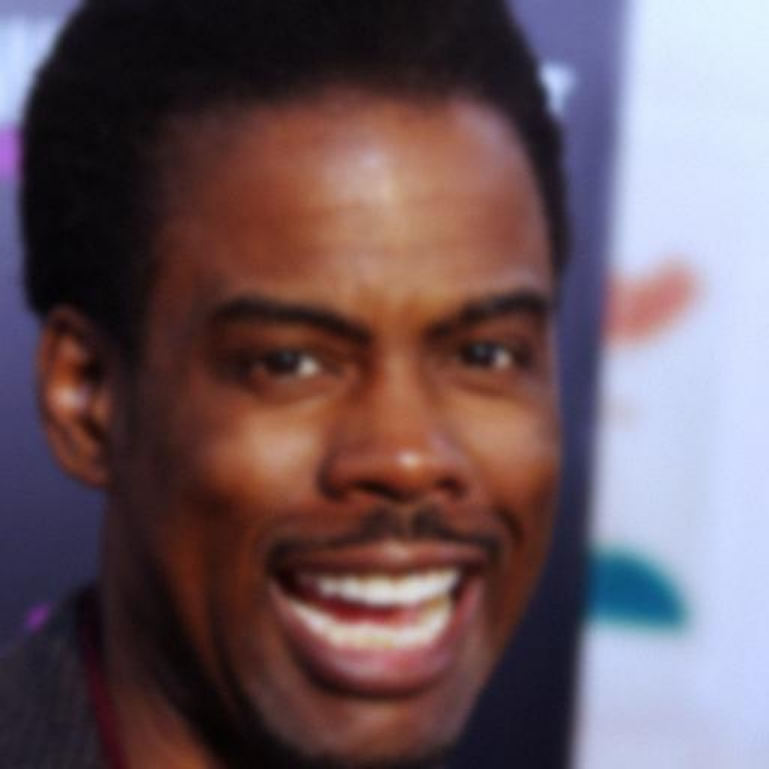} \\

            \sidecap{Ground truth}&\includegraphics[width=0.40\textwidth]{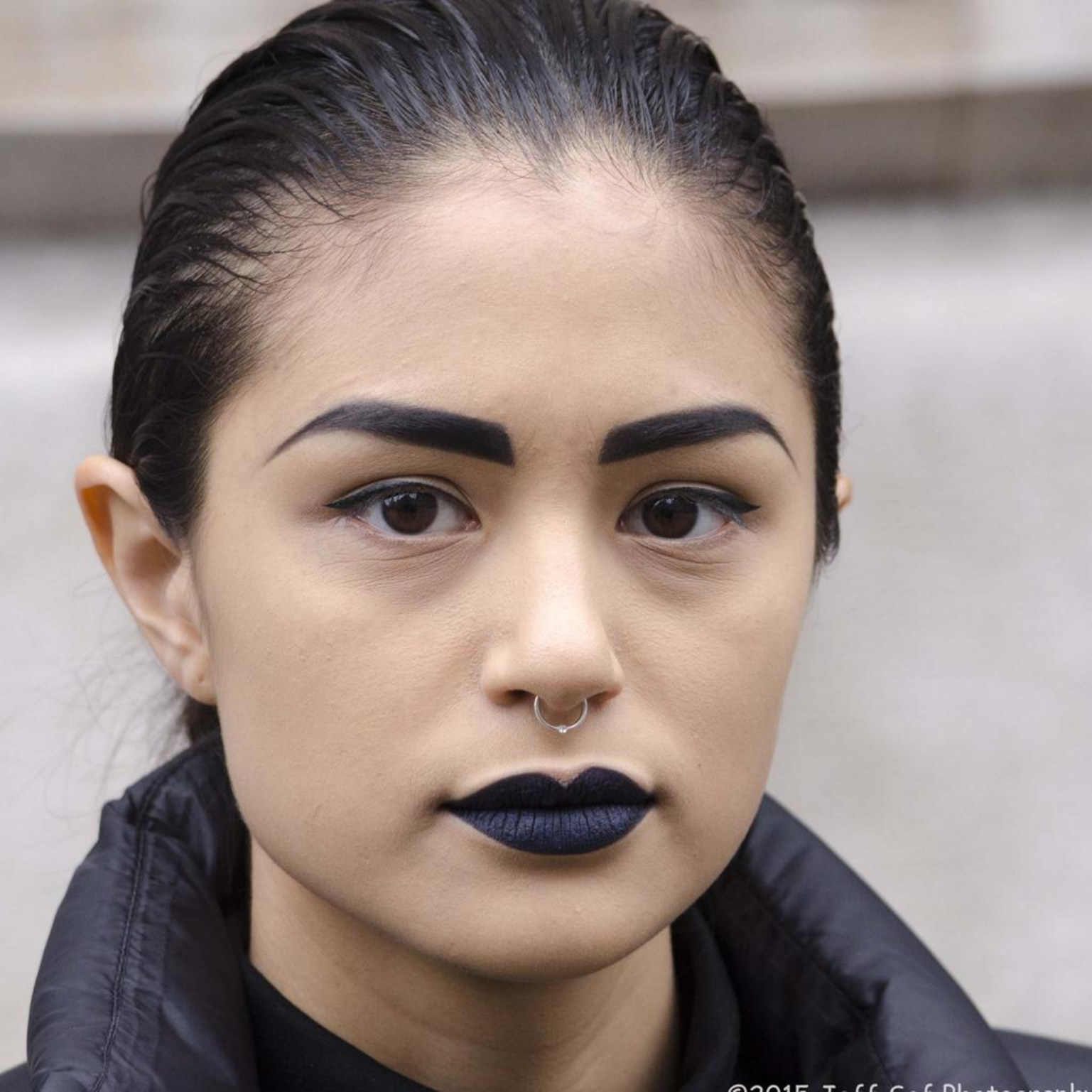} & 
            \includegraphics[width=0.40\textwidth]{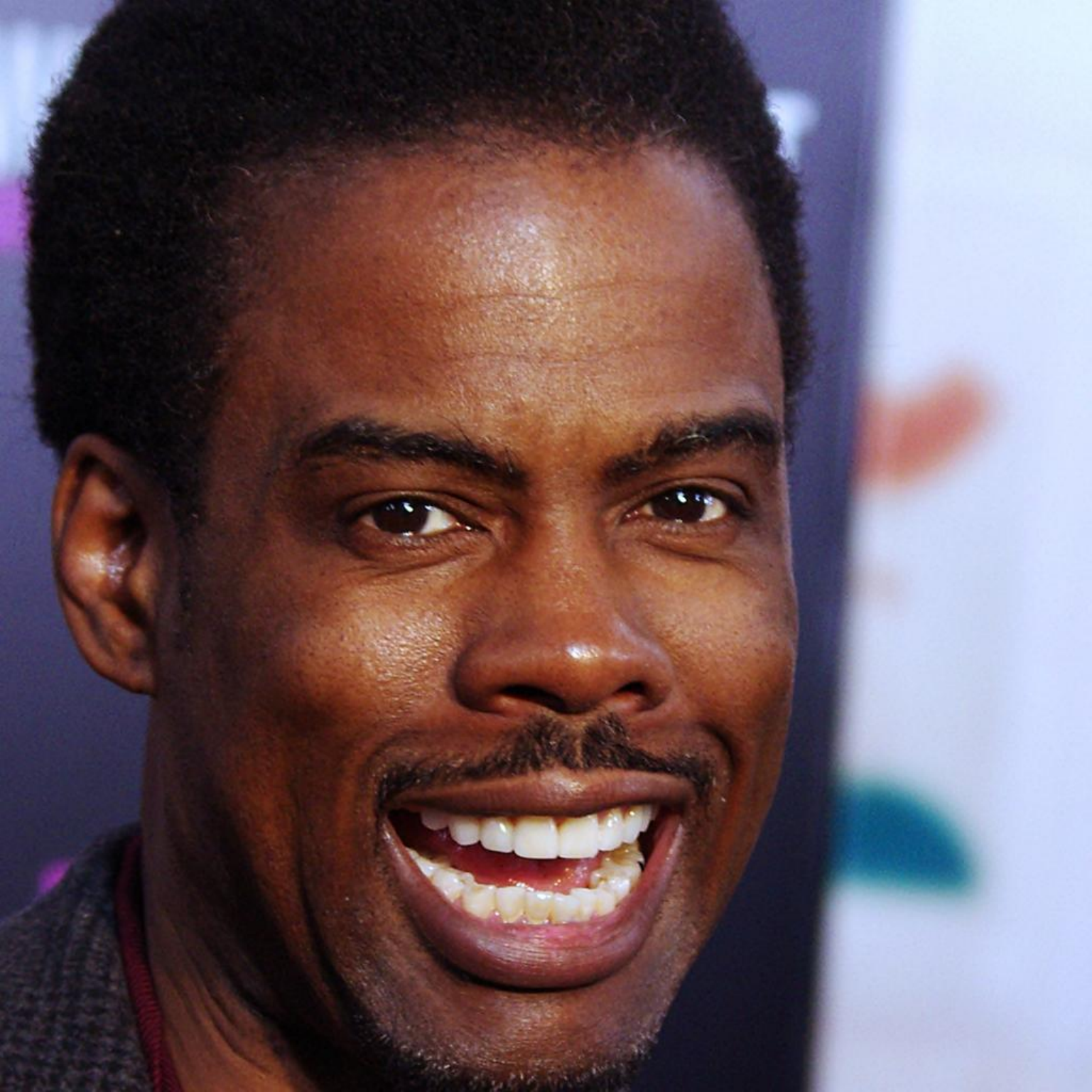} \\

            \sidecap{Restored}&\includegraphics[width=0.40\textwidth]{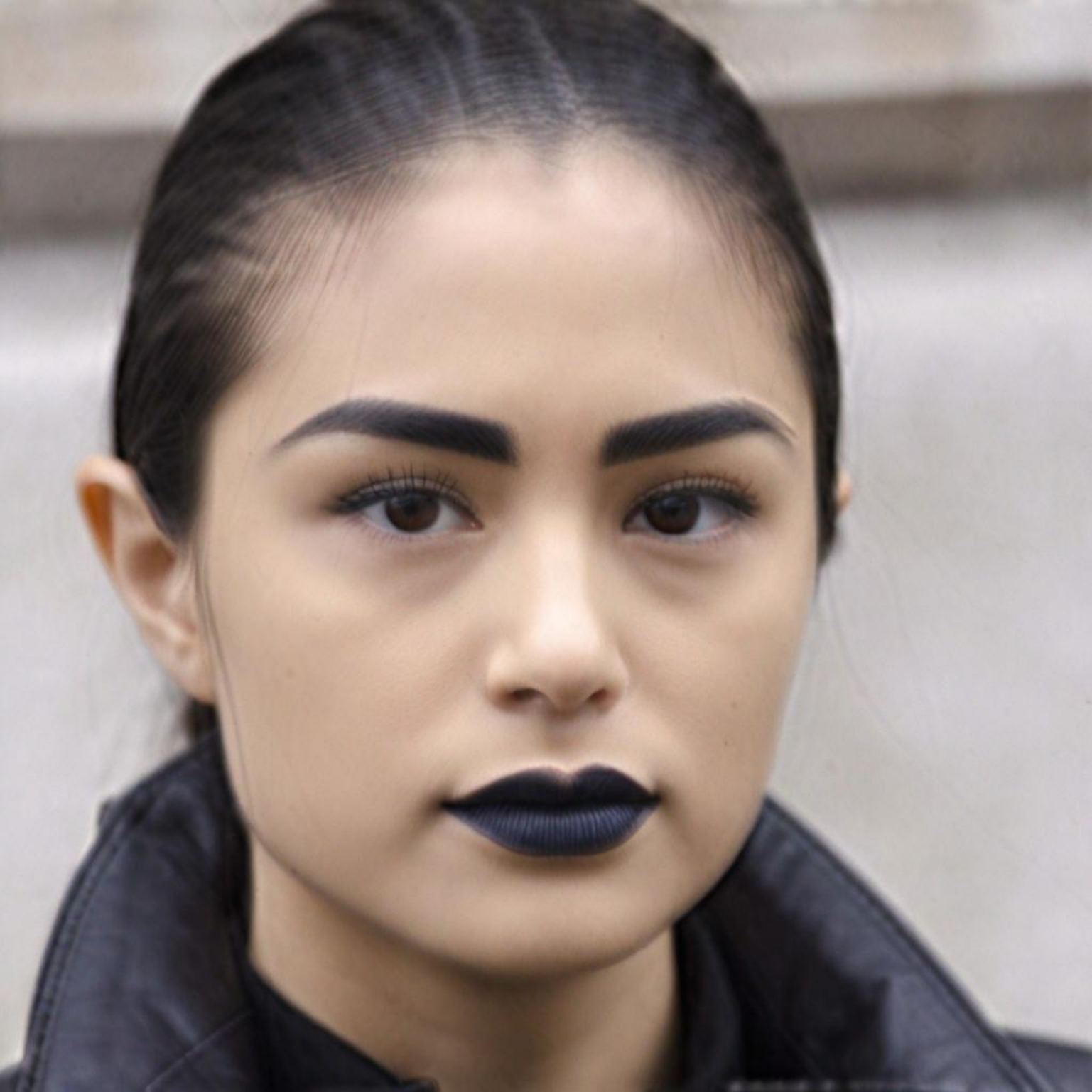} & 
            \includegraphics[width=0.40\textwidth]{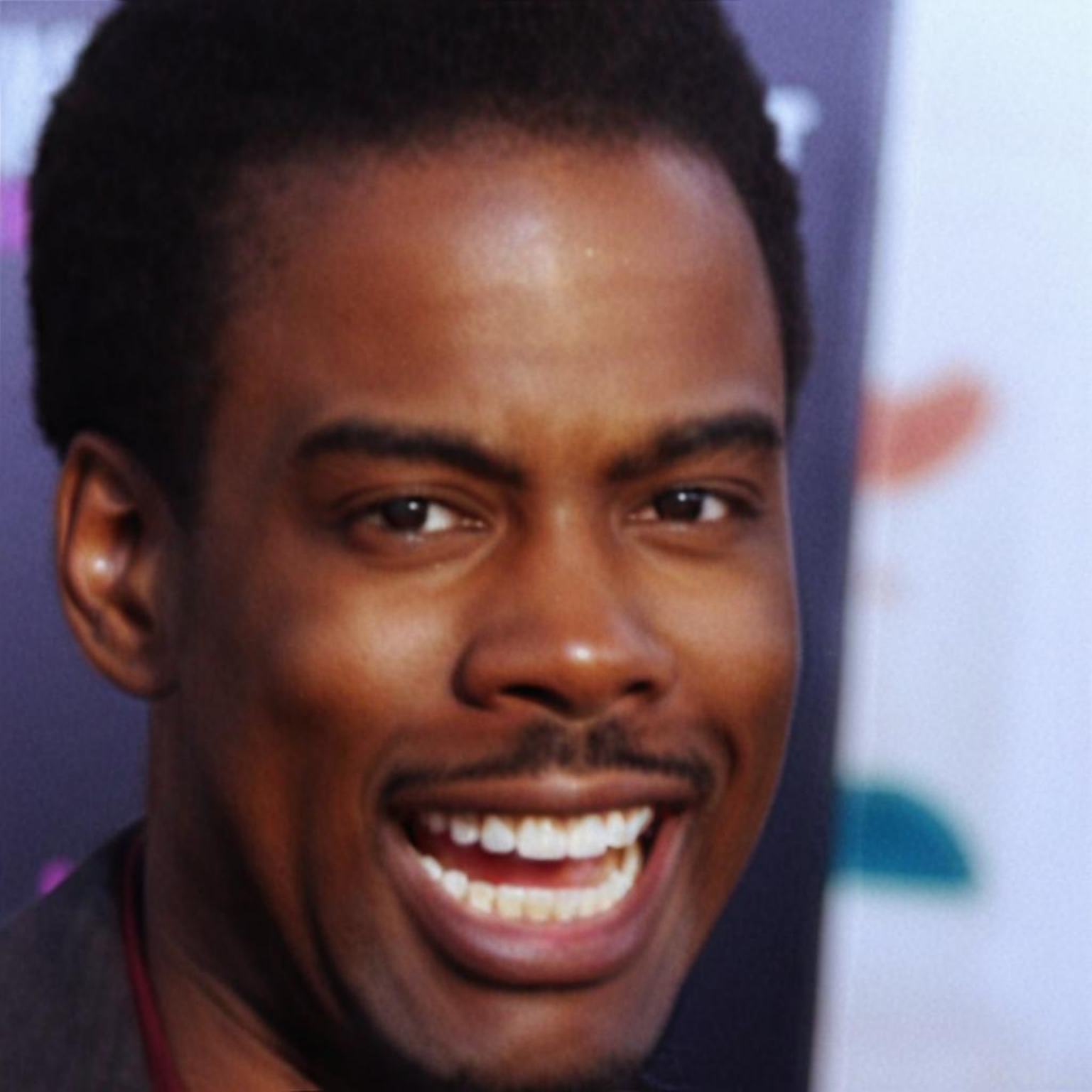} \\
        \end{tabular}
    \end{center}

    \caption{Gaussian deblur FFHQ-1024 LATINO-PRO.}
    \label{fig:FFHQ_Deblur}
\end{figure*}

\begin{figure*}[!h]
    \centering
    \begin{center} 
        \begin{tabular}{c c c} 
            \centering
             \sidecap{Measurements}&\includegraphics[width=0.40\textwidth]{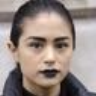} & 
            \includegraphics[width=0.40\textwidth]{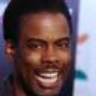} \\

             \sidecap{Ground truth}&\includegraphics[width=0.40\textwidth]{images/Tests/Additional/clean1.pdf} & 
            \includegraphics[width=0.40\textwidth]{images/Tests/Additional/clean0.pdf} \\

             \sidecap{Restored}&\includegraphics[width=0.40\textwidth]{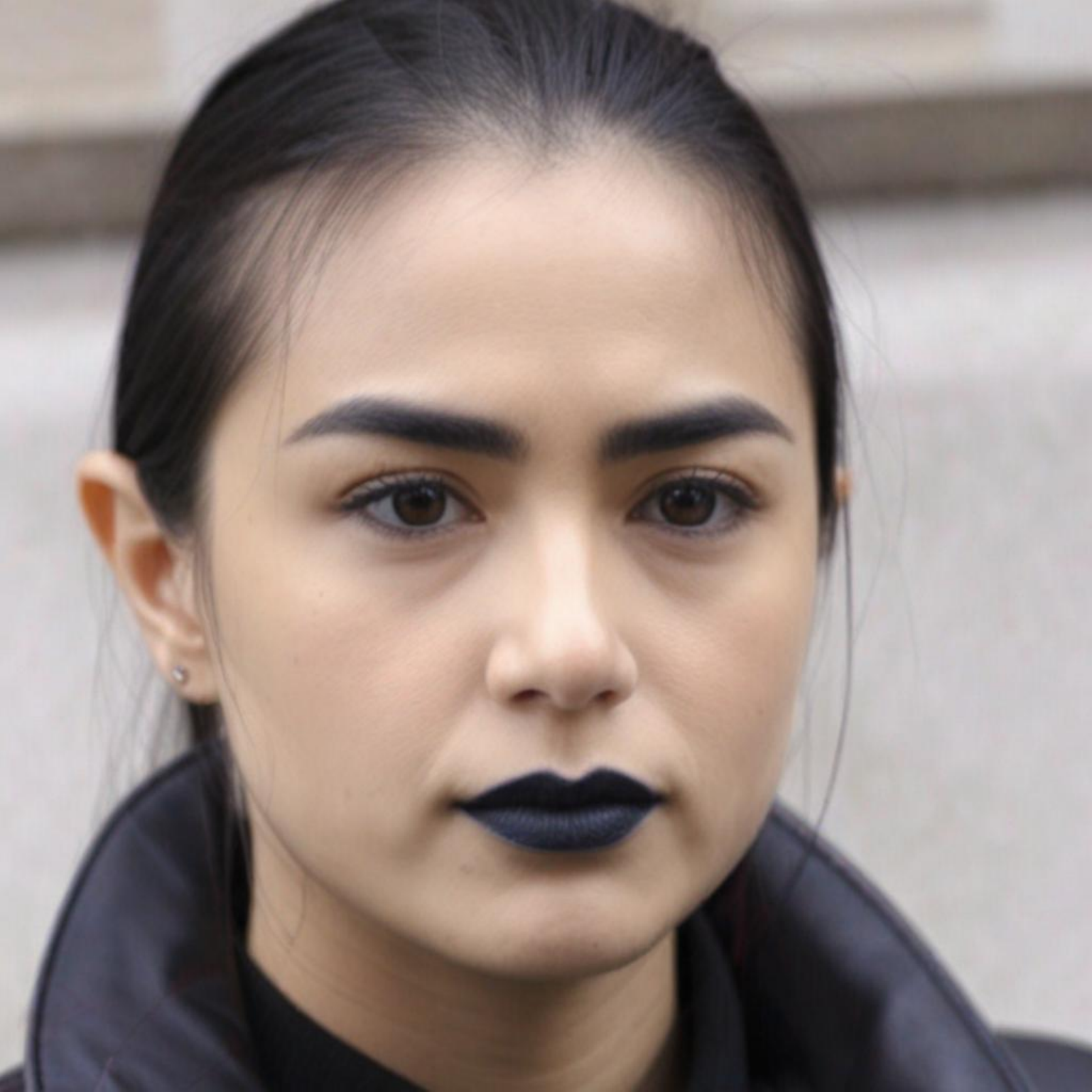} & 
            \includegraphics[width=0.40\textwidth]{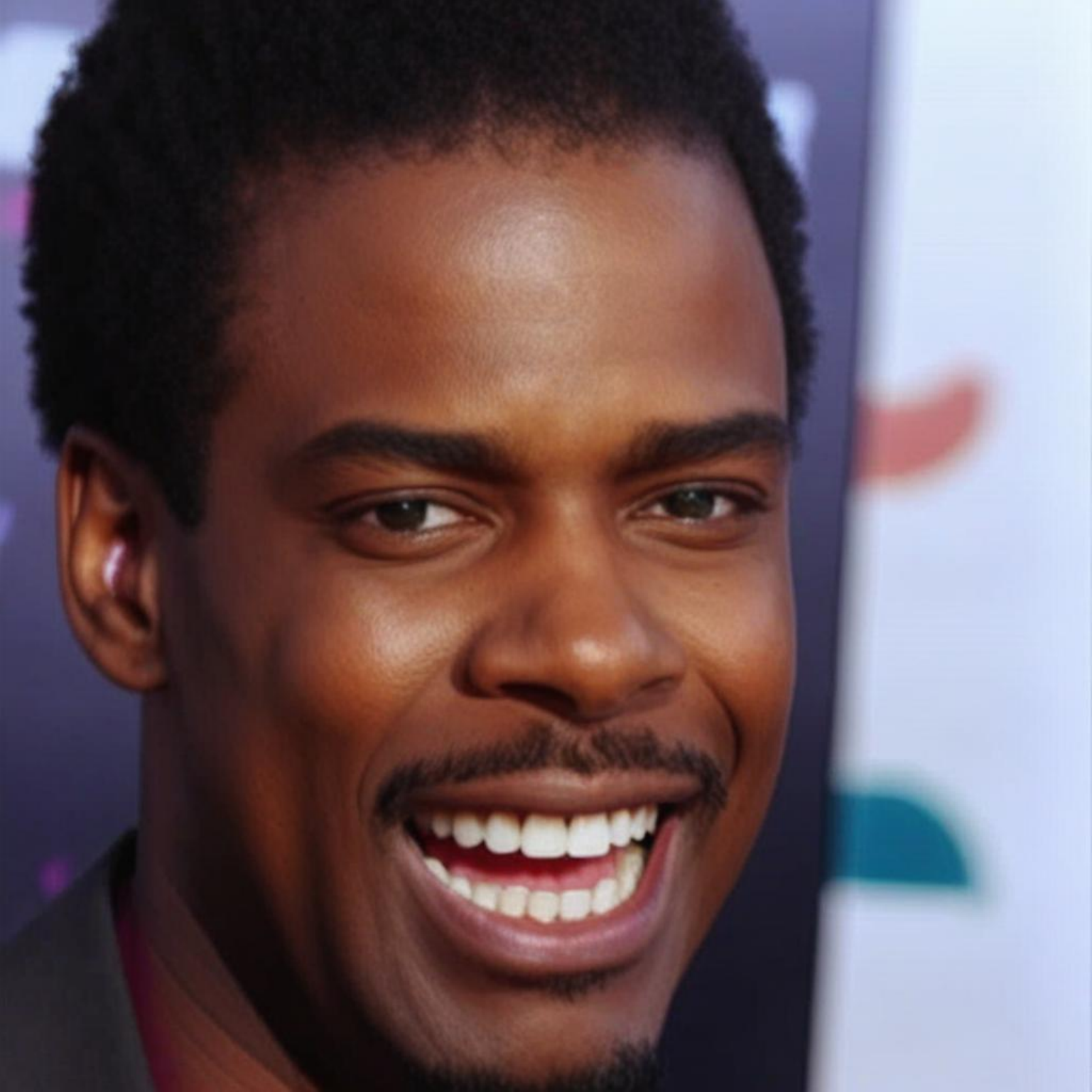} \\
        \end{tabular}
    \end{center}

    \caption{SR$\times16$ FFHQ-1024 LATINO-PRO.}
    \label{fig:FFHQ_SR_sx16}
\end{figure*}

\begin{figure*}[!h]
    \centering
    \begin{center} 
        \begin{tabular}{c c c} 
            \centering
            \sidecap{Measurements}&\includegraphics[width=0.40\textwidth]{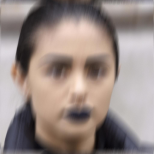} & 
            \includegraphics[width=0.40\textwidth]{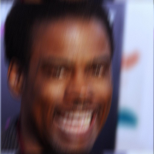} \\

            \sidecap{Ground truth}&\includegraphics[width=0.40\textwidth]{images/Tests/Additional/clean1.pdf} & 
            \includegraphics[width=0.40\textwidth]{images/Tests/Additional/clean0.pdf} \\

            \sidecap{Restored}&\includegraphics[width=0.40\textwidth]{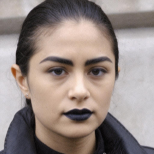} & 
            \includegraphics[width=0.40\textwidth]{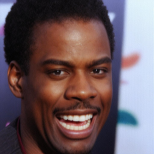} \\
        \end{tabular}
    \end{center}

    \caption{Motion Deblurring FFHQ-1024 LATINO-PRO.}
    \label{fig:FFHQ_motion}
\end{figure*}

\section{Additional results: FFHQ1024}
\label{sec:additional}

In Figures \ref{fig:FFHQ_Deblur}, \ref{fig:FFHQ_SR_sx16} and \ref{fig:FFHQ_motion} we show more tests at the $1024\times 1024$ resolution. We also provide in Table \ref{tab:1024_comparison_FFHQ} the metrics for the FFHQ-1024 1k val dataset, to allow comparisons with future works. The tasks considered are the same of Section \ref{sec:experiments}, adapted to the new resolution, as discussed in \ref{sec:equivalent-HR}.
\newpage

\section{Memory and time consumption}
\label{sec:Memory}

In Table \ref{tab:gpu_time_comparison} we provide an exhaustive comparison of our models with respect to current SOTA in terms of memory consumption and time needed. We implemented versions of P2L and TReg starting from the official codebases as described in \cite{Chung2023PrompttuningLD, Kim2023RegularizationBT}. For estimating the memory consumption and speed in the XL versions of these two methods, we adapted such codes to the SDXL prior. As an important remark, we highlight that SD provides float16 implementations (half precision) to speed up and reduce memory usage. However, this implementation does not allow taking gradients with respect to the score and decoder network, as it results in an integer overflow error during the \texttt{torch.autograd()} call. This forces all the implementations below (except for TReg, as it does not require such gradients) to make use of the float32 (full precision) implementation of SDXL, which explains the even bigger overhead in GPU usage and time.
\begin{table}[!h]
\centering
\small
\begin{tabular}{lccc}
\toprule
\textbf{Method} & \textbf{GPU (Gb)} & \textbf{Time (s)} & \textbf{Resolution}\\
\hline
\textbf{LATINO} & 13.6 & 5.53 & $1024^2$ \\

\textbf{LATINO-PRO} & 23.4 & 48.8 & $1024^2$ \\

\textbf{LATINO-LoRA} & 4.16 & 2.89 & $512^2$ \\
\hline
\textbf{TReg} & 7.75 & 60.5 & $512^2$ \\

\textbf{P2L} & 8.18 & 402 & $512^2$ \\

\textbf{LDPS} & 8.16 & 279 & $512^2$ \\

\textbf{PSLD} & 9.44 & 326 & $512^2$ \\

\textbf{LDPS-XL} & 56.6 & 1670 & $1024^2$ \\

\textbf{PSLD-XL} & 69.5 & 2200 & $1024^2$ \\

\textbf{TReg-XL} & 33.5 & 240 & $1024^2$ \\

\textbf{P2L-XL} & 57.1 & 6800 & $1024^2$ \\
\bottomrule
\end{tabular}
\caption{GPU Memory and Time consumption comparison}
\label{tab:gpu_time_comparison}
\end{table}

The NFEs considered for each algorithm are the same as shown in Table \ref{tab:512_comparison_AFHQ}. As a reference, we also provide in Table \ref{tab:gpu_priors} the GPU memory consumption of the respective priors, which can be considered as lower bounds. The times are those obtained by running the algorithms on a single Nvidia A100 GPU, averaging the times of the different inverse problems considered.

\begin{table}[!h]
\centering
\begin{tabular}{lcc}
\toprule
\textbf{Prior} \textbf{Method} & \textbf{GPU (Gb)} & \textbf{Resolution}\\
\hline
\textbf{DMD2} & 10.7 & $1024^2$ \\

\textbf{SD1.5} & 3.25 & $512^2$ \\

\textbf{SD1.5LoRA} & 3.84 & $512^2$ \\
\bottomrule
\end{tabular}
\caption{GPU Memory consumption when sampling an image with different generative models priors).}
\label{tab:gpu_priors}
\end{table}

\section{Comparison with TReg}
\label{sec:TReg}

One of the main strengths of SOTA algorithms such as TReg is the possibility of shifting the semantic domain of the reconstruction through the prompt $c$. As we described in Section \ref{sec:SOUL}, our algorithm can obtain the same type of results, providing a useful fast and light tool for image editing. To prove this, we performed experiments on the Food-101 dataset \cite{Bossard2014Food101M} as done in \cite{Kim2023RegularizationBT} as shown in Figure \ref{fig:Food}. The degradations used are Gaussian Deblurring of intensity $\sigma=5.0$ and super-resolution $\times 16$, both with an additional white noise of $\sigma_n=0.01$. All the images are at $512\times512$ resolution, meaning that, as for the AFHQ dataset, we have to rescale the problems to their equivalent $1024\times 1024$ resolution versions as seen in \ref{sec:deblur},\ref{sec:super-res}. In particular, the images obtained are at a higher resolution than the original ones, i.e. $1024\times 1024$.

To further show the improvements in the reconstruction quality with respect to both TReg and P2L, we provide a comparison of the results obtained on the AFHQ val dataset. Together with Table \ref{tab:512_comparison_AFHQ}, we show in Figure \ref{fig:AFHQ_dogs} and \ref{fig:AFHQ_cats} visual examples.

{
\begin{figure*}[!h]
    \centering
    \makebox[0.15\textwidth]{\textbf{Measurement}}
    \hfill
    \makebox[0.15\textwidth]{\textbf{GT}}
    \hfill
    \makebox[0.15\textwidth]{\textbf{Spaghetti}}
    \hfill
    \makebox[0.15\textwidth]{\textbf{Macarons}}
    \hfill
    \makebox[0.15\textwidth]{\textbf{Hamburger}}
    \hfill
    \makebox[0.15\textwidth]{\textbf{Fried rice}}
    
    \begin{minipage}{0.15\textwidth}
        \includegraphics[width=\textwidth]{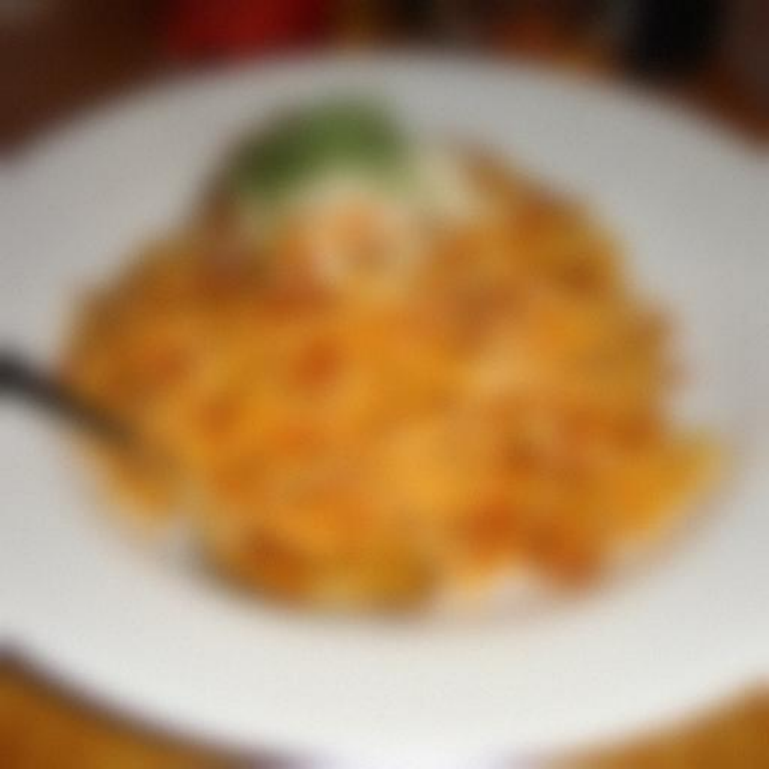}\\
        \includegraphics[width=\textwidth]{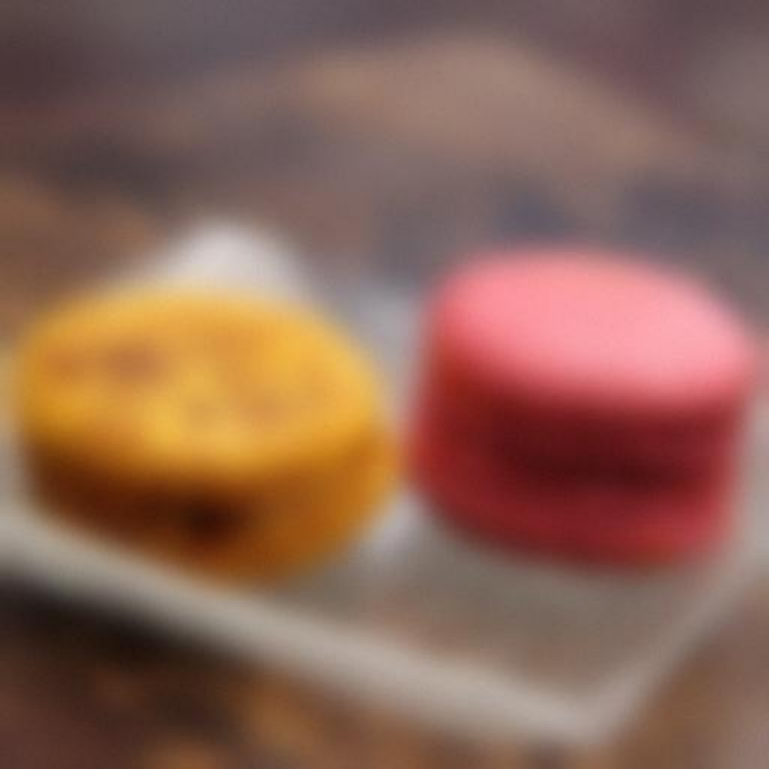}\\
        \includegraphics[width=\textwidth]{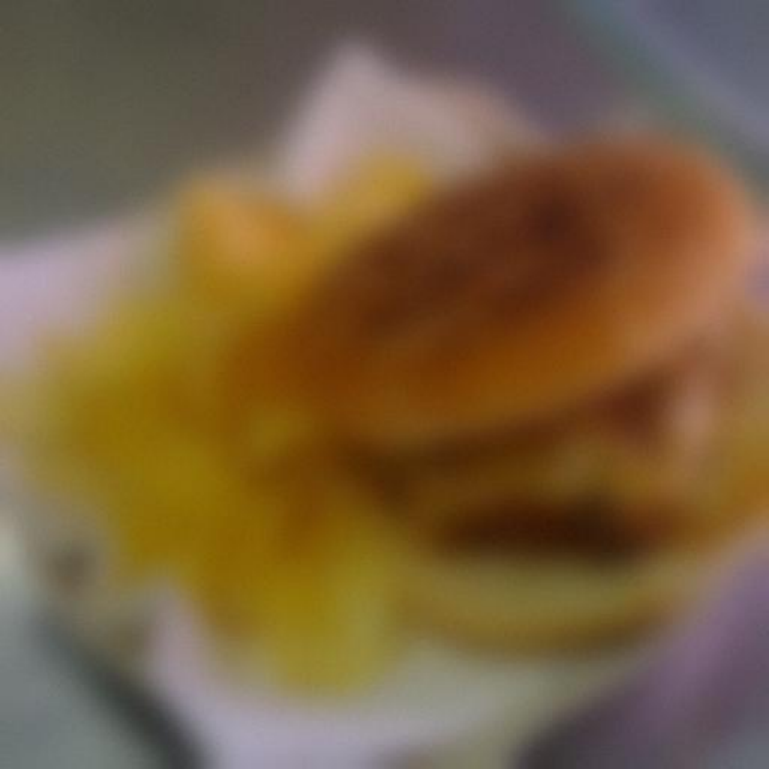}\\
        \includegraphics[width=\textwidth]{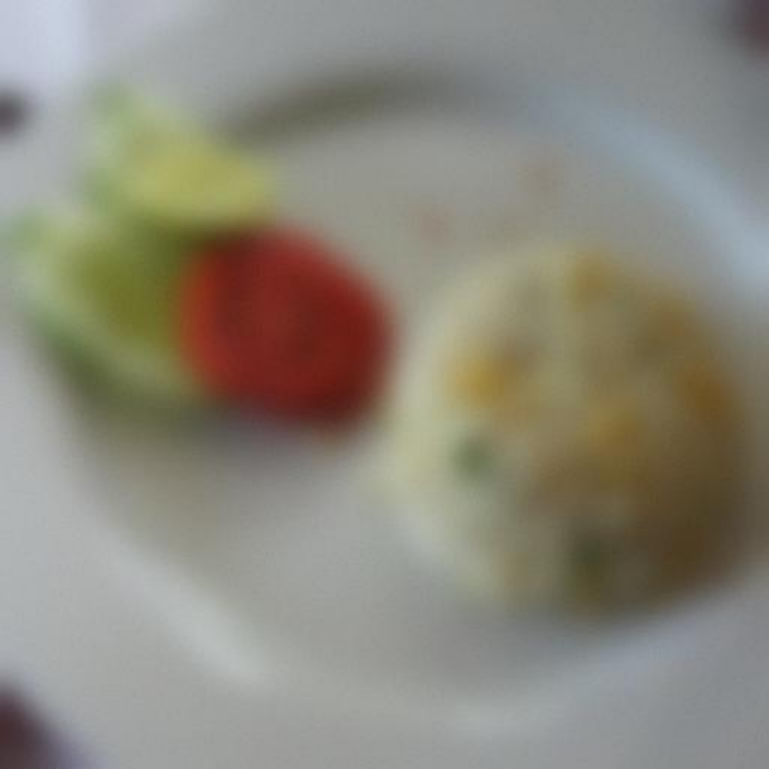}\\
        \includegraphics[width=\textwidth]{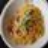}\\
        \includegraphics[width=\textwidth]{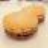}\\
        \includegraphics[width=\textwidth]{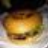}\\
        \includegraphics[width=\textwidth]{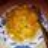}
    \end{minipage}
    \hfill
    \begin{minipage}{0.15\textwidth}
        \includegraphics[width=\textwidth]{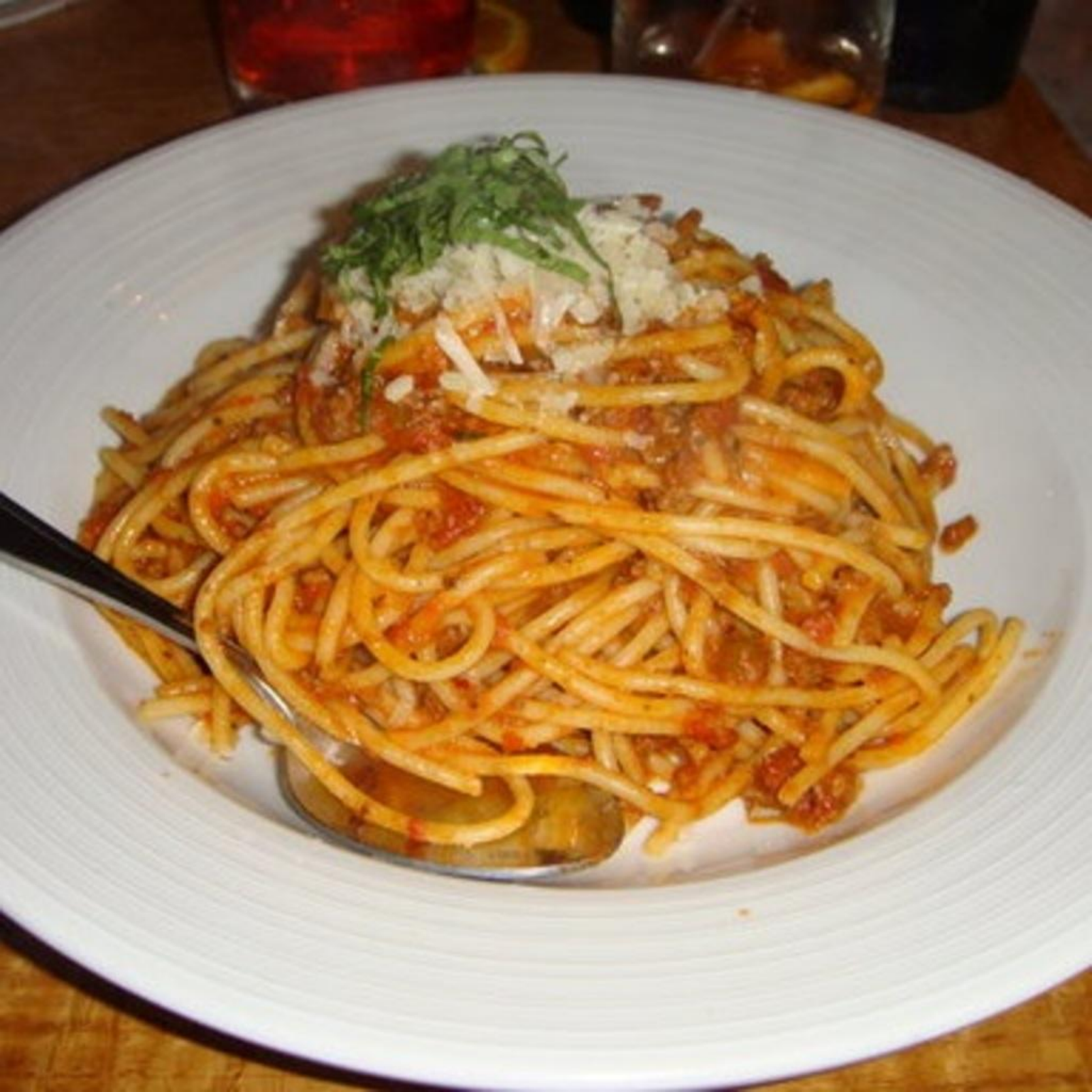} \\
        \includegraphics[width=\textwidth]{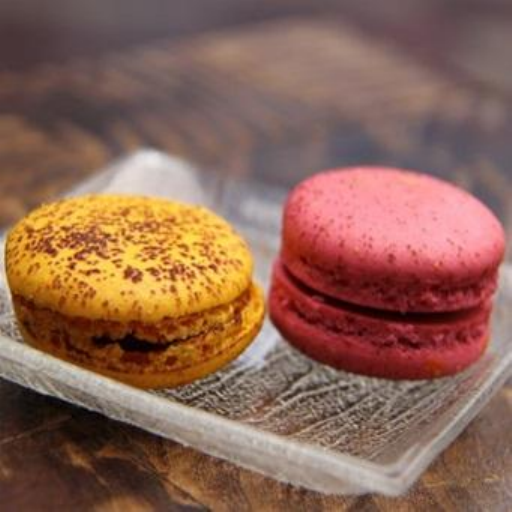} \\
        \includegraphics[width=\textwidth]{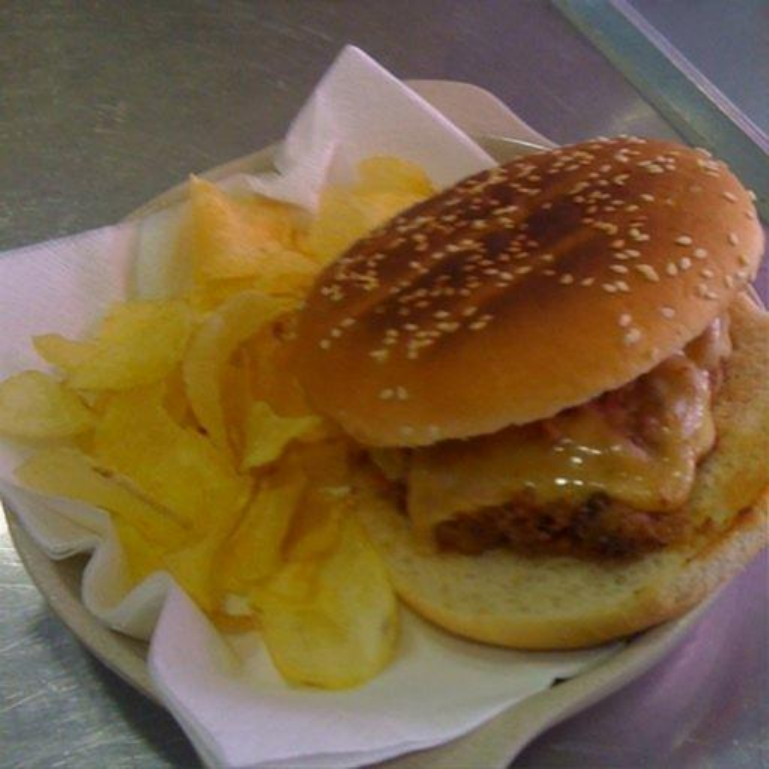} \\
        \includegraphics[width=\textwidth]{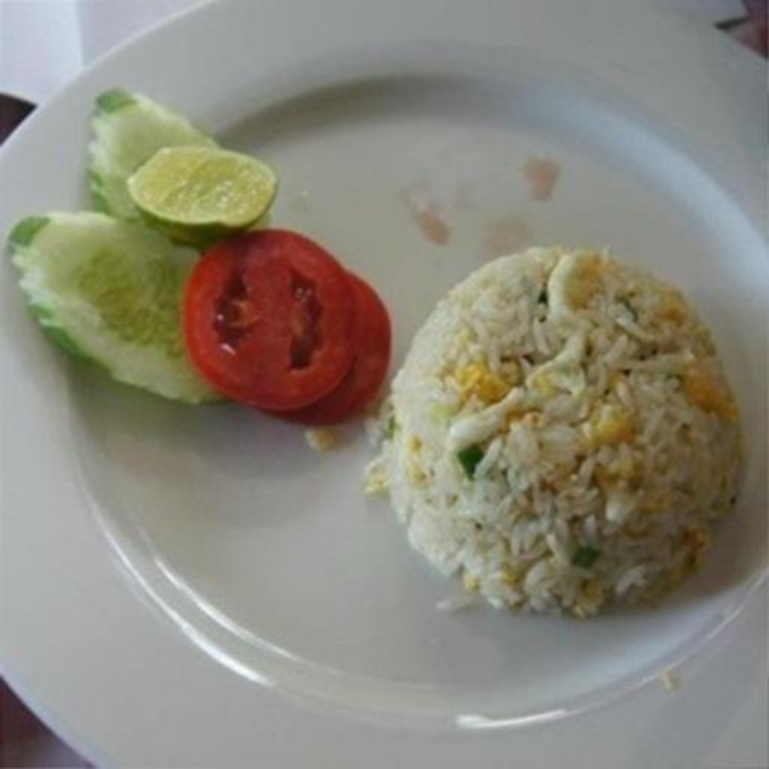} \\
        \includegraphics[width=\textwidth]{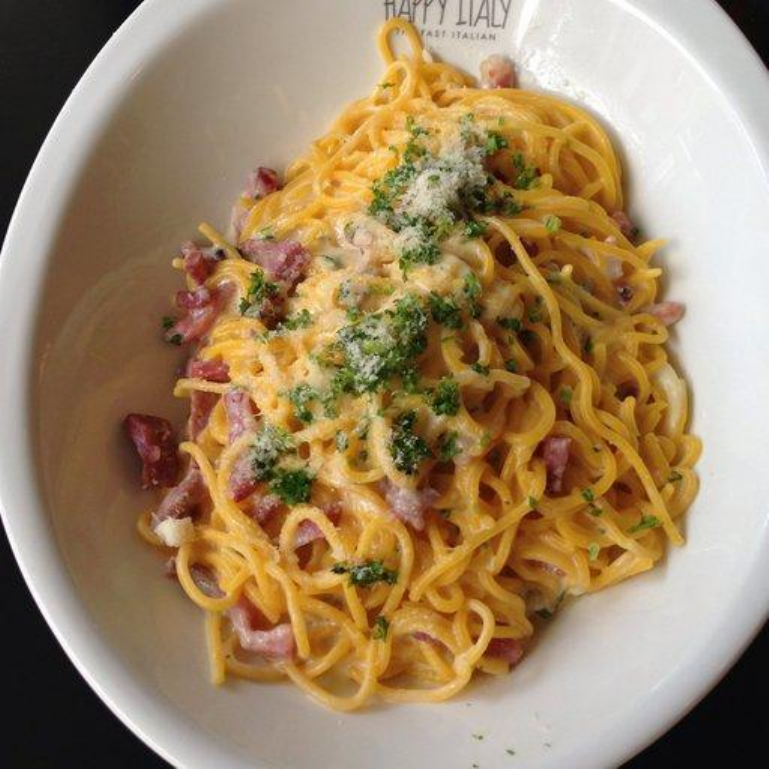} \\
        \includegraphics[width=\textwidth]{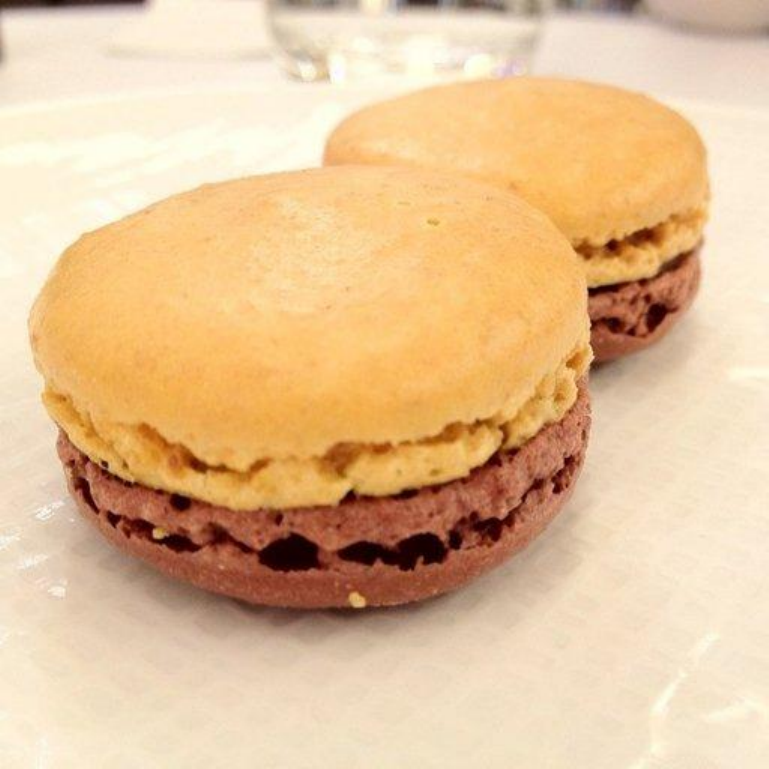} \\
        \includegraphics[width=\textwidth]{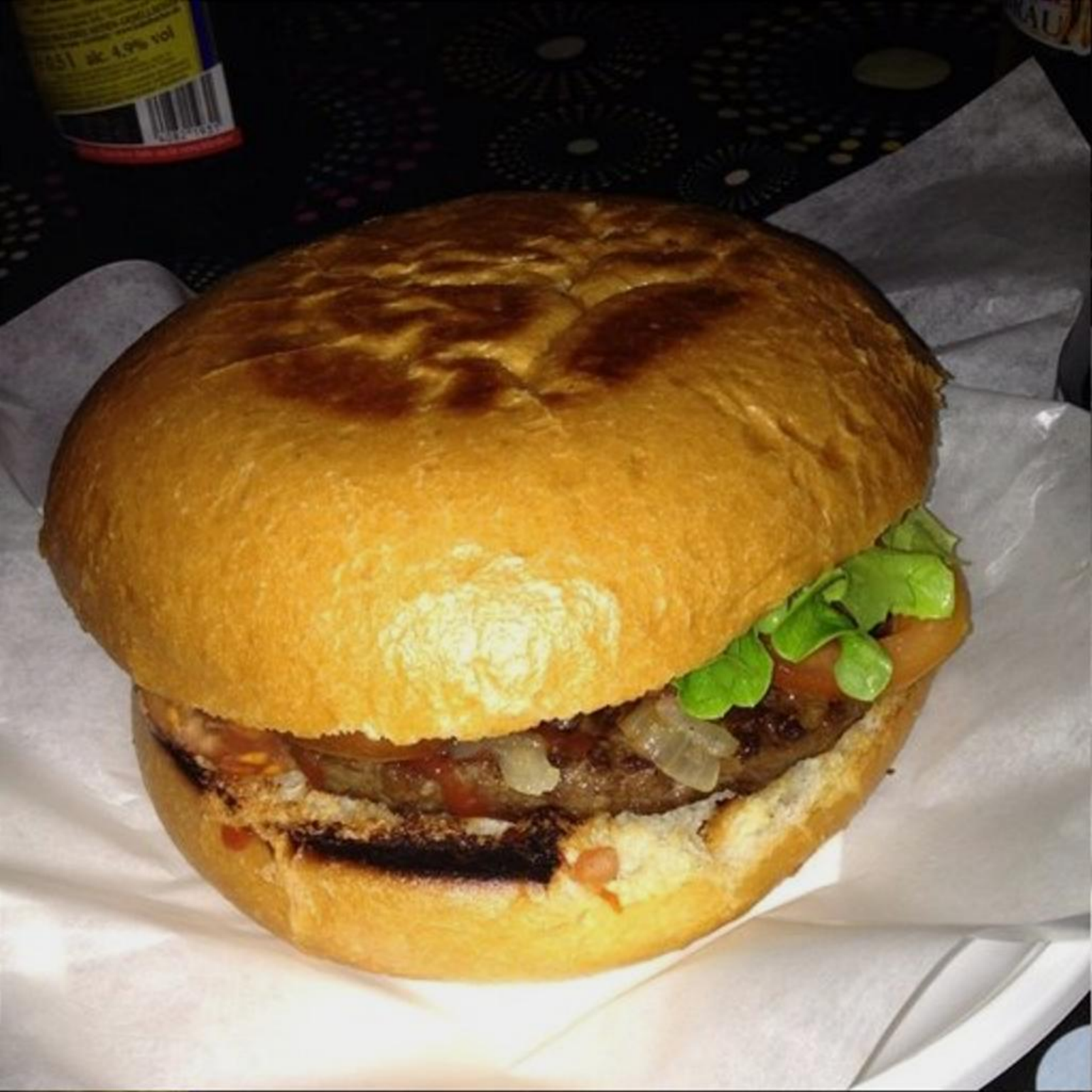} \\
        \includegraphics[width=\textwidth]{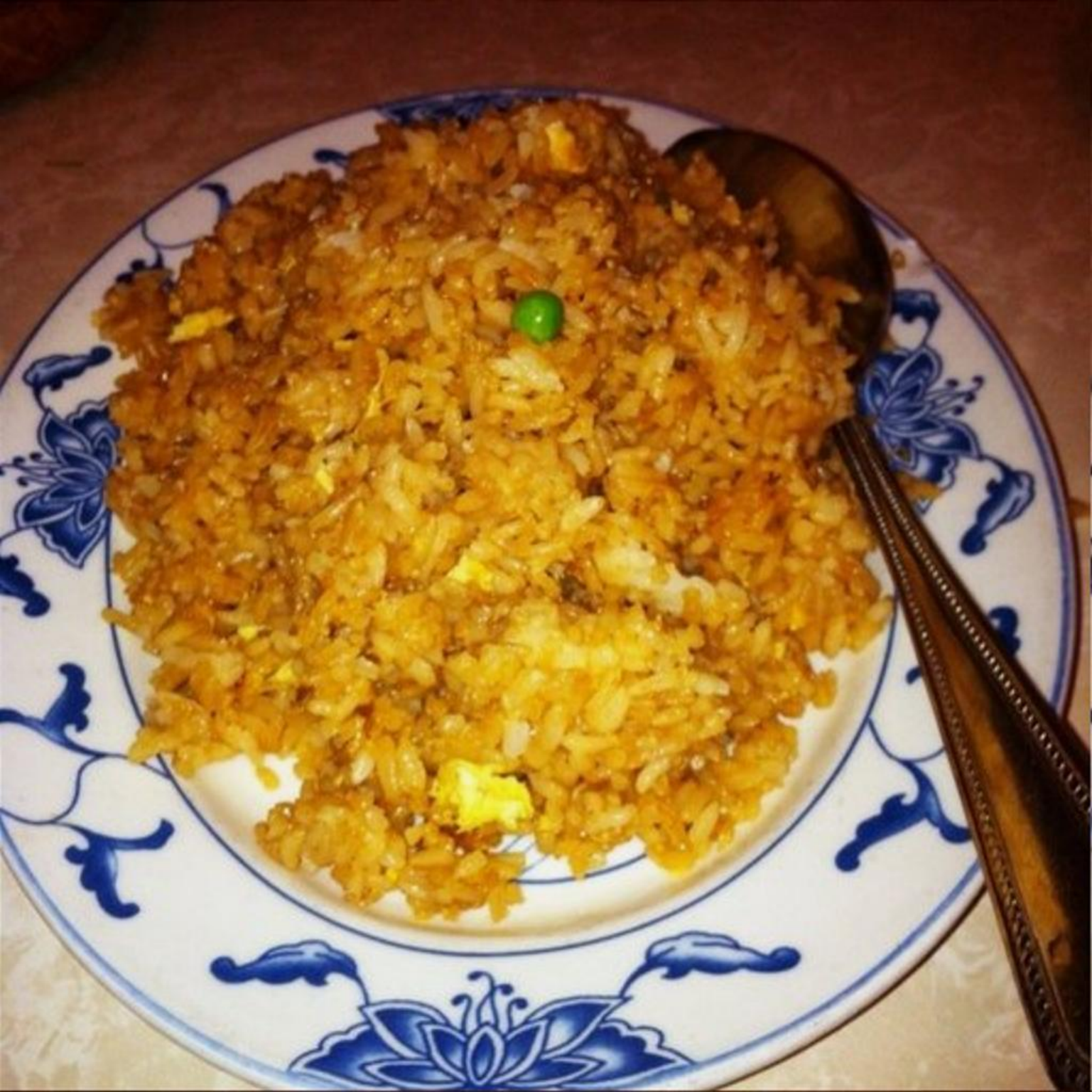}
    \end{minipage}
    \hfill
    \begin{minipage}{0.15\textwidth}
        \includegraphics[width=\textwidth]{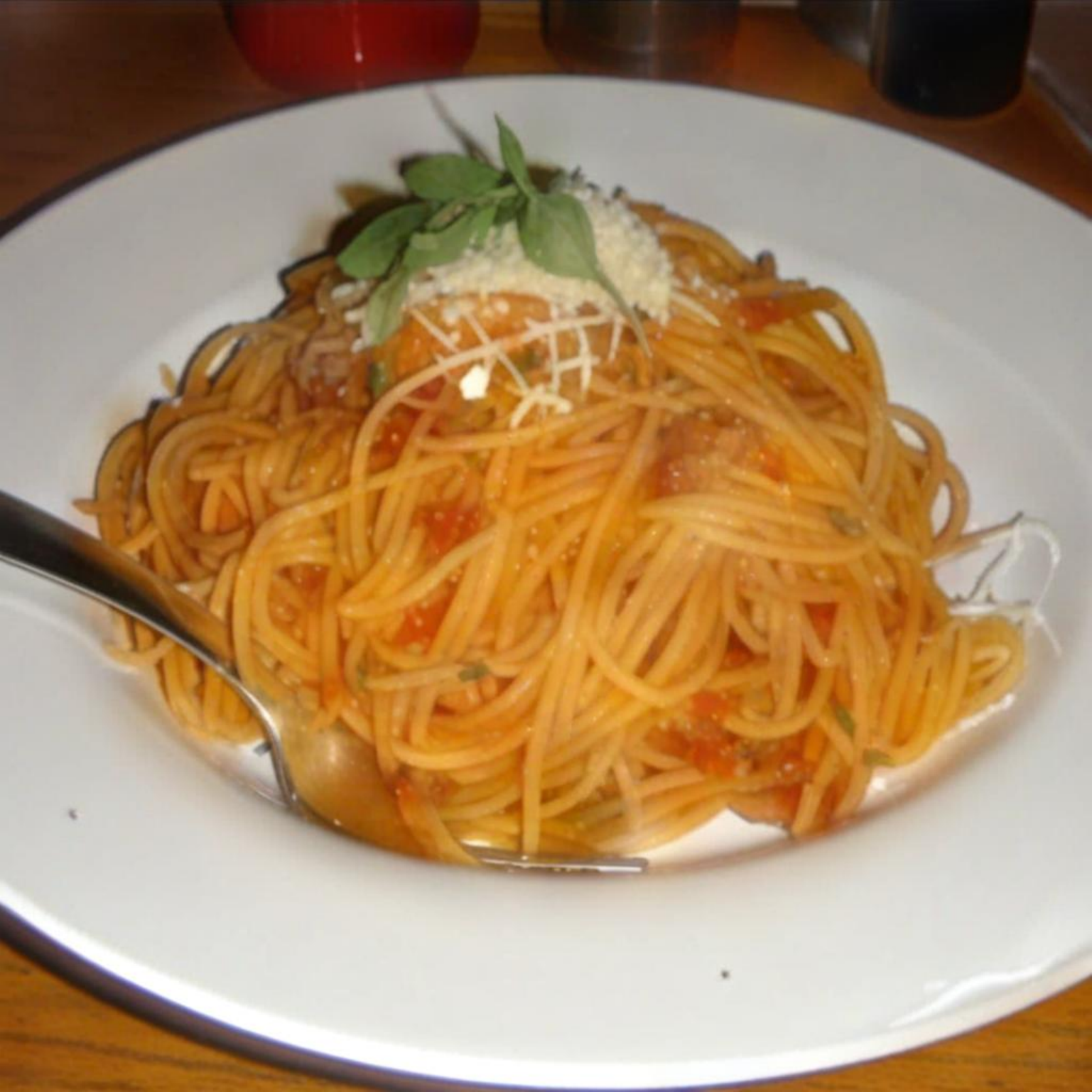} \\
        \includegraphics[width=\textwidth]{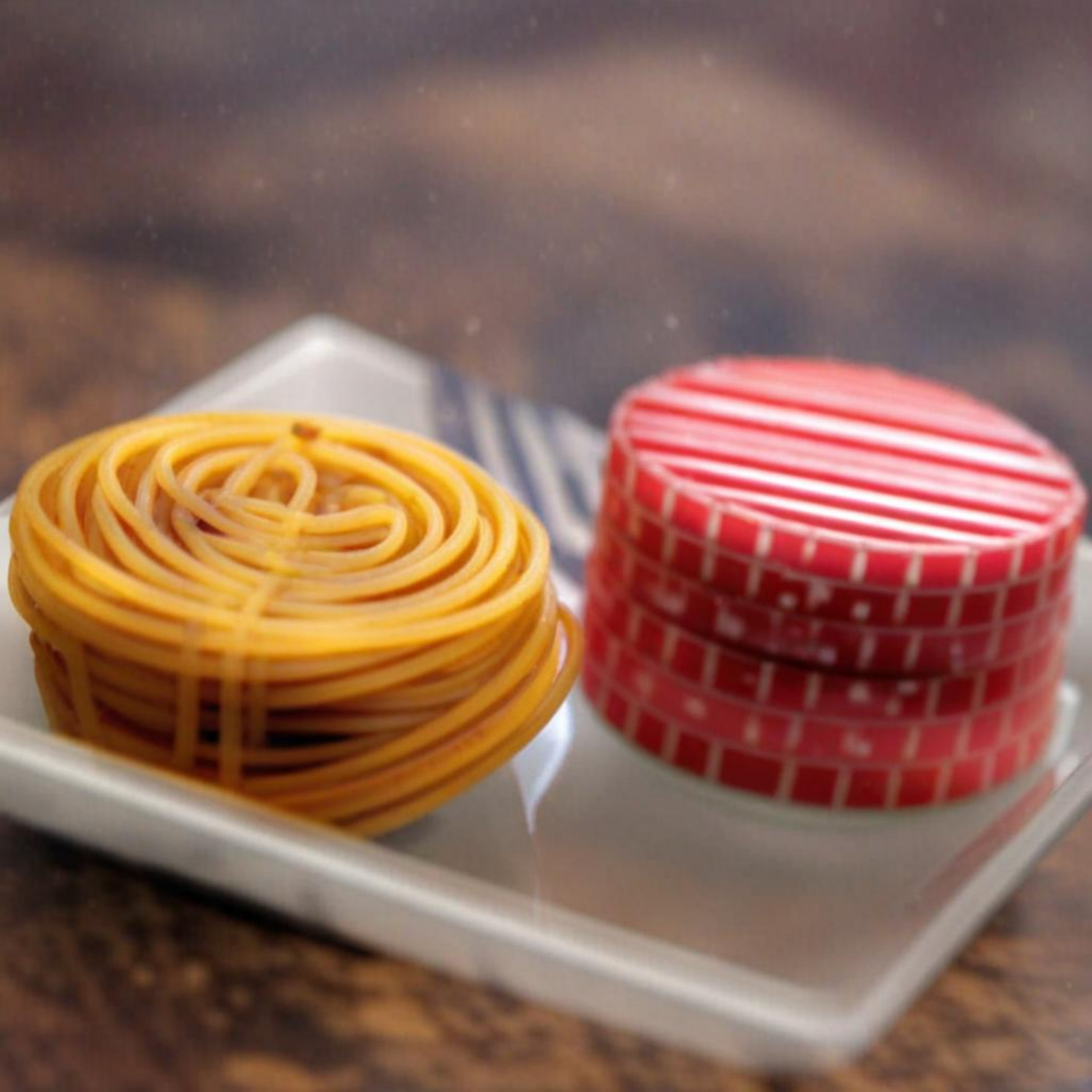} \\
        \includegraphics[width=\textwidth]{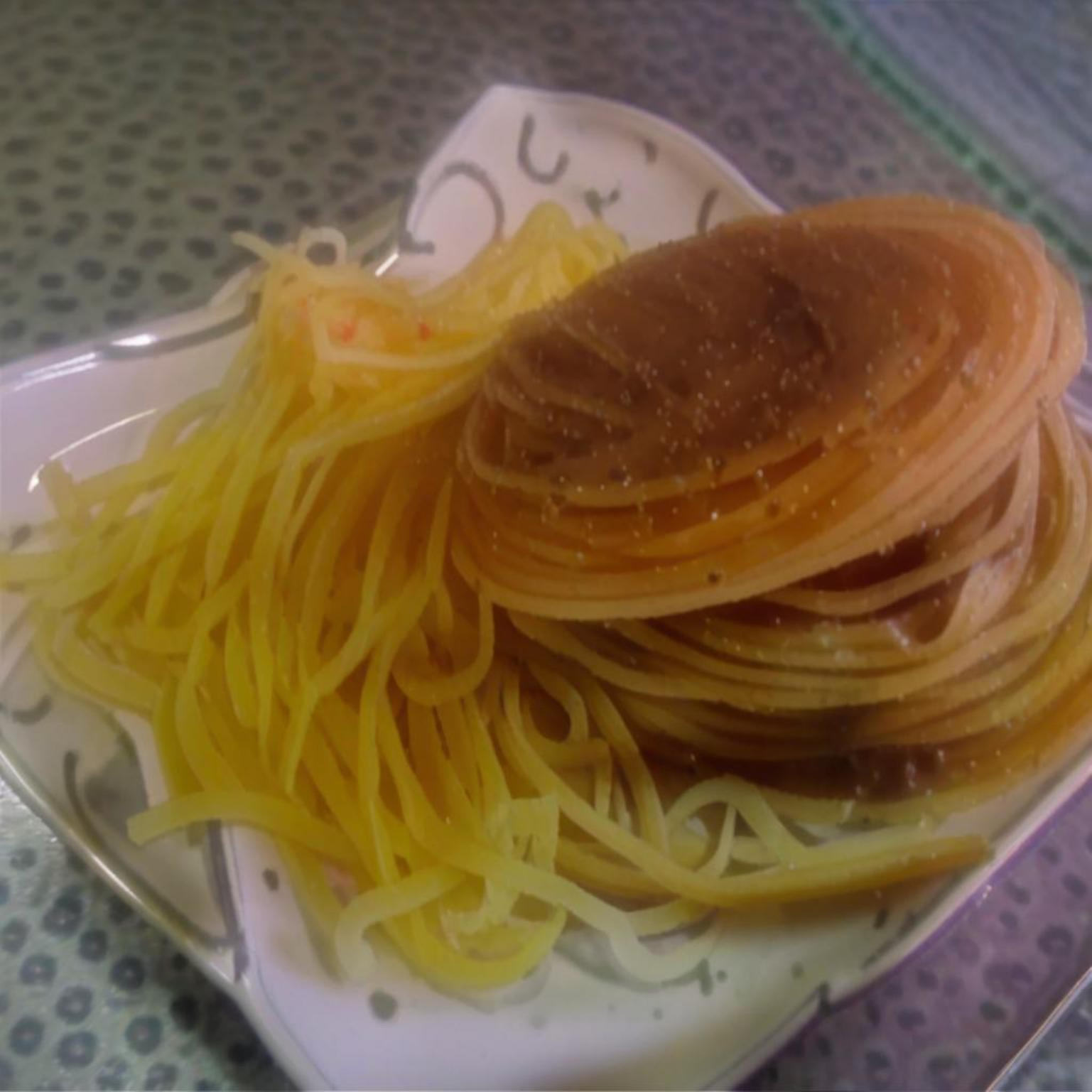} \\
        \includegraphics[width=\textwidth]{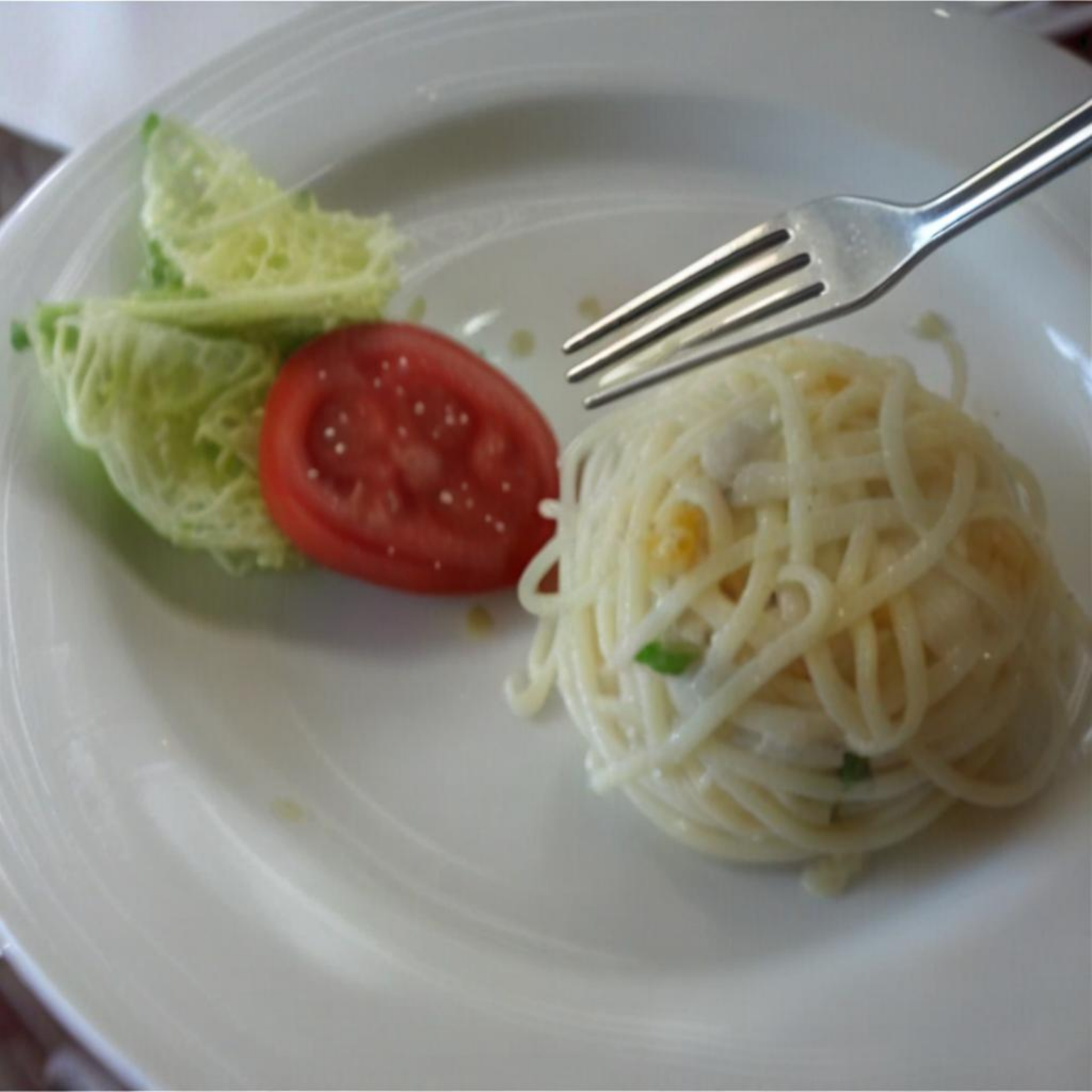} \\
        \includegraphics[width=\textwidth]{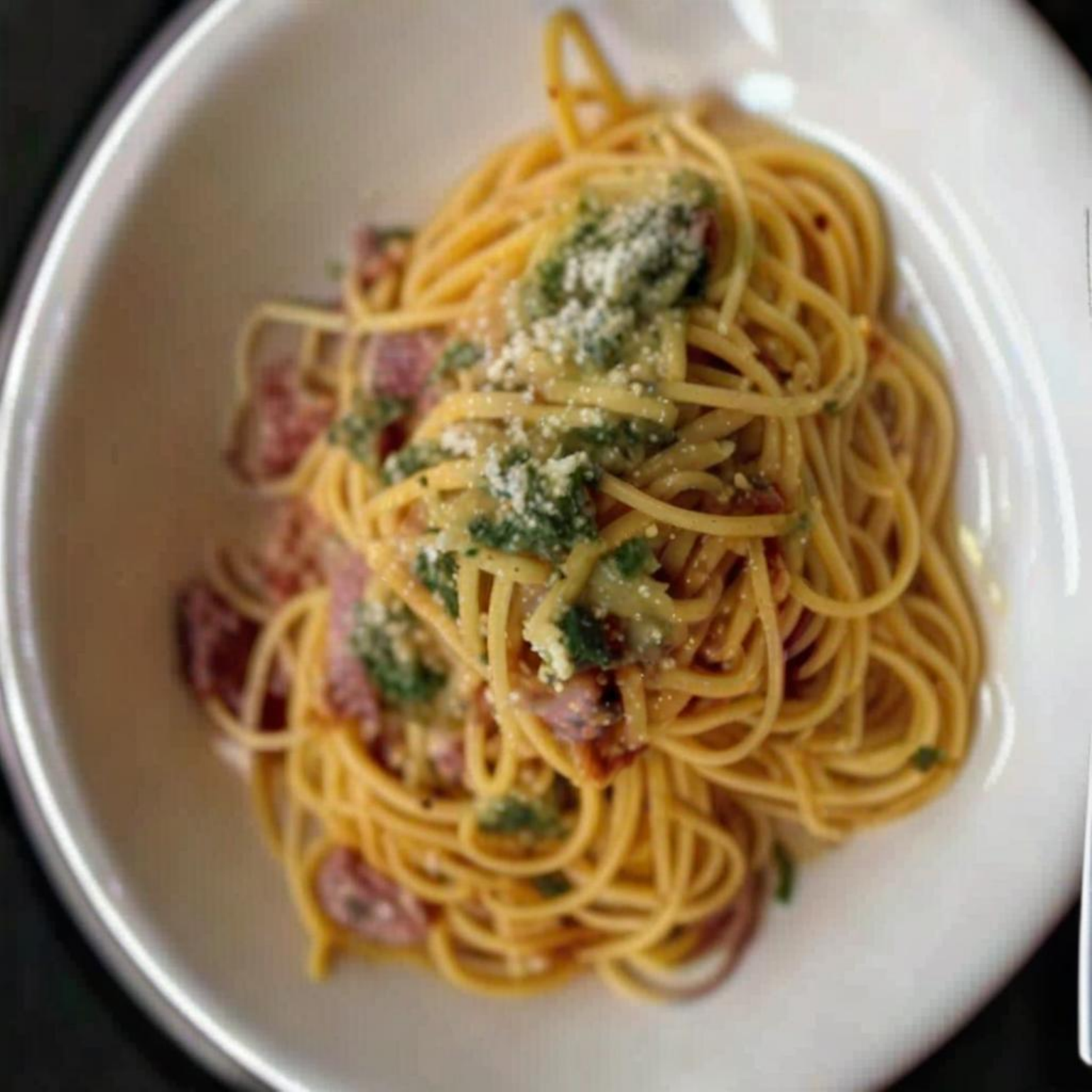}\\
        \includegraphics[width=\textwidth]{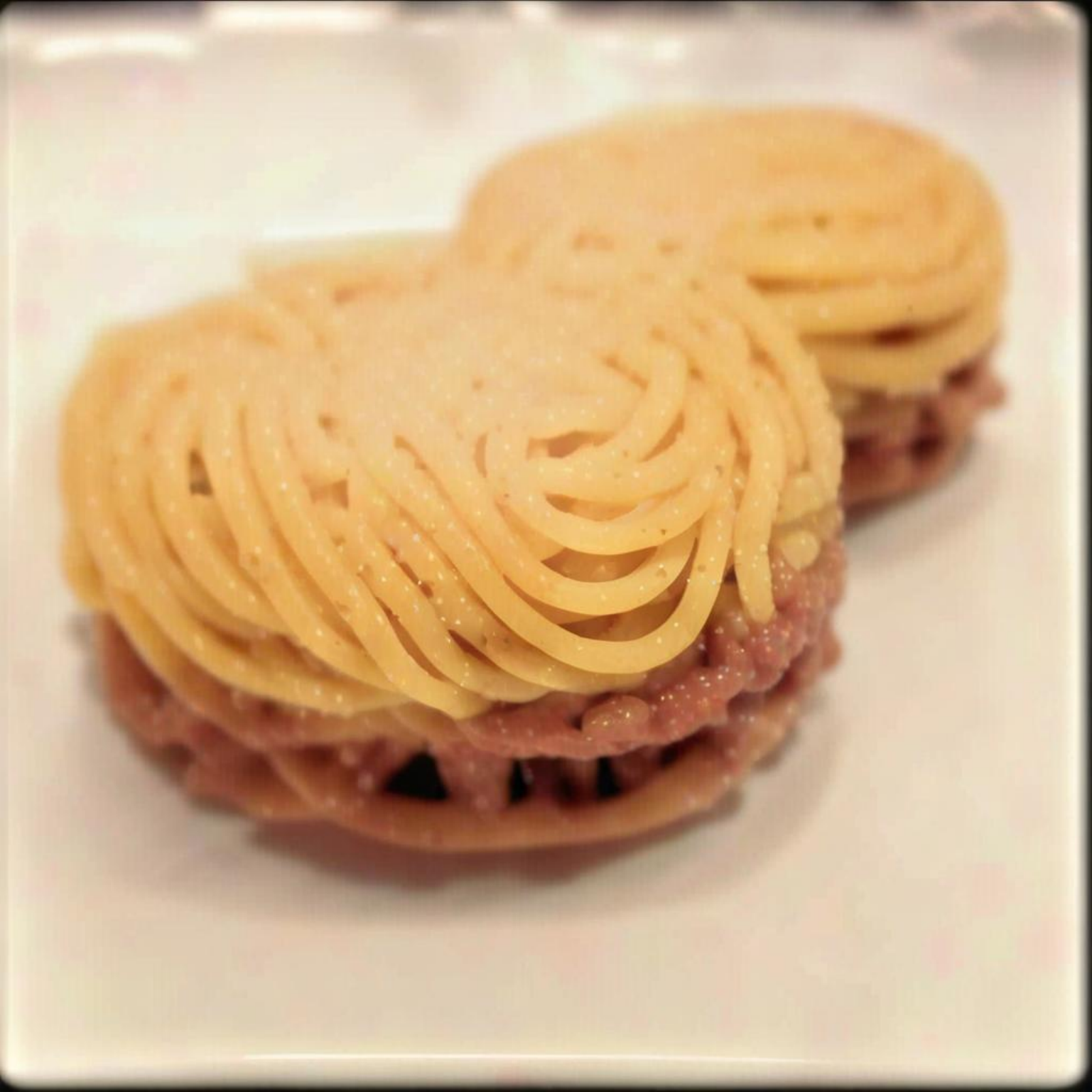}\\
        \includegraphics[width=\textwidth]{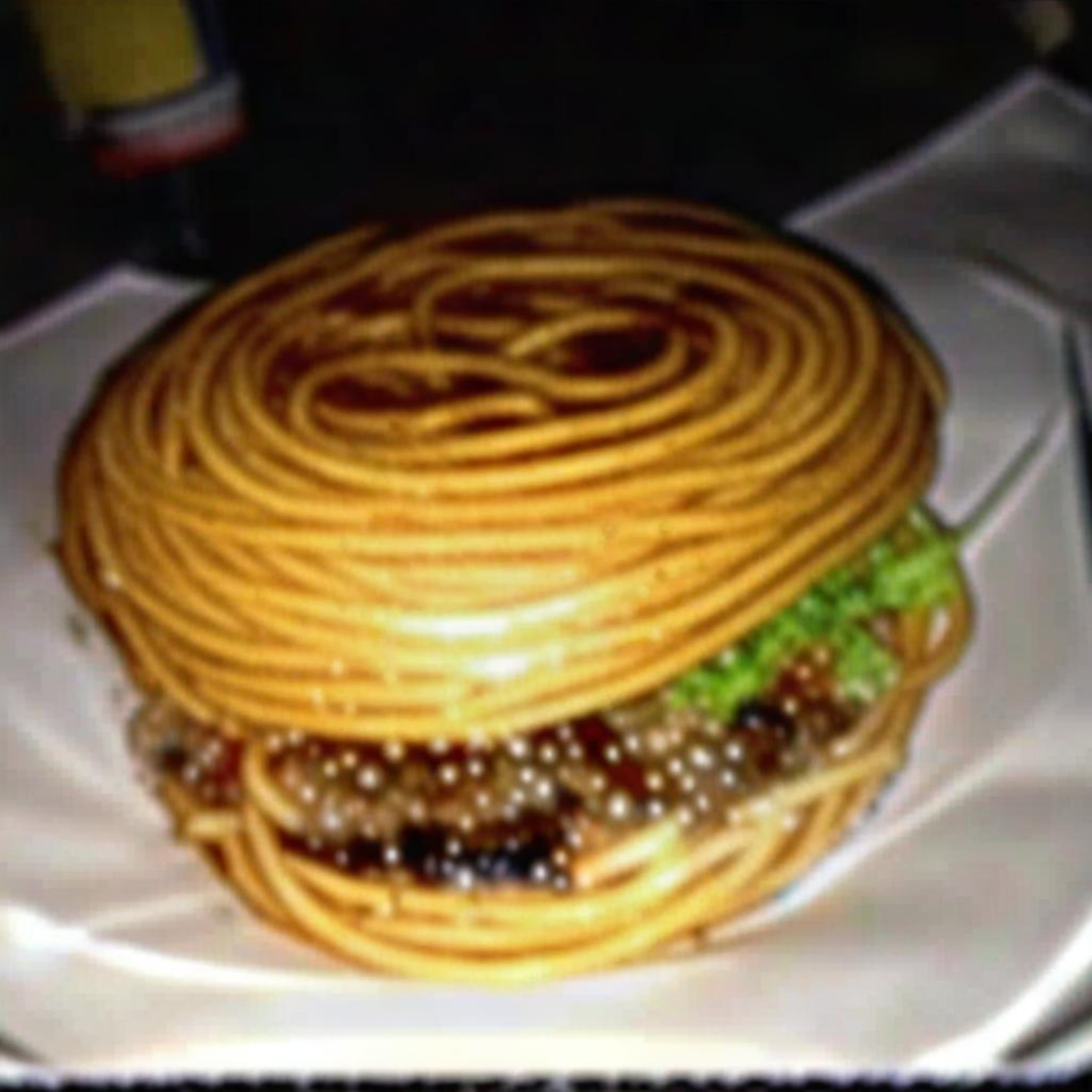}\\
        \includegraphics[width=\textwidth]{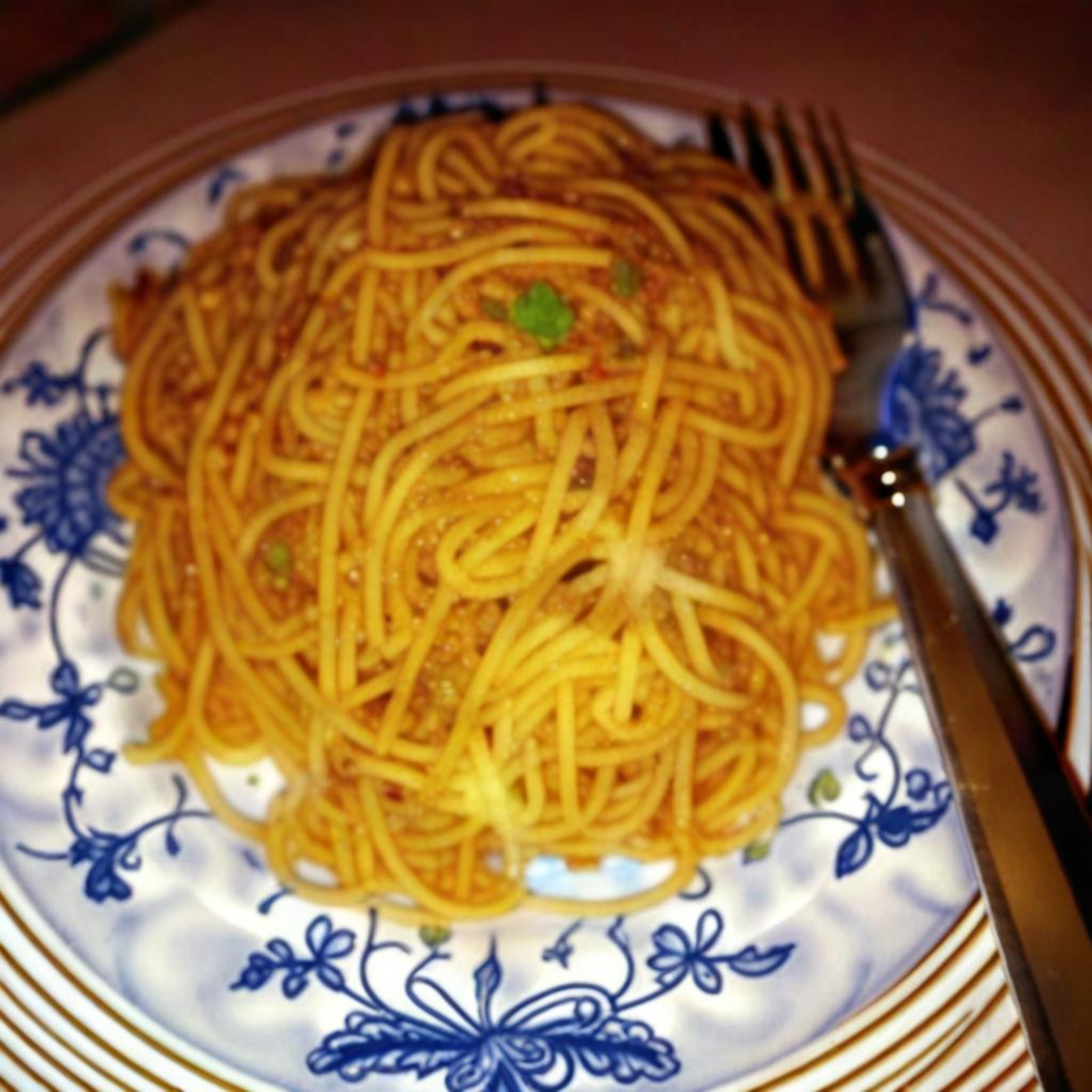}
    \end{minipage}
    \hfill
    \begin{minipage}{0.15\textwidth}
        \includegraphics[width=\textwidth]{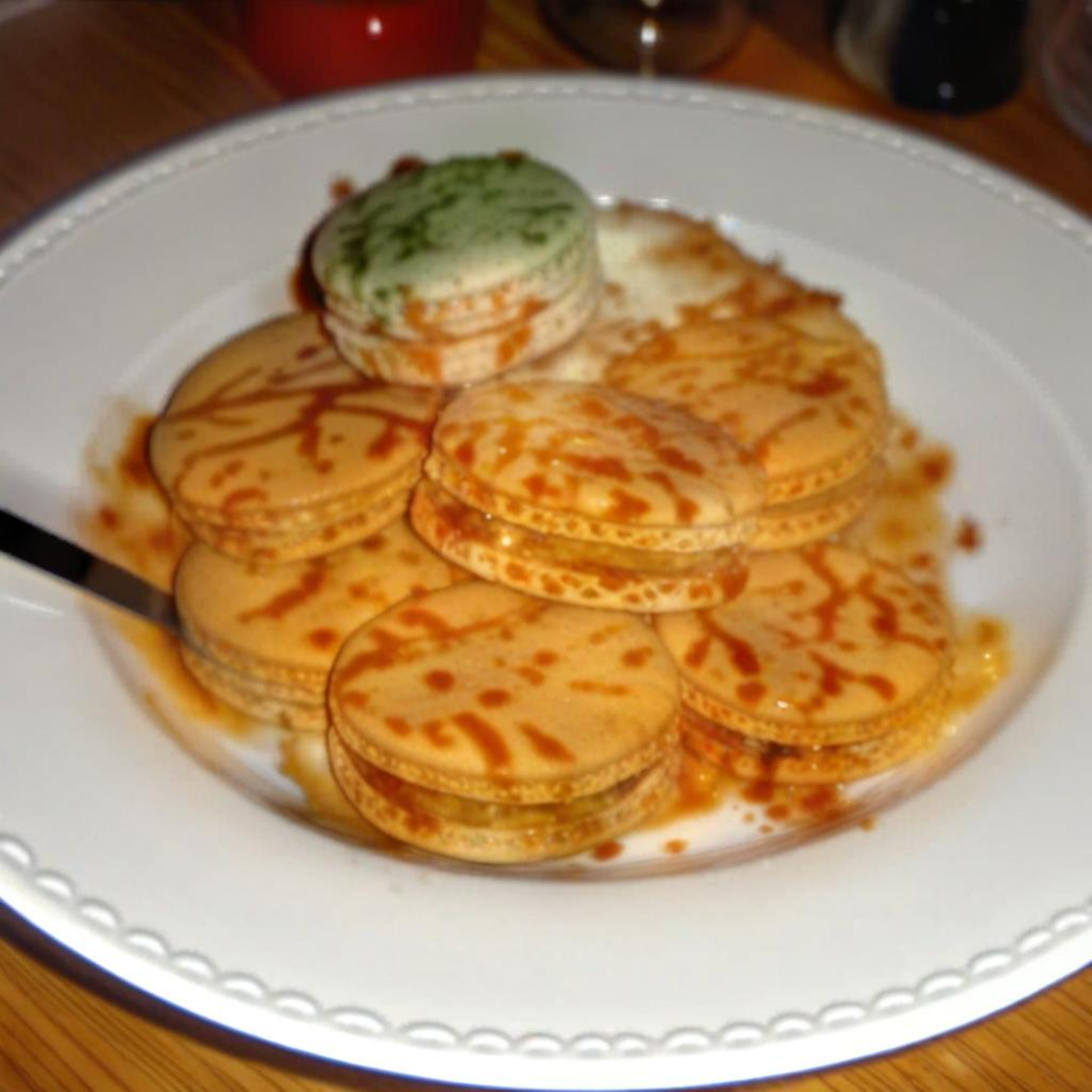}\\
        \includegraphics[width=\textwidth]{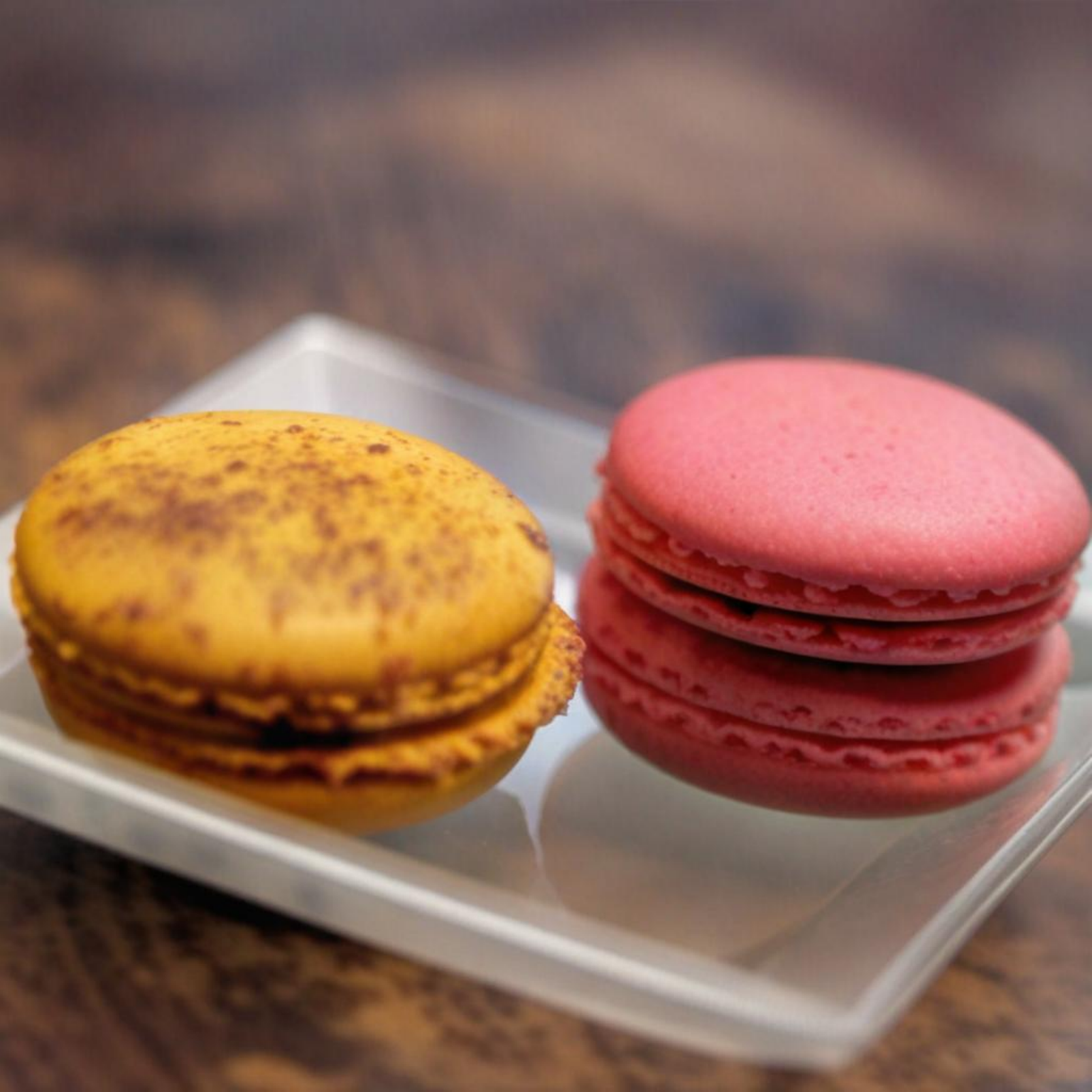}\\
        \includegraphics[width=\textwidth]{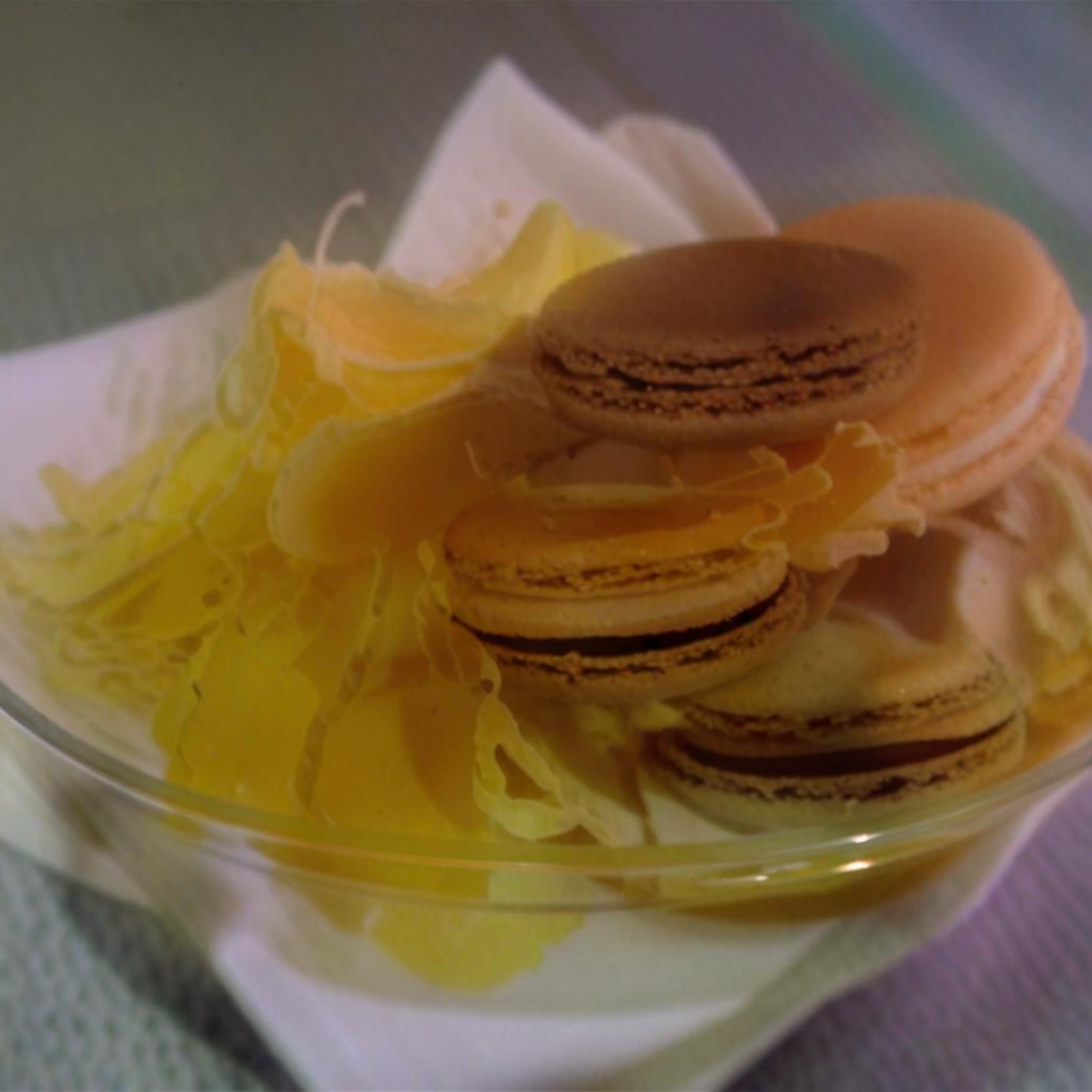}\\
        \includegraphics[width=\textwidth]{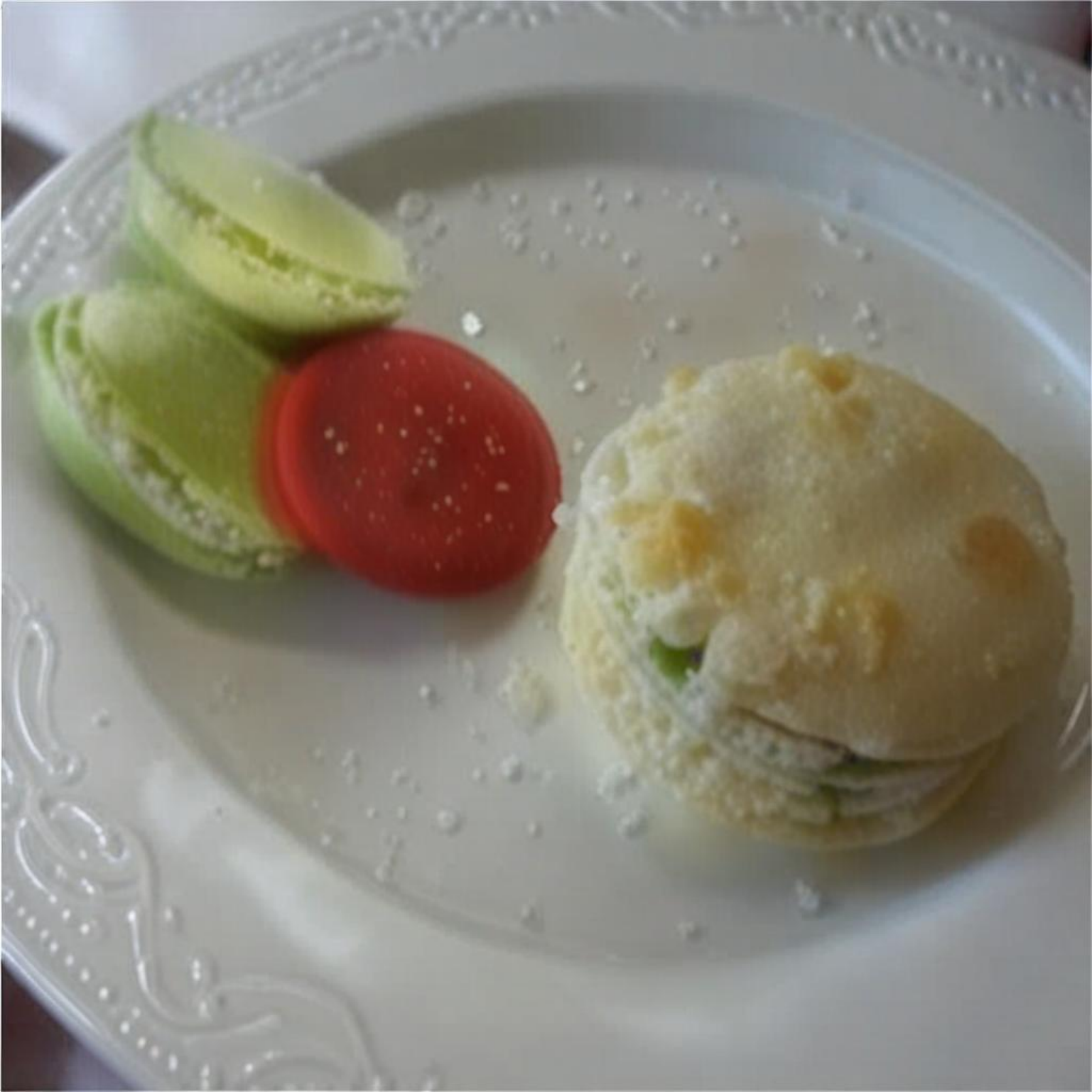}\\
        \includegraphics[width=\textwidth]{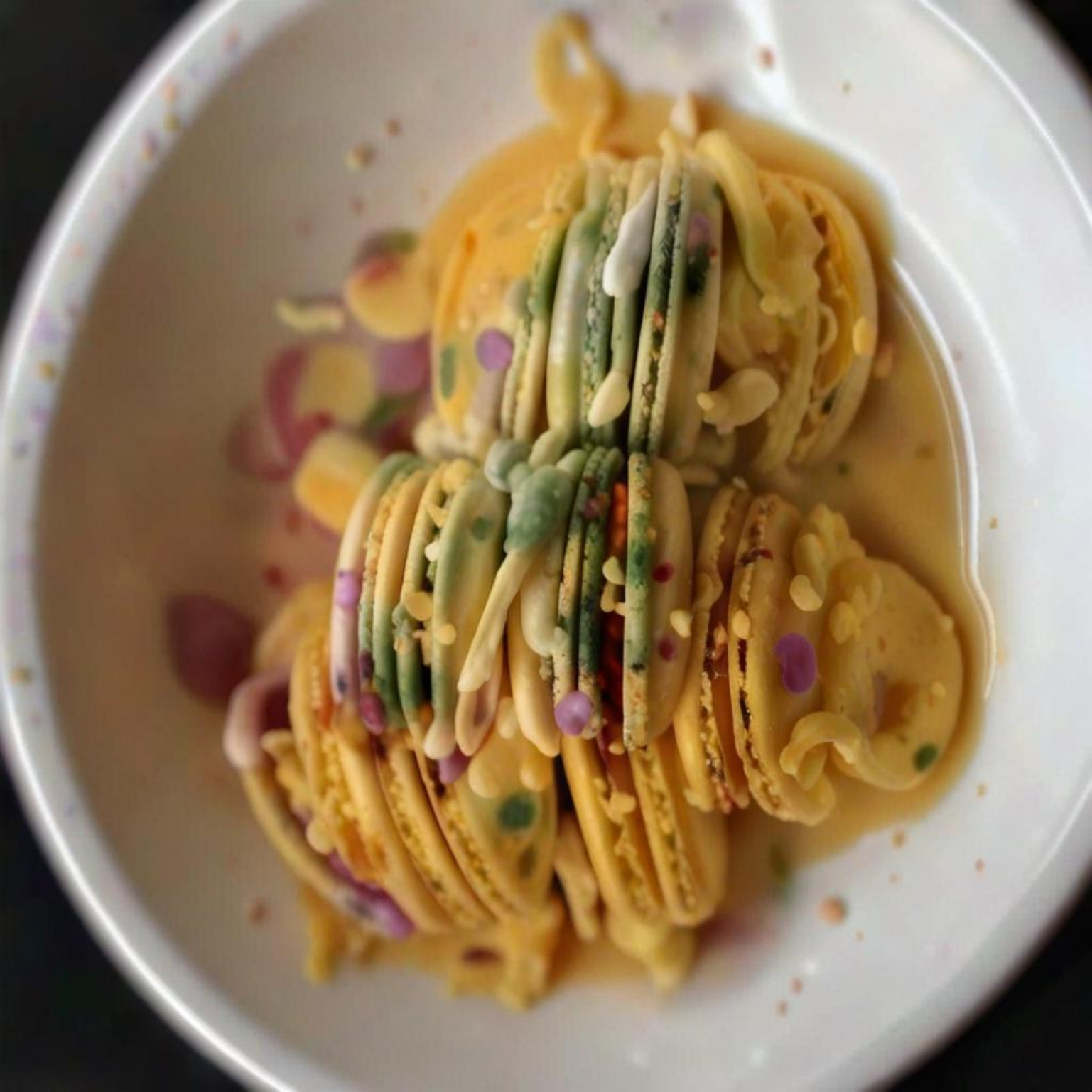}\\
        \includegraphics[width=\textwidth]{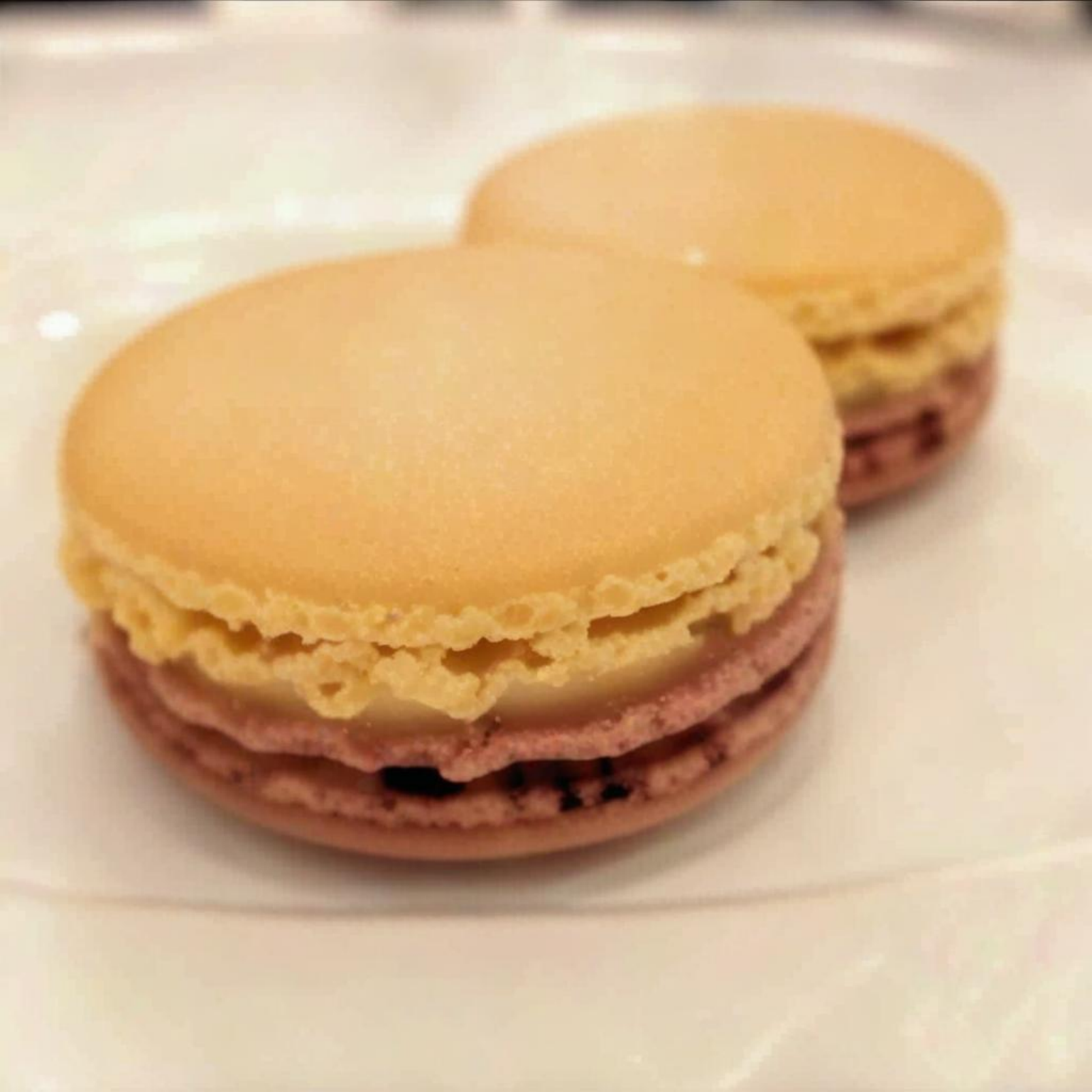}\\
        \includegraphics[width=\textwidth]{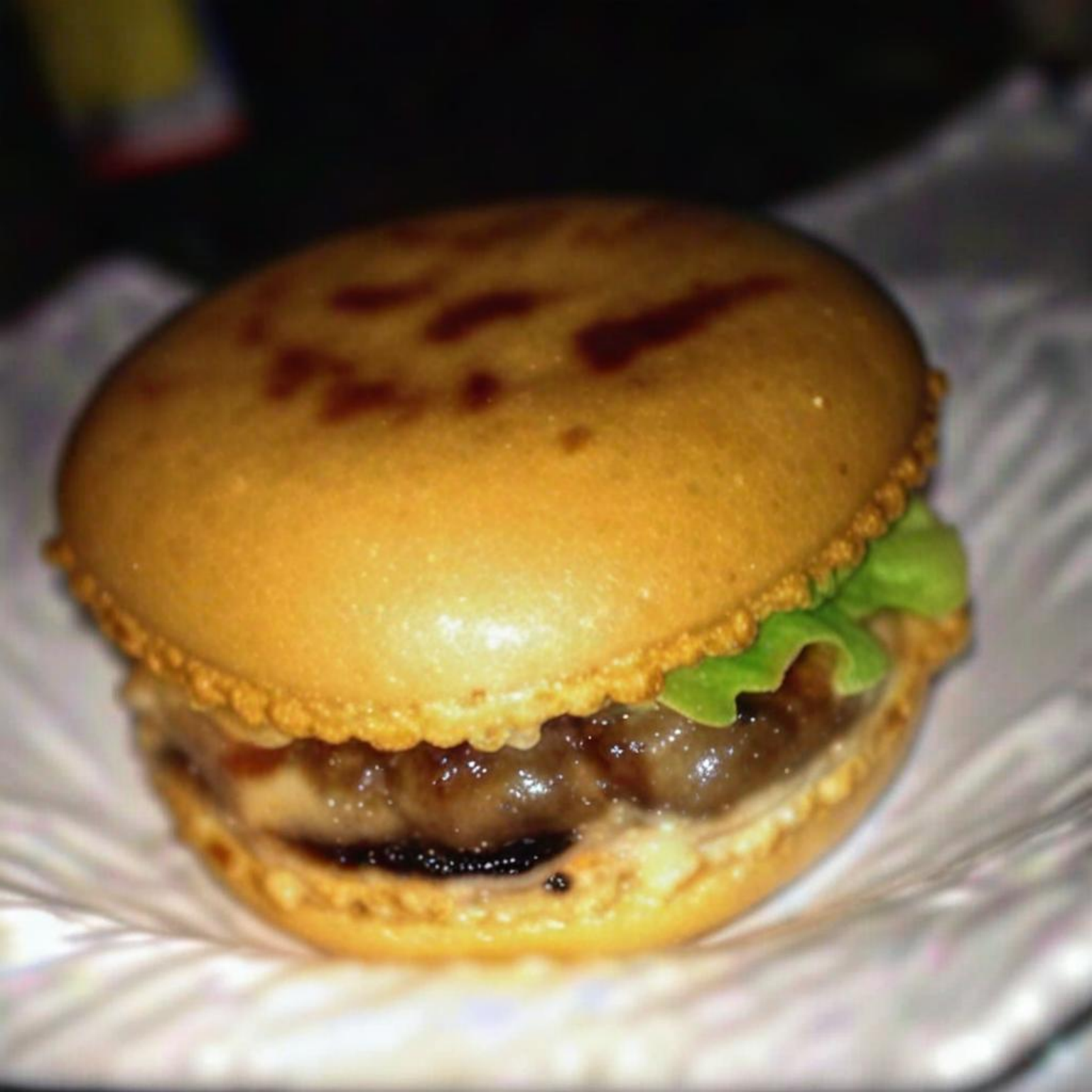}\\
        \includegraphics[width=\textwidth]{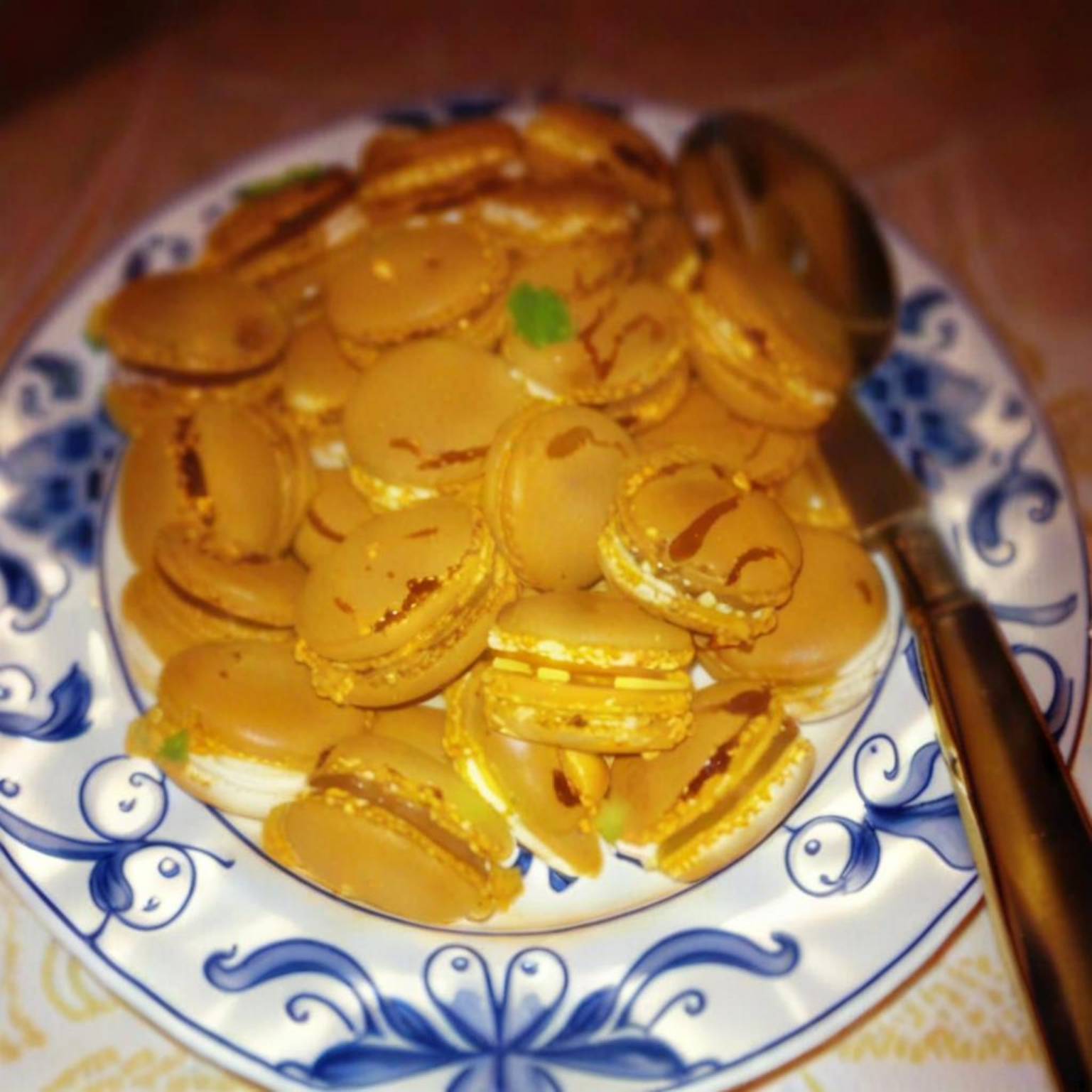}
    \end{minipage}
    \hfill
    \begin{minipage}{0.15\textwidth}
        \includegraphics[width=\textwidth]{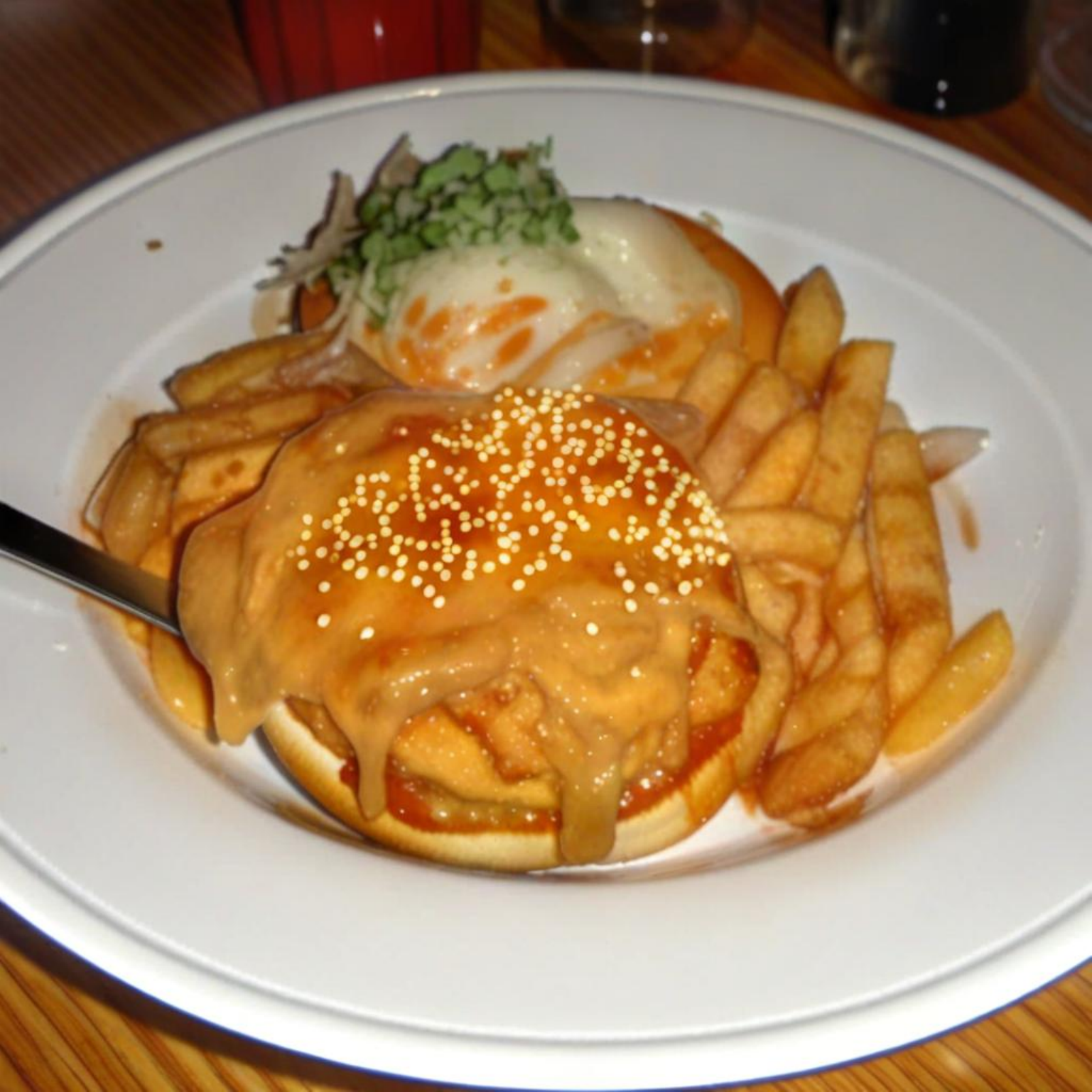}\\
        \includegraphics[width=\textwidth]{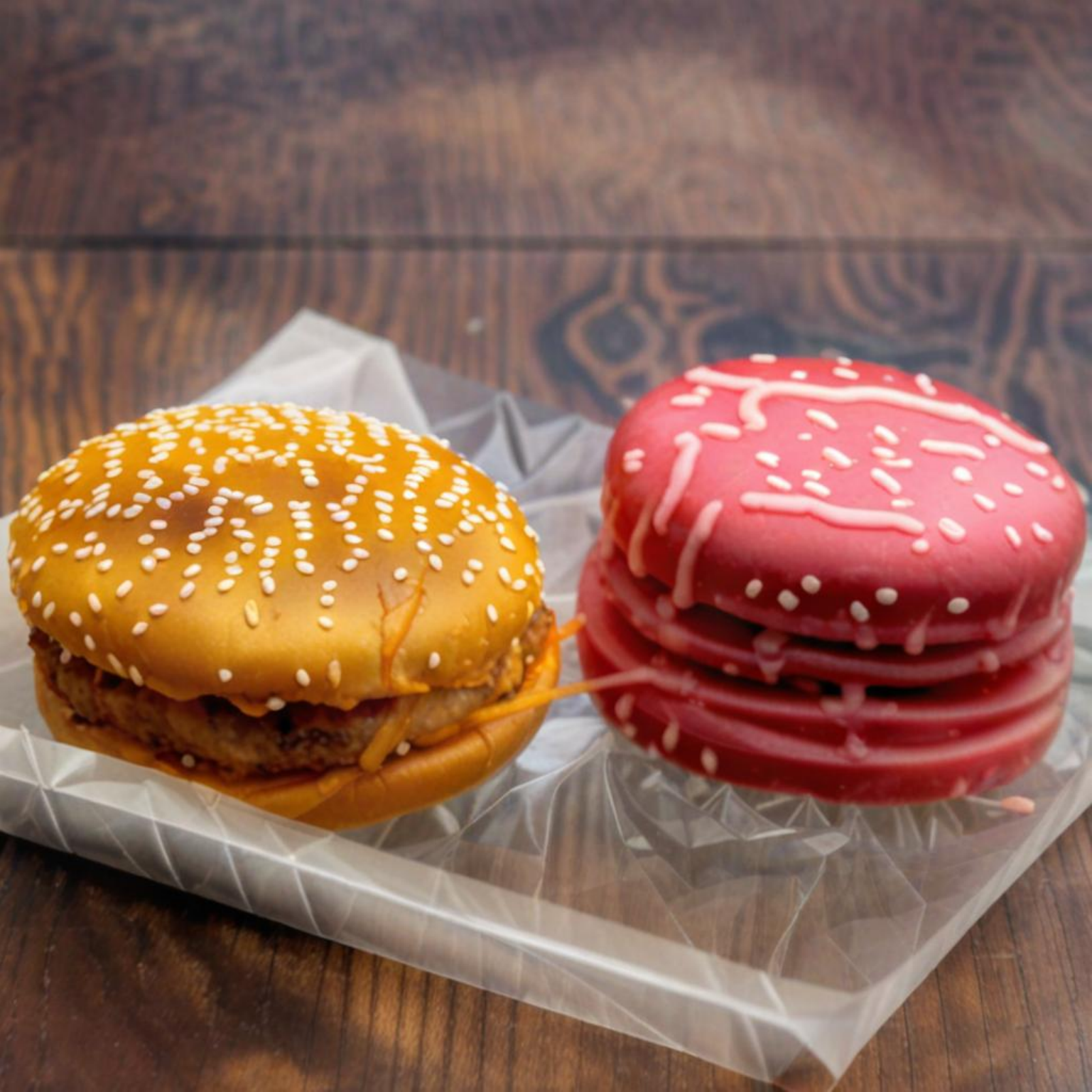}\\
        \includegraphics[width=\textwidth]{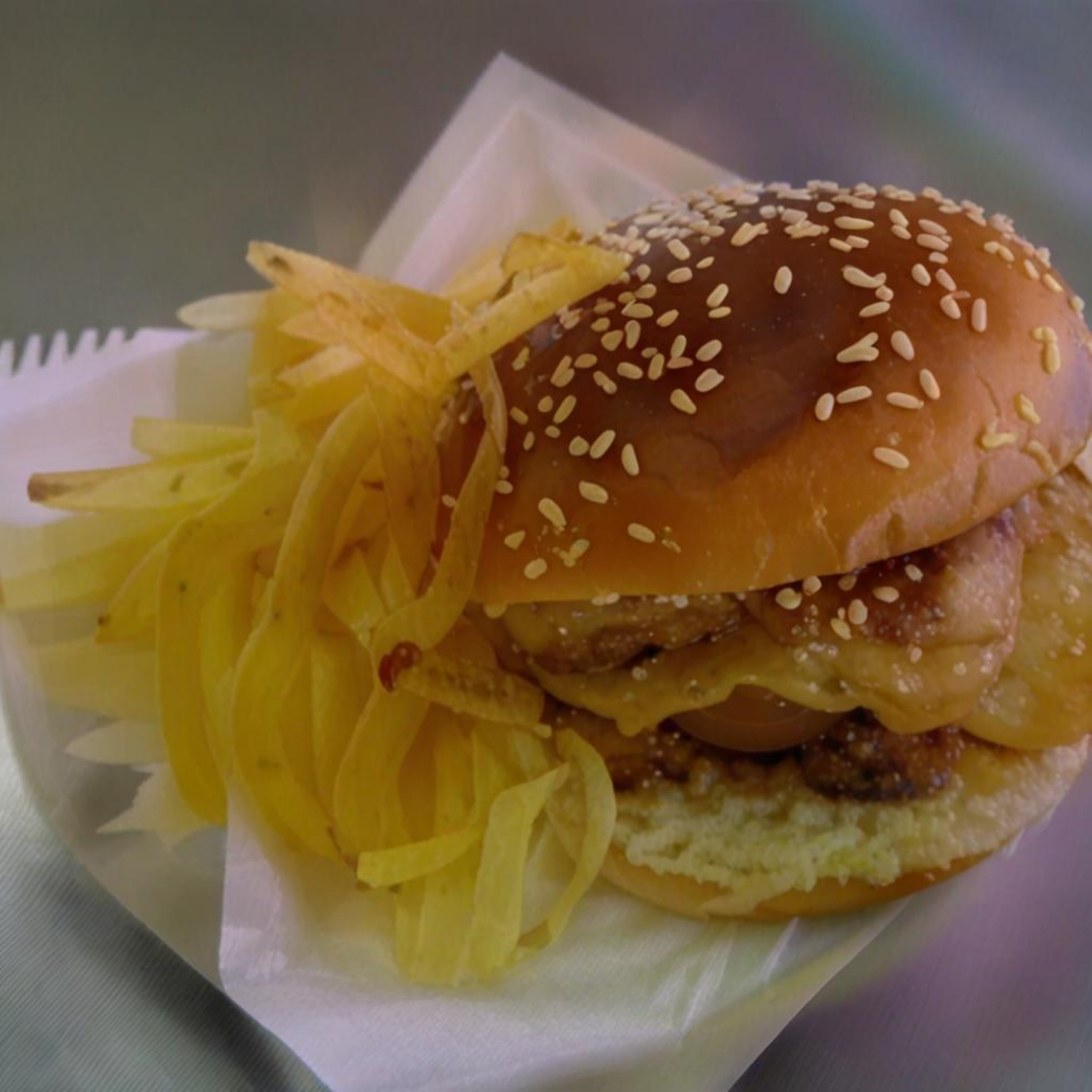}\\
        \includegraphics[width=\textwidth]{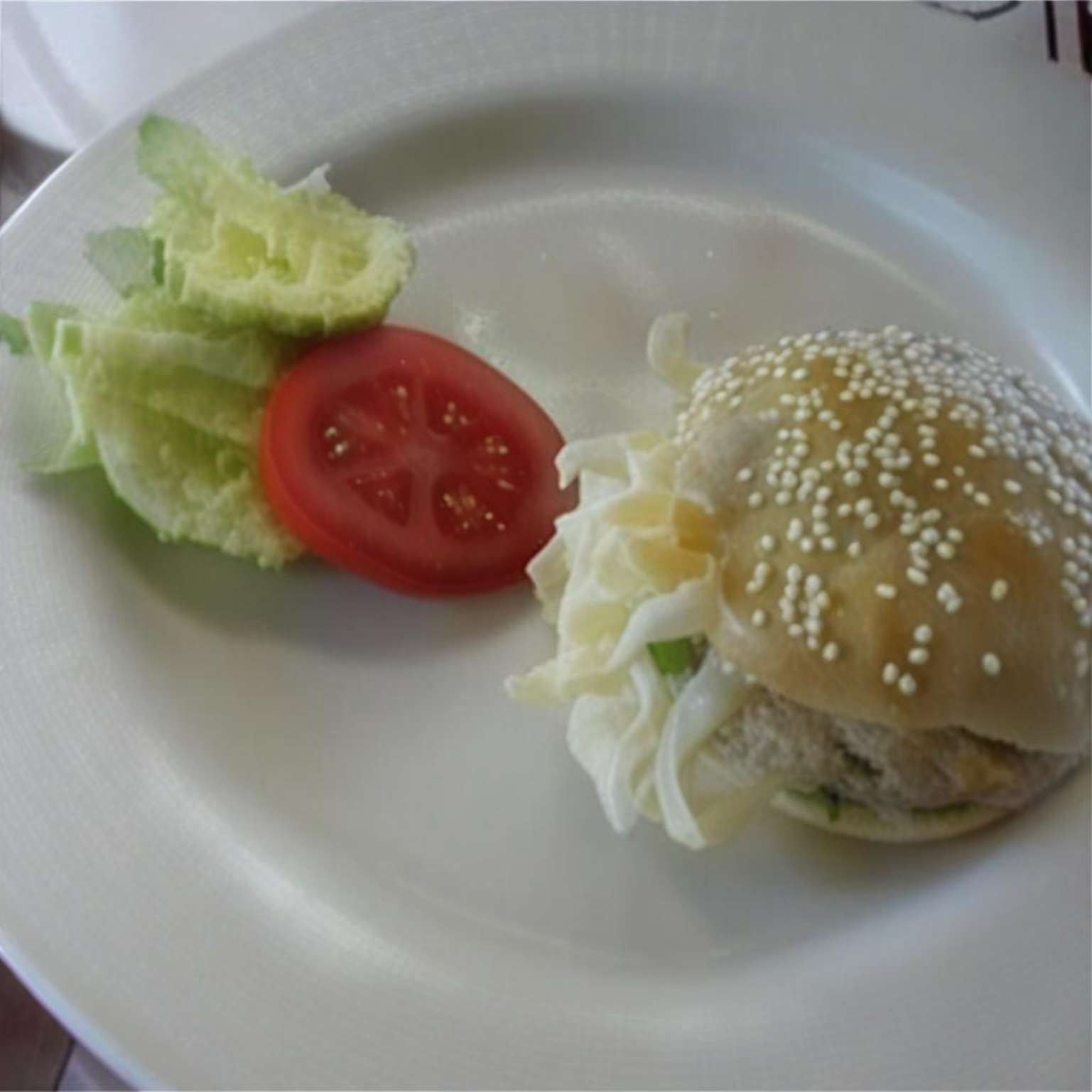}\\
        \includegraphics[width=\textwidth]{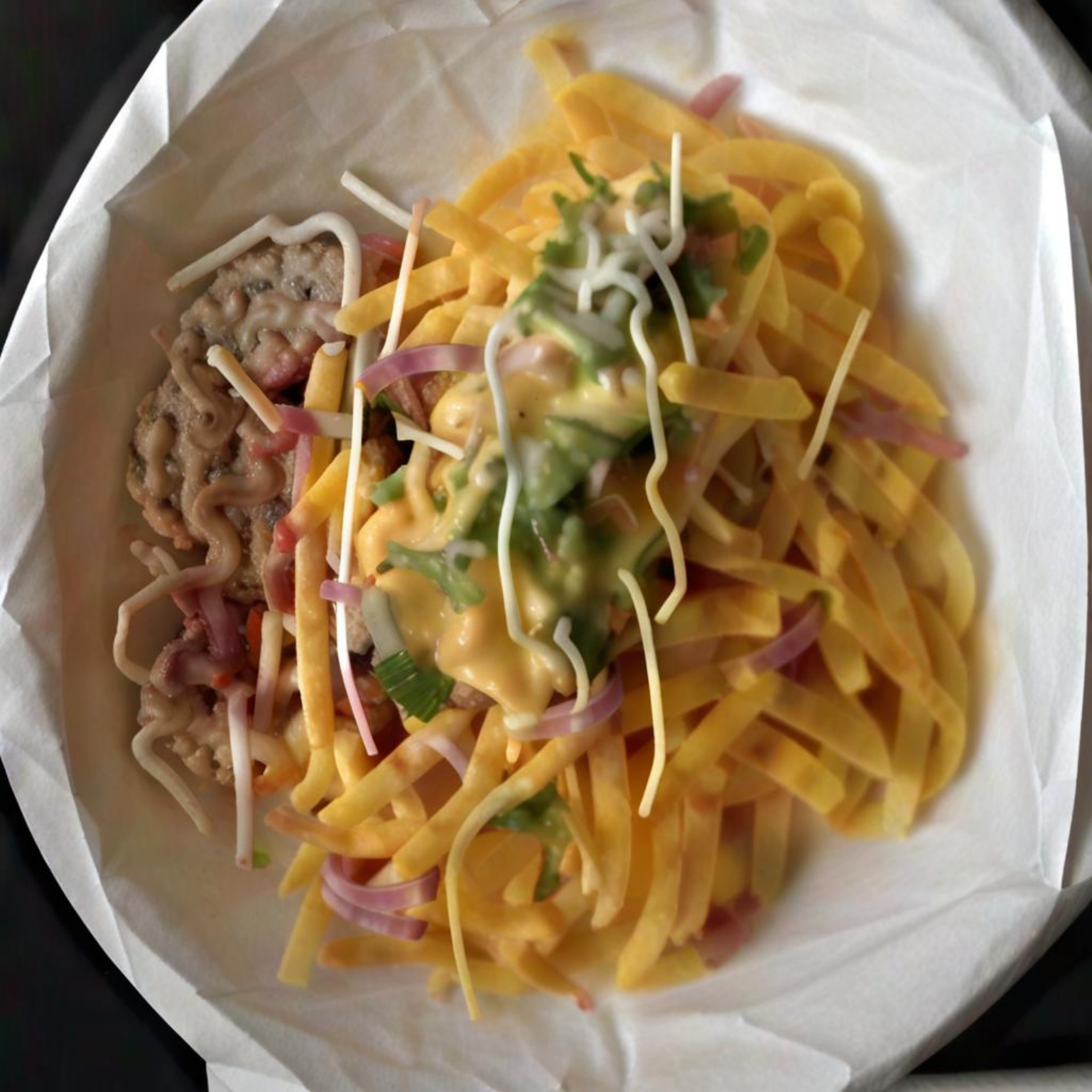}\\
        \includegraphics[width=\textwidth]{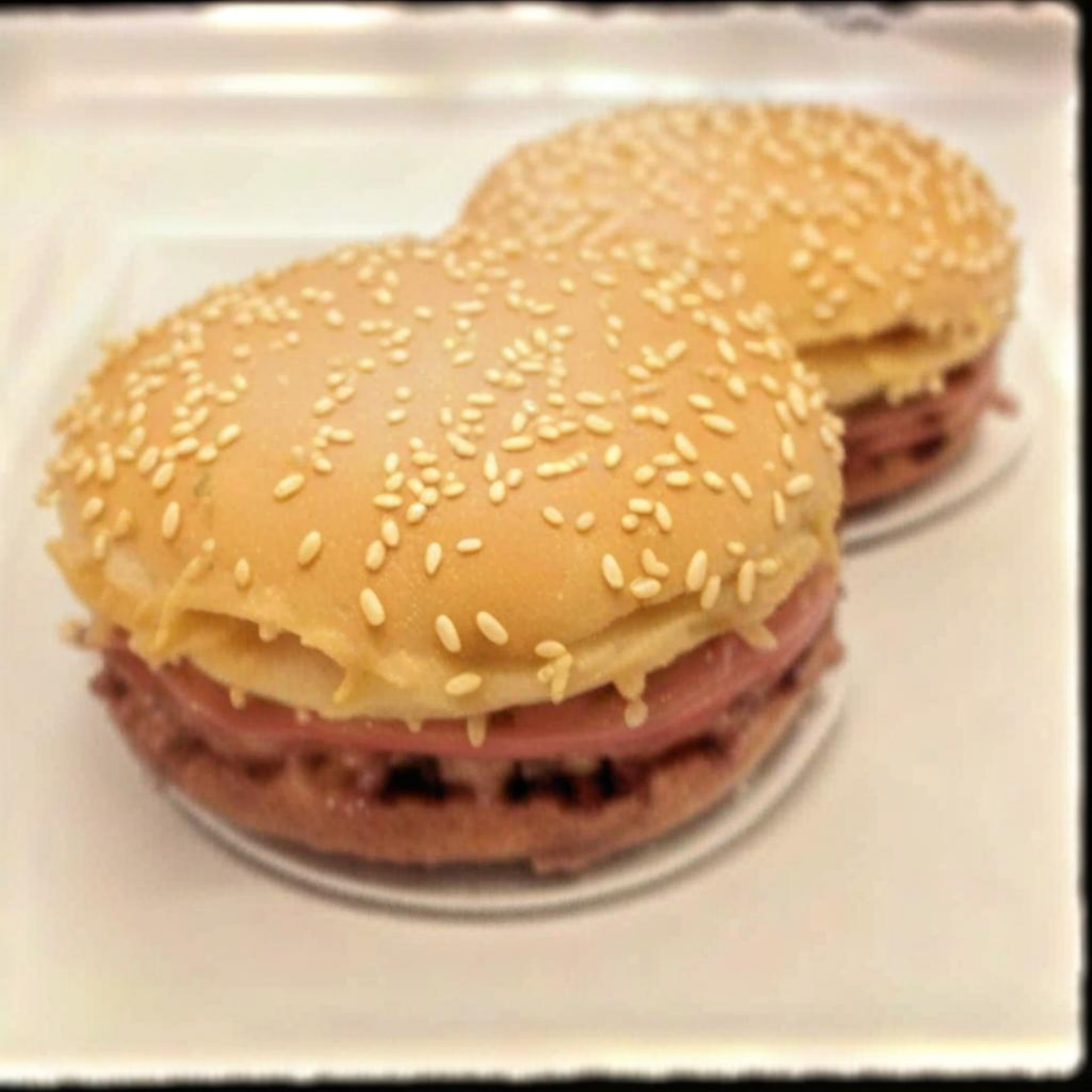}\\
        \includegraphics[width=\textwidth]{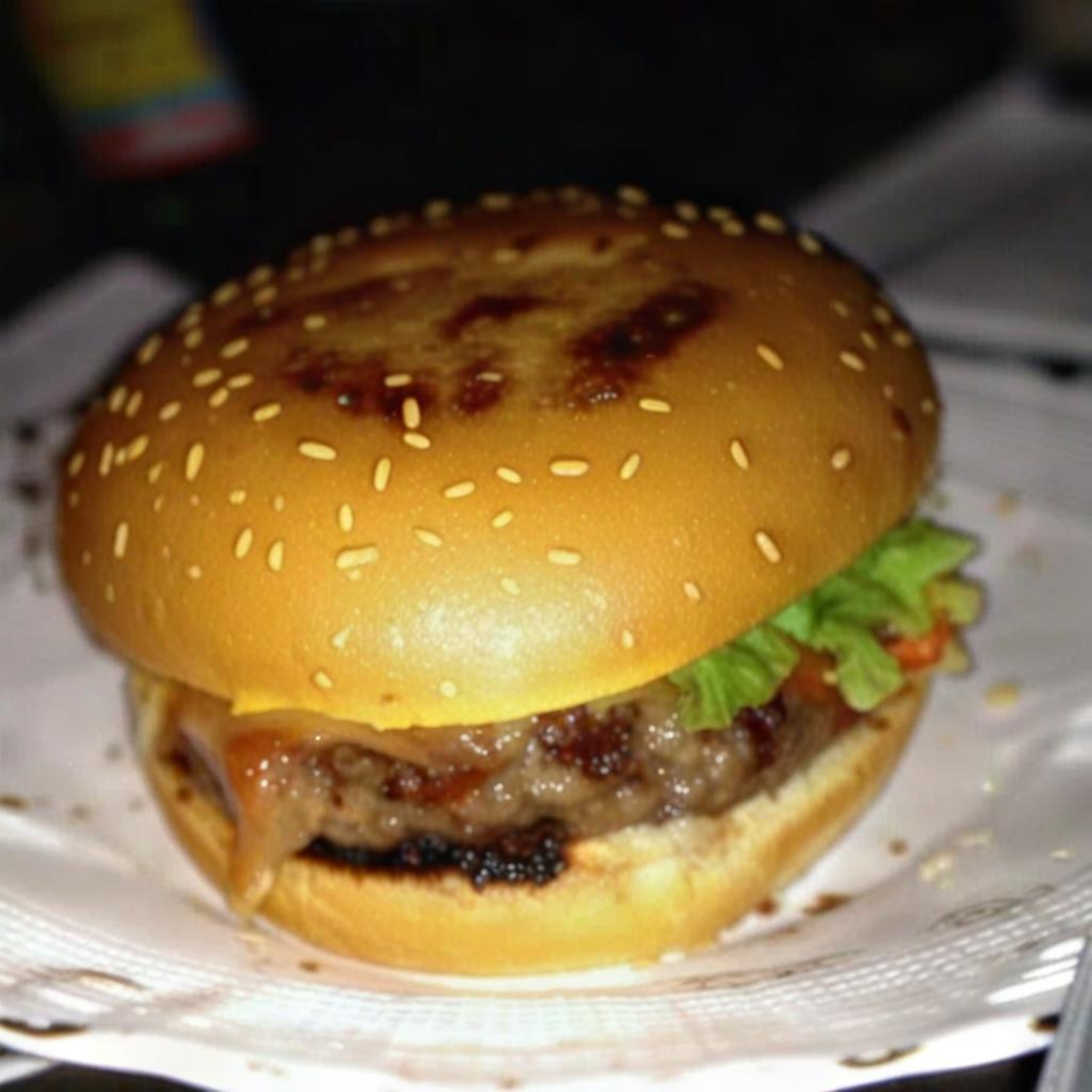}\\
        \includegraphics[width=\textwidth]{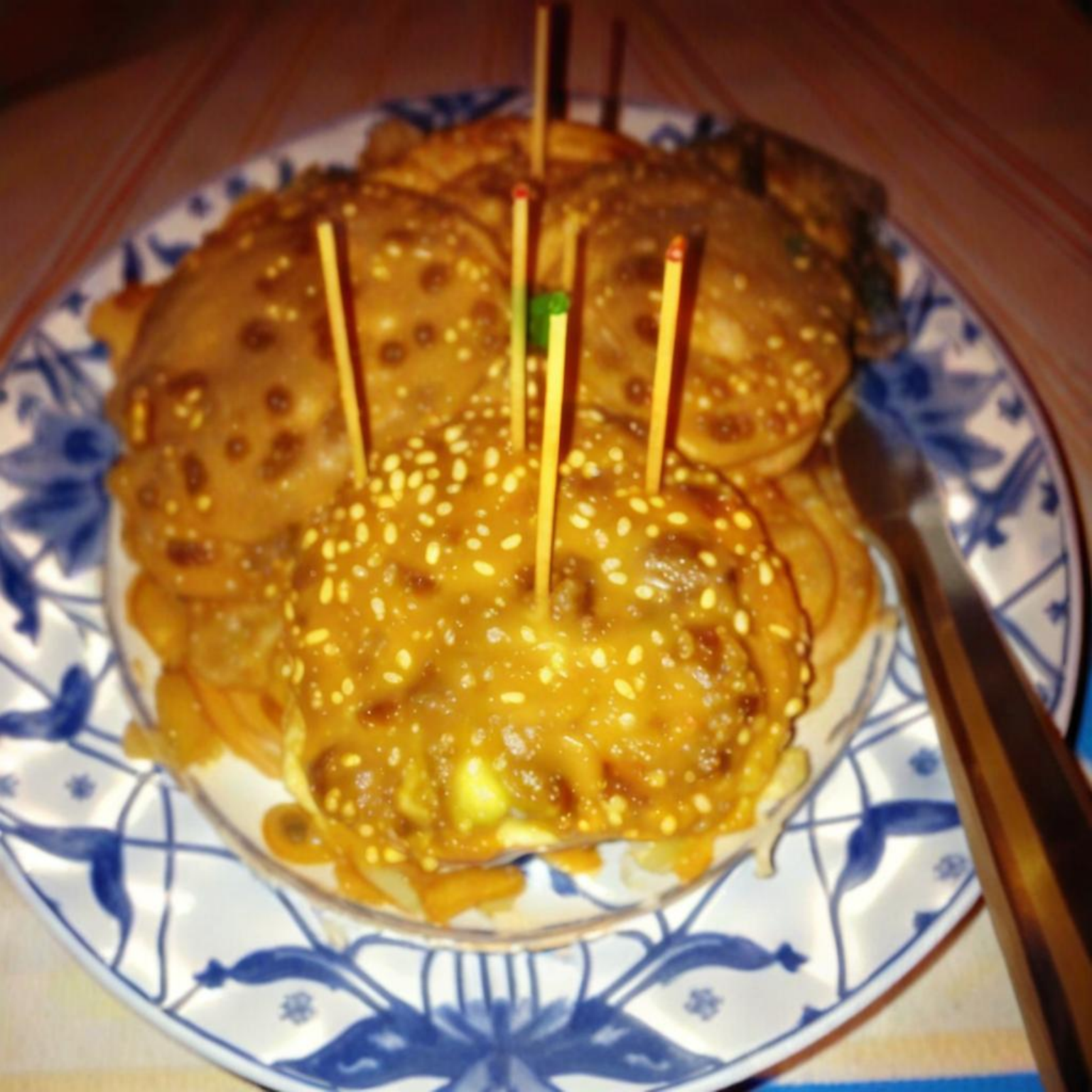}
    \end{minipage}
    \hfill
    \begin{minipage}{0.15\textwidth}
        \includegraphics[width=\textwidth]{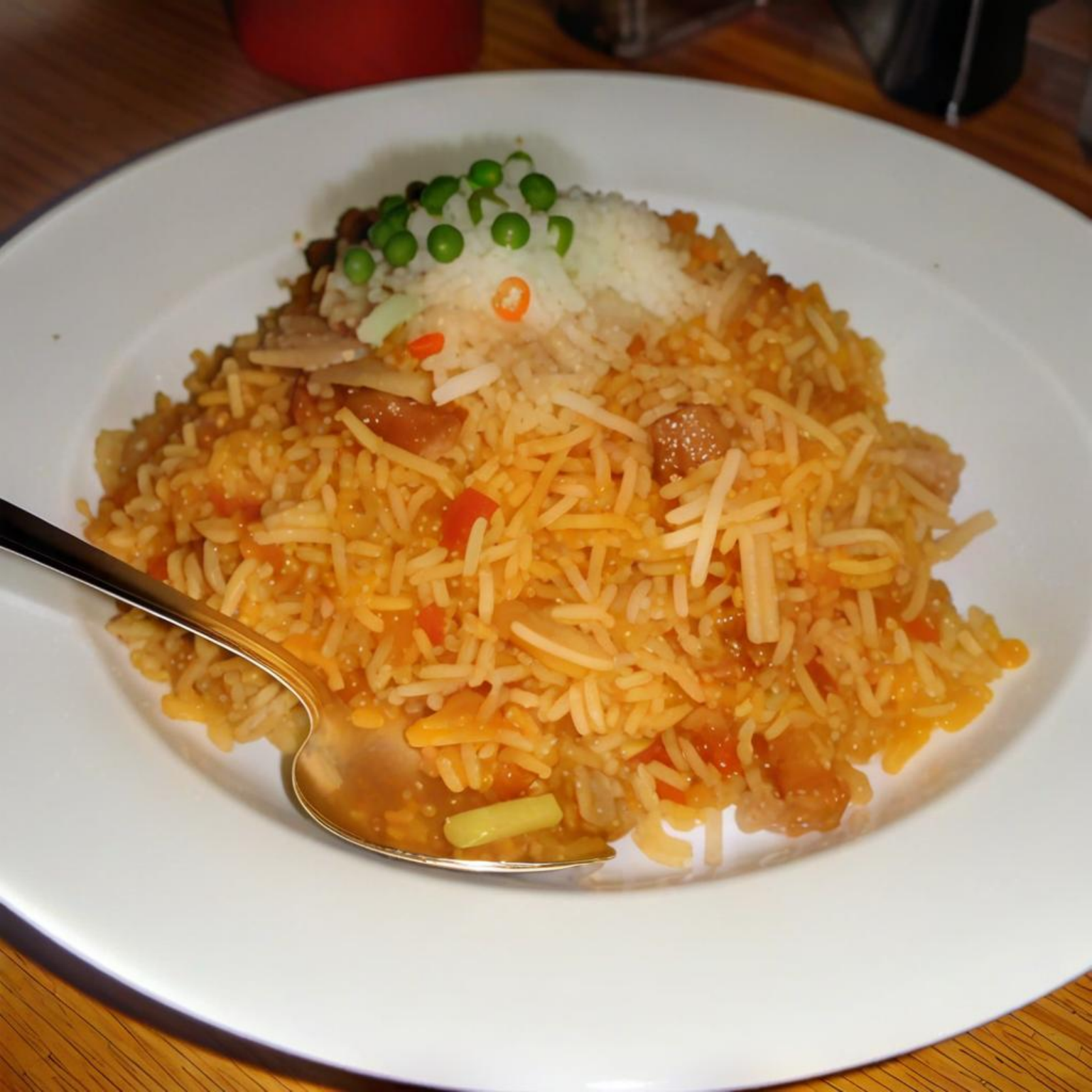}\\
        \includegraphics[width=\textwidth]{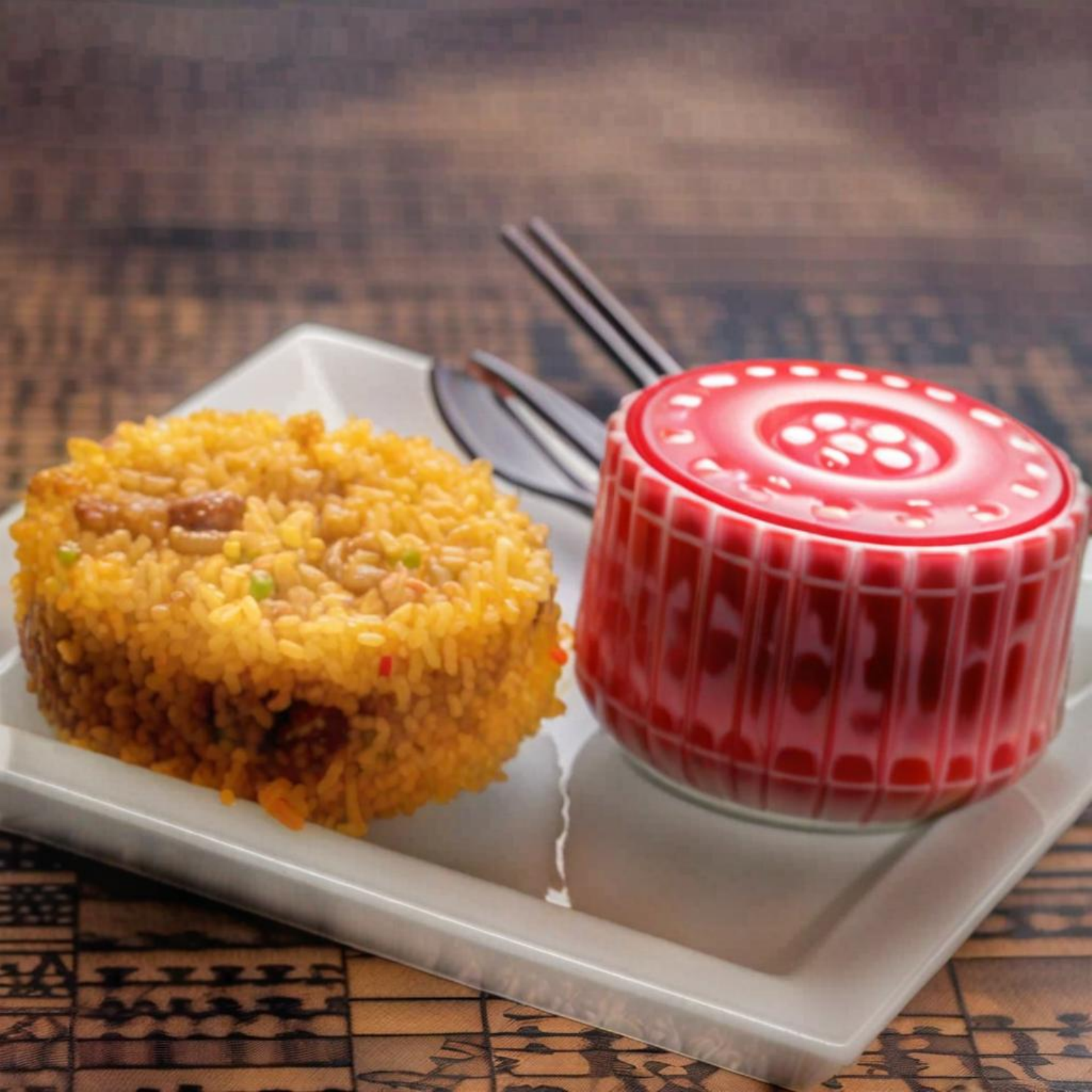}\\
        \includegraphics[width=\textwidth]{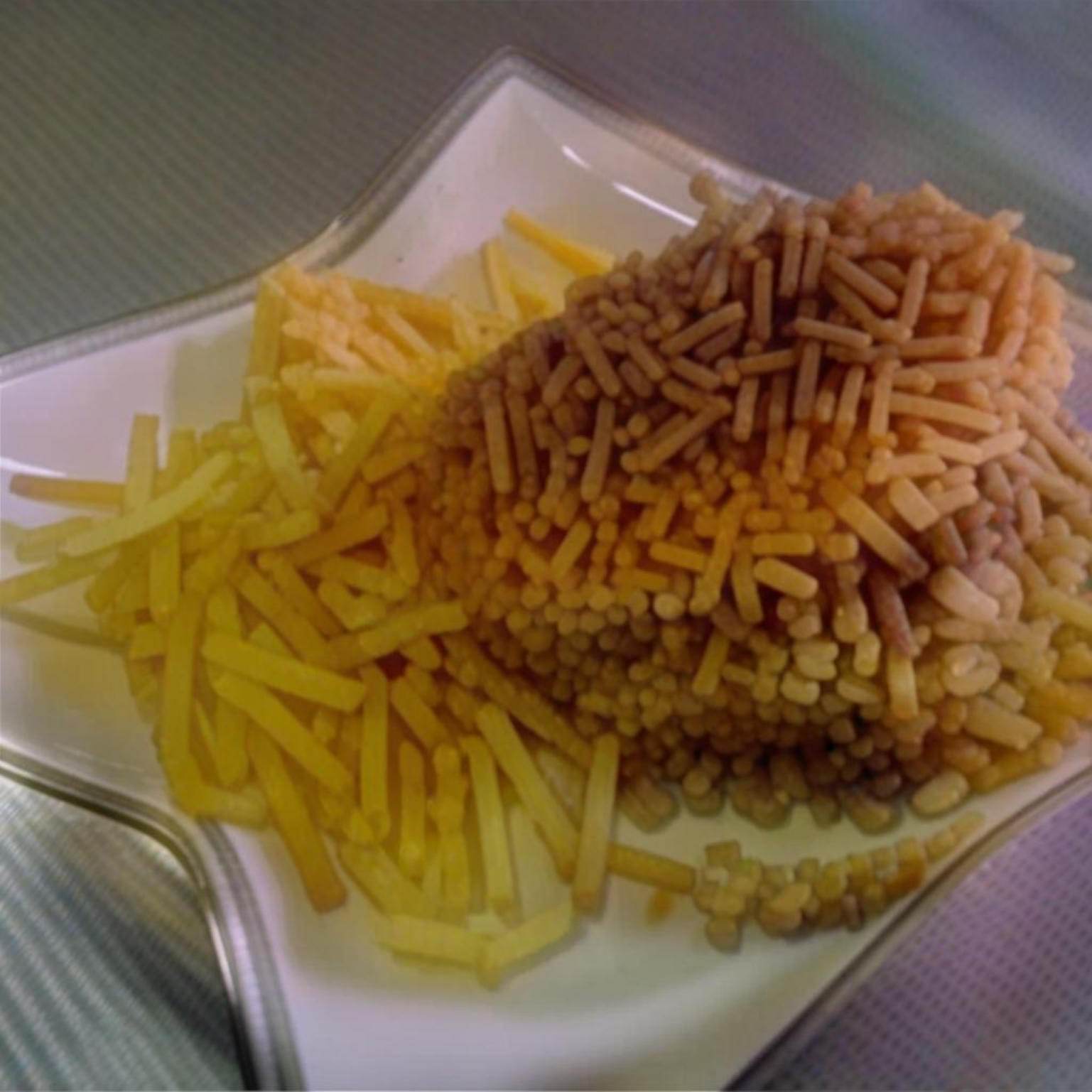}\\
        \includegraphics[width=\textwidth]{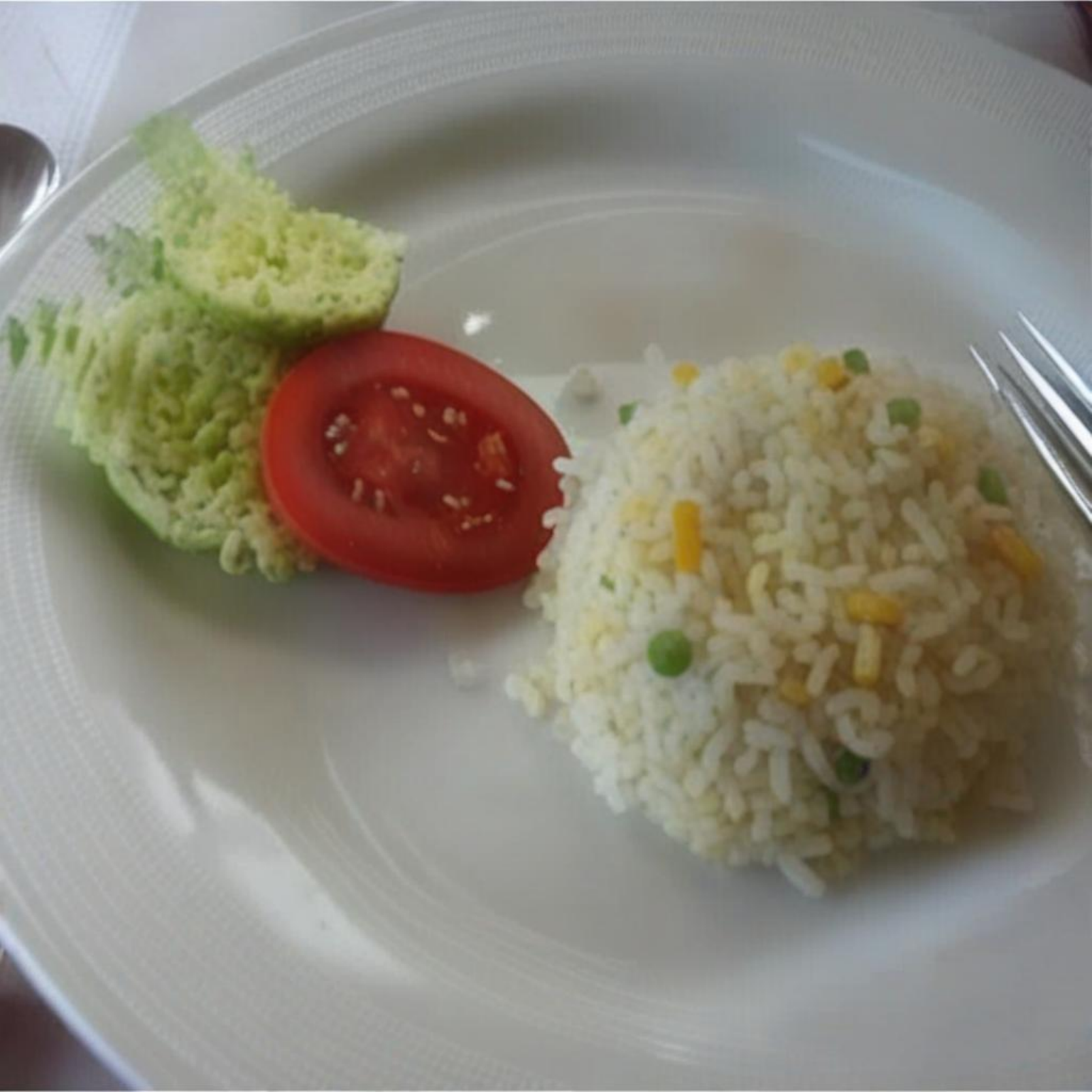}\\
        \includegraphics[width=\textwidth]{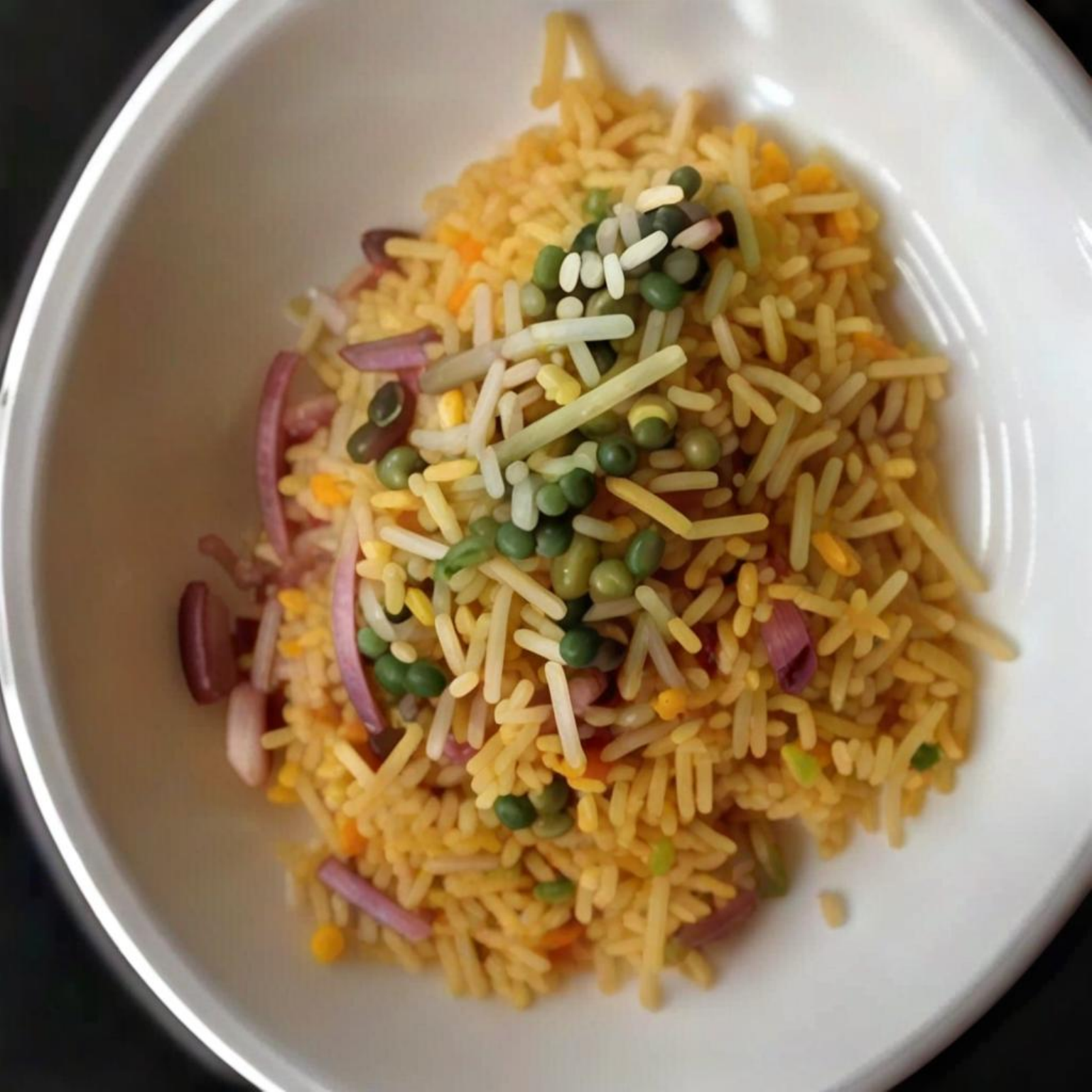}\\
        \includegraphics[width=\textwidth]{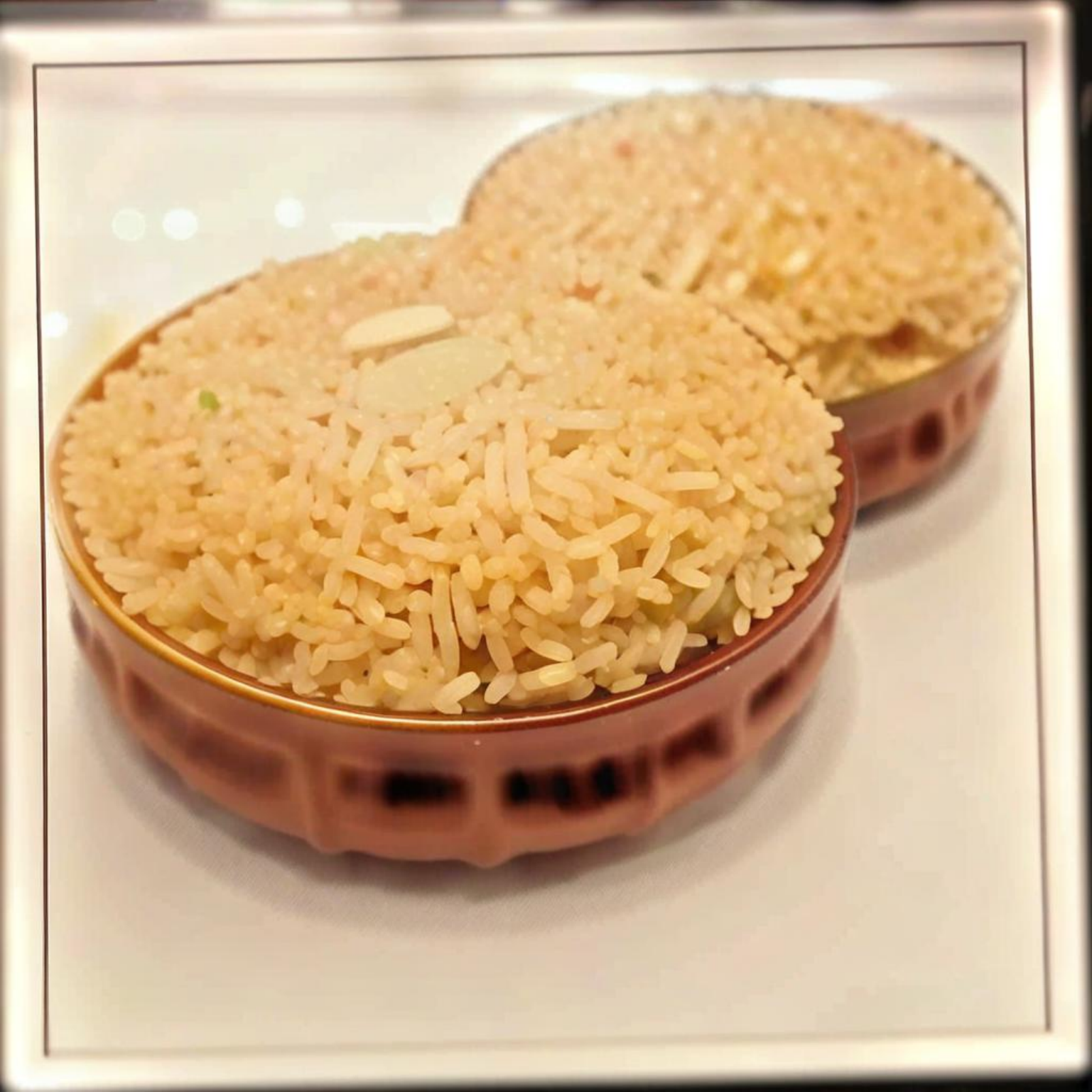}\\
        \includegraphics[width=\textwidth]{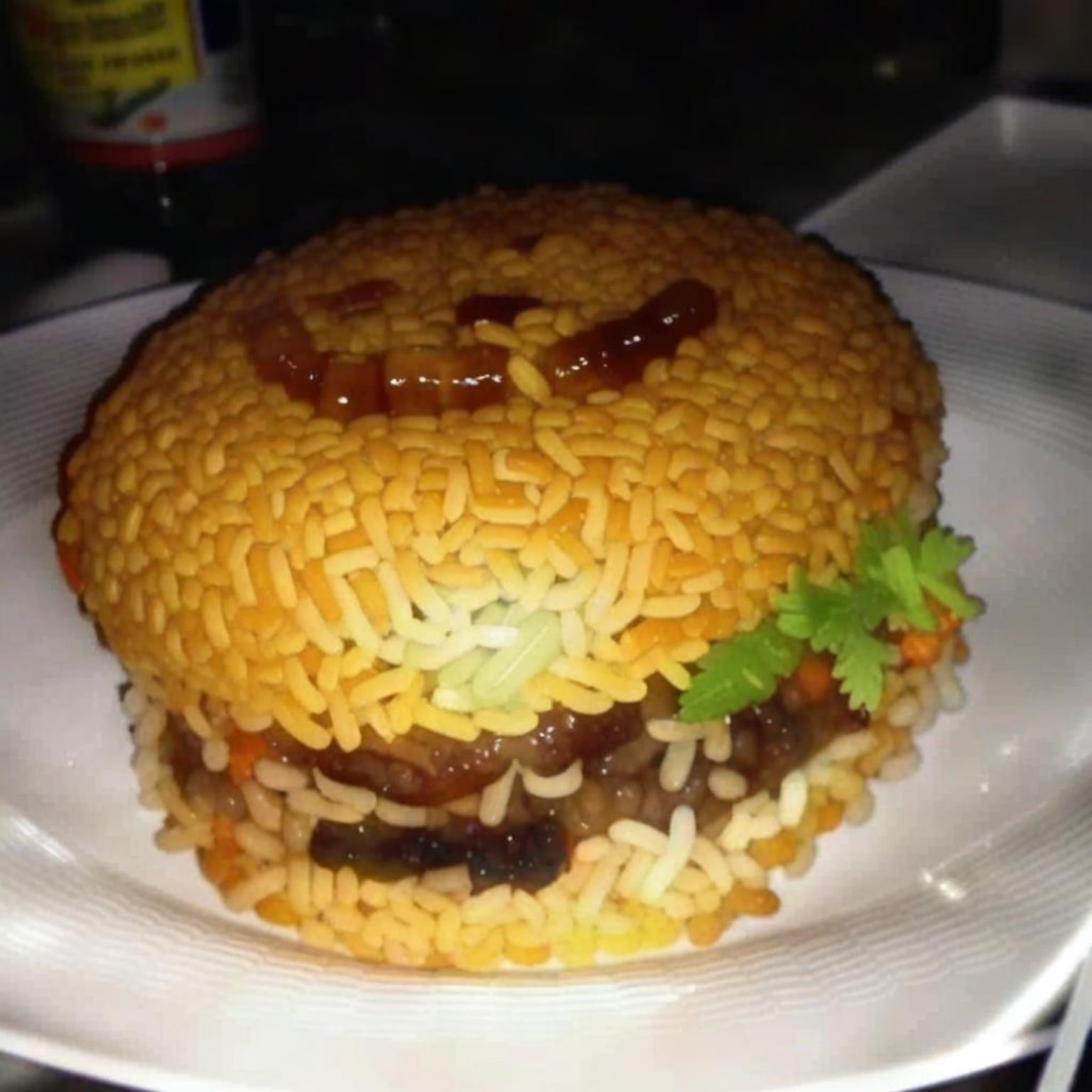}\\
        \includegraphics[width=\textwidth]{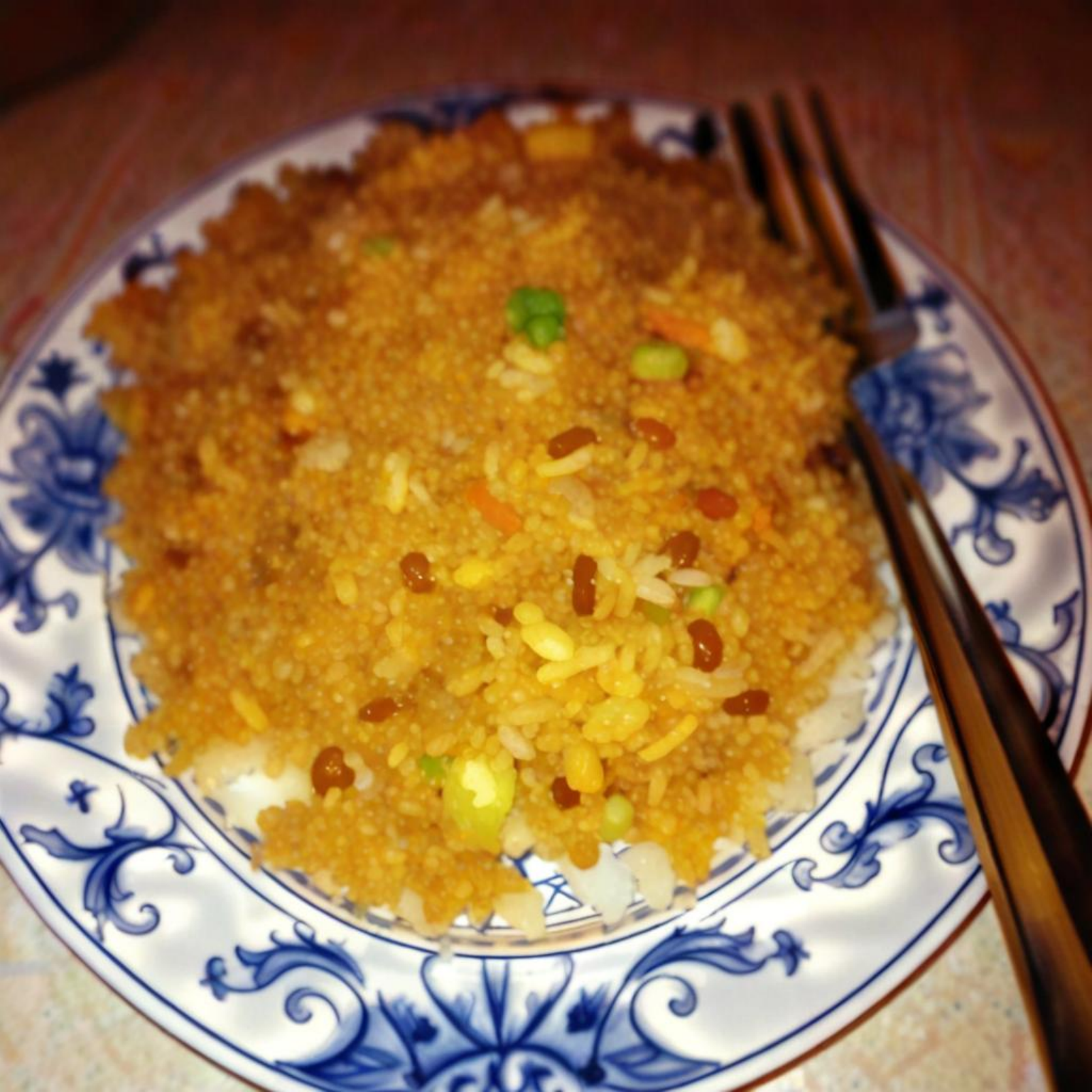}
    \end{minipage}

    \caption{Visual results of the 8-steps LATINO on Food101 dataset for semantic shift task.}
    \label{fig:Food}
\end{figure*}
}

\begin{figure*}[!h]
    \centering
    \makebox[0.16\textwidth]{\textbf{Measurement}}
    \hfill
    \makebox[0.16\textwidth]{\textbf{GT}}
    \hfill
    \makebox[0.16\textwidth]{\textbf{LATINO-PRO}}
    \hfill
    \makebox[0.16\textwidth]{\textbf{TREG}}
    \hfill
    \makebox[0.16\textwidth]{\textbf{P2L}}
    \hfill
    \makebox[0.16\textwidth]{\textbf{PSLD}}
    
    \vspace{2mm} 

    \begin{minipage}{0.16\textwidth}
        \includegraphics[width=\textwidth]{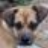}\\
        \includegraphics[width=\textwidth]{images/Tests/TReg_comparisons/degraded5.pdf}\\
        \includegraphics[width=\textwidth]{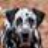}\\
        \includegraphics[width=\textwidth]{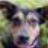}\\
        \includegraphics[width=\textwidth]{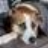}\\
        \includegraphics[width=\textwidth]{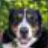}
    \end{minipage}
    \hfill
    \begin{minipage}{0.16\textwidth}
        \includegraphics[width=\textwidth]{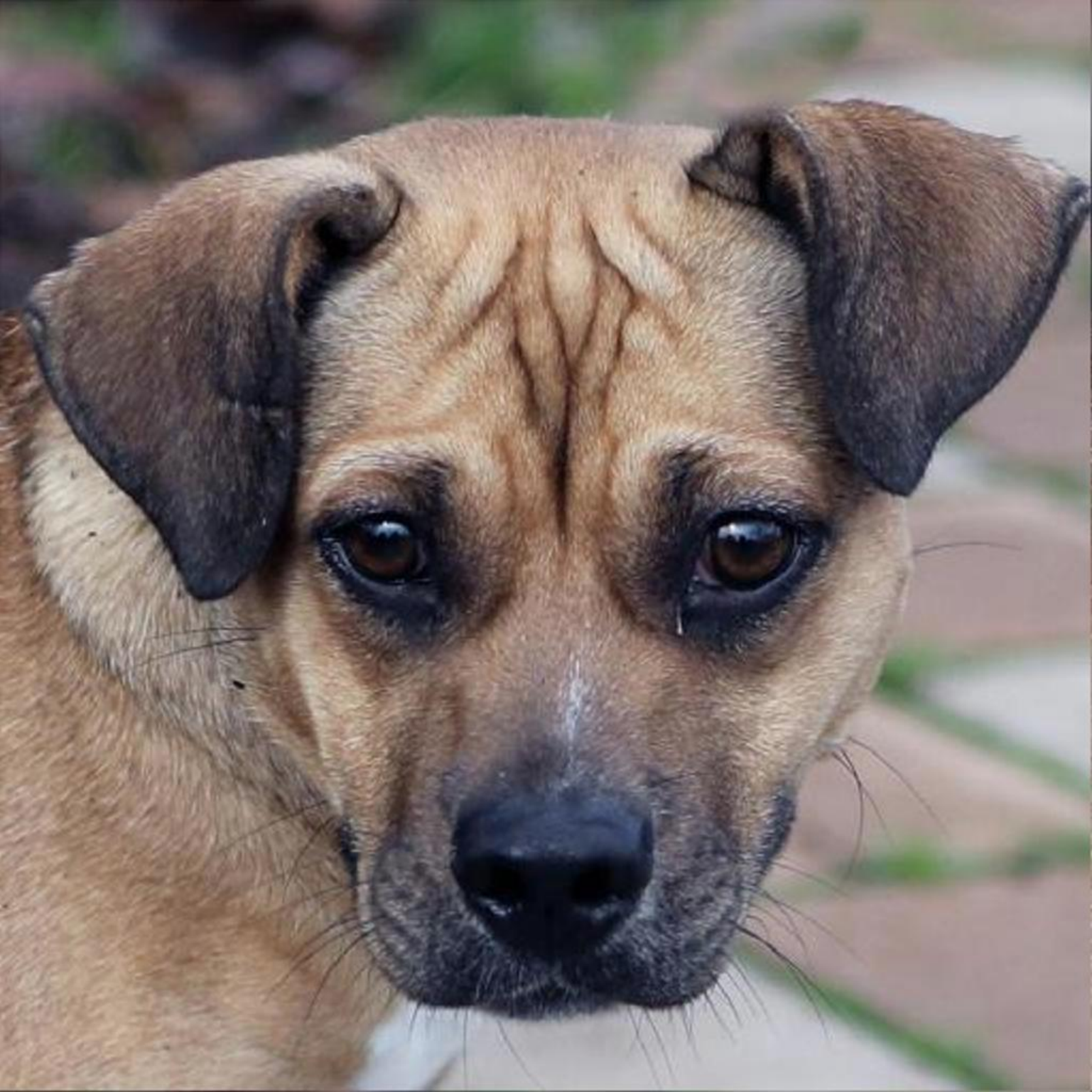} \\
        \includegraphics[width=\textwidth]{images/Tests/TReg_comparisons/clean5.pdf} \\
        \includegraphics[width=\textwidth]{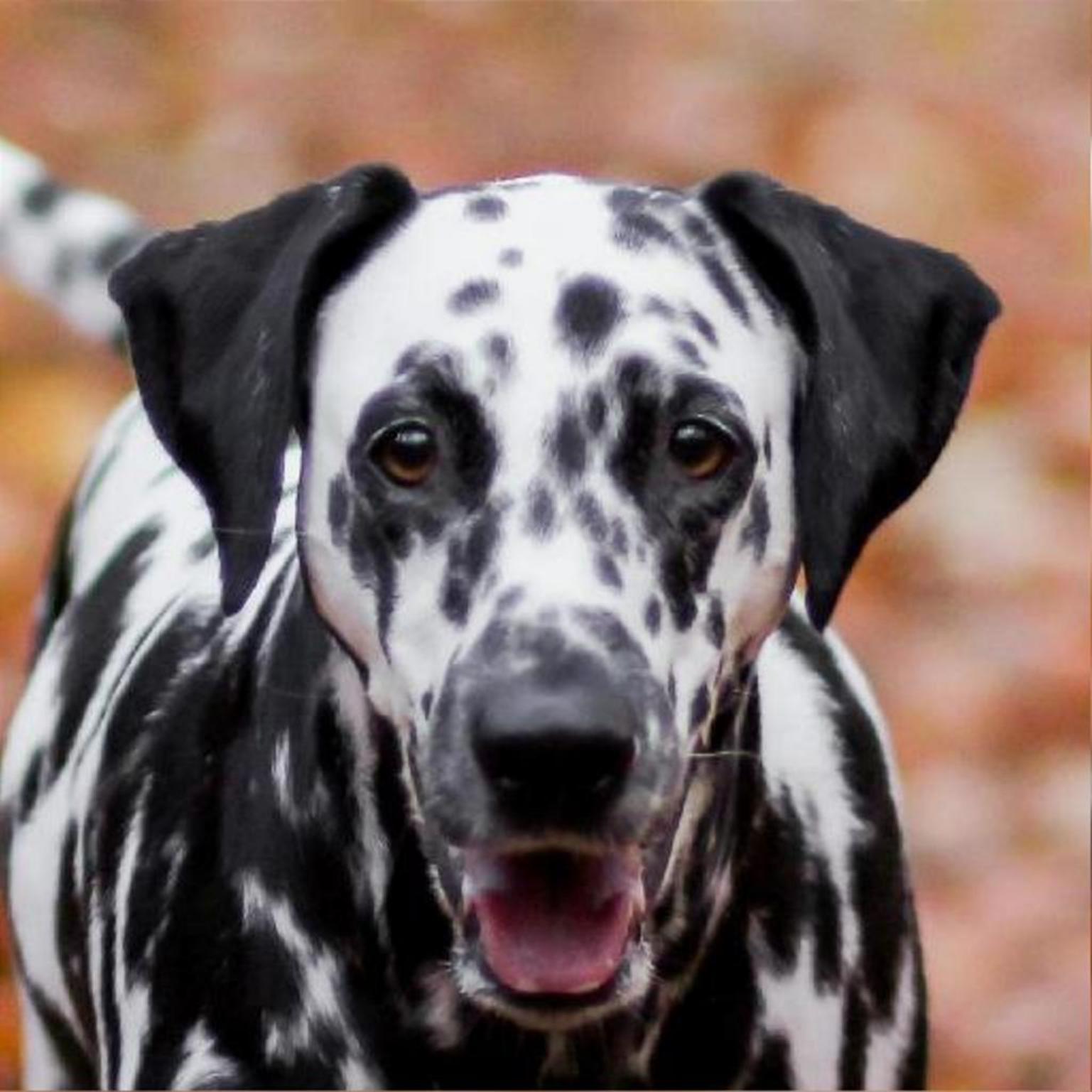} \\
        \includegraphics[width=\textwidth]{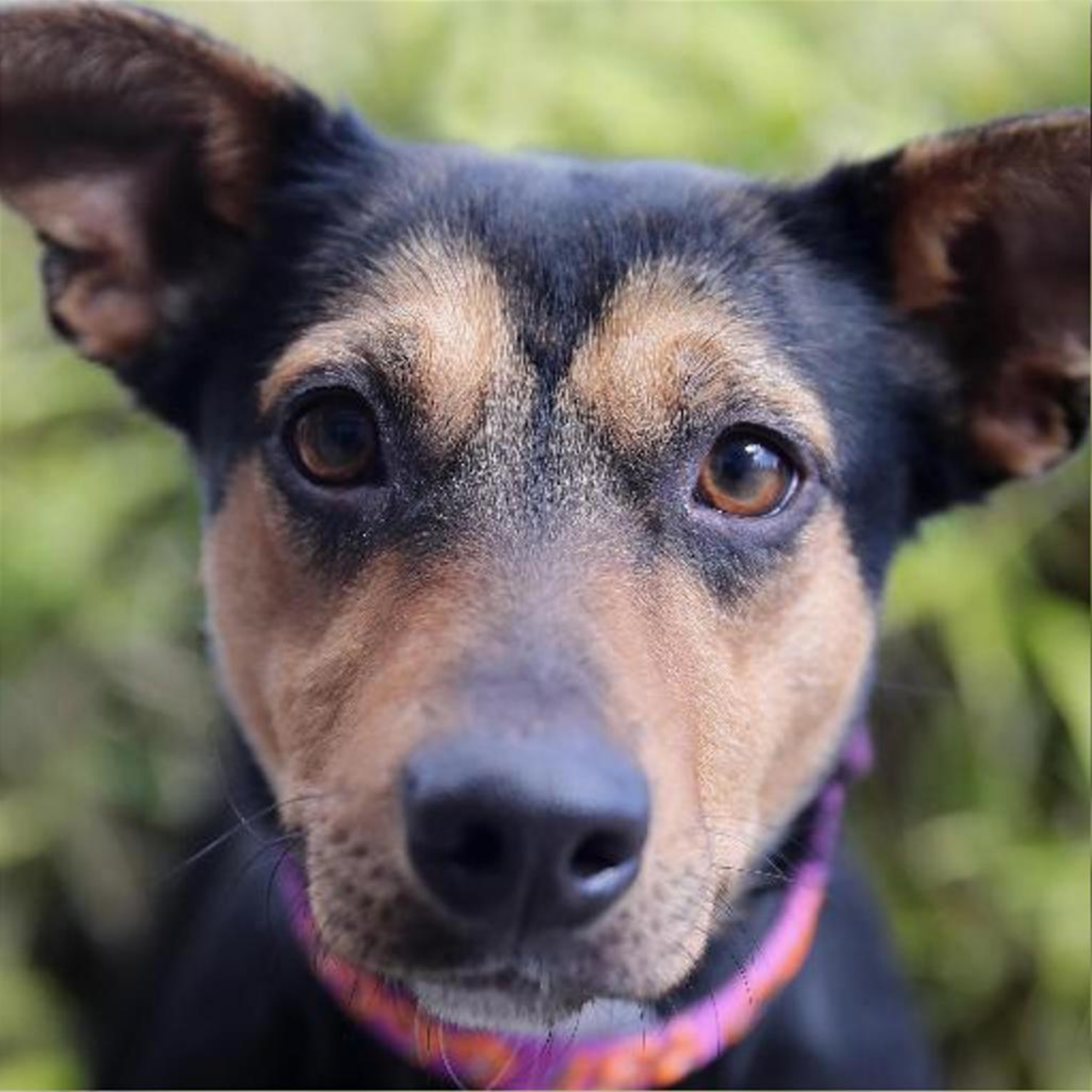} \\
        \includegraphics[width=\textwidth]{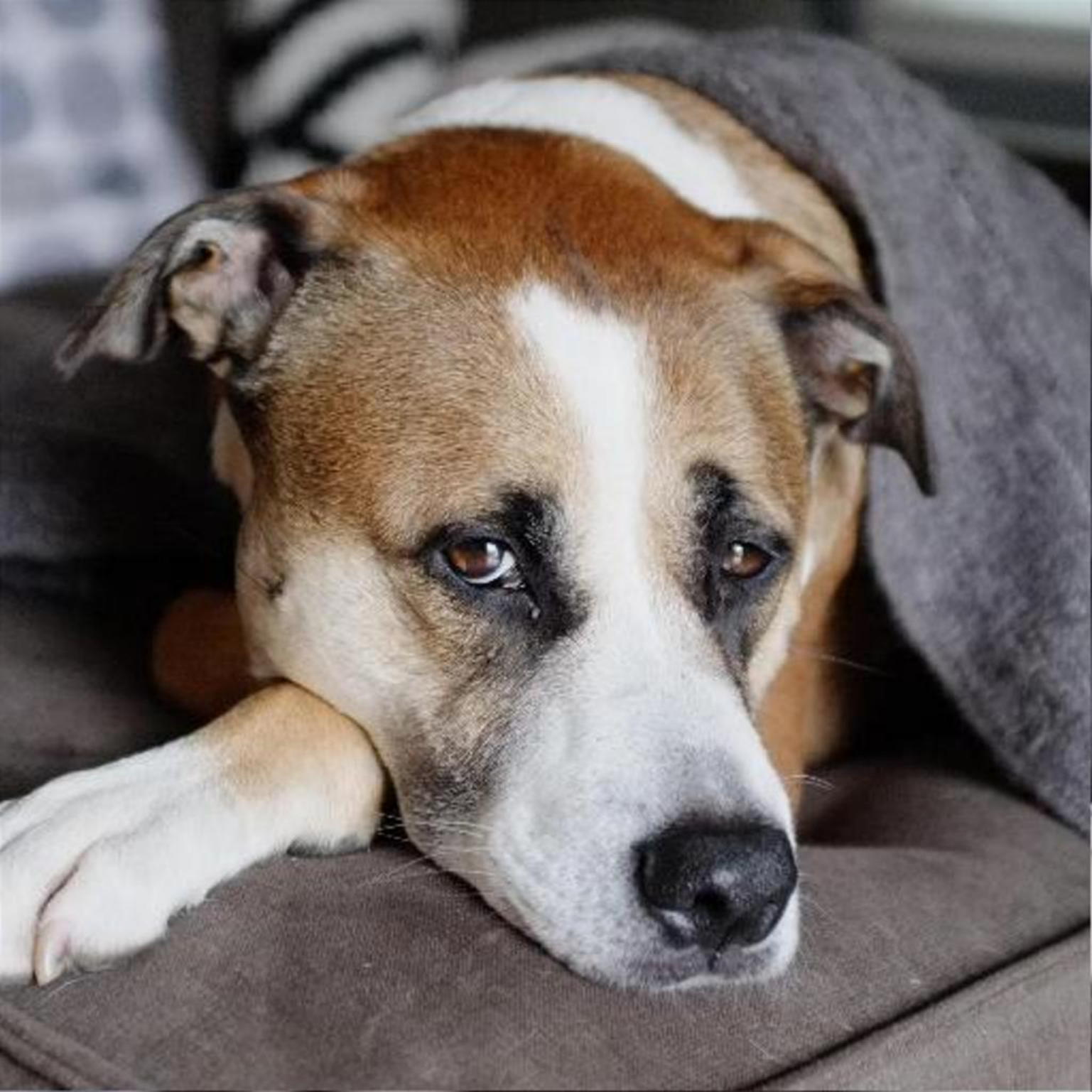} \\
        \includegraphics[width=\textwidth]{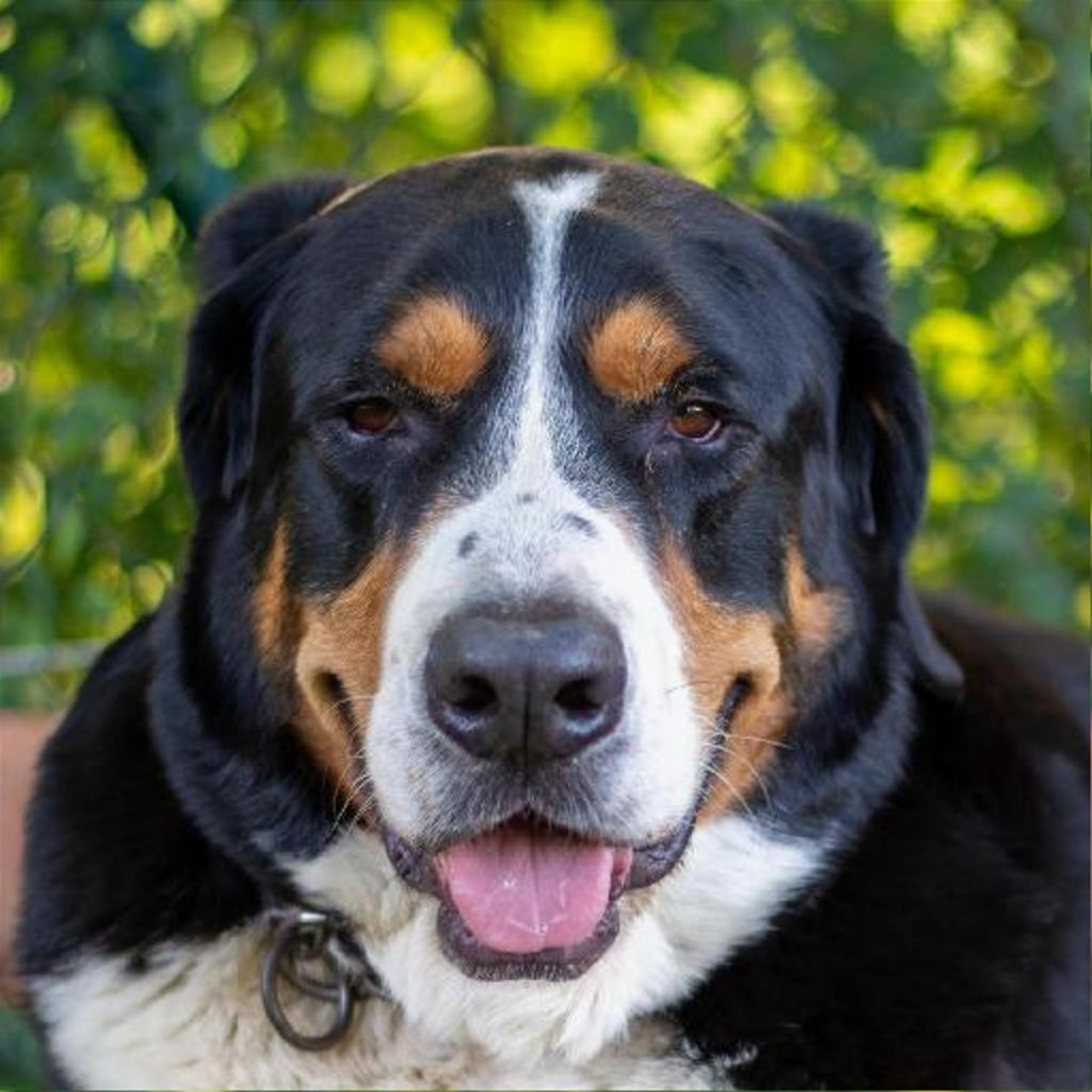}
    \end{minipage}
    \hfill
    \begin{minipage}{0.16\textwidth}
        \includegraphics[width=\textwidth]{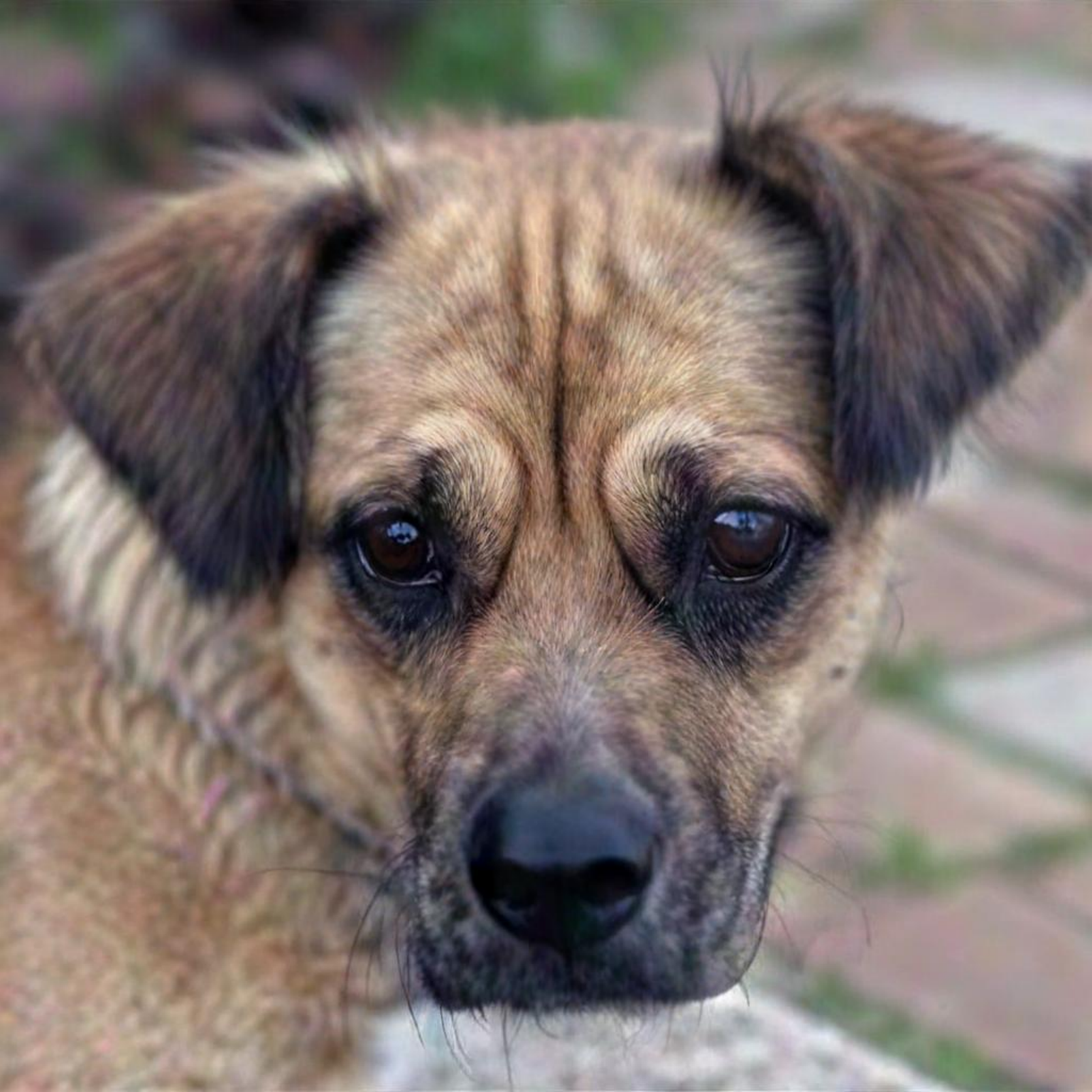} \\
        \includegraphics[width=\textwidth]{images/Tests/TReg_comparisons/restored5.pdf} \\
        \includegraphics[width=\textwidth]{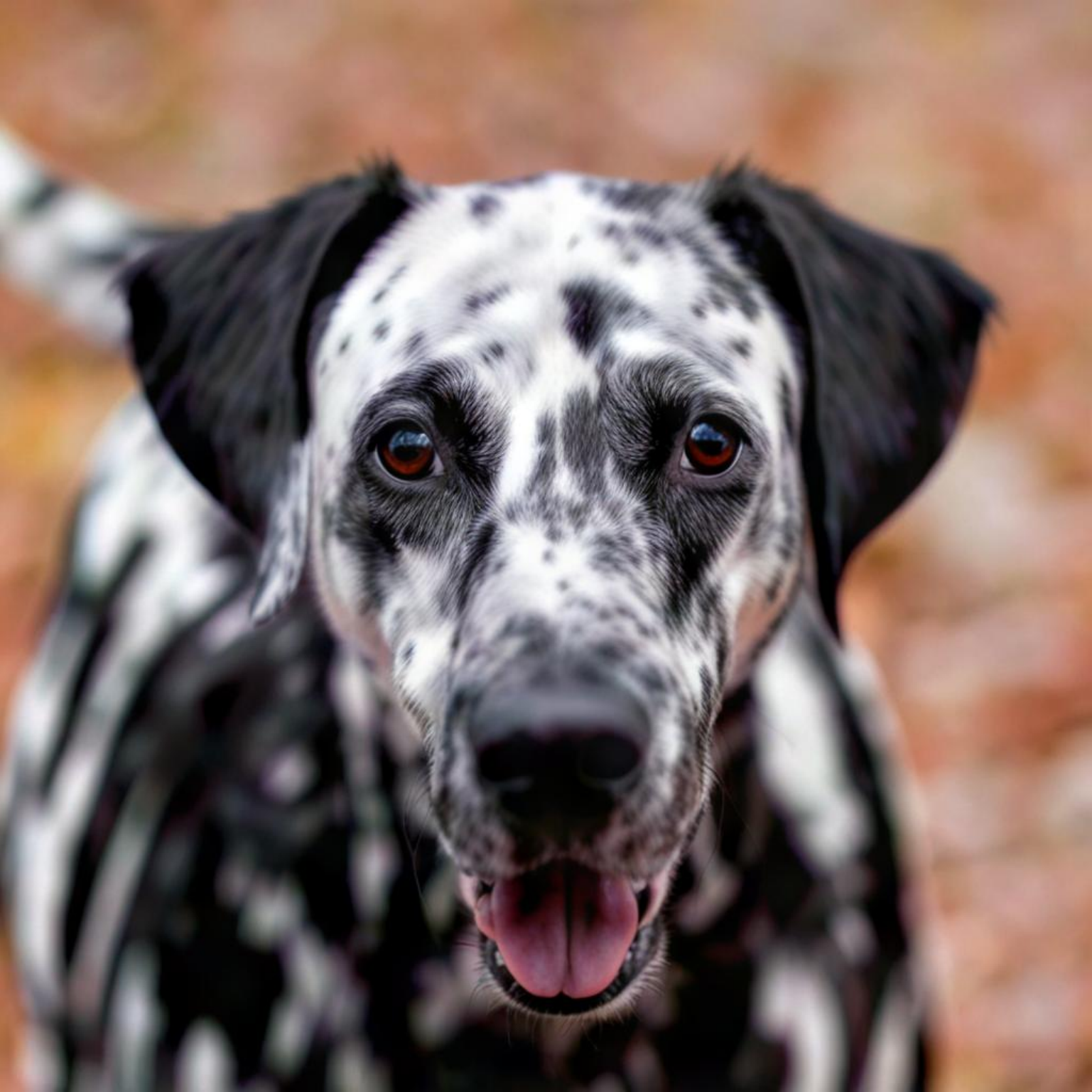} \\
        \includegraphics[width=\textwidth]{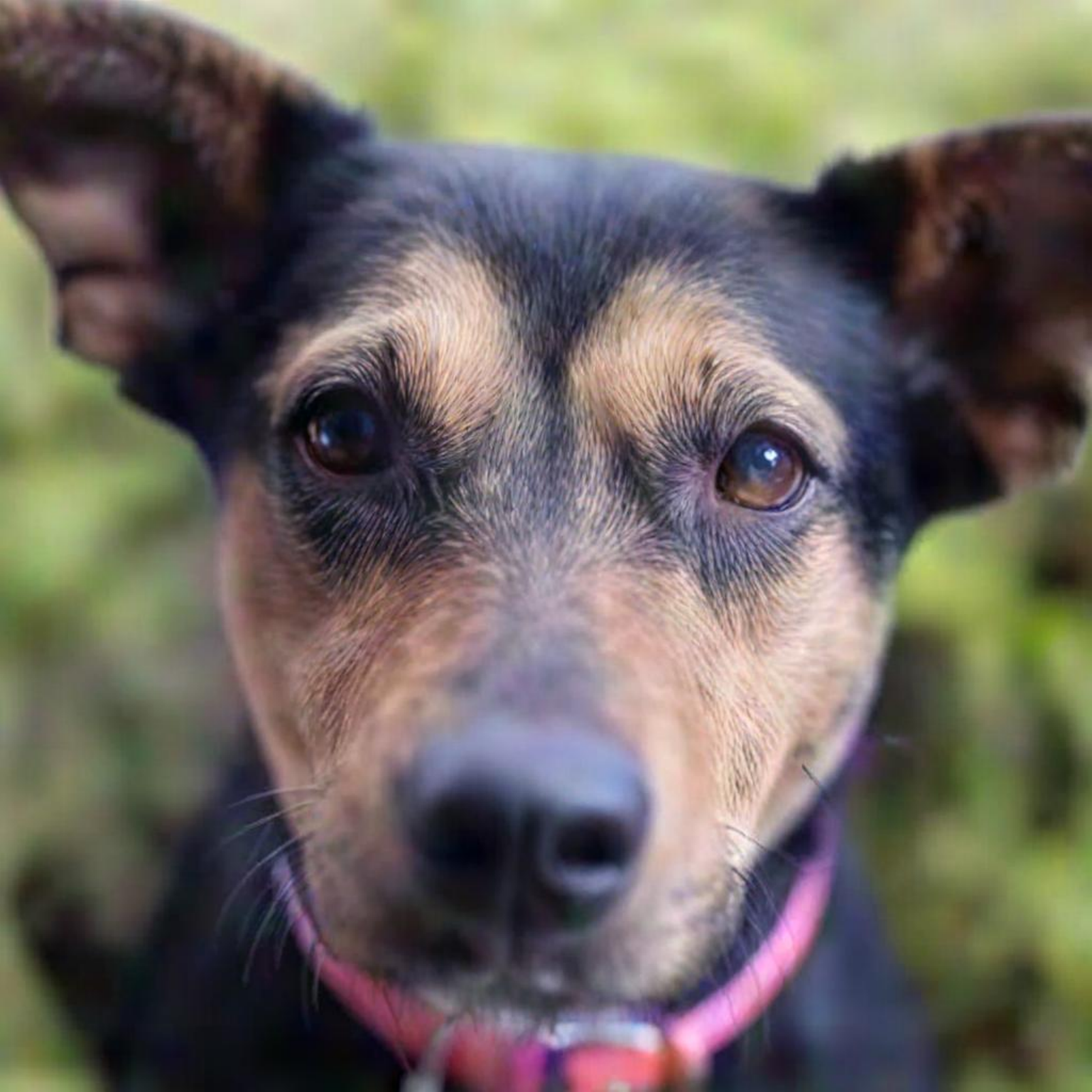} \\
        \includegraphics[width=\textwidth]{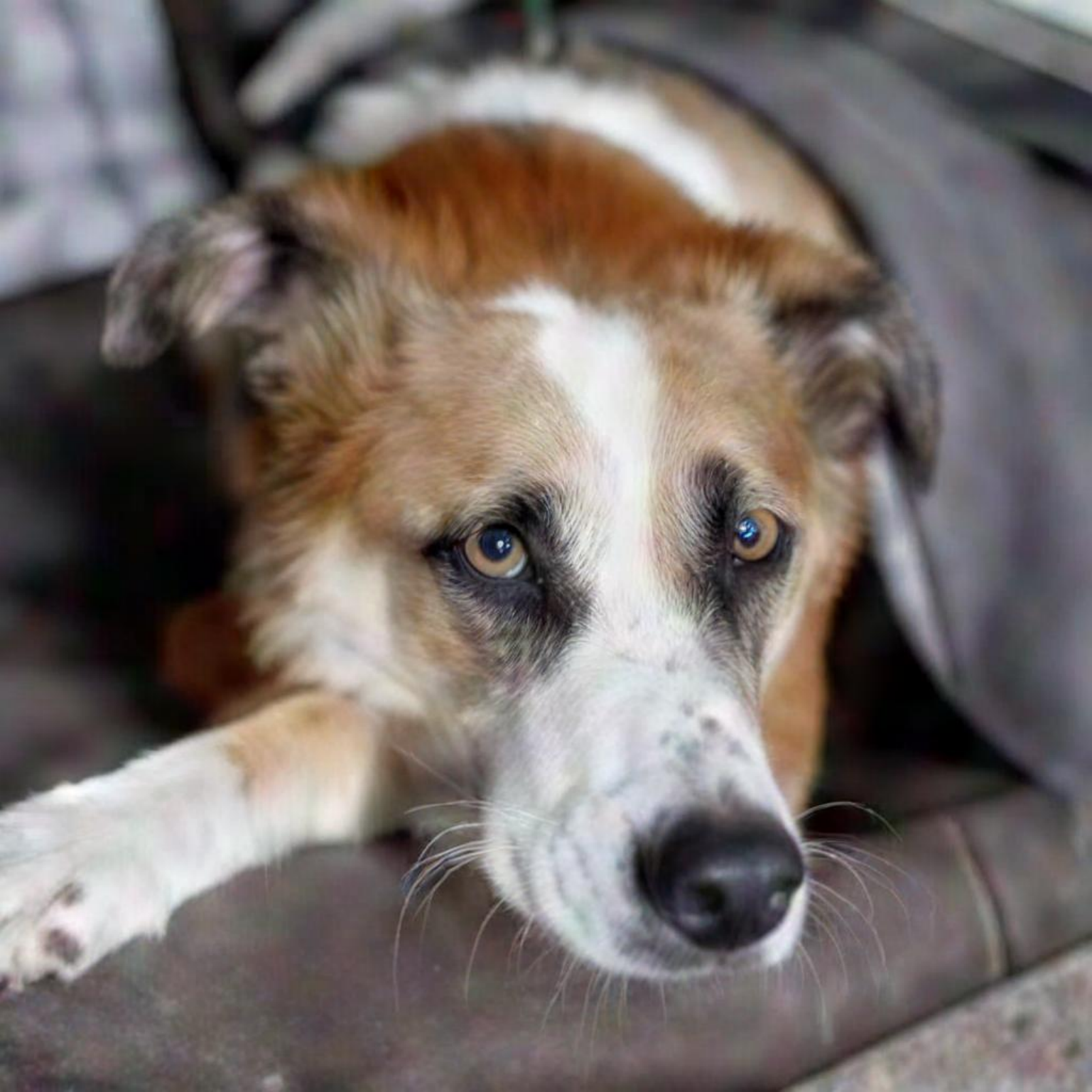} \\
        \includegraphics[width=\textwidth]{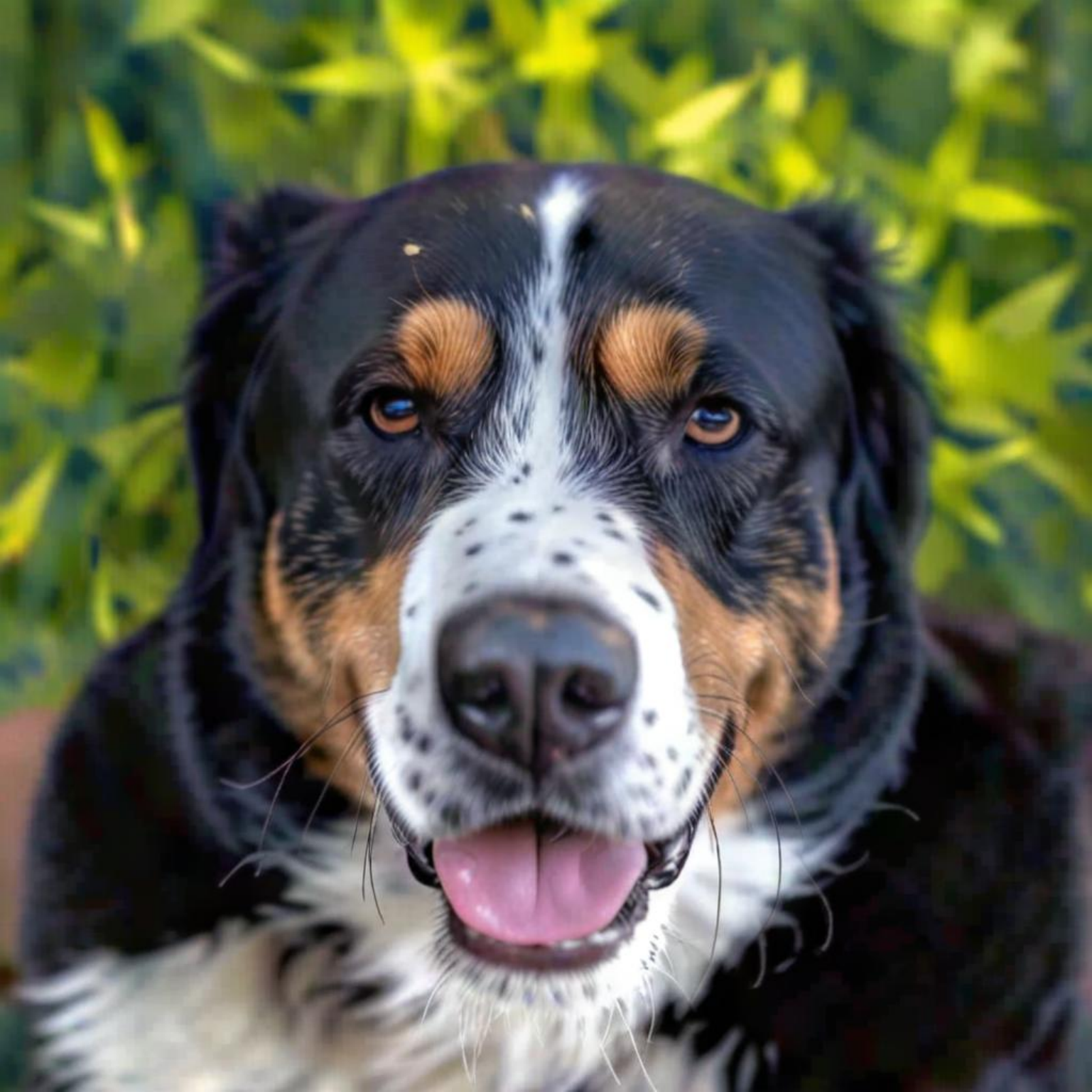}
    \end{minipage}
    \hfill
    \begin{minipage}{0.16\textwidth}
        \includegraphics[width=\textwidth]{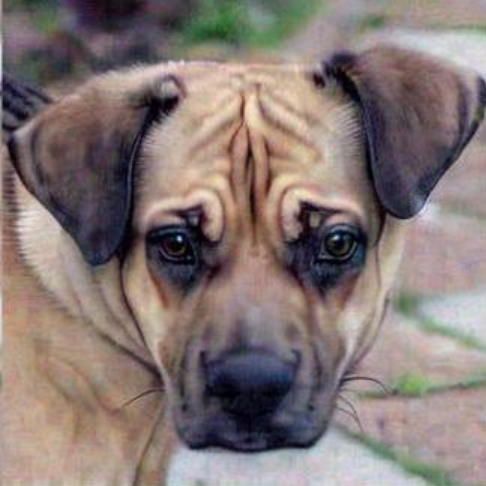}\\
        \includegraphics[width=\textwidth]{images/Tests/TReg_comparisons/dog_row2_col3.pdf}\\
        \includegraphics[width=\textwidth]{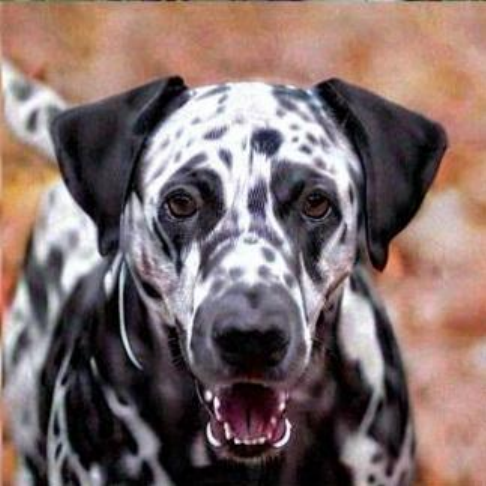}\\
        \includegraphics[width=\textwidth]{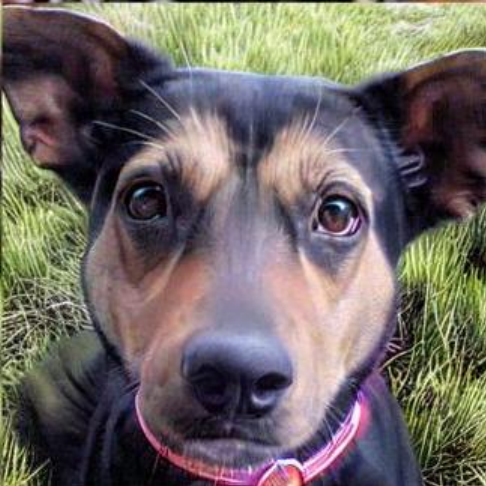}\\
        \includegraphics[width=\textwidth]{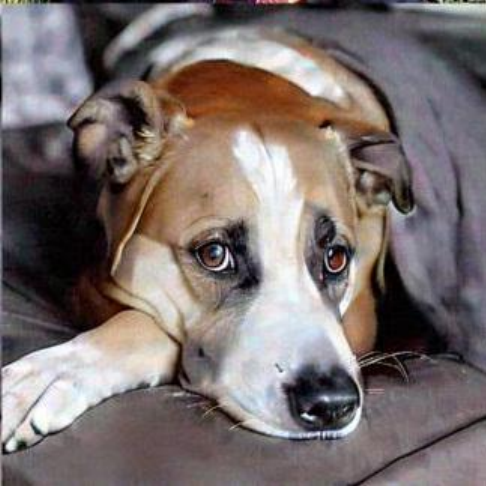}\\
        \includegraphics[width=\textwidth]{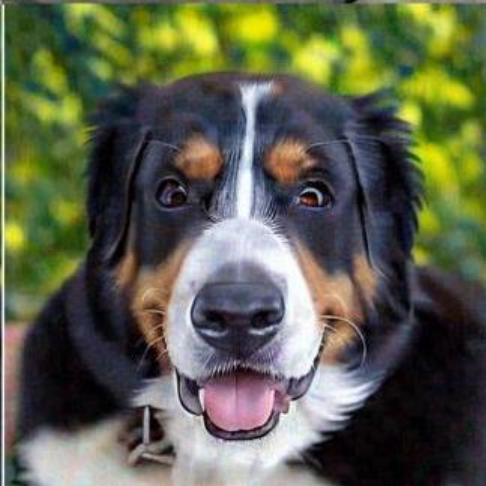}
    \end{minipage}
    \hfill
    \begin{minipage}{0.16\textwidth}
        \includegraphics[width=\textwidth]{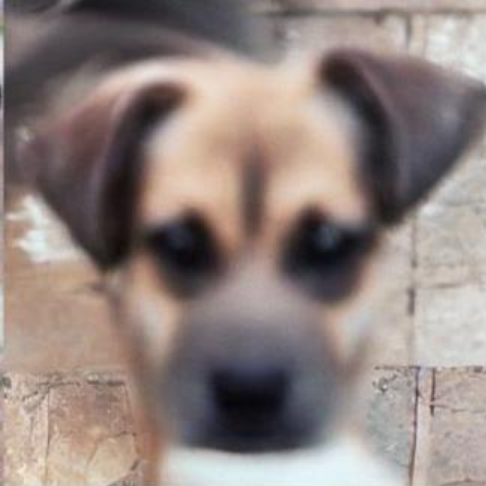}\\
        \includegraphics[width=\textwidth]{images/Tests/TReg_comparisons/dog_row2_col4.pdf}\\
        \includegraphics[width=\textwidth]{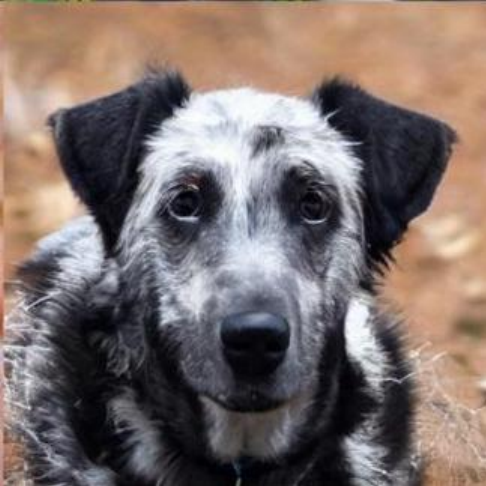}\\
        \includegraphics[width=\textwidth]{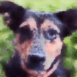}\\
        \includegraphics[width=\textwidth]{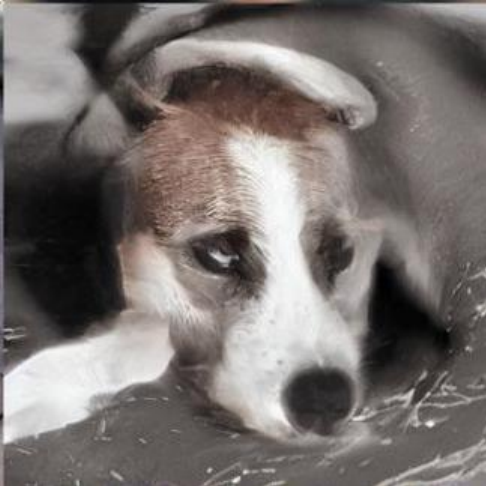}\\
        \includegraphics[width=\textwidth]{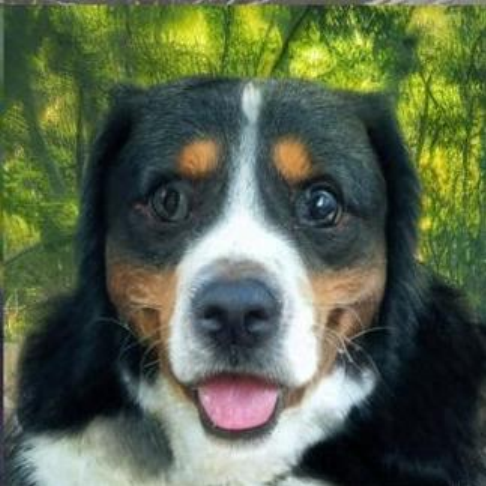}
    \end{minipage}
    \hfill
    \begin{minipage}{0.16\textwidth}
        \includegraphics[width=\textwidth]{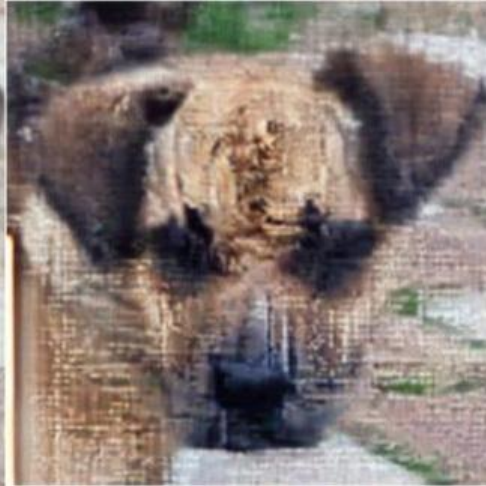}\\
        \includegraphics[width=\textwidth]{images/Tests/TReg_comparisons/dog_row2_col5.pdf}\\
        \includegraphics[width=\textwidth]{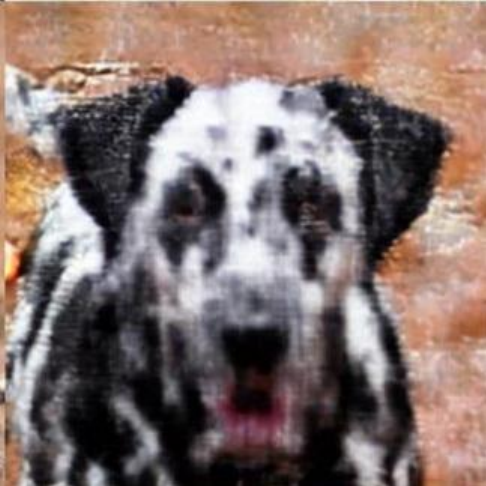}\\
        \includegraphics[width=\textwidth]{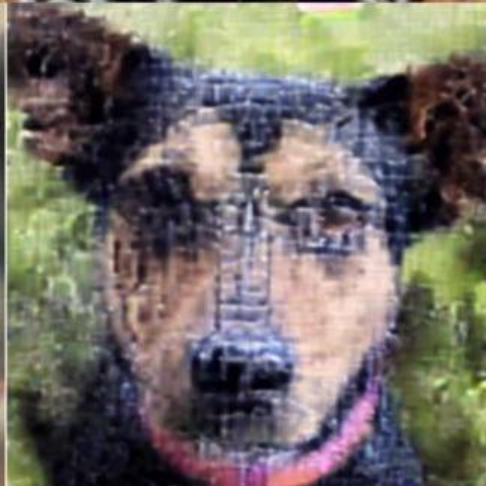}\\
        \includegraphics[width=\textwidth]{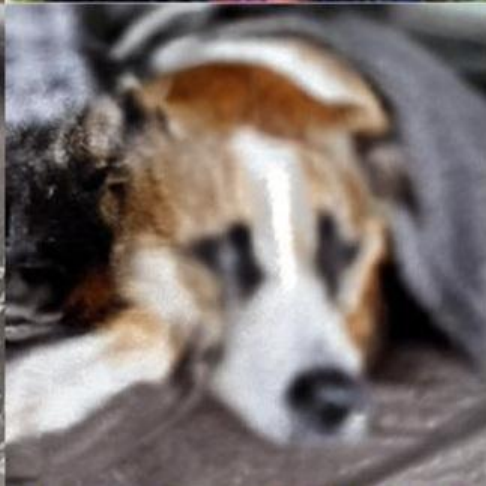}\\
        \includegraphics[width=\textwidth]{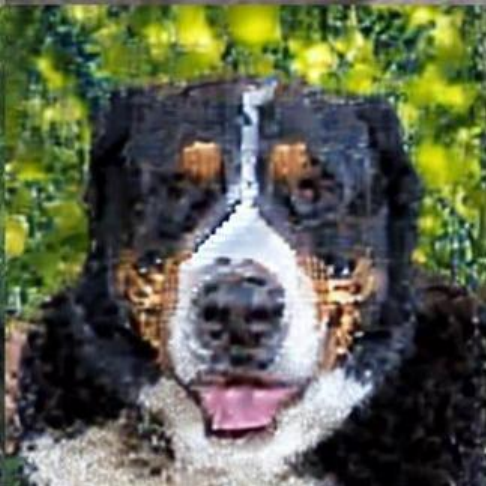}
    \end{minipage}

    \caption{Comparisons between LATINO-PRO, TReg and P2L.\vspace{4cm}}
    \label{fig:AFHQ_dogs}
\end{figure*}

\begin{figure*}[!h]
    \centering
    \makebox[0.16\textwidth]{\textbf{Measurement}}
    \hfill
    \makebox[0.16\textwidth]{\textbf{GT}}
    \hfill
    \makebox[0.16\textwidth]{\textbf{LATINO-PRO}}
    \hfill
    \makebox[0.16\textwidth]{\textbf{TREG}}
    \hfill
    \makebox[0.16\textwidth]{\textbf{P2L}}
    \hfill
    \makebox[0.16\textwidth]{\textbf{PSLD}}
    
    \vspace{2mm} 

    \begin{minipage}{0.16\textwidth}
        \includegraphics[width=\textwidth]{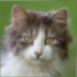}\\
        \includegraphics[width=\textwidth]{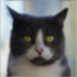}\\
        \includegraphics[width=\textwidth]{images/Tests/TReg_comparisons/degraded9.pdf}\\
        \includegraphics[width=\textwidth]{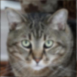}\\
        \includegraphics[width=\textwidth]{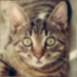}
    \end{minipage}
    \hfill
    \begin{minipage}{0.16\textwidth}
        \includegraphics[width=\textwidth]{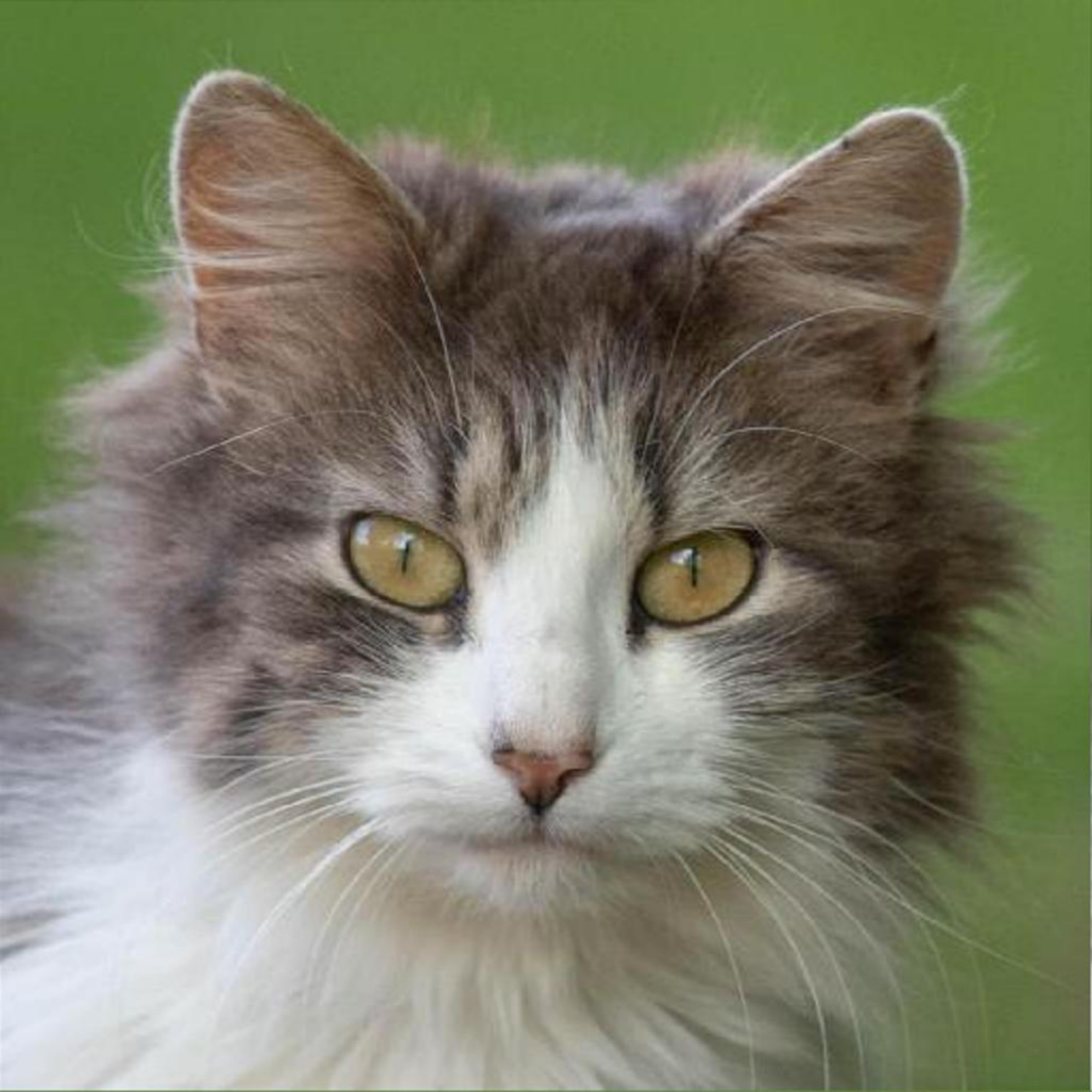} \\
        \includegraphics[width=\textwidth]{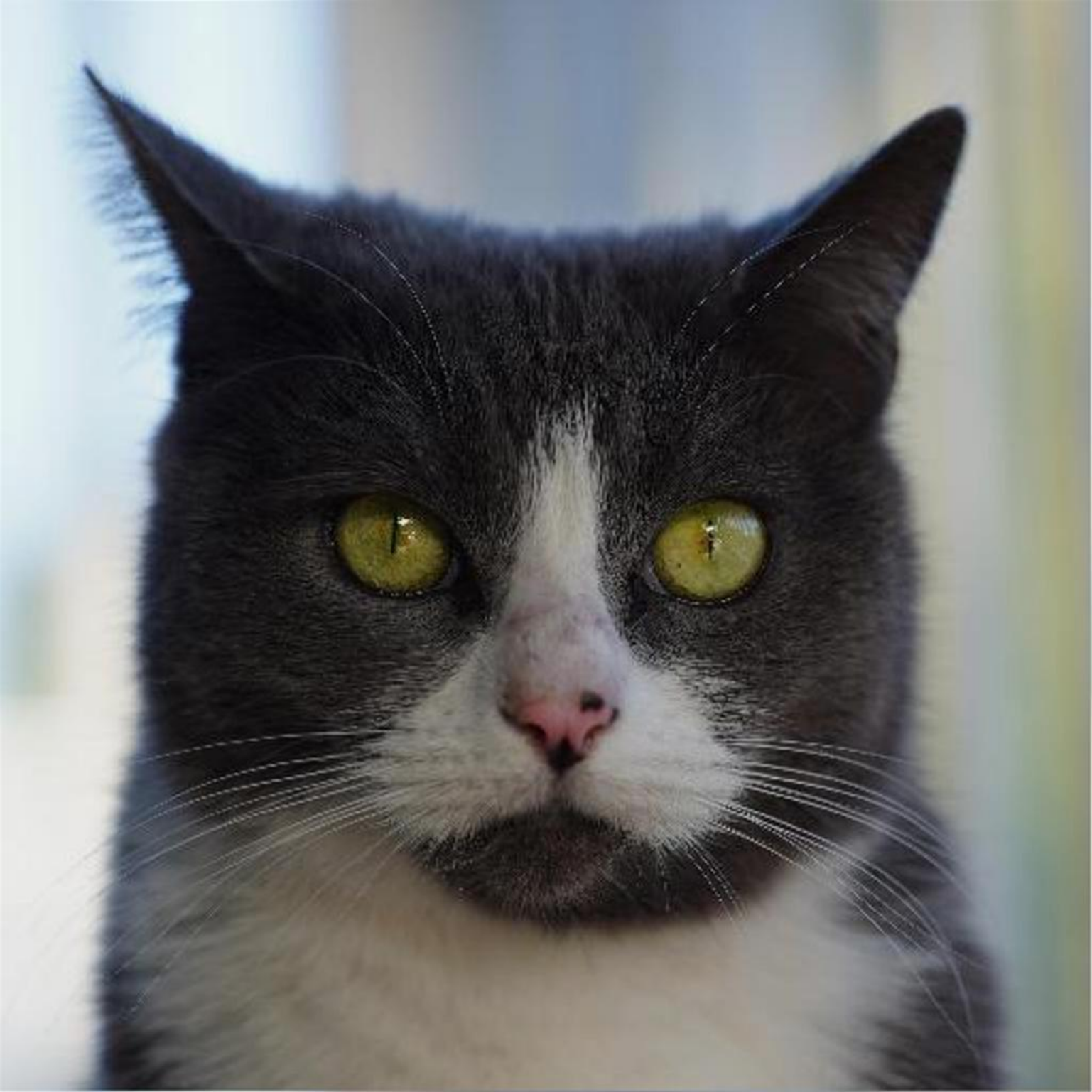} \\
        \includegraphics[width=\textwidth]{images/Tests/TReg_comparisons/clean9.pdf} \\
        \includegraphics[width=\textwidth]{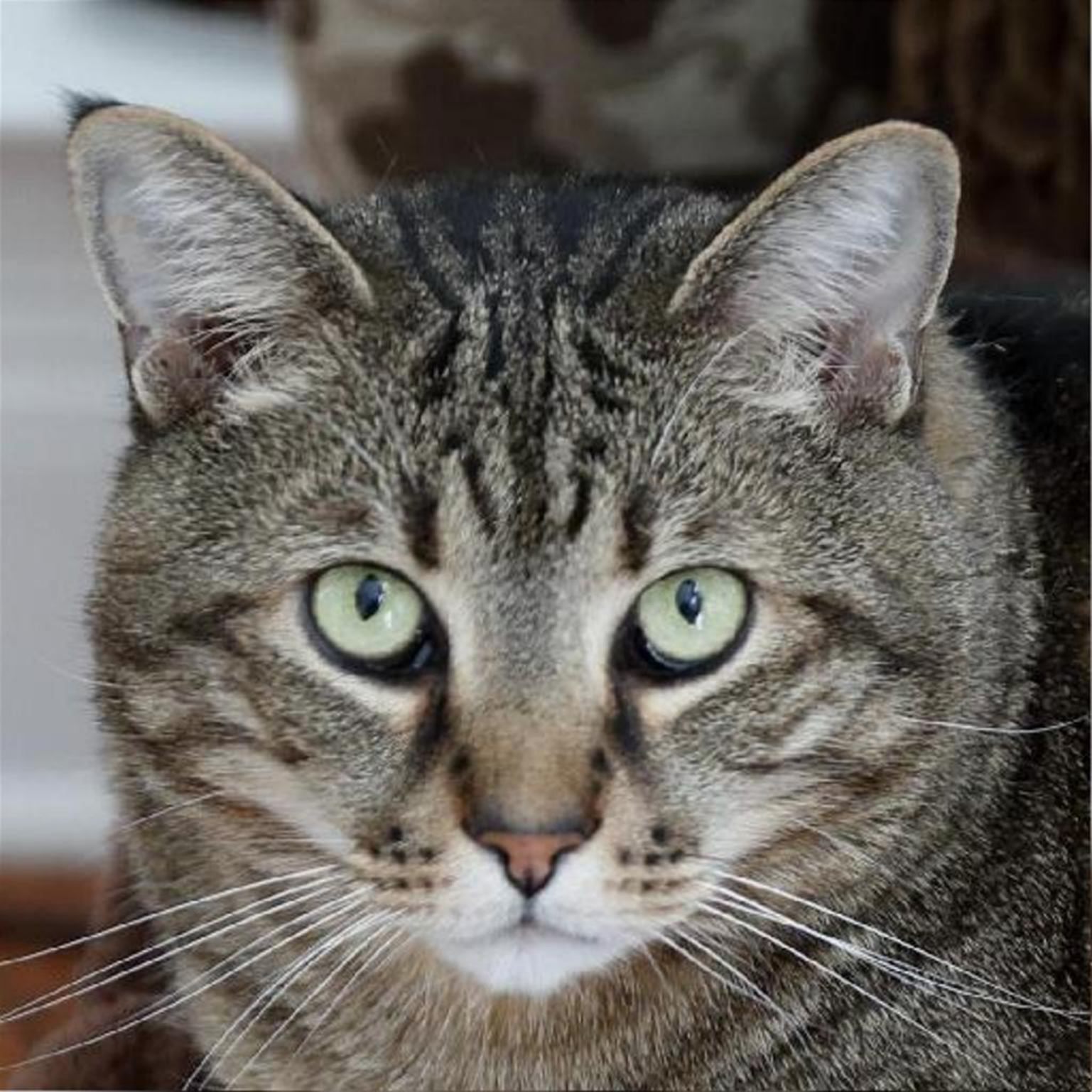} \\
        \includegraphics[width=\textwidth]{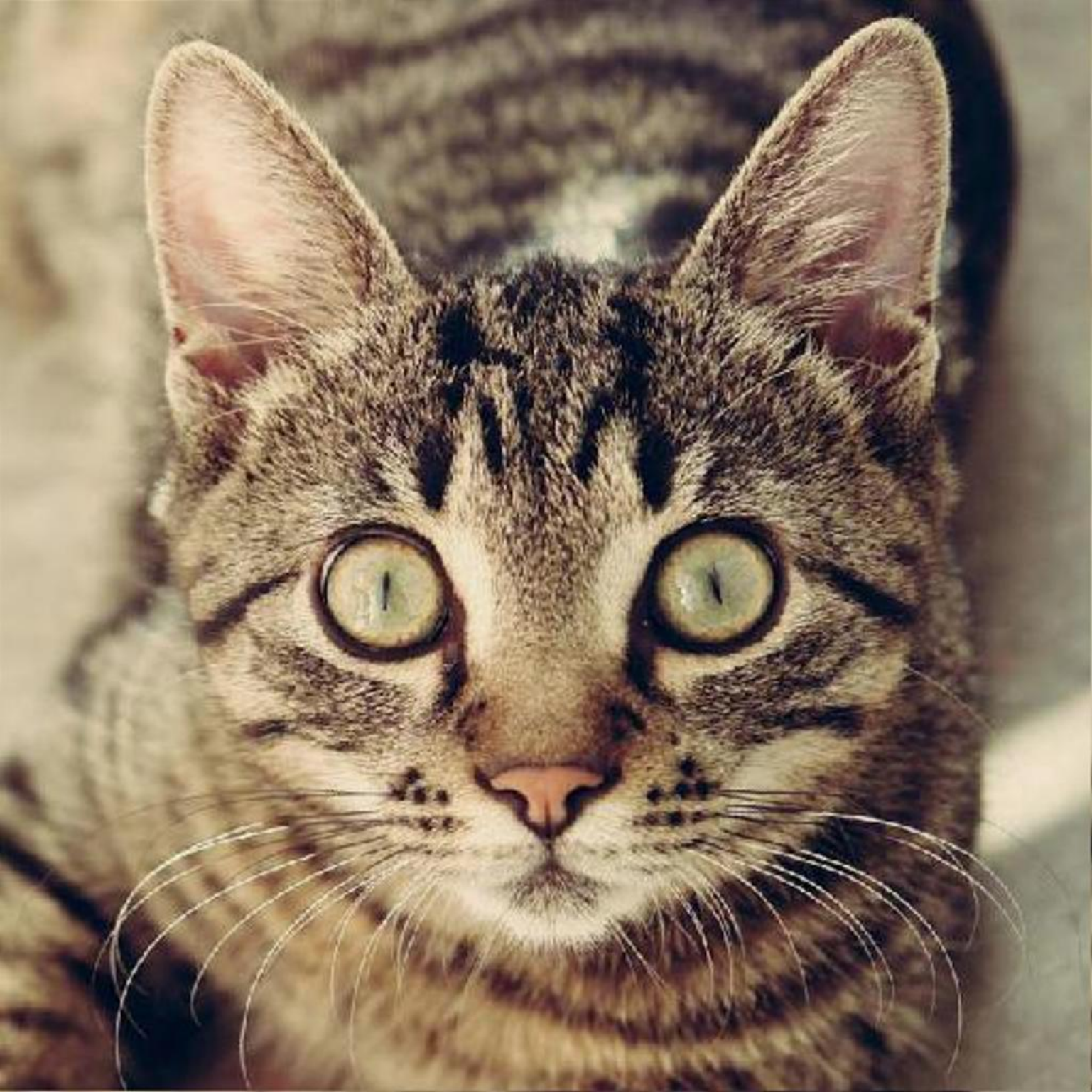}
    \end{minipage}
    \hfill
    \begin{minipage}{0.16\textwidth}
        \includegraphics[width=\textwidth]{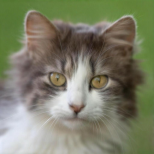} \\
        \includegraphics[width=\textwidth]{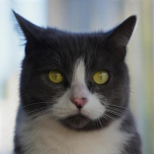} \\
        \includegraphics[width=\textwidth]{images/Tests/TReg_comparisons/restored9.pdf} \\
        \includegraphics[width=\textwidth]{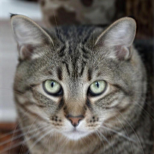} \\
        \includegraphics[width=\textwidth]{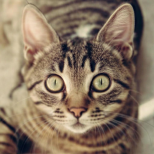}
    \end{minipage}
    \hfill
    \begin{minipage}{0.16\textwidth}
        \includegraphics[width=\textwidth]{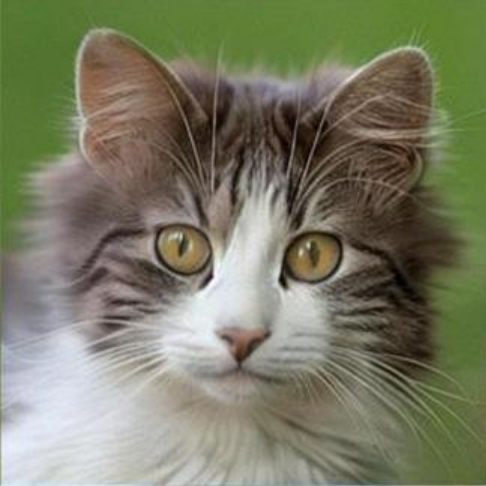}\\
        \includegraphics[width=\textwidth]{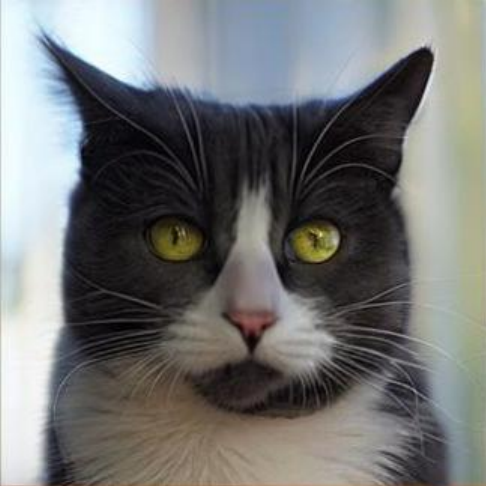}\\
        \includegraphics[width=\textwidth]{images/Tests/TReg_comparisons/cat_row3_col3.pdf}\\
        \includegraphics[width=\textwidth]{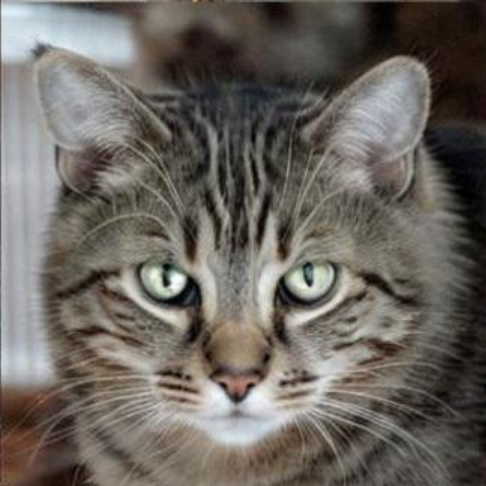}\\
        \includegraphics[width=\textwidth]{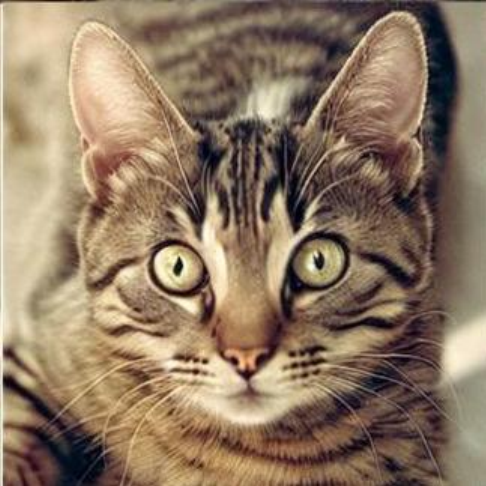}
    \end{minipage}
    \hfill
    \begin{minipage}{0.16\textwidth}
        \includegraphics[width=\textwidth]{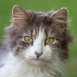}\\
        \includegraphics[width=\textwidth]{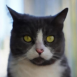}\\
        \includegraphics[width=\textwidth]{images/Tests/TReg_comparisons/cat_row3_col4.pdf}\\
        \includegraphics[width=\textwidth]{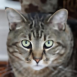}\\
        \includegraphics[width=\textwidth]{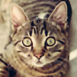}
    \end{minipage}
    \hfill
    \begin{minipage}{0.16\textwidth}
        \includegraphics[width=\textwidth]{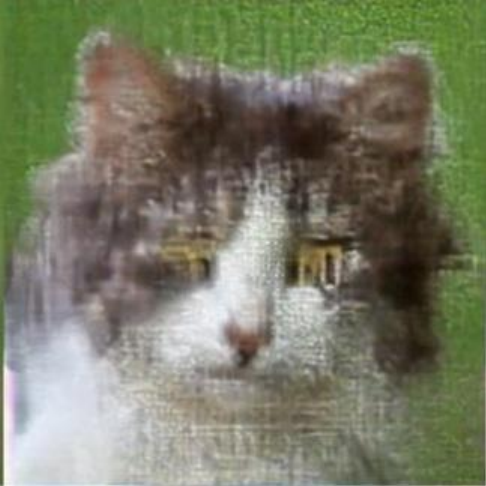}\\
        \includegraphics[width=\textwidth]{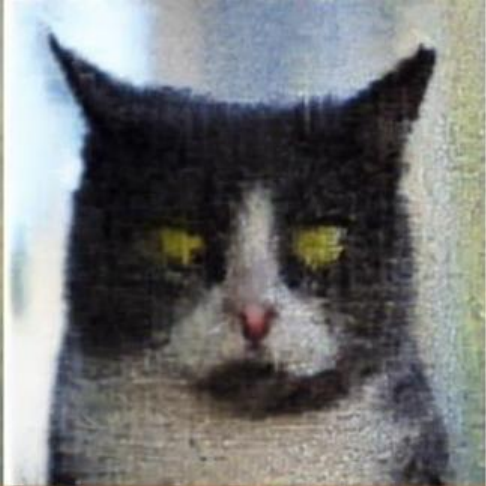}\\
        \includegraphics[width=\textwidth]{images/Tests/TReg_comparisons/cat_row3_col5.pdf}\\
        \includegraphics[width=\textwidth]{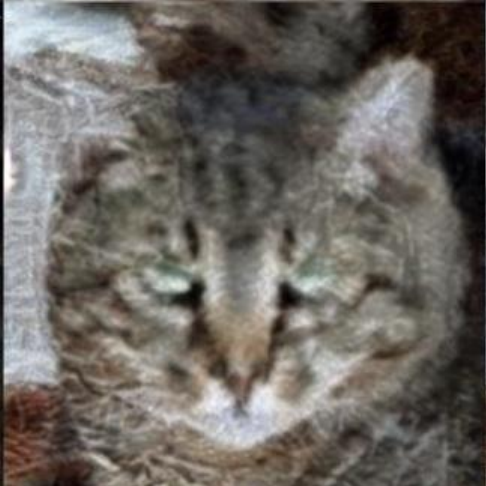}\\
        \includegraphics[width=\textwidth]{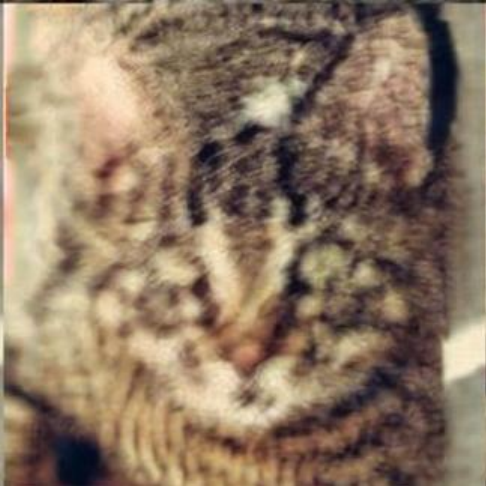}
    \end{minipage}

    \caption{Comparisons between LATINO-PRO, TReg and P2L.\vspace{1cm}}
    \label{fig:AFHQ_cats}
\end{figure*}

\begin{table*}[!h]\small
\centering 
\begin{tabular}{lccccccc}
\toprule

 & & \multicolumn{3}{c}{\textbf{Deblur (Gaussian)}} & \multicolumn{3}{c}{\textbf{SR$\times16$}} \\ 
  \cmidrule(lr){3-5} \cmidrule(lr){6-8}
  \textbf{Method} &  \textbf{NFE↓} & \textbf{FID↓} & \textbf{PSNR↑} & \textbf{LPIPS↓} & \textbf{FID↓} & \textbf{PSNR↑} & \textbf{LPIPS↓} \\ \hline
\textbf{LATINO-PRO}  & \underline{68} & \textbf{24.56} & \textbf{27.77} & \textbf{0.347} &  \textbf{39.85}   & \textbf{22.28} & \textbf{0.463} \\ 
\textbf{LATINO}   & \textbf{8}   & \underline{28.02}   & 27.21   & \underline{0.360} & \underline{59.03} & 20.84 & \underline{0.496} \\ \hline
\textbf{LATINO-PRO "a cat"}   & \underline{68}  & 48.93     & \underline{27.42}   & 0.382 & 125.63      & \underline{22.23} & 0.502 \\
\textbf{LATINO "a cat"}   & \textbf{8}  & 82.45    & 26.34   & 0.420 & 190.04         & 20.09    & 0.547             \\ \hline
LDPS   & 1000  & 81.18 &  24.86  & 0.502  & 154.3    & 18.26  & 0.667    \\ 
PSLD \cite{Rout2023}  & 1000 & 41.04  & 26.12 & 0.455  & 92.35  &  22.20  & 0.585 \\ \bottomrule
\end{tabular}
\caption{Results for Gaussian Deblurring with $\sigma = 5.0$, and $\times 16$ super-resolution, both with noise $\sigma_y = 0.01$ on the AFHQ-512 dogs val dataset. Base prompt when not specified: \texttt{a sharp photo of a dog}. \textbf{Bold}:  best, \underline{underline}: second best.\vspace{.3cm}}
\label{tab:512_comparison_AFHQ_prompt}
\end{table*}

{

\begin{figure*}[!h]
    \centering
    \makebox[0.19\textwidth]{\textbf{Prior sample at Step 0}}
    \hfill
    \makebox[0.19\textwidth]{\textbf{Optimized prior sample}}
    \hfill
    \makebox[0.19\textwidth]{\textbf{GT}}
    \hfill
    \makebox[0.19\textwidth]{\textbf{Measurement}}
    \hfill
    \makebox[0.19\textwidth]{\textbf{Restored}}
    
    \vspace{2mm}

    \begin{minipage}{0.19\textwidth}
        \includegraphics[width=\textwidth]{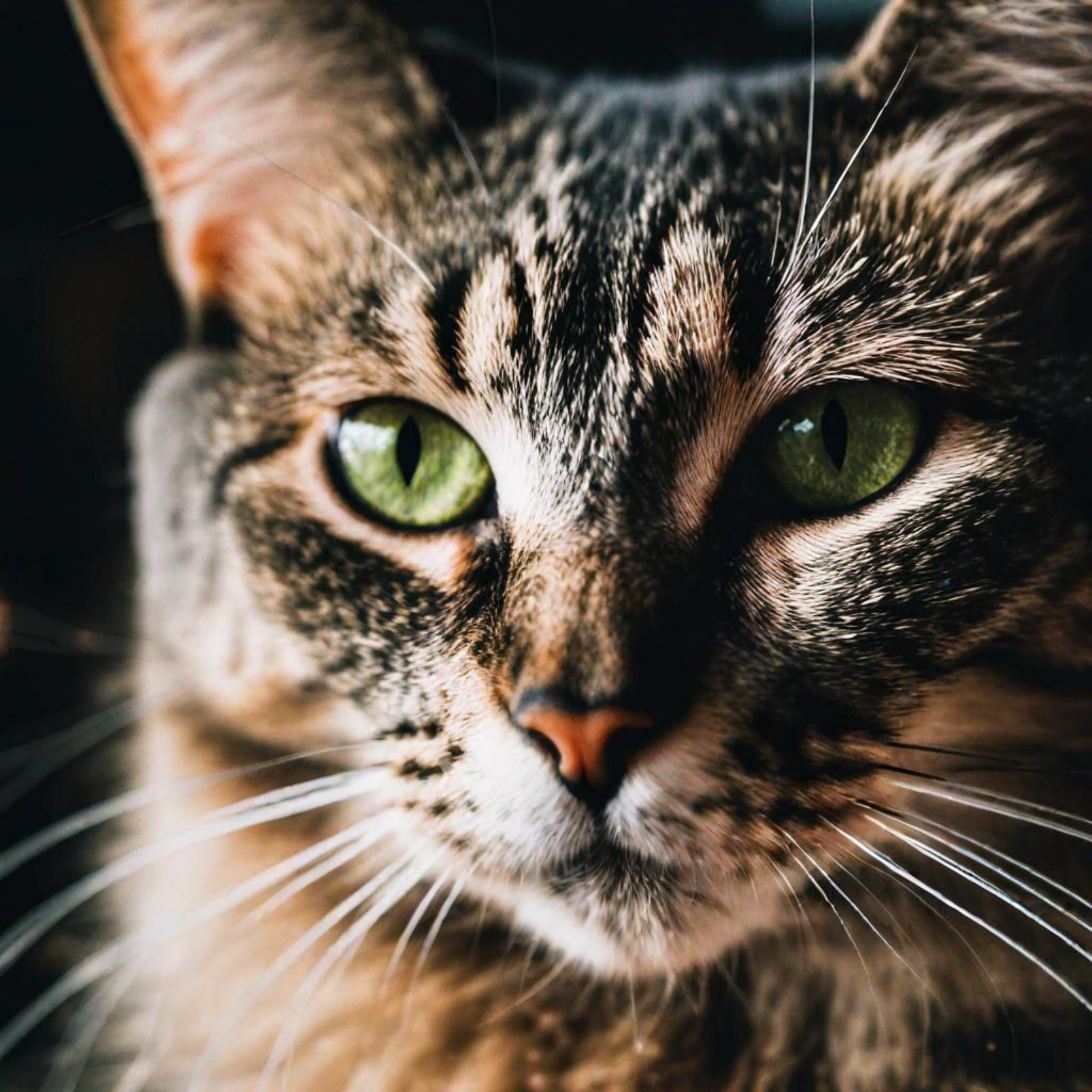}
    \end{minipage}
    \hfill
    \begin{minipage}{0.19\textwidth}
        \includegraphics[width=\textwidth]{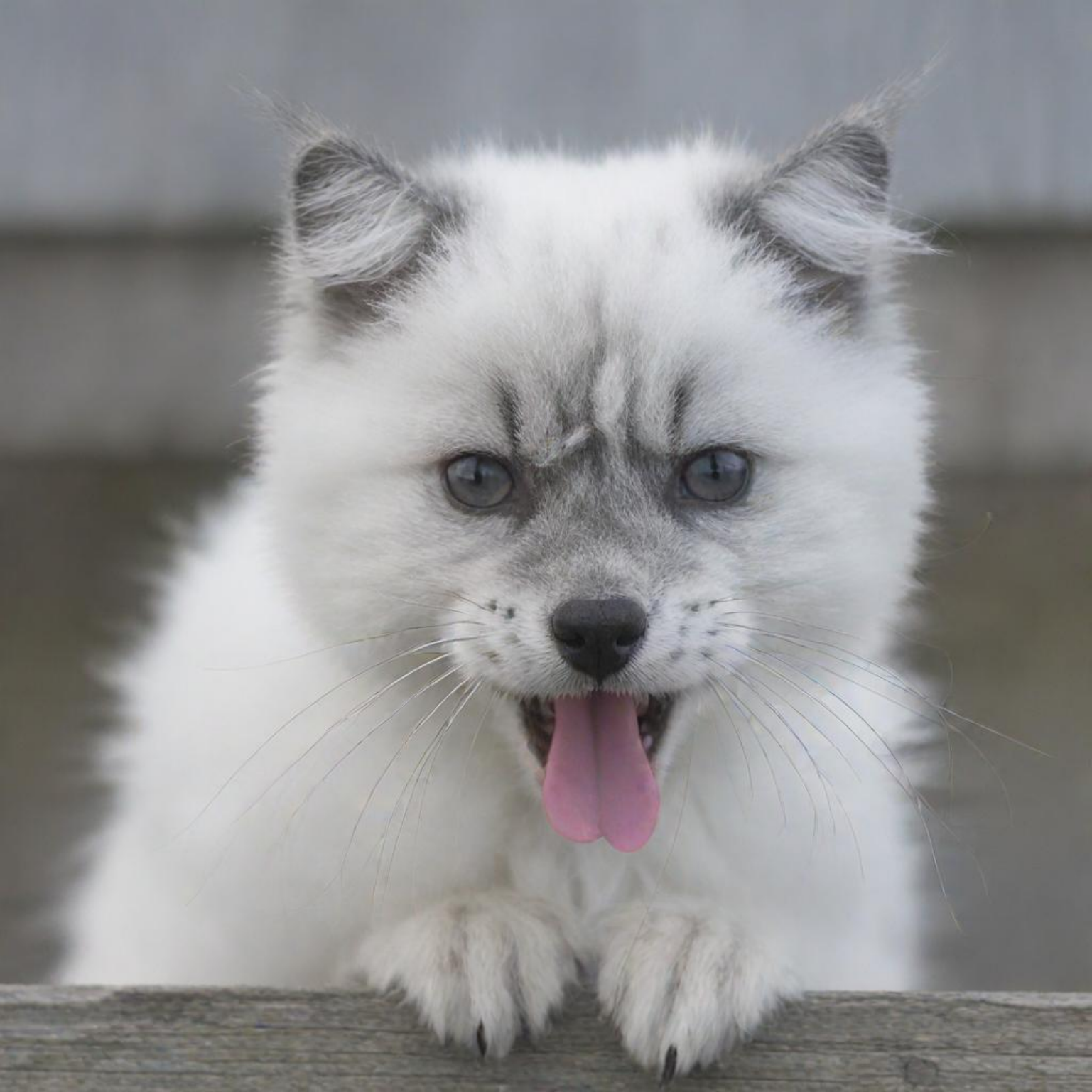}
    \end{minipage}
    \hfill
    \begin{minipage}{0.19\textwidth}
        \includegraphics[width=\textwidth]{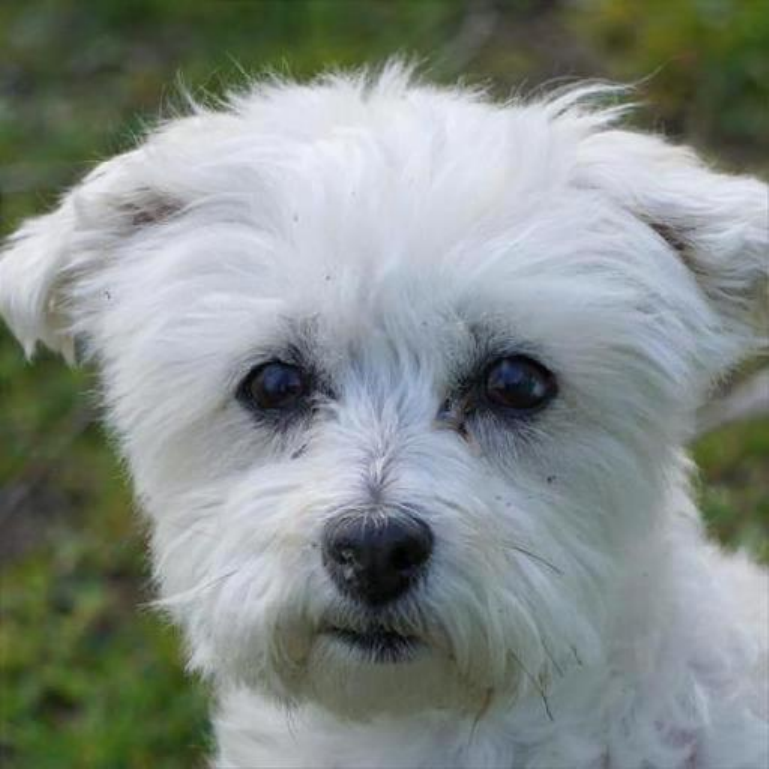}
    \end{minipage}
    \hfill
    \begin{minipage}{0.19\textwidth}
        \includegraphics[width=\textwidth]{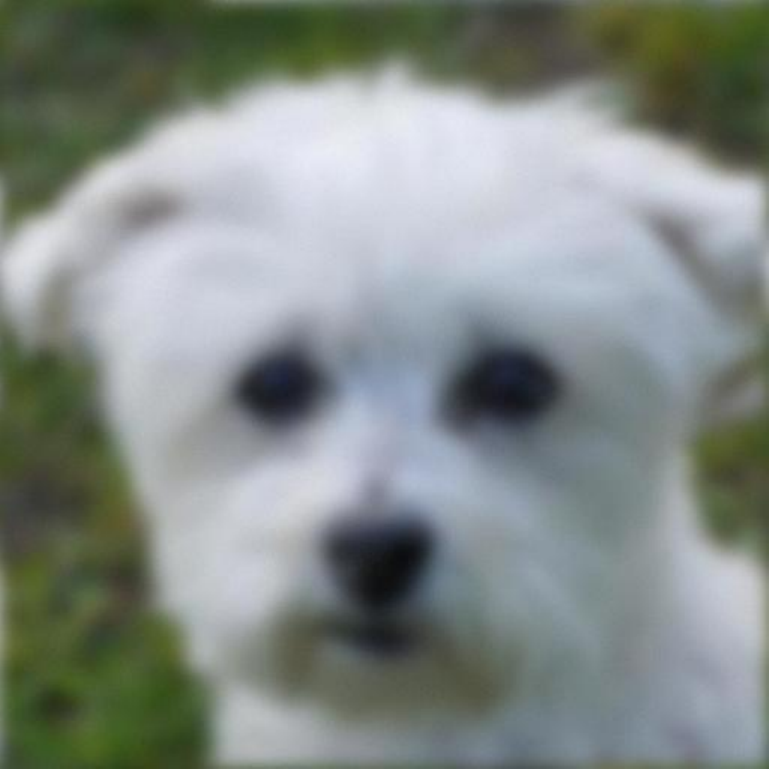}
    \end{minipage}
    \hfill
    \begin{minipage}{0.19\textwidth}
        \includegraphics[width=\textwidth]{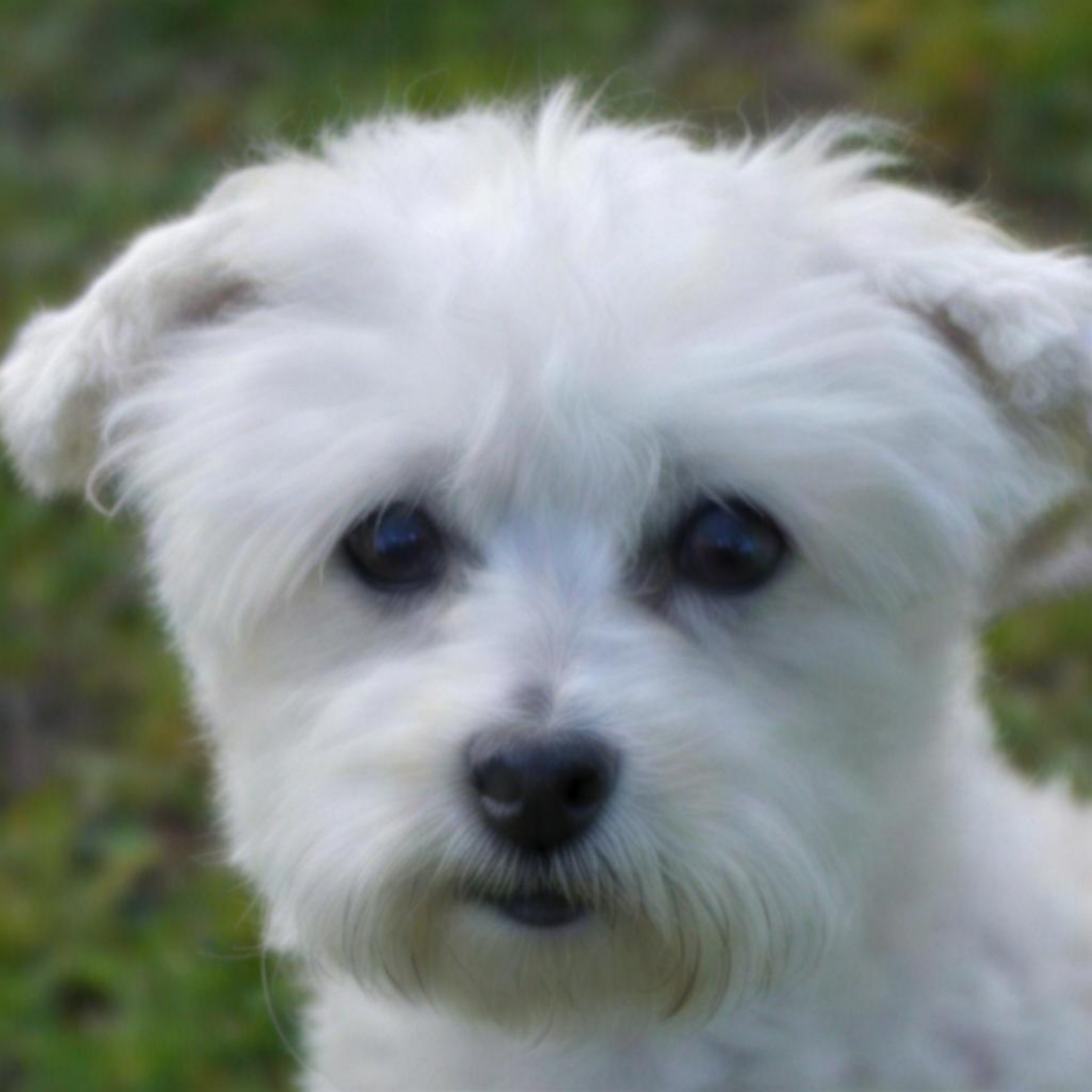}
    \end{minipage}

    \caption{Effect of prompt optimization on the AFHQ-dogs val dataset. Initial prompt: \texttt{a sharp photo of a cat}.\vspace{-0.2cm}}
    \label{fig:AFHQ_cat_dog}
\end{figure*}

\begin{figure*}[!h]
    \centering
    \begin{minipage}{0.47\textwidth}
        \includegraphics[width=\textwidth]{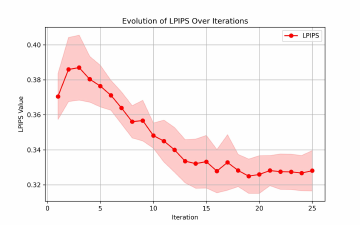}
    \end{minipage}
    \hfill
    \begin{minipage}{0.47\textwidth}
        \includegraphics[width=\textwidth]{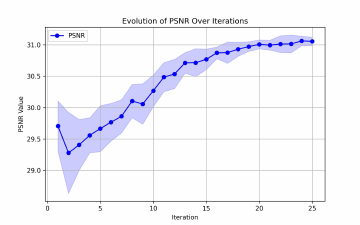}
    \end{minipage}
\vspace*{-0.2cm}
    \caption{Metrics evolution during LATINO-PRO iterations for the example in Figure \ref{fig:AFHQ_cat_dog}. Initial prompt: \texttt{a sharp photo of a cat}.\vspace{-0.2cm}}
    \label{fig:AFHQ_cat_dog_metrics}
\end{figure*}

\begin{figure*}[!h]
    \centering
    \makebox[0.19\textwidth]{\textbf{Prior sample at Step 0}}
    \hfill
    \makebox[0.19\textwidth]{\textbf{Optimized prior sample}}
    \hfill
    \makebox[0.19\textwidth]{\textbf{GT}}
    \hfill
    \makebox[0.19\textwidth]{\textbf{Measurement}}
    \hfill
    \makebox[0.19\textwidth]{\textbf{Restored}}
    
    \vspace{2mm}

    \begin{minipage}{0.19\textwidth}
        \includegraphics[width=\textwidth]{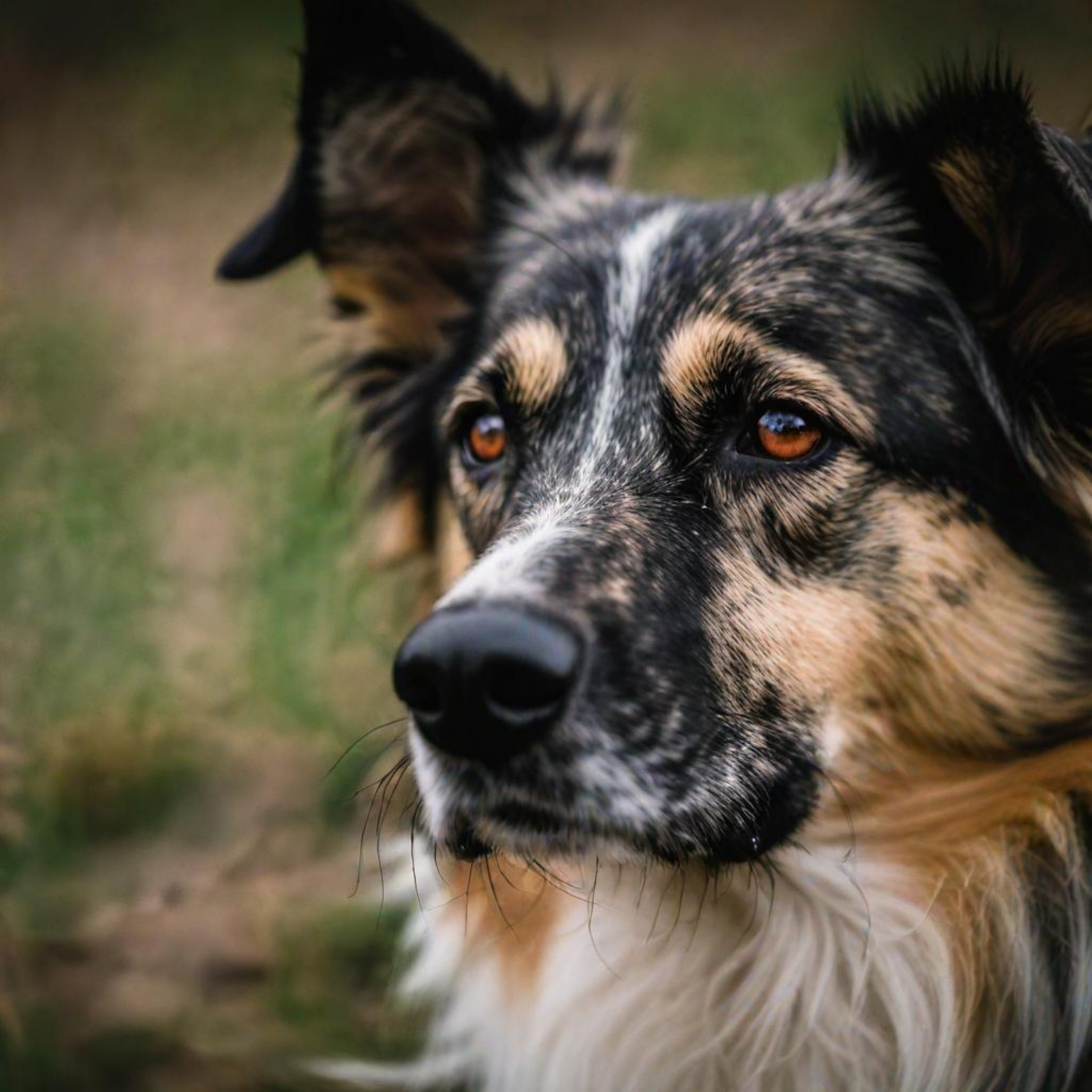}\\
        \includegraphics[width=\textwidth]{images/Tests/Deblurring_3.0_0.01_512_prompt/prior_0.pdf}\\
        \includegraphics[width=\textwidth]{images/Tests/Deblurring_3.0_0.01_512_prompt/prior_0.pdf}
    \end{minipage}
    \hfill
    \begin{minipage}{0.19\textwidth}
        \includegraphics[width=\textwidth]{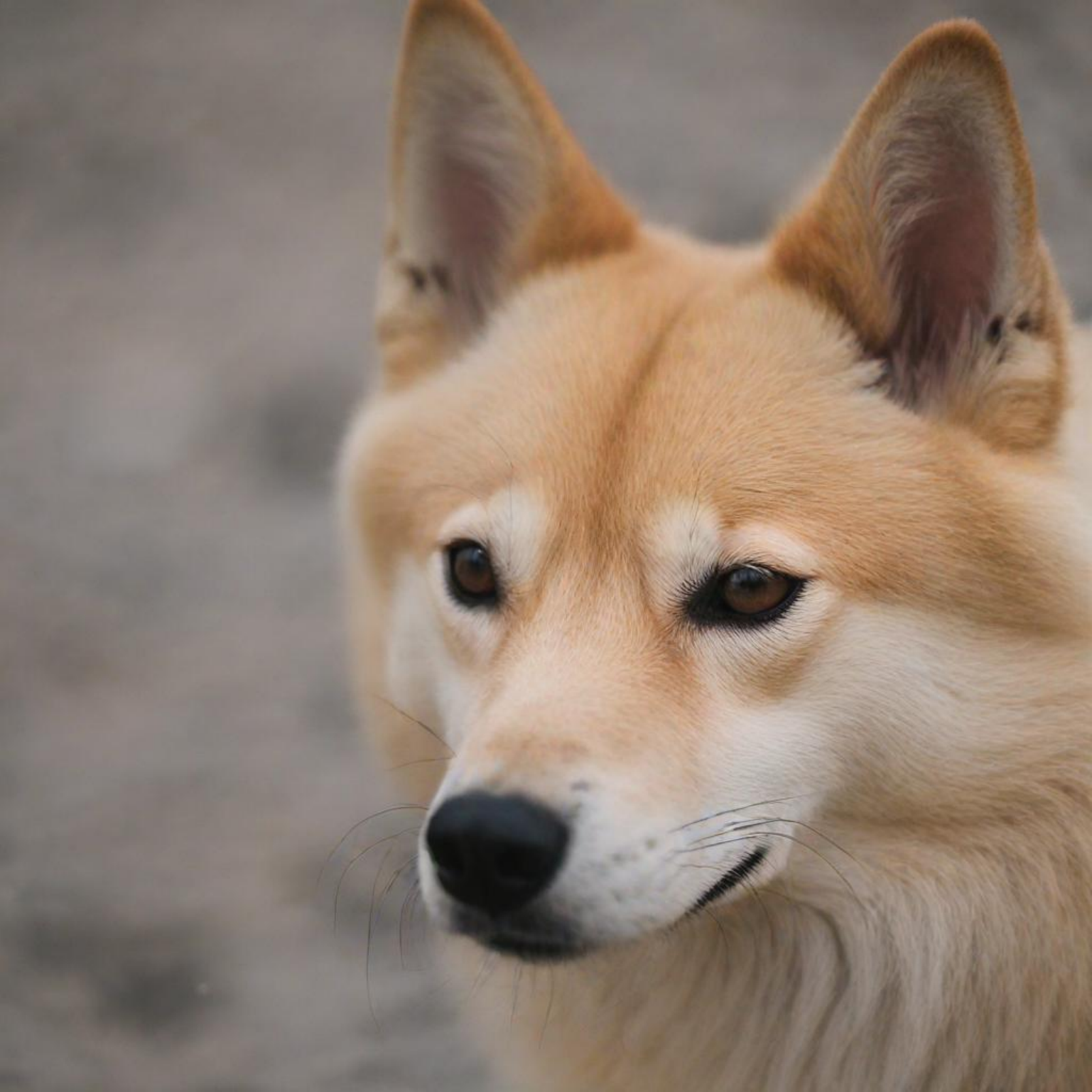} \\
        \includegraphics[width=\textwidth]{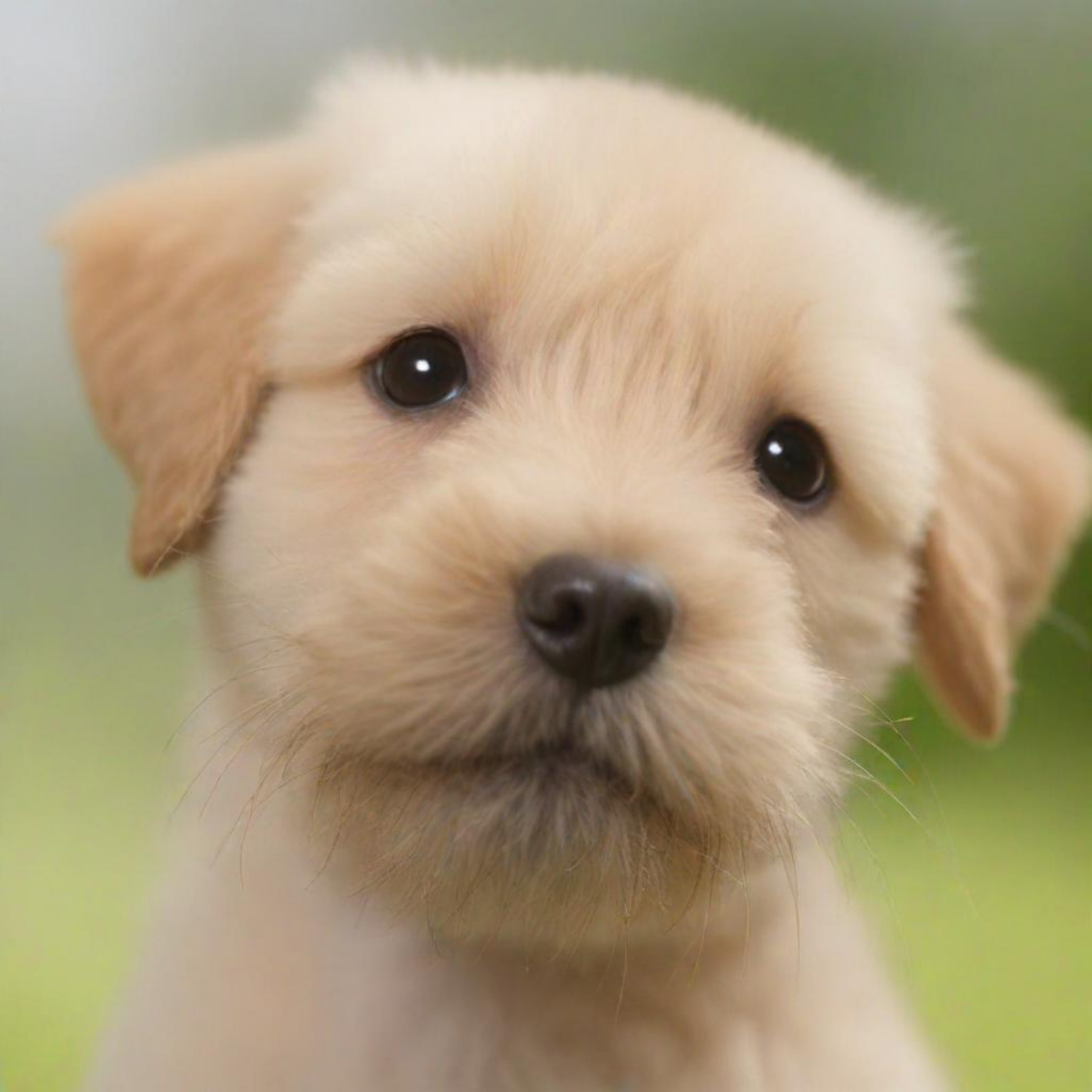} \\
        \includegraphics[width=\textwidth]{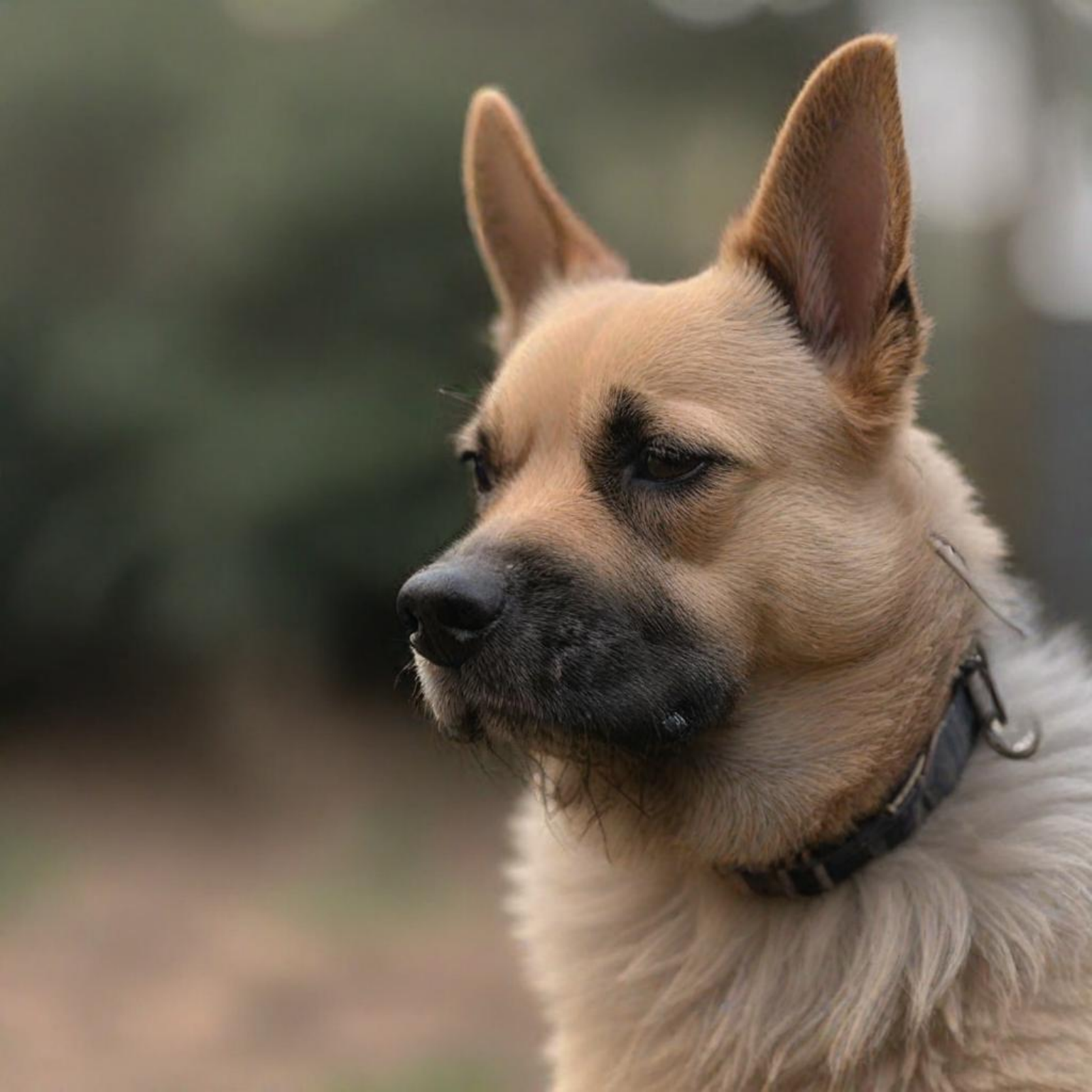}
    \end{minipage}
    \hfill
    \begin{minipage}{0.19\textwidth}
        \includegraphics[width=\textwidth]{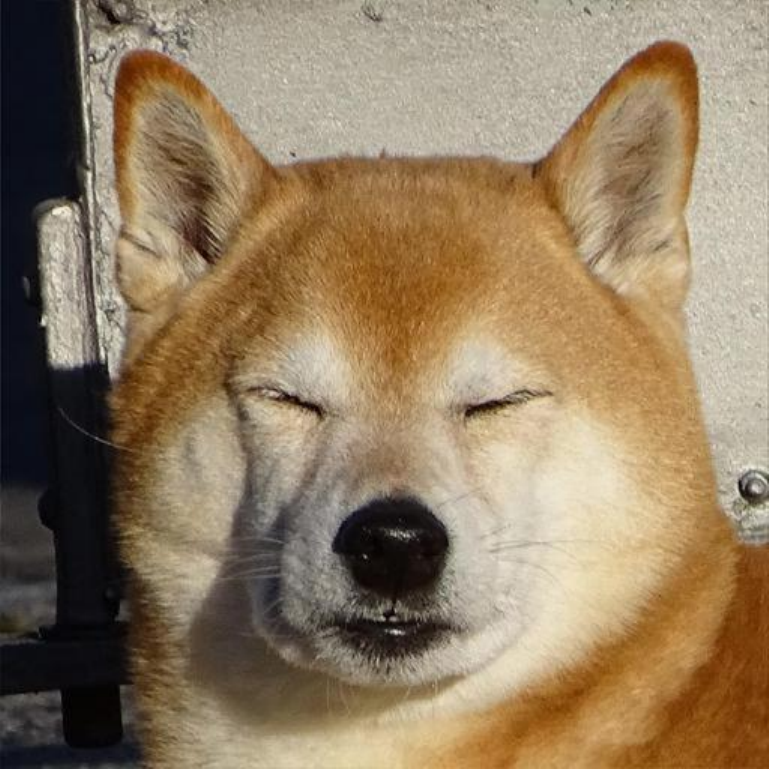} \\
        \includegraphics[width=\textwidth]{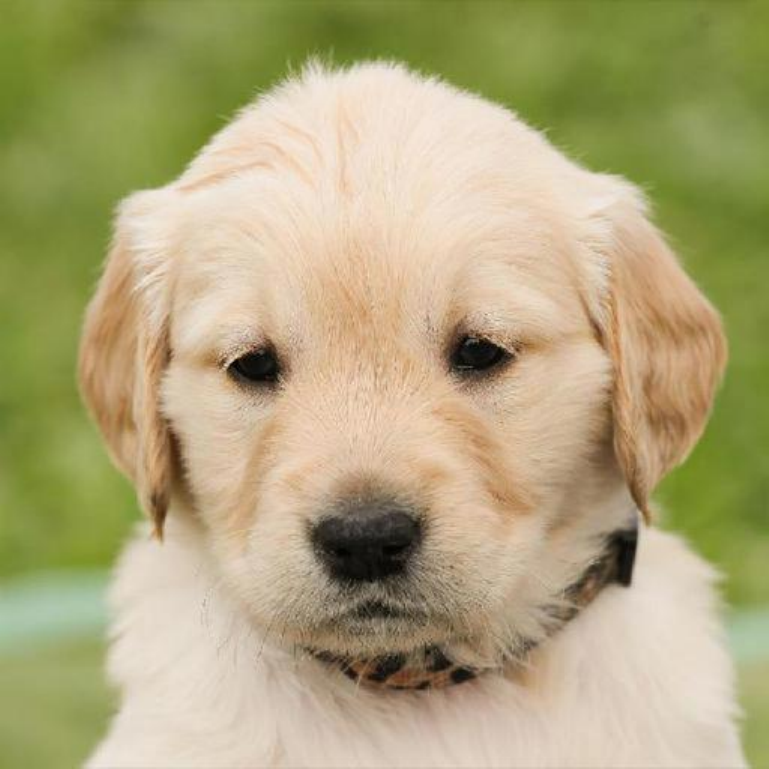} \\
        \includegraphics[width=\textwidth]{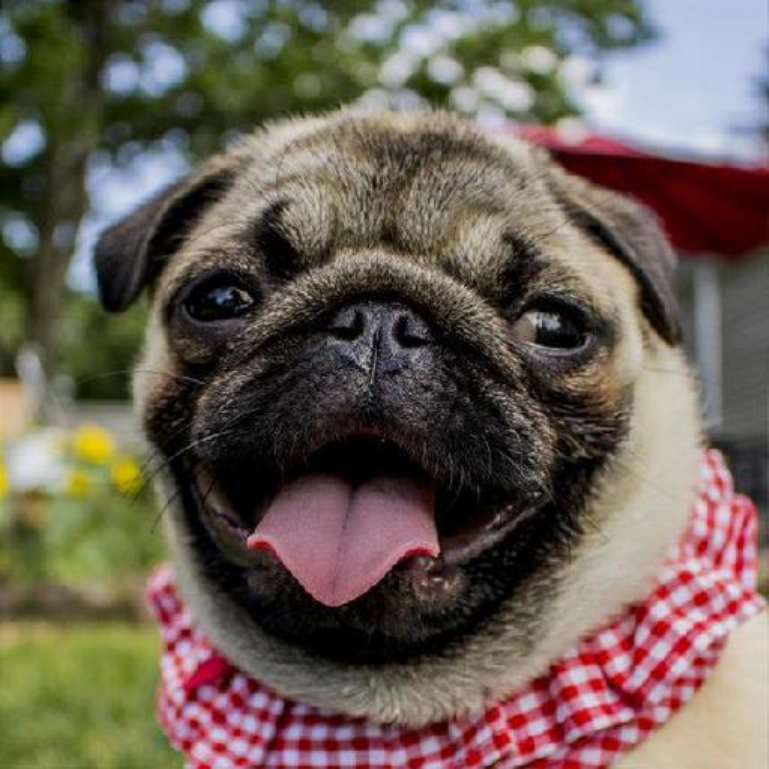}
    \end{minipage}
    \hfill
    \begin{minipage}{0.19\textwidth}
        \includegraphics[width=\textwidth]{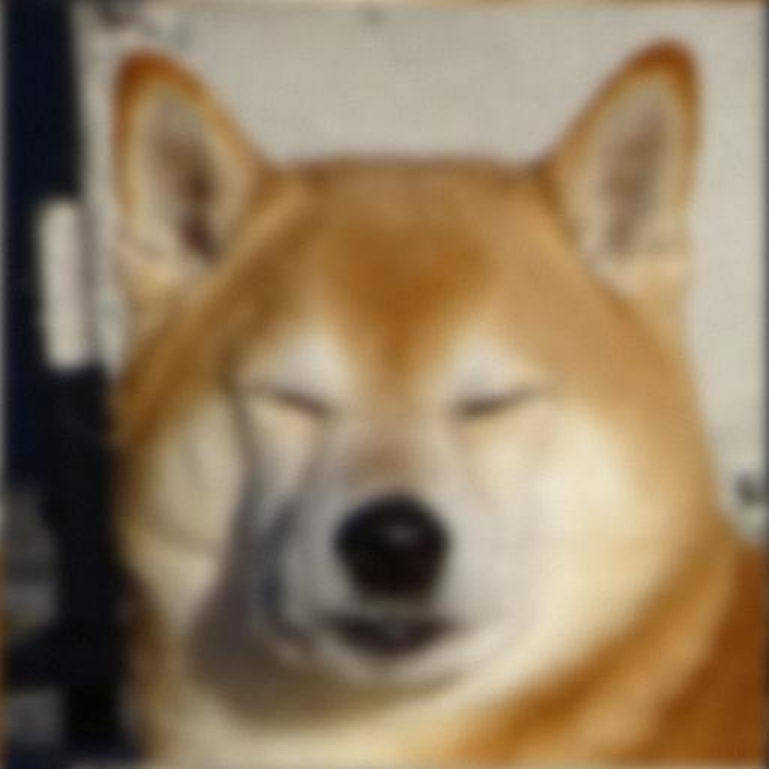} \\
        \includegraphics[width=\textwidth]{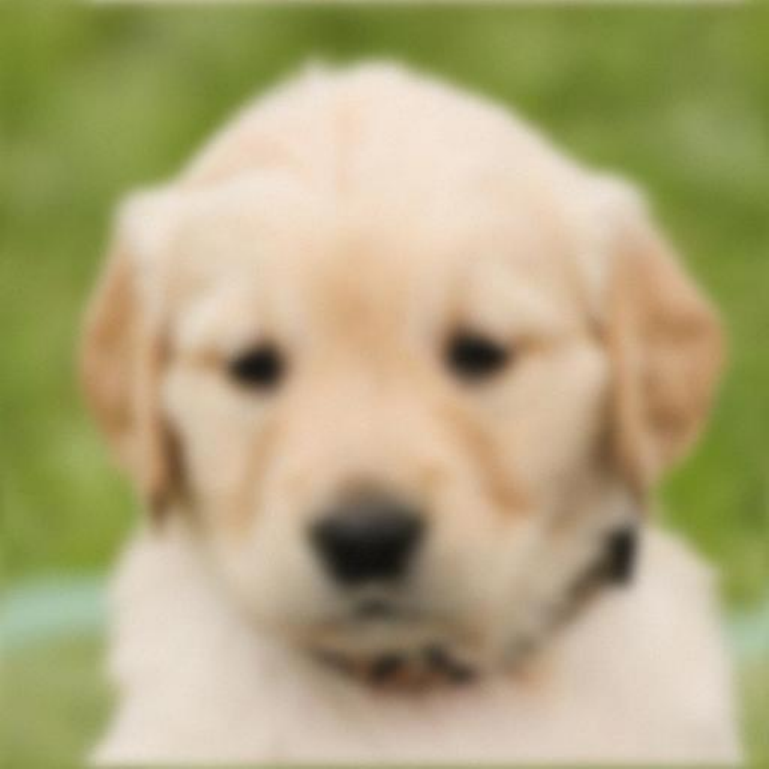} \\
        \includegraphics[width=\textwidth]{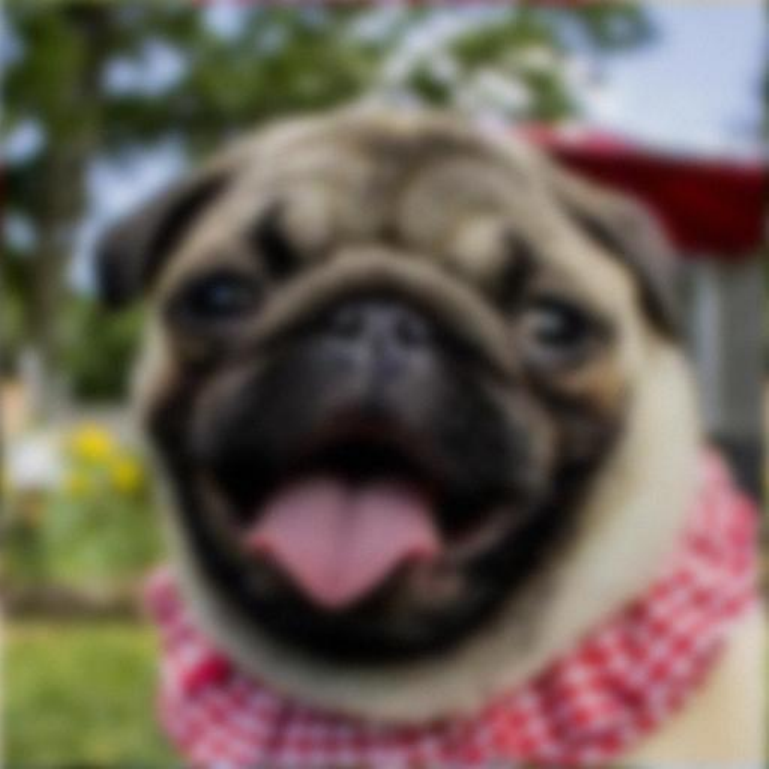}
    \end{minipage}
    \hfill
    \begin{minipage}{0.19\textwidth}
        \includegraphics[width=\textwidth]{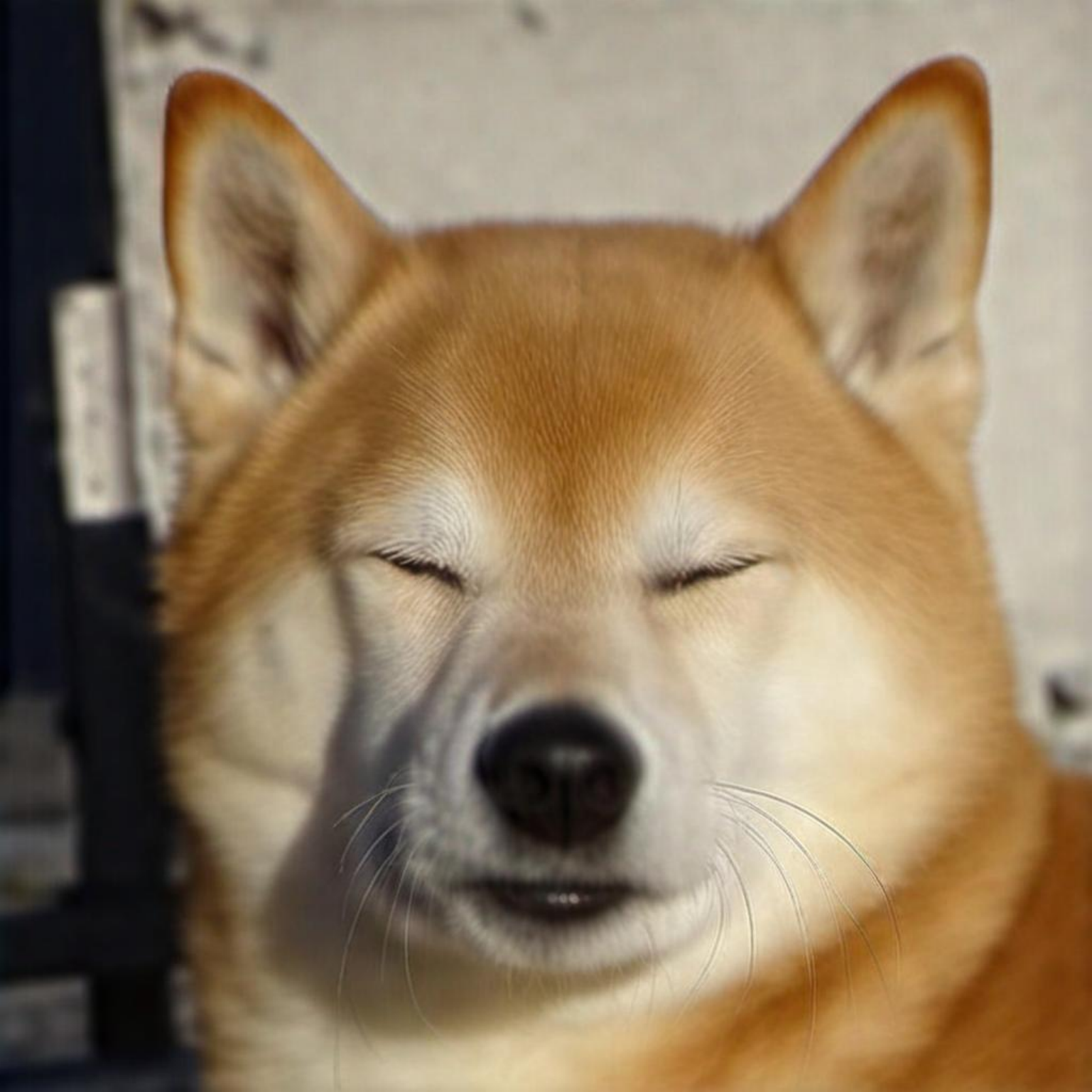} \\
        \includegraphics[width=\textwidth]{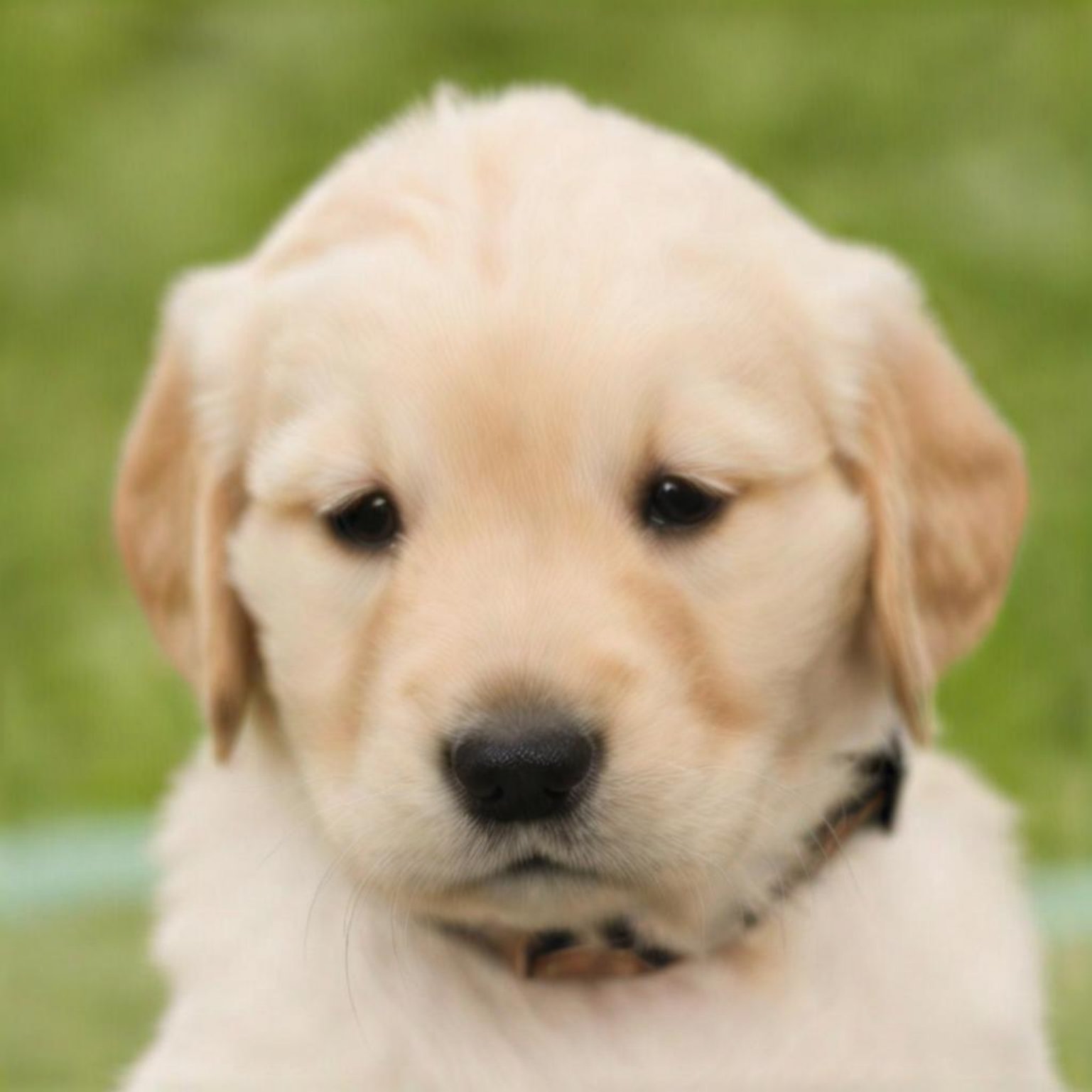} \\
        \includegraphics[width=\textwidth]{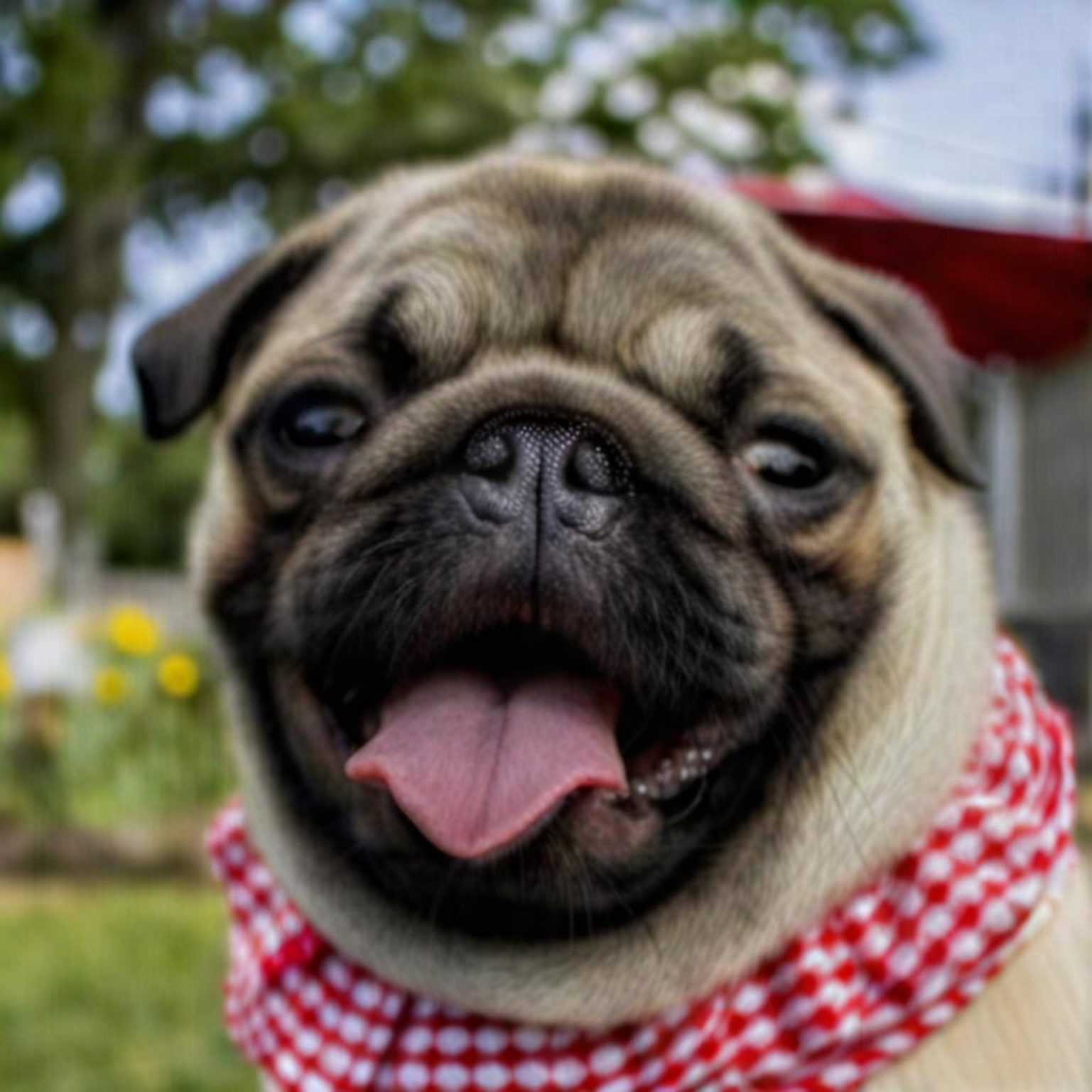}
    \end{minipage}
\vspace*{-0.2cm}
    \caption{Effect of prompt optimization on the AFHQ-dogs val dataset. Initial prompt: \texttt{a sharp photo of a dog}.}
    \label{fig:AFHQ_dogs_prompt}
\end{figure*}

}

\section{Prompt Tuning: experimental results}
\label{sec:prompt-exp}

In the Experimental section \ref{sec:experiments}, we explored prompt tuning as a way to improve the reconstructions when the prompt is already aligned with the image semantic (e.g., \texttt{a sharp photo of a face} for the FFHQ dataset). We will now show the capabilities of LATINO-PRO to significantly improve the reconstructions in cases where the given prompt is not aligned. \\

In Table \ref{tab:512_comparison_AFHQ_prompt}, we can see how when we try to reconstruct images of dogs with the prompt \texttt{a sharp photo of a cat}, the PSNR, SSIM and LPIPS metrics are much worse than with the prompt \texttt{a sharp photo of a dog}. However, LATINO-PRO brings the results closer to the optimal case. In particular, we can appreciate the effectiveness in the Gaussian deblurring case, while the SR $\times16$ seems harder to re-align. This is in fact in accordance with what we would expect, since the amount of information contained in the degraded observation $\vy$ is small, and thus the prior has more influence on the reconstruction. In practice, a high-resolution image of a cat could be compatible with a low-resolution image of a dog, and thus the prompt is not able to learn the actual ground truth. \\

In Figure \ref{fig:AFHQ_cat_dog} we can see how the prior can learn from the measurement the main features. The semantic shift from cat $\rightarrow$ dog is appreciable, as well as how the main colors are learned by the prior. Figure \ref{fig:AFHQ_cat_dog_metrics} shows the trend of LPIPS and PSNR during the LATINO-PRO iterations, averaged over 20 different seeds. We decided to run the algorithm for $25$ steps to show why it is preferable to early-stop it at around $15-20$ steps, the LPIPS metric indeed tends to usually rise after this interval, and the PSNR does not show any significant improvement. 

A similar experiment has been conducted on less extreme cases, as shown in Figure \ref{fig:AFHQ_dogs_prompt}, where the prompt given was \texttt{a sharp photo of a dog}. We can observe the ability of the prior to learn characteristics like the breed of the dog or whether it is a puppy or an adult.

\section{Inpainting}
\label{sec:Inpainting}

While in Section \ref{sec:experiments}, we focused on the deblurring and super-resolution tasks, we here apply our model to the inpainting case. We consider box impainting tasks where we cover the eyes of the animals in the AFHQ-512 dataset and both the eyes and the mouth for the faces in the FFHQ-512 dataset, as done in TReg \cite{Kim2023RegularizationBT}. We show the FID and PSNR metrics for the FFHQ and AFHQ 1k validation datasets in Table \ref{tab:box_inpainting_comparison_FFHQ}.

In Figure \ref{fig:qualitative_comparison_FFHQ_box} and Figure \ref{fig:qualitative_comparison_AFHQ_box}, we show a comparison of the available methods. The entries in the table are those advertised in the paper \cite{Kim2023RegularizationBT}. The last rows were obtained using the LATINO-PRO model, giving as prompt \texttt{a photo of} + \texttt{a face} or \texttt{a dog} + the specific caption. The results can be interpreted in the following way: the reduced number of steps of LATINO makes the inpainting task more challenging, especially for more complex images such as faces.
 The prompt optimization done through SAPG helps to mitigate this phenomenon with better visual results, especially for the AFHQ case. The high performance of the proposed method on the averaged metrics shows that similar problems are present in current SOTA methods.

\begin{table}[!h]
\centering \small
\begin{tabular}{lcccc}
\toprule
& \multicolumn{2}{c}{\textbf{FFHQ}} & \multicolumn{2}{c}{\textbf{AFHQ}} \\ 
\cmidrule(lr){2-3} \cmidrule(lr){4-5}
\textbf{Method} & \textbf{FID↓} & \textbf{PSNR↑} & \textbf{FID↓} & \textbf{PSNR↑} \\
\hline
\textbf{LATINO-PRO} & 67.79 & \textbf{20.55} & \textbf{19.91} & \textbf{21.05} \\

\textbf{LATINO} & 87.78 & \underline{20.01}  & \underline{27.01} & \underline{19.92} \\
\hline
P2L \cite{Chung2023PrompttuningLD} & 85.32 & 16.84 & 138.4 & 16.07 \\

TReg \cite{Kim2023RegularizationBT} & \underline{66.93} & 19.95 & 51.97 & 17.39 \\

PSLD \cite{Rout2023} & \textbf{60.97} & 19.76 & 104.7 & 16.93 \\
\bottomrule\vspace{-0.3cm}
\end{tabular}
\caption{Box Inpainting results on FFHQ (left) and AFHQ (right). \textbf{Bold}:  best, \underline{underline}: second best.\vspace{-0.3cm}}
\label{tab:box_inpainting_comparison_FFHQ}
\end{table}
\begin{figure}[!h]
\centering
\begin{minipage}{0.15\textwidth}
    \centering \textbf{Measurement} \\ 
    \includegraphics[width=\textwidth]{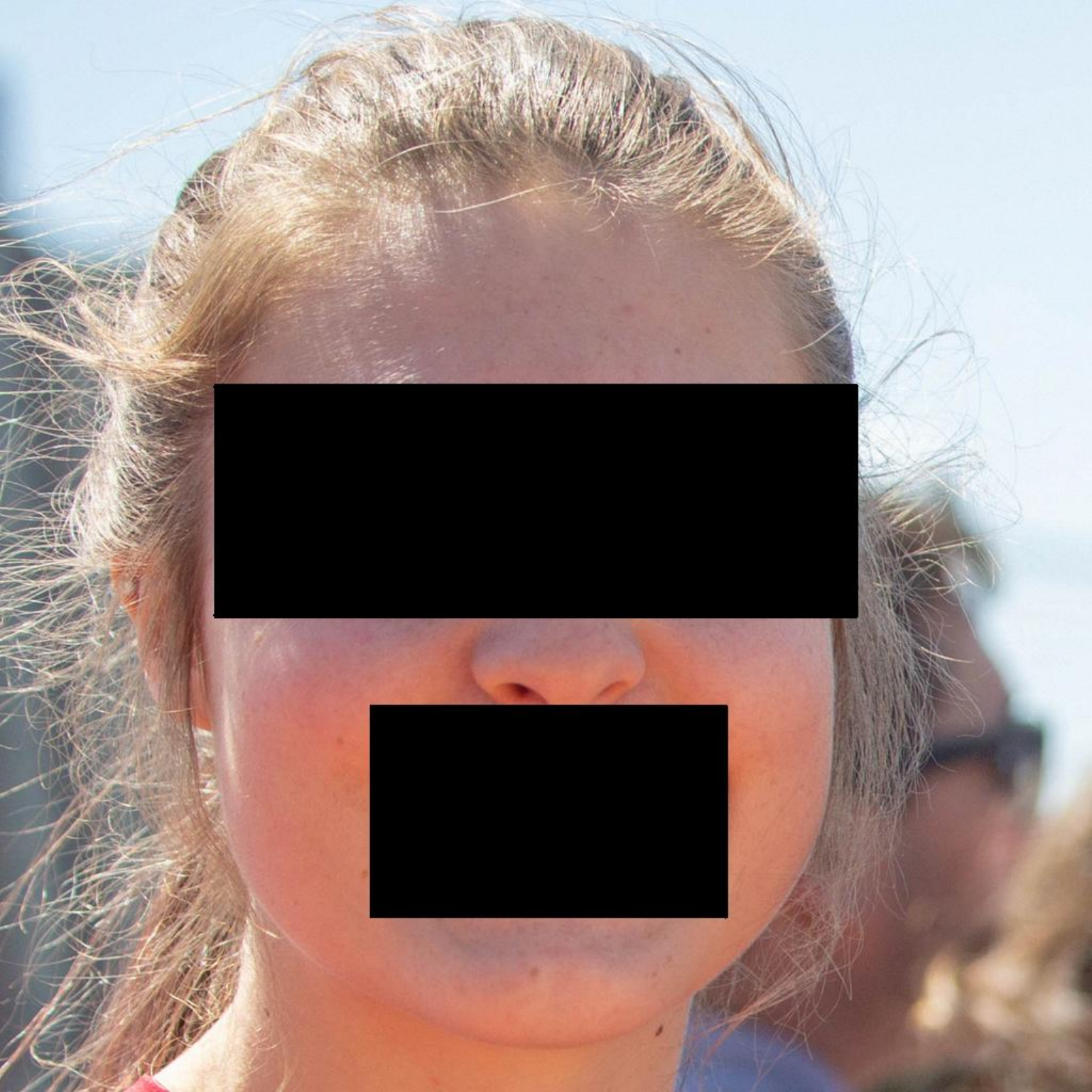}\\
    \centering \textbf{LATINO-PRO} \\ 
    \includegraphics[width=\textwidth]{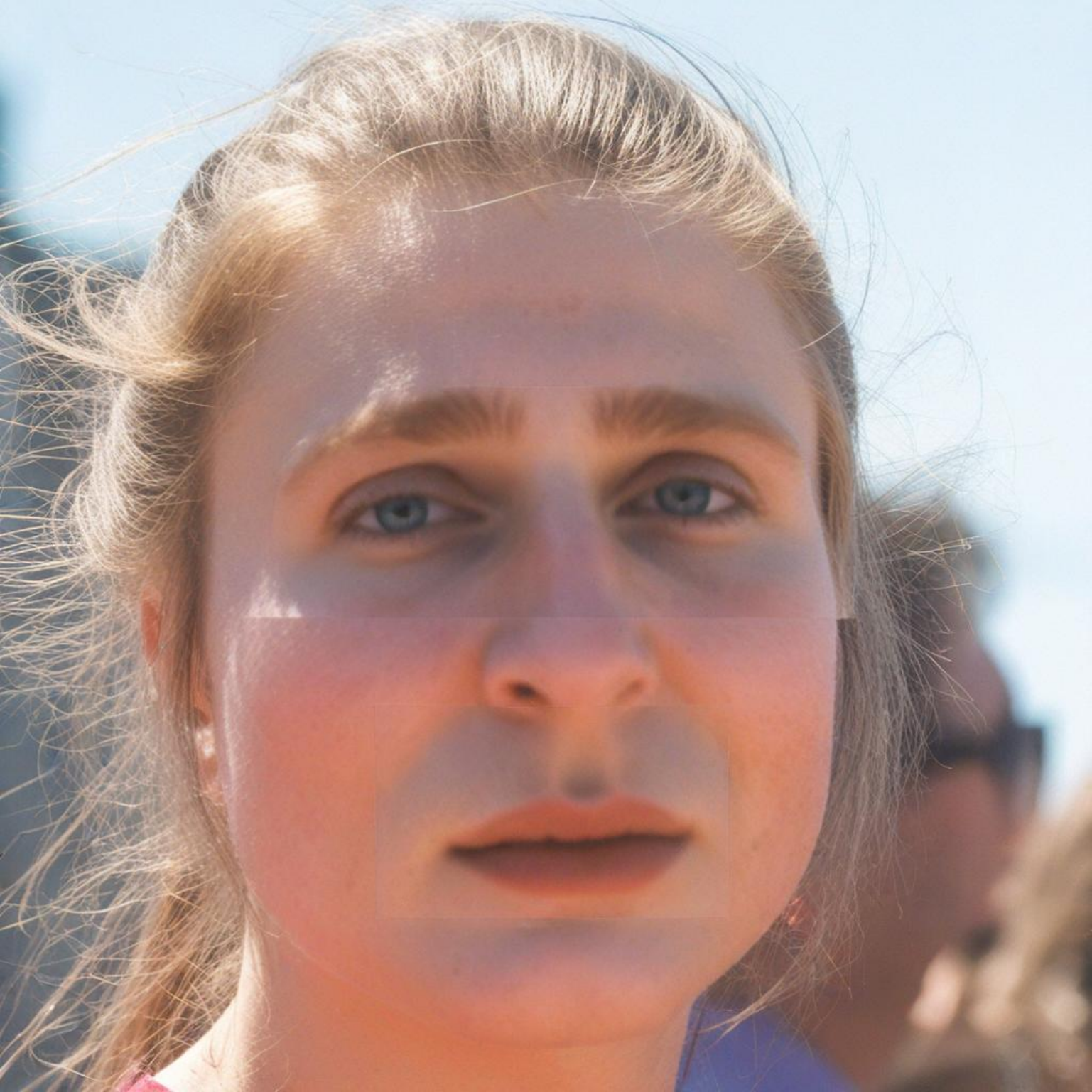}\\
    \centering \textbf{"sad"} \\ 
    \includegraphics[width=\textwidth]{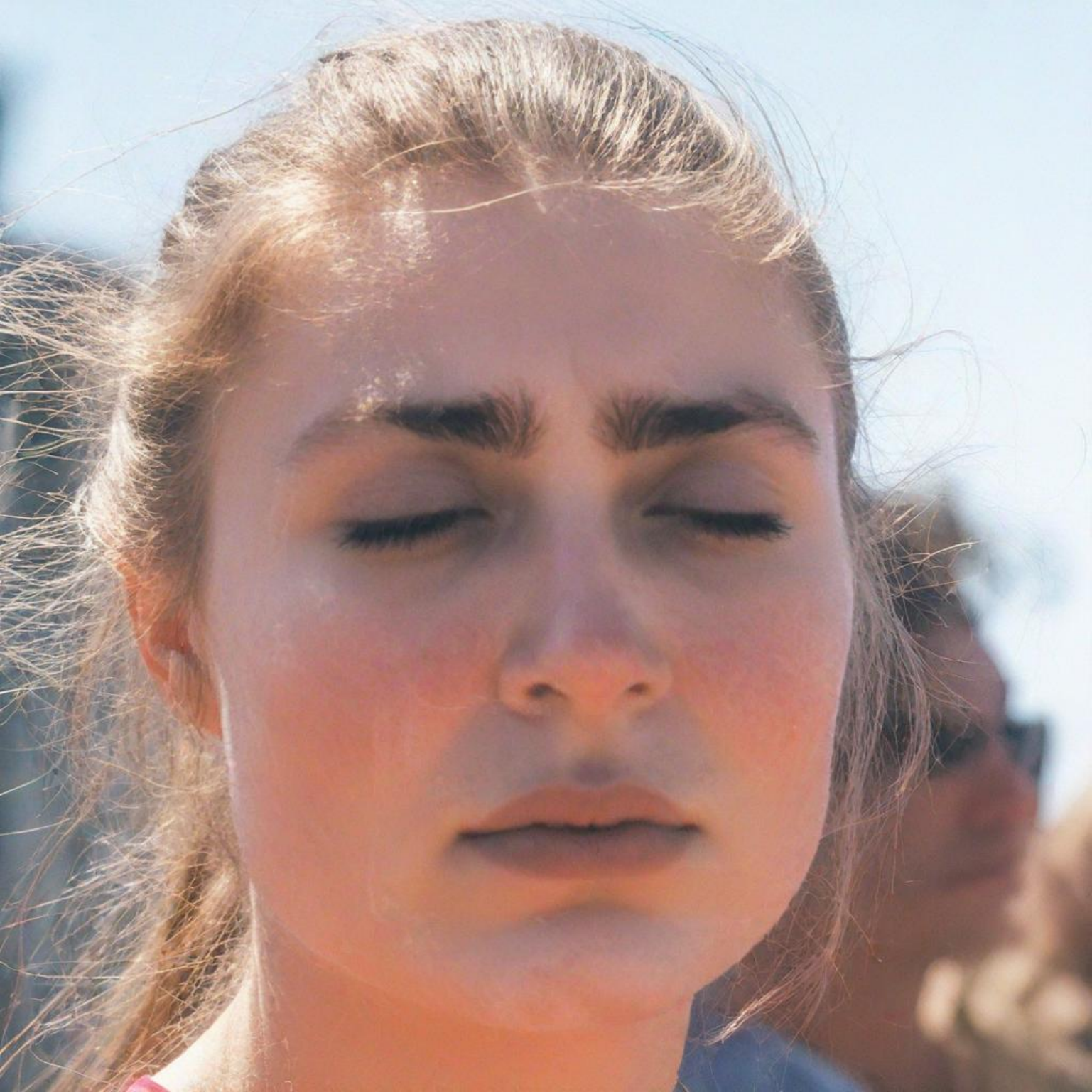}
\end{minipage}%
\hfill
\begin{minipage}{0.15\textwidth}
    \centering \textbf{GT} \\ 
    \includegraphics[width=\textwidth]{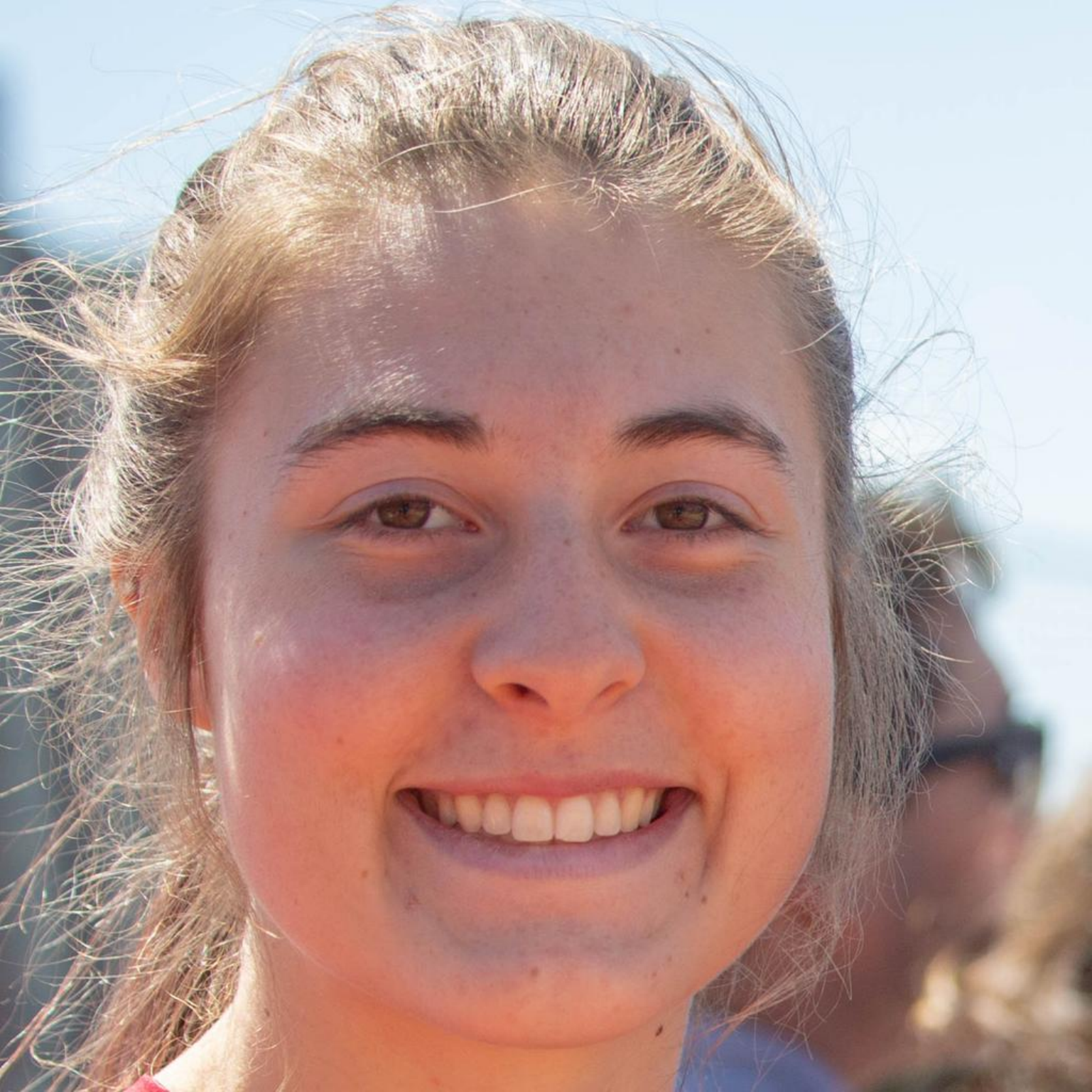}\\
    \centering \textbf{P2L} \\ 
    \includegraphics[width=\textwidth]{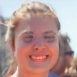}\\
    \centering \textbf{"with a smile"} \\ 
    \includegraphics[width=\textwidth]{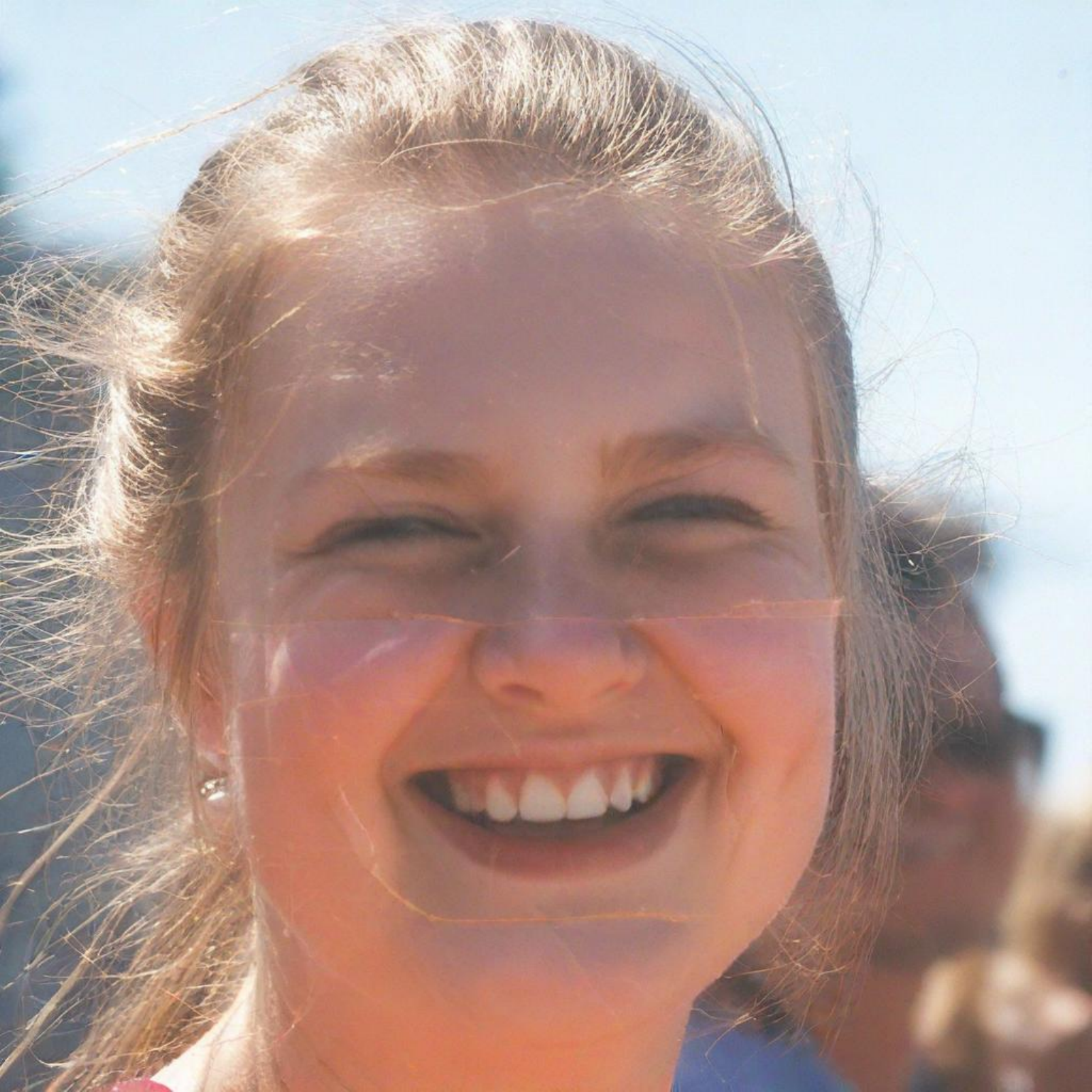}
\end{minipage}%
\hfill
\begin{minipage}{0.15\textwidth}
    \centering \textbf{LATINO} \\ 
    \includegraphics[width=\textwidth]{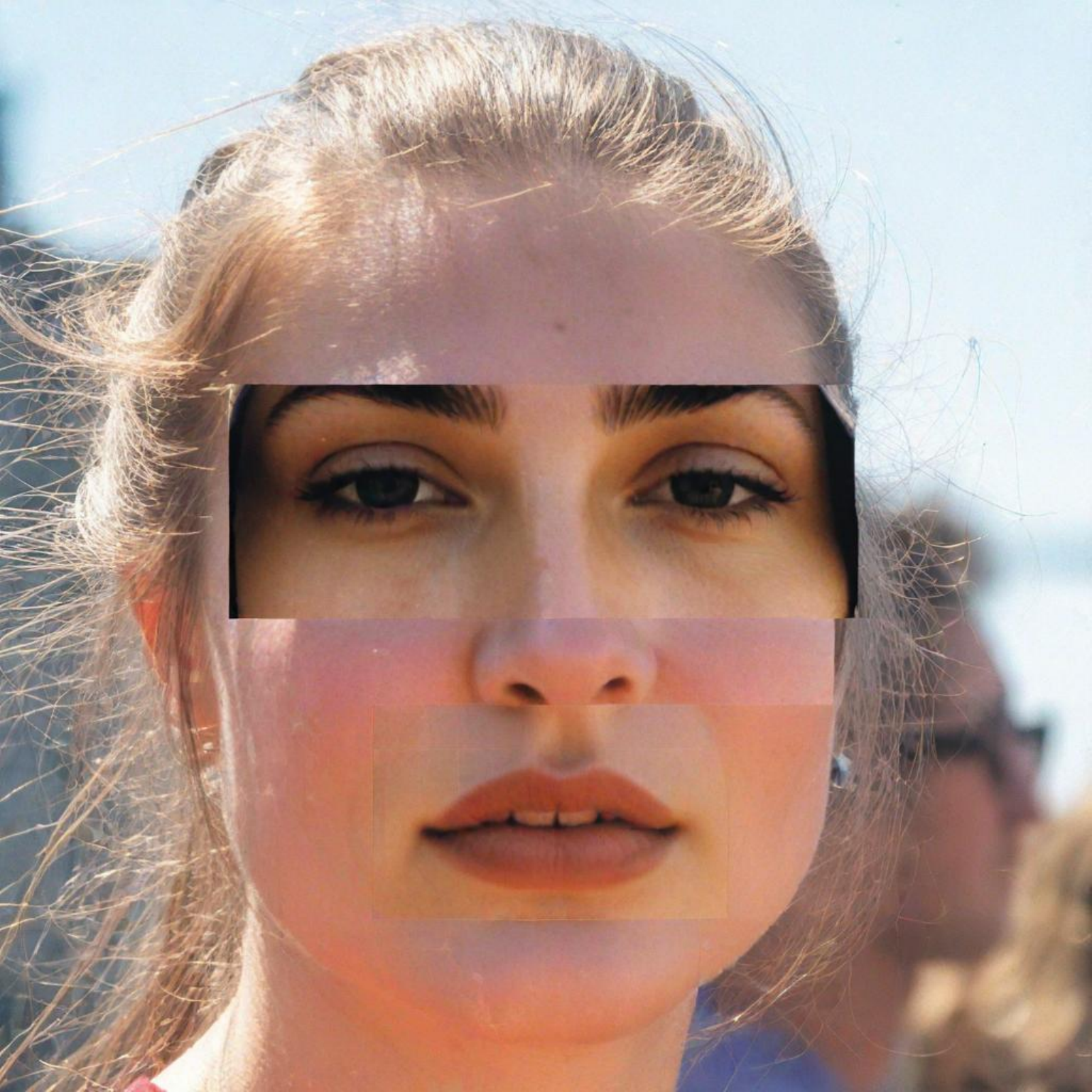}\\
    \centering \textbf{PSLD} \\ 
    \includegraphics[width=\textwidth]{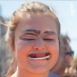}\\
    \centering \textbf{"with sunglasses"} \\ 
    \includegraphics[width=\textwidth]{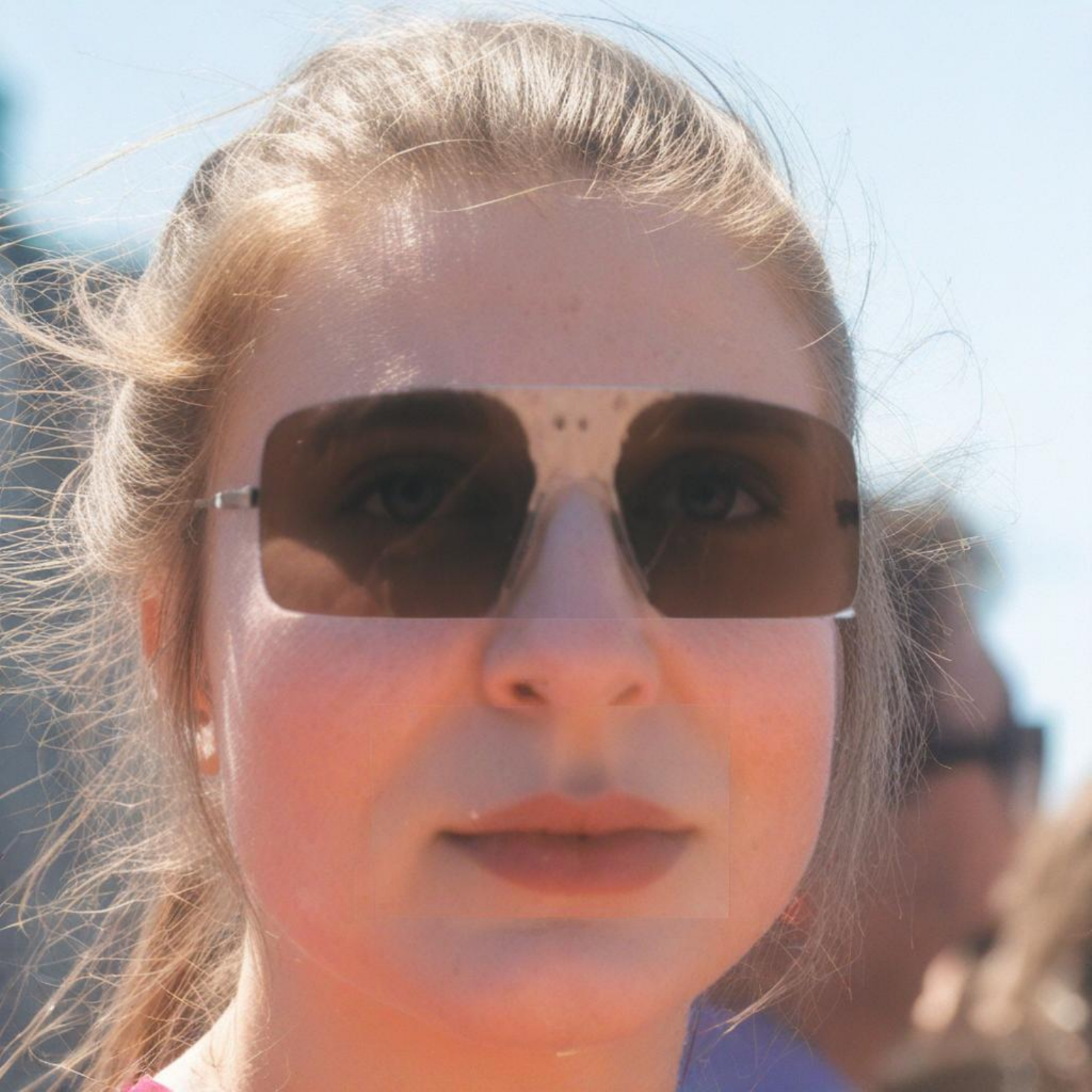}
\end{minipage}%
\vspace*{-.cm}
\caption{Box inpainting results on FFHQ-512.
Middle row: LATINO-PRO with the prompt \texttt{a sharp photo of a face}, P2L and PSLD. Bottom row: LATINO-PRO with different prompts.\vspace{-0.cm}}
\label{fig:qualitative_comparison_FFHQ_box}
\end{figure}
\begin{figure}[!h]
\centering
\begin{minipage}{0.15\textwidth}
    \centering \textbf{Measurement} \\ 
    \includegraphics[width=\textwidth]{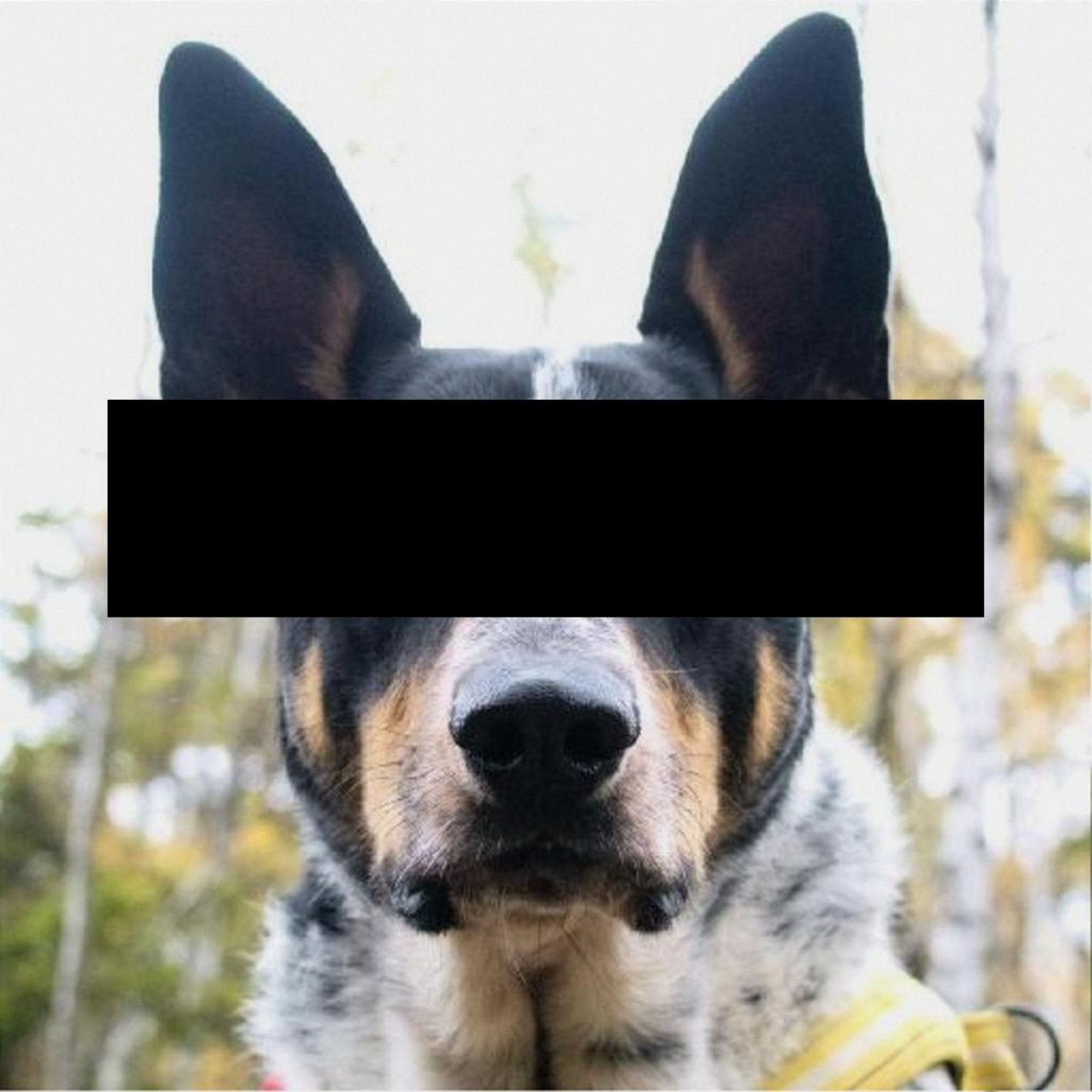}\\
    \centering \textbf{LATINO-PRO} \\ 
    \includegraphics[width=\textwidth]{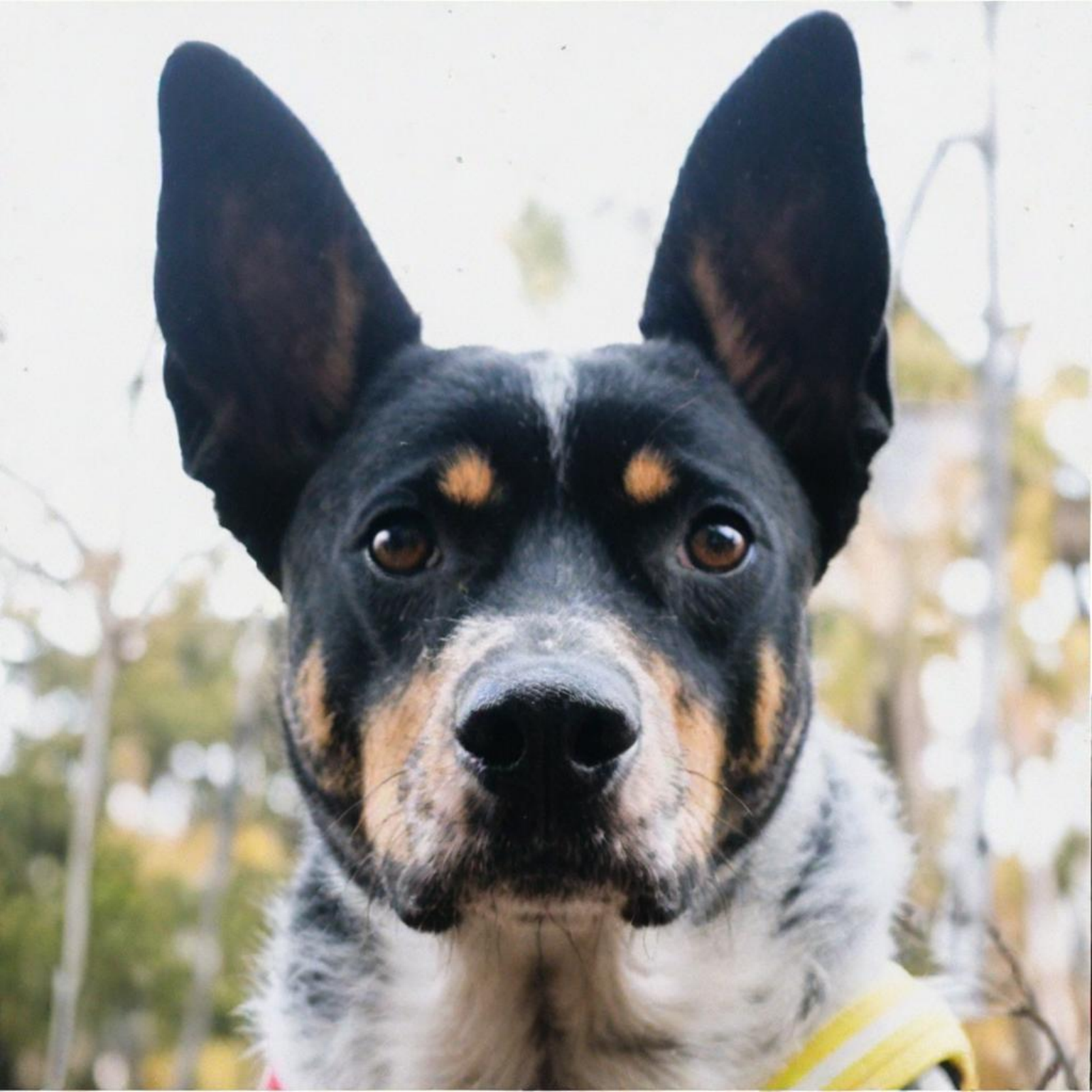}\\
    \centering \textbf{``with blue eyes''} \\ 
    \includegraphics[width=\textwidth]{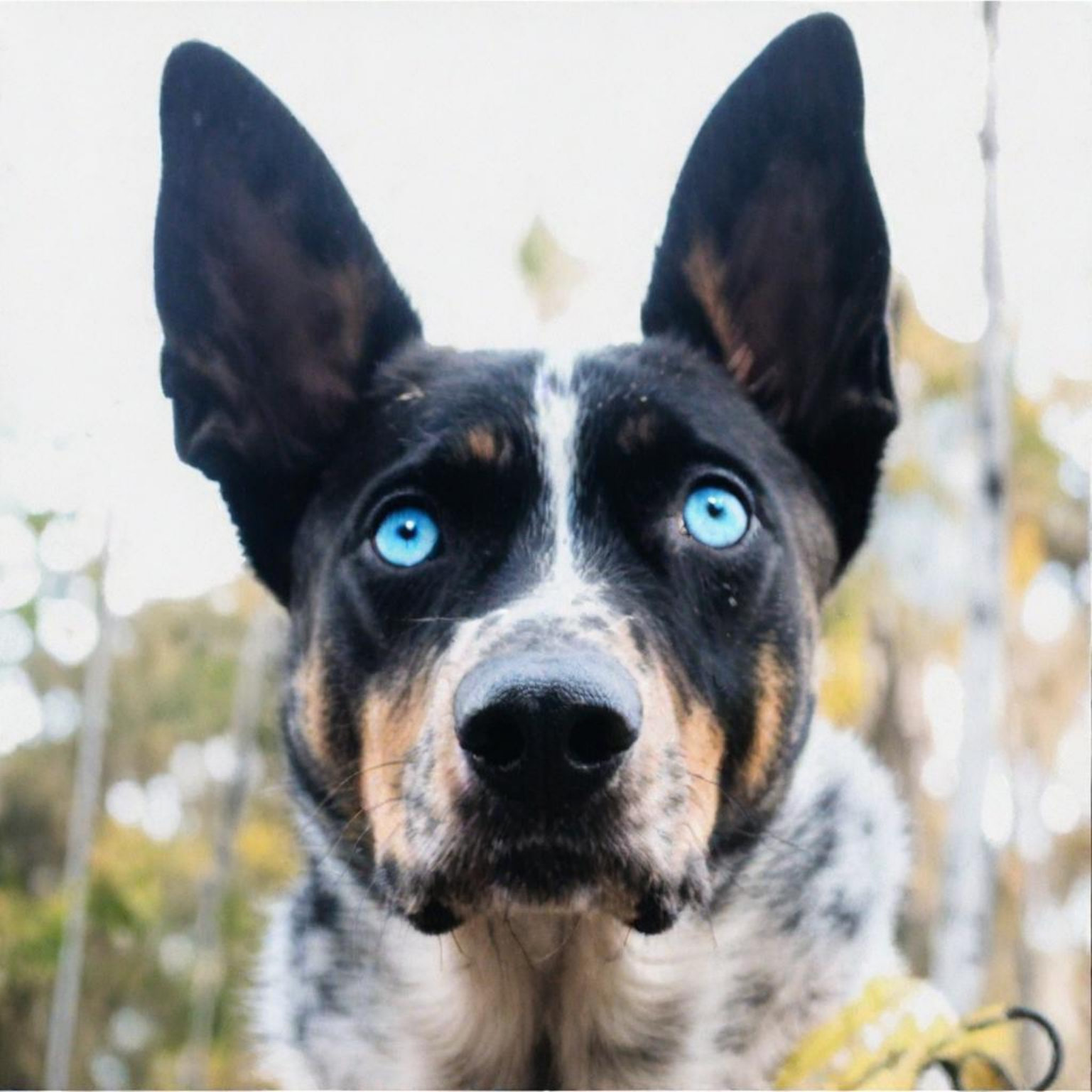}
\end{minipage}%
\hfill
\begin{minipage}{0.15\textwidth}
    \centering \textbf{GT} \\ 
    \includegraphics[width=\textwidth]{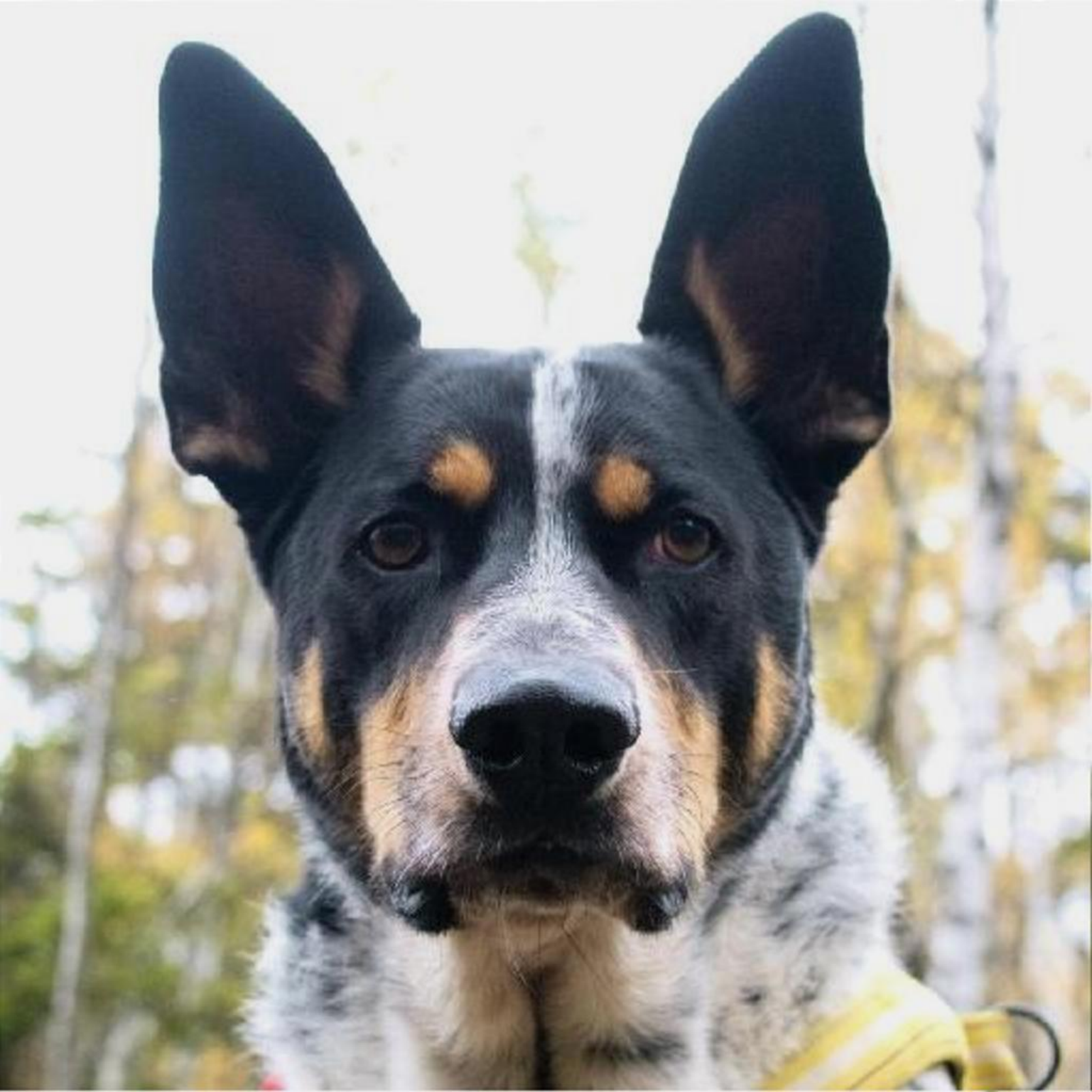}\\
    \centering \textbf{P2L} \\ 
    \includegraphics[width=\textwidth]{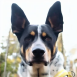}\\
    \centering \textbf{``with green eyes''} \\ 
    \includegraphics[width=\textwidth]{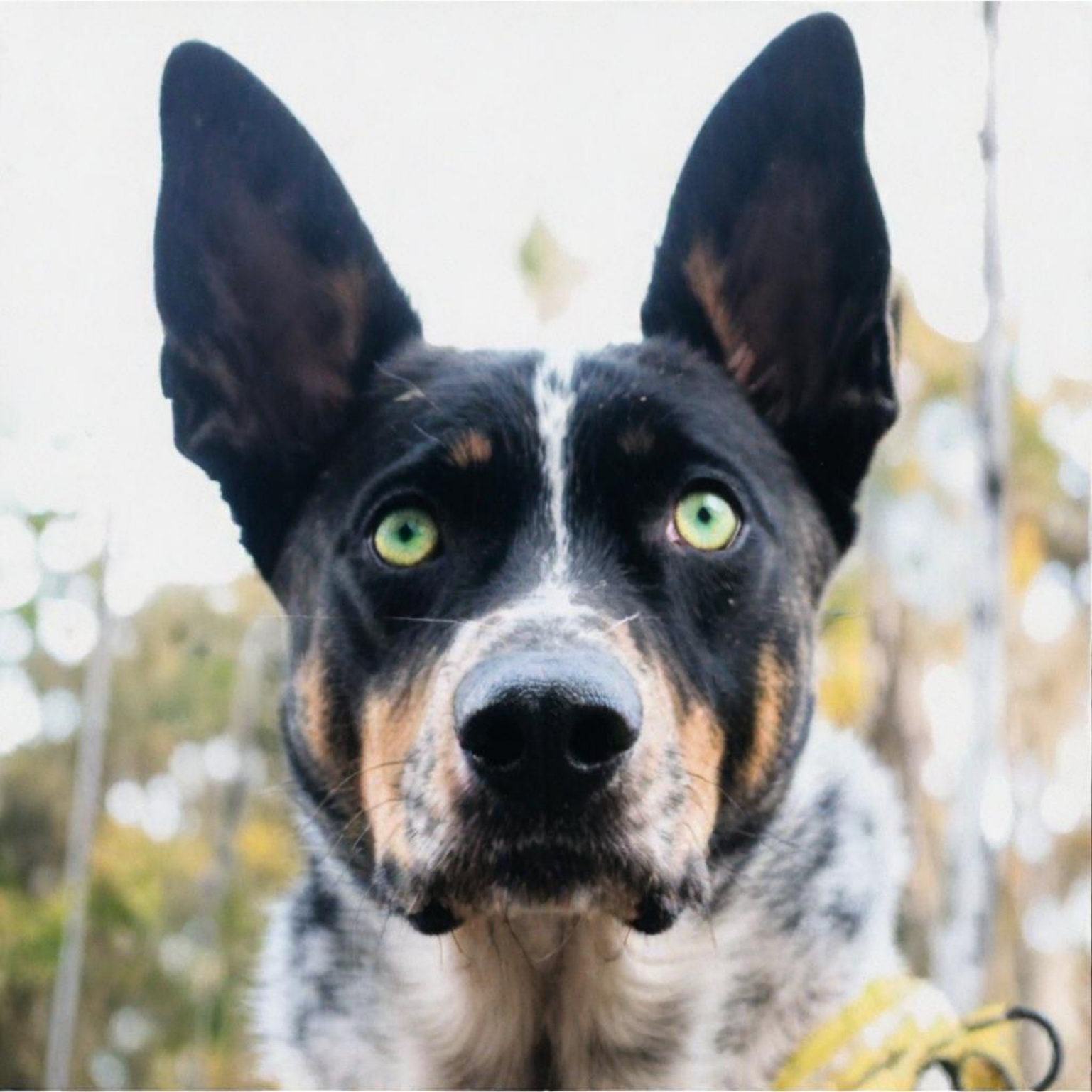}
\end{minipage}%
\hfill
\begin{minipage}{0.15\textwidth}
    \centering \textbf{LATINO} \\ 
    \includegraphics[width=\textwidth]{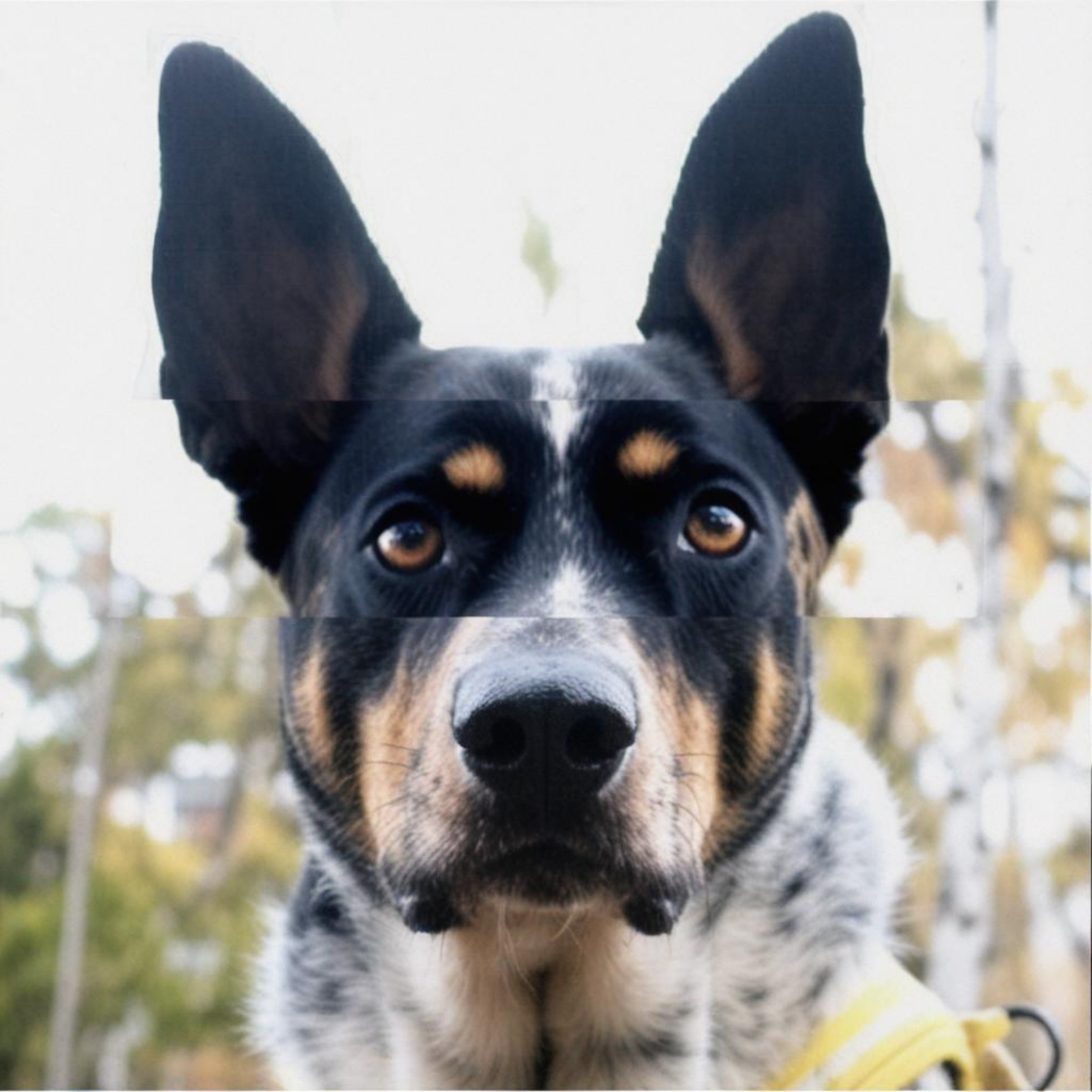}\\
    \centering \textbf{PSLD} \\ 
    \includegraphics[width=\textwidth]{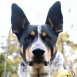}\\ 
    \centering \textbf{``with sunglasses''} \\ 
    \includegraphics[width=\textwidth]{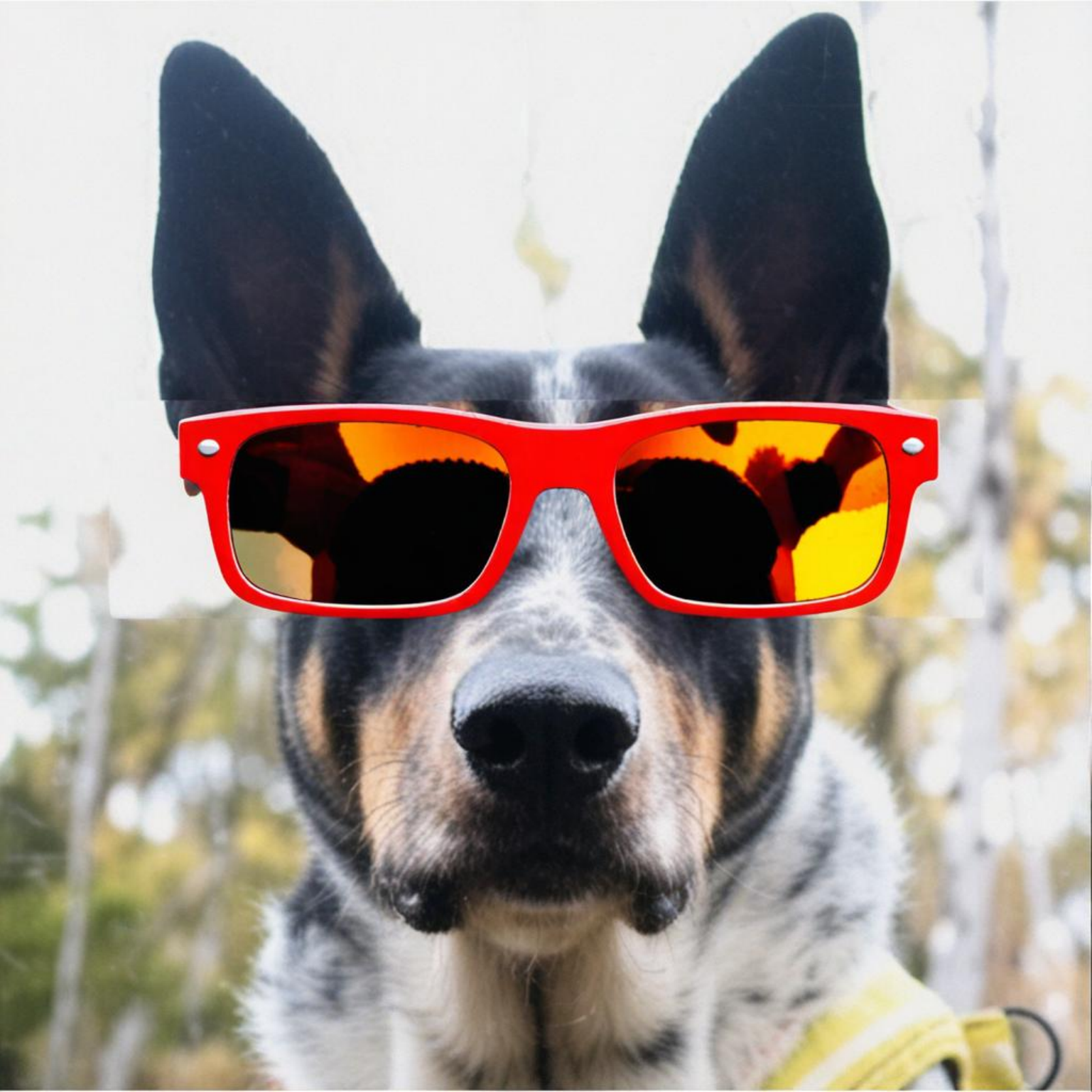}
\end{minipage}%
\vspace*{-.2cm}
\caption{Box inpainting (AFHQ-512).
Second row: LATINO-PRO with the prompt \texttt{a sharp photo of a dog}, P2L and PSLD. Last row: LATINO-PRO with various prompt initializations.\vspace{-0.cm}}
\label{fig:qualitative_comparison_AFHQ_box}
\end{figure}

\section{Harder cases}
\label{sec:extreme}

As shown in the the preview Figure \ref{fig:preview}, LATINO-PRO is able to solve also harder restoration tasks. Here we provide in Figure  \ref{fig:extreme} more visual results on the FFHQ1024 in this direction. In particular we observe an high level of consistency for aggressive tasks like Gaussian Deblurring with $\sigma = 20.0$ pixels and $\times 32$ super-resolution. We identify as a limitation the inability of our algorithm to keep the same level of consistency when the noise level is increased to $\sigma_n = 0.1$. The prior tendency to dominate is not contrasted enough by the proximity operator, and the results tend to deviate more from the actual ground truth.

\begin{figure}[!h]
\centering
\begin{minipage}{0.3\columnwidth}
    \centering \textbf{Measurement} \\ 
    \includegraphics[width=\columnwidth]{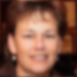} \\
    \includegraphics[width=\columnwidth]{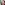}\\
    \includegraphics[width=\columnwidth]{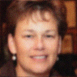}\\
    \includegraphics[width=\columnwidth]{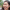}
\end{minipage}%
\begin{minipage}{0.3\columnwidth}
    \centering \textbf{GT} \\ 
    \includegraphics[width=\columnwidth]{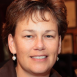} \\
    \includegraphics[width=\columnwidth]{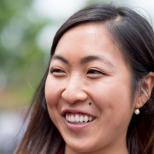}\\
    \includegraphics[width=\columnwidth]{images/Tests/extreme/clean3.pdf}\\
    \includegraphics[width=\columnwidth]{images/Tests/extreme/clean4.pdf}
\end{minipage}%
\begin{minipage}{0.3\columnwidth}
    \centering \textbf{LATINO-PRO} \\ 
    \includegraphics[width=\columnwidth]{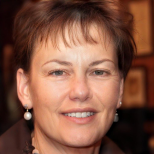} \\
    \includegraphics[width=\columnwidth]{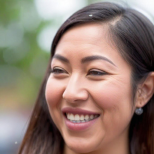}\\
    \includegraphics[width=\columnwidth]{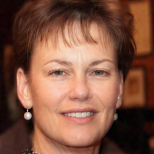}\\
    \includegraphics[width=\textwidth]{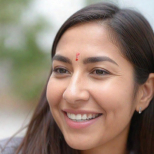}
\end{minipage}
\caption{Qualitative comparison of LATINO-PRO on hard image restoration on FFHQ-1024. Tasks: Gaussian deblur $\sigma=20.0$ and $\times 32$ super-resolution with noise $\sigma_n=0.01$. Gaussian deblur $\sigma=10.0$ and $\times 16$ super-resolution with noise $\sigma_n=0.1$.}
\label{fig:extreme}
\end{figure}

\end{document}